\RequirePackage{fix-cm}
\documentclass[smallextended]{svjour3}       % onecolumn (second format)
\smartqed  % flush right qed marks, e.g. at end of proof

\usepackage{graphicx}
\newtheorem{thm}{Theorem}%[section]
\newtheorem{lem}{Lemma}

\newtheorem{cor}{Corollary}
\newtheorem{defn}{Definition}%[section]
\newtheorem{assu}{Assumption}
\newtheorem{cond}{Condition}
\usepackage{subfig}
\newcommand{\norm}[1]{\left\lVert#1\right\rVert}
\usepackage{algorithm,algorithmic}

\usepackage{xcolor}
\usepackage{amsmath,bm}
\usepackage{hyperref}
\usepackage{amsfonts}       % blackboard math symbols

\begin{document}

\title{ Faster Riemannian Newton-type Optimization \\ by Subsampling and Cubic Regularization%\thanks{Grants or other notes
}
\subtitle{}


\author{Yian Deng         \and
        Tingting Mu %etc.
}


\institute{Yian Deng, Tingting Mu  \at
              Department of Computer Science \\
              University of Manchester \\
              \email{ \{yian.deng, tingting.mu\}@manchester.ac.uk}           %  \\
}

\date{Manuscript was submitted to Machine Learning in 10/2021 and accepted in 01/2023.}

\maketitle

\begin{abstract}
This work is on constrained large-scale non-convex optimization where the constraint set implies a manifold structure. 
Solving such problems is important in a multitude of fundamental machine learning tasks. 
Recent advances on Riemannian optimization have enabled the convenient recovery of solutions by adapting unconstrained optimization algorithms over manifolds. 
However, it remains challenging to scale up and meanwhile maintain stable convergence rates and handle saddle points. 
We propose a new second-order Riemannian optimization algorithm, aiming at improving convergence rate and reducing computational cost. 
It enhances the Riemannian trust-region algorithm that explores curvature information to escape saddle points through a mixture of subsampling and cubic regularization techniques. 
We conduct rigorous  analysis to study the convergence behavior of the proposed algorithm. 
We also perform extensive experiments to evaluate it based on two general machine learning tasks using multiple datasets.  
The proposed algorithm exhibits improved computational speed and convergence behavior compared to a large set of state-of-the-art Riemannian optimization algorithms. 

\end{abstract}

\section{Introduction}
In  modern machine learning, many learning tasks are  formulated as non-convex optimization problems. 
This is because, as  compared to linear or convex formulations,  they  can often capture more accurately the underlying structures within the data, and model more precisely the learning performance (or losses).
There is an important class of non-convex problems of which the constraint sets  possess manifold structures, e.g., to optimize over a set of orthogonal matrices. 
A manifold in mathematics refers to a topological space that locally behaves as an Euclidean space near each point. 
Over a $D$-dimensional manifold $\mathcal{M}$ in a $d$-dimensional ambient space ($D< d$), each local patch around each data point (a subset of $\mathcal{M}$) is homeomorphic to a local open subset of the Euclidean space $\mathbb{R}^D$. 
This special structure enables a straightforward adoption of any unconstrained optimization algorithm for solving a constrained problem over a manifold, simply by applying a systematic way to modify the  gradient and Hessian calculation. 
The modified calculations are called the Riemannian gradients and the Riemannian Hessians, which will be rigorously defined later.
Such an accessible method  for developing optimized solutions has benefited many applications and encouraged the implementation of various optimization libraries.

A representative learning task that gives rise to non-convex problems over manifolds is low-rank matrix completion, widely applied in signal denoising, recommendation systems and image recovery \cite{liu2019convolution}. 
It is formulated as optimization problems constrained on fixed-rank matrices that belong to a Grassmann manifold.
Another  example task is principal component analysis (PCA) popularly used in statistical data analysis and dimensionality reduction   \cite{shahid2015robust}.  
It seeks an optimal orthogonal projection matrix  from a Stiefel manifold. 
A more general problem setup than PCA is the subspace learning \cite{mishra2019riemannian}, where a low-dimensional space is an instance of a Grassmann manifold. 
When training a neural network, in order to reduce overfitting, the orthogonal constraint  that provides a Stiefel manifold structure is sometimes imposed over the network weights  \cite{anandkumar2016efficient}. 
Additionally, in  hand gesture recognition  \cite{nguyen2019neural}, optimizing over a symmetric definite positive (SDP)  manifold has been shown effective. 

Recent developments in optimization on Riemannian manifolds  \cite{absil2009optimization} have offered a convenient and unified solution framework for solving the aforementioned class of non-convex  problems.
The Riemannian optimization techniques  translate the constrained  problems into unconstrained ones on the manifold whilst preserving the geometric structure of the solution. 
For example, one Riemannian way to implement a PCA   is to   preserve the SDP geometric structure of the solutions without explicit constraints   \cite{horev2016geometry}. 
A simplified description of how Riemannian optimization works is that it first applies a straightforward way to modify the calculation of the first-order and second-order gradient information, then it adopts  an unconstrained optimization algorithm that uses the modified gradient information.
There are systematic ways to compute these modifications by analyzing the geometric structure of the manifold.  
Various libraries have implemented these methods and  are available to practitioners, e.g., Manopt \cite{boumal2014manopt} and Pymanopt \cite{townsend2016pymanopt}. 

Among such techniques, Riemannian gradient descent (RGD)  is the simplest. 
To handle large-scale computation with a finite-sum structure, Riemannian stochastic gradient descent (RSGD) \cite{bonnabel2013stochastic} has been proposed to estimate the gradient from a single sample (or a sample batch)  in each iteration of the optimization. 
Here, an iteration refers to the process by which an incumbent solution is updated with gradient and (or) higher-order derivative information; for example, Eq. (\ref{r_sgd_update}) in the upcoming text defines an RSGD iteration.
Convergence rates of RGD and RSGD are compared in \cite{zhang2016first} together with a global complexity analysis. 
The work concludes that  RGD can converge linearly while RSGD converges sub-linearly, but RSGD becomes computationally cheaper when there is a significant increase in the size of samples to process, also it can potentially prevent overfitting.  
By using RSGD to optimize over the Stiefel manifold, \cite{politz2016interpretable} attempts to improve interpretability of domain adaptation  and  has demonstrated  its benefits for text classification.

A major drawback of RSGD is the variance issue, where the  variance of the update direction can slow down the convergence and result in poor solutions. Typical techniques for variance reduction include the Riemannian stochastic variance reduced gradient  (RSVRG) \cite{zhang2016riemannian} and the Riemannian stochastic recursive gradient (RSRG) \cite{kasai2018riemannian}. RSVRG
reduces the gradient variance by using a momentum term, which takes into account  the gradient information obtained from both RGD and RSGD. RSRG follows a different strategy and considers only the information in the last and current iterations. This has the benefit of avoiding large cumulative errors, which can be caused by transporting the gradient vector along a distant path  when aligning two vectors at the same tangent plane. It has been shown by \cite{kasai2018riemannian} that RSRG performs better than RSVRG particularly for large-scale problems.

The RSGD variants can  suffer from oscillation across the slopes of a ravine \cite{kumar2018geometry}. 
This also happens when performing  stochastic gradient descent in Euclidean spaces.  
To address this, various adaptive algorithms have been proposed. 
The core idea is to control the learning process with adaptive learning rates in addition to the gradient momentum. 
Riemannian techniques of this kind include R-RMSProp \cite{kumar2018geometry}, R-AMSGrad \cite{cho2017riemannian}, R-AdamNC \cite{becigneul2018riemannian}, RPG \cite{huang2021riemannian} and RAGDsDR \cite{alimisis2021momentum}.

Although improvements have been made  for first-order optimization, they might still be insufficient for handling saddle points in non-convex problems \cite{mokhtari2018escaping}. 
They can only guarantee convergence to stationary points and do not have control over getting trapped at saddle points due to the lack of higher-order information. 
As an alternative, second-order algorithms are normally good at escaping saddle points by exploiting curvature information \cite{kohler2017sub,tripuraneni2018stochastic}.
Representative examples of this are the trust-region (TR) methods. 
Their capacity for handling saddle points and  improved convergence over many first-order methods has been demonstrated in \cite{weiwei2013newton}  for various non-convex problems. 
The TR technique has been extended to a Riemannian setting for the first time by \cite{absil2007trust}, referred to as  the Riemannian TR (RTR) technique.

It is well known that what prevents the wide use of the second-order Riemannian techniques in large-scale problems is the high cost of computing the exact Hessian matrices. 
Inexact techniques are therefore proposed to iteratively search for solutions without explicit Hessian computations. 
They can also handle non-positive-definite Hessian matrices and improve operational robustness. 
Two representative inexact examples are the conjugate gradient and the Lanczos methods \cite{zhu2017riemannian,xu2016matrix}.  
However, their reduced complexity is still proportional to the sample size, and they can still be computationally costly when working with large-scale problems. 
To address this issue, the subsampling technique has been proposed, and its core idea is to approximate the gradient and Hessian using a batch of samples. 
It has been proved by \cite{shen2019stochastic} that the TR method with subsampled gradient and Hessian can achieve a convergence rate of order $\mathcal{O}\left(\frac{1}{k^{2/3}}\right)$ with $k$ denoting the iteration number.
A sample-efficient stochastic TR approach is proposed by  \cite{shen2019stochastic} which finds an $(\epsilon, \sqrt{\epsilon})$-approximate local minimum within  a number of $\mathcal{O}(\sqrt{n}/\epsilon^{1.5})$ stochastic Hessian oracle queries where $n$ denotes the sample number.
The subsampling technique  has been applied  to improve the second-order Riemannian optimization for the first time by \cite{kasai2018inexact}.  
Their proposed inexact RTR algorithm employs subsampling over the Riemannian manifold and achieves faster convergence than the standard RTR method. 
Nonetheless, subsampling can be sensitive to the configured batch size.  
Overly small batch sizes can lead to poor convergence.

In the latest development of second-order unconstrained  optimization,  it has been shown that the adaptive cubic regularization technique \cite{cartis2011adaptive} can improve the standard and subsampled TR algorithms and the Newton-type methods, resulting in, for instance, improved convergence and effectiveness at escaping strict saddle points  \cite{kohler2017sub,xu2020newton}.
To improve performance,   the variance reduction techniques have been combined into cubic regularization and  extended  to cases with inexact solutions. Example works of this include  \cite{zhou2020stochastic,zhou2019stochastic}  which were the first to rigorously demonstrate the advantage of variance reduction for second-order optimization algorithms.
Recently, the potential of cubic regularization for solving non-convex problems over constraint sets with Riemannian manifold structures has been shown by \cite{zhang2018cubic,agarwal2021adaptive}.

We aim at  improving the RTR optimization by taking advantage of both the adaptive cubic regularization and subsampling techniques.  
Our problem of interest   is to find a local minimum of a non-convex finite-sum minimization problem constrained on a set endowed with a Riemannian manifold structure. 
Letting $f_i:\mathcal{M}\rightarrow\mathbb{R}$ be a  real-valued function defined on a Riemannian manifold $\mathcal{M}$, we consider a twice differentiable finite-sum objective function:
\begin{equation}
	\begin{split}
		\min_{\mathbf{x}\in\mathcal{M}}\ f(\mathbf{x})\:= \frac{1}{n}\sum_{i=1}^{n}f_i(\mathbf{x}).
	\end{split}
	\label{general_optimization_problem}
\end{equation}
In the machine learning context, $n$  denotes the sample number, and  $f_i(\mathbf{x})$  is a  smooth real-valued and twice differentiable cost (or loss) function computed for the $i$-th sample. 
The $n$ samples are assumed to be uniformly sampled, and thus $\mathbb{E}\left[f_i(\mathbf{x})\right ] = \lim_{n \to \infty} \frac{1}{n}\sum_{i=1}^{n}f_i(\mathbf{x})$.

We propose a cubic Riemannian Newton-like (RN) method  to solve more effectively the  problem in Eq. (\ref{general_optimization_problem}). Specifically, we enable two key improvements in the Riemannian space, including  (1) to approximate the Riemannian gradient and Hessian using the subsampling technique and (2) to improve the subproblem formulation by replacing the trust-region constraint with a cubic regularization term. 
The resulting algorithm is named  Inexact Sub-RN-CR\footnote{The abbreviation Sub-RN-CR comes from Sub-sampled Riemannian Newton-like Cubic Regularization. We follow the tradition of referring to a TR method enhanced by cubic regularization as a Newton-like method \cite{cartis2011adaptive}. The implementation of the Inexact Sub-RN-CR is provided in \href{https://github.com/xqdi/isrncr}{https://github.com/xqdi/isrncr}.}.
After introducing cubic regularization, it becomes more challenging to solve the subproblem,  for which we demonstrate two effective solvers based on the Lanczos and conjugate gradient methods.
We provide convergence analysis for the proposed Inexact Sub-RN-CR algorithm and present the main results in Theorems \ref{theorem2} and \ref{thm_opt_complex}.  
Additionally, we provide analysis for the subproblem solvers, regarding their solution quality, e.g., whether and how  they meet a set of desired conditions as presented in Assumptions \ref{assu_cauchy_eigen_point}-\ref{assu_cg}, and their convergence property.  
The key results are presented  in Lemma \ref{lem_krylov}, Theorem \ref{theorem1}, Lemmas \ref{lem_lanczos_sol_property} and \ref{lem_tcg_sol_property}.
Overall, our results are satisfactory.
The proposed  algorithm finds an $\left(\epsilon_g,\epsilon_H\right)$-optimal solution (defined in Section \ref{sec:Inexact Sub-RN-CR}) in fewer iterations than the state-of-the-art RTR \cite{kasai2018inexact}.  
Specifically, the required number of iterations is reduced from $\mathcal{O}\left(\max\left(\epsilon_g^{-2}\epsilon_H^{-1},\epsilon_H^{-3}\right)\right)$   to $\mathcal{O}\left(\max\left(\epsilon_g^{-2},\epsilon_H^{-3}\right)\right)$. 
When being tested by extensive experiments on  PCA and matrix completion tasks with different datasets and applications in image analysis, our algorithm shows much better performance than most state-of-the-art and popular Riemannian optimization algorithms, in terms of both the solution quality and computing time.

\section{Notations and Preliminaries}
We start by familiarizing the readers with the notations and concepts that will be used in the paper, and recommend  \cite{absil2009optimization} for a more detailed explanation of the relevant concepts.
The manifold $\mathcal{M}$ is equipped with a smooth inner product $\langle\cdot,\cdot \rangle_{\mathbf{x}}$ associated with the tangent space $T_{\mathbf{x}}\mathcal{M}$ at any $\mathbf{x}\in\mathcal{M}$, and this inner product is  referred to as the Riemannian metric.
The norm  of a tangent vector $\bm{\eta}\in T_{\mathbf{x}}\mathcal{M}$ is denoted by $\left \|\bm{\eta}  \right \|_{\mathbf{x}} $, which is computed by $\left \|\bm{\eta}  \right \|_{\mathbf{x}} =\sqrt{\langle \bm{\eta},\bm{\eta} \rangle_{\mathbf{x}}}$.
When we use the notation $\left \|\bm{\eta}  \right \|_{\mathbf{x}}$, $\bm{\eta}$ by default belongs to the tangent space $T_{\mathbf{x}}\mathcal{M}$.
We use $\mathbf{0}_{\mathbf{x}}\in T_{\mathbf x} \mathcal{M}$ to denote the zero vector  of the tangent space at $\mathbf x$.
The \textit{retraction mapping} denoted by $R_{\mathbf{x}}\left(\bm{\eta}\right):T_{\mathbf{x}}\mathcal{M}\rightarrow\mathcal{M}$  is used  to move $\mathbf{x}\in\mathcal{M}$ in the direction $\bm{\eta}\in T_{\mathbf{x}}\mathcal{M}$ while remaining on $\mathcal{M}$,  and it is an equivalent version  of  $\mathbf{x} +\bm{\eta}$ in an Euclidean space.
The  \textit{pullback} of $f$ at $\mathbf x$ is defined by $ \hat{f}(\bm\eta)  =  f(R_{\mathbf x}(\bm\eta))$, and $\hat{f}(\mathbf 0_{\mathbf x}) = f(\mathbf x)$.
The \textit{vector transport} operator $ \mathcal{T}_{\mathbf x}^{\mathbf y}(\mathbf v): T_\mathbf{x}\mathcal{M} \rightarrow T_\mathbf{x}\mathcal{M}$ moves a tangent vector $\mathbf v \in T_{\mathbf{x}}\mathcal{M}$ from a point $\mathbf x\in\mathcal{M} $ to another $\mathbf y\in\mathcal{M}$.
We also use a shorthand notation $\mathcal{T}_{\bm{\eta}}({\mathbf{v}})$ to describe $\mathcal{T}_{\mathbf x}^{\mathbf y}(\mathbf v)$ for a moving direction $\bm\eta \in T_{\mathbf{x}}\mathcal{M}$ from $\mathbf x$ to $\mathbf y$ satisfying $R_{\mathbf{x}}\left(\bm{\eta}\right) = \mathbf y$.
The \textit{parallel transport} operator $\mathcal{P}_{\bm{\eta},\gamma} (\mathbf{v})$ is a special instance of the vector transport. It moves  $\mathbf{v}\in T_{\mathbf{x}}\mathcal{M}$ in the direction of $\bm{\eta}\in T_{\mathbf{x}}\mathcal{M}$ along a geodesic $\gamma:[0,1]\rightarrow \mathcal{M}$, where $\gamma(0)=\mathbf{x}$, $\gamma(1)=\mathbf{y}$ and $\gamma^{\prime}(0) = \bm{\eta}$, and during the movement, 
it has to satisfy the parallel condition on the geodesic curve.
We simplify the notation $\mathcal{P}_{\bm{\eta},\gamma} (\mathbf{v})$ to $\mathcal{P}_{\bm{\eta}} (\mathbf{v})$.
Fig.  \ref{fig_manifold}   illustrates a manifold and  the operations over it.
Additionally, we use   $\|\cdot\|$  to denote the $l_2$-norm operation in a Euclidean space.

\begin{figure}[h]
	\centering
	\includegraphics[width=0.6\textwidth]{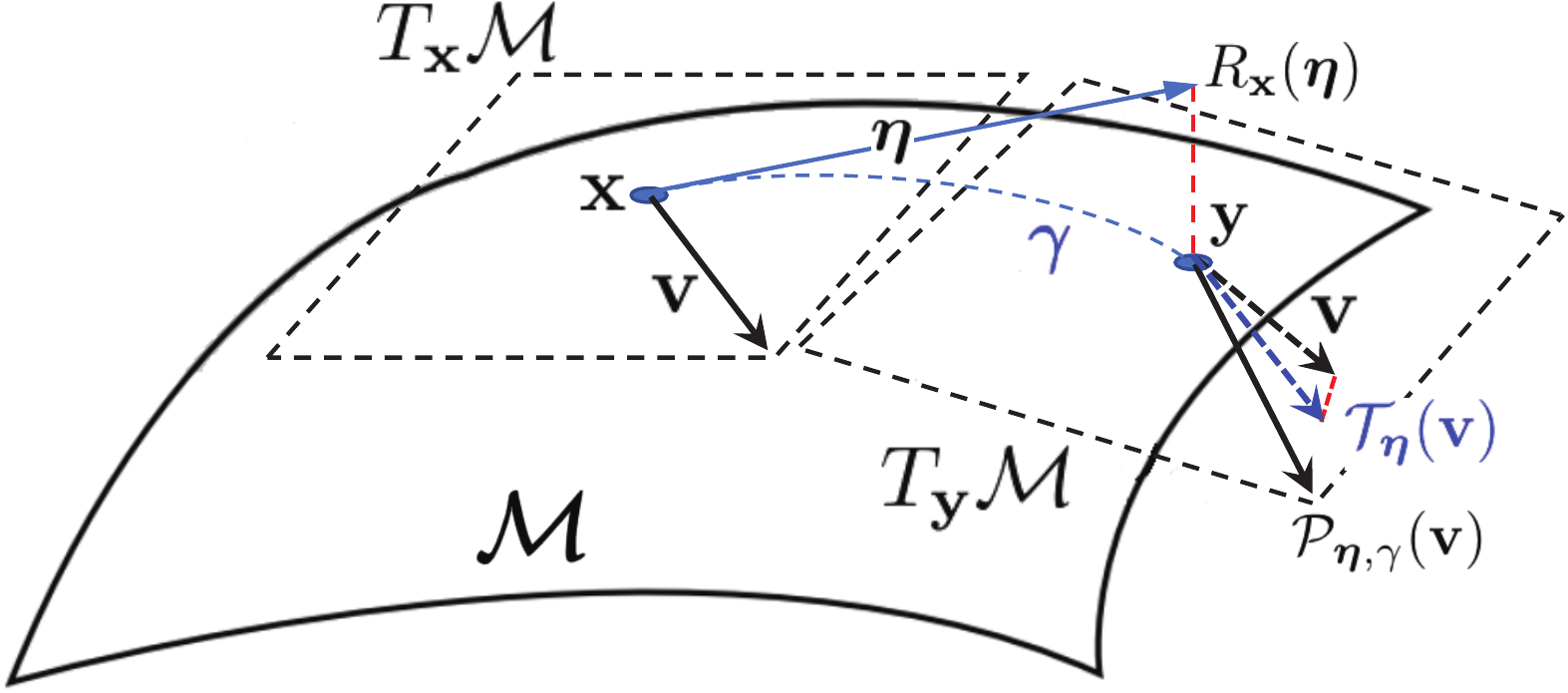}
	\caption{\label{fig_manifold}Illustration of the retraction and parallel transport operations.}
\end{figure}

The Riemannian gradient of a real-valued differentiable function $f$ at $\mathbf{x}\in\mathcal{M}$, denoted by ${\rm grad}f(\mathbf{x})$, is defined as the unique element of $\mathcal{T}_\mathbf{x}\mathcal{M}$ satisfying $\langle {\rm grad}f(\mathbf{x}),\bm\xi \rangle_\mathbf{x}=\mathcal{D}f(\mathbf{x})[\bm\xi],\ \forall \bm\xi\in T_\mathbf{x}\mathcal{M}$. 
Here, $\mathcal{D}f(\mathbf{x})[\bm\xi]$ generalizes the notion of the directional derivative to a manifold, defined  as the derivative of $f(\gamma(t))$ at $t=0$ where $\gamma(t)$ is a curve on $\mathcal{M}$ such that $\gamma(0)=\mathbf{x}$ and $\dot{\gamma}(0)=\bm\xi$.  
When operating in an Euclidean space, we  use the same notation $\mathcal{D}f(\mathbf{x})[\bm\xi]$ to denote the classical directional derivative.
The Riemannian Hessian of a real-valued differentiable function $f$ at $\mathbf{x}\in\mathcal{M}$, denoted by ${\rm Hess} f(\mathbf{x})[\bm\xi]: T_\mathbf{x}\mathcal{M} \rightarrow T_\mathbf{x}\mathcal{M}$, is a linear mapping defined based on the Riemannian connection, as ${\rm Hess} f(\mathbf{x})[\bm\xi] =\tilde{\nabla}_{\bm\xi}  {\rm grad} f(\mathbf{x})$. The Riemannian connection $\tilde{\nabla}_{\bm\xi}\bm \eta :  T_\mathbf{x}\mathcal{M}\times  T_\mathbf{x}\mathcal{M} \rightarrow  T_\mathbf{x}\mathcal{M}$, generalizes the notion of the directional derivative of a vector field.  
For a function $f$ defined over an embedded manifold, its Riemannian gradient can be computed by projecting the Euclidean gradient   $\nabla f(\mathbf{x}) $ onto the tangent space, as ${\rm grad}f(\mathbf{x})=P_{T_\mathbf{x}\mathcal{M}}\left[\nabla f(\mathbf{x})\right]$ where $P_{T_\mathbf{x}\mathcal{M}}[\cdot]$ is the orthogonal projection onto $T_\mathbf{x}\mathcal{M}$.  Similarly, its Riemannian Hessian can be computed by projecting the  classical directional derivative of ${\rm grad}f(\mathbf{x})$, defined by $\nabla ^2 f(\mathbf{x})[\bm\xi] =  \mathcal{D} {\rm grad}f(\mathbf{x}) [\bm\xi]$,  onto the tangent space, resulting in ${\rm Hess}f(\mathbf{x})[\bm\xi]=P_{T_\mathbf{x}\mathcal{M}}\left[\nabla ^2 f(\mathbf{x})[\bm\xi]\right]$.
When the function $f$ is defined over a quotient manifold, the Riemannian gradient and Hessian can be computed by projecting  $\nabla f(\mathbf{x}) $ and $\nabla^2 f(\mathbf{x})[\bm\xi]$ onto the horizontal space of the manifold.

Taking the PCA problem as an example (see Eqs (\ref{eq_pca_formula_0}) and (\ref{eq_pca_formula}) for its formulation), the general objective in Eq. (\ref{general_optimization_problem}) can be instantiated by  $\min_{\mathbf{U}\in {\rm Gr}\left(r,d\right)}$ $ \frac{1}{n}\sum_{i=1}^n\left\|\mathbf{z}_i-\mathbf{UU}^T\mathbf{z}_i\right\|^2$, where $\mathcal{M}={\rm Gr}\left(r,d\right)$ is a Grassmann manifold.  %
The function of interest is $f_i(\mathbf{U}) = \left\|\mathbf{z}_i-\mathbf{UU}^T\mathbf{z}_i\right\|^2$.
The Grassmann  manifold ${\rm Gr}\left(r,d\right)$ contains the set of $r$-dimensional linear subspaces of the $d$-dimensional vector space. 
Each subspace corresponds to  a point on the manifold that is an equivalence class of $d\times r$ orthogonal matrices, expressed as 
$\mathbf{x} = \left[\mathbf{U}\right]:=\left\{\mathbf{UO}:\mathbf{U}^T\mathbf{U}=\mathbf{I},\mathbf{O}\in {\rm O}\left(r\right)\right\}$ where ${\rm O}\left(r\right)$ denotes the orthogonal group in $\mathbb{R}^{r\times r}$. 
A tangent vector  $\bm\eta \in T_{\mathbf{x}} {\rm Gr}(r,d)$ of the Grassmann manifold has the form 
$\bm\eta=\mathbf{U}^\bot\mathbf{B}$ \cite{edelman1998geometry}, where $\mathbf{B}\in \mathbb{R}^{(d-r)\times r}$,  and $\mathbf{U}^\bot\in\mathbb{R}^{d\times (d-r)}$ is the orthogonal complement of $\mathbf{U}$ such that $\left[\mathbf{U},\mathbf{U}^{\bot}\right]$ is orthogonal. 
A commonly used Riemannian metric for the Grassmann manifold  is the canonical inner product 
$\langle\bm{\eta},\bm{\bm\xi}\rangle_{\mathbf{x}}={\rm tr}(\bm{\eta}^T\bm{\bm\xi})$ given  $\bm{\eta},\bm{\bm\xi}\in T_{\mathbf{x}}{\rm Gr}(r,d)$, resulting in  
$\left\|\bm\eta\right\|_{\mathbf{x}}=\|\bm\eta\|$ 
%Yian: Double check and confirm. 
(Section 2.3.2 of \cite{edelman1998geometry}).
As we can see, the Riemannian metric and the norm here are equivalent to the Euclidean inner product and norm. The same result can be derived from another commonly used  metric  of the Grassmann manifold, i.e., 
$\langle\bm\eta, \bm\xi\rangle_{\mathbf{x}}={\rm tr}\left(\bm\eta^T\left(\mathbf{I}-\frac{1}{2}\mathbf{U}\mathbf{U}^{T}\right)\bm\xi\right)$ for $\bm\eta,\bm\xi\in T_{\mathbf{x}}{\rm Gr}(r,d)$ (Section 2.5 of \cite{edelman1998geometry}).
Expressing  two  given tangent vectors as $\bm\eta=\mathbf{U}^{\bot}\mathbf{B}_{\eta}$ and $\bm\xi =\mathbf{U}^{\bot}\mathbf{B}_{\xi}$ with $\mathbf{B}_{\eta}, \mathbf{B}_{\xi}\in \mathbb{R}^{(d-r)\times r}$, we have
\begin{align}
	\nonumber
	\langle\bm\eta,\bm\xi \rangle_\mathbf{x} =& \; {\rm tr}\left(\left(\mathbf{U}^{\bot}\mathbf{B}_{\eta}\right)^T\left(\mathbf{I}-\frac{1}{2}\mathbf{U}\mathbf{U}^{T}\right)\mathbf{U}^{\bot}\mathbf{B}_{\xi}\right)={\rm tr}\left(\left(\mathbf{U}^{\bot}\mathbf{B}_{\eta}\right)^T\mathbf{U}^{\bot}\mathbf{B}_{\xi}\right)\\
	=&\; {\rm tr}\left(\bm\eta^T\bm\xi\right).
\end{align}
Here we provide a few examples  of the key operations explained earlier on the Grassmann manifold, taken from \cite{boumal2014manopt}.
Given a data point $[\mathbf{U}]$, a moving direction $\bm\eta$ and the step size $t$, one way to construct the retraction mapping is  through performing singular value decomposition (SVD) on $\mathbf{U}+t\bm\eta$, i.e., $\mathbf{U}+t\bm\eta=\bar{\mathbf{U}}\mathbf{S}\bar{\mathbf{V}}^T$, and the new data point after moving is $\left[\bar{\mathbf{U}}\bar{\mathbf{V}}^T\right]$.
A transport operation can be implemented by projecting a given tangent vector using the orthogonal projector $\mathbf{I} - \mathbf{U}\mathbf{U}^T$.
Both Riemannian gradient and Hessian can be computed by projecting the Euclidean gradient and Hessian of $f(\mathbf{U})$ using the same  projector $\mathbf{I} - \mathbf{U}\mathbf{U}^T$.

\subsection{First-order Algorithms}

To optimize the problem in Eq. (\ref{general_optimization_problem}), the first-order Riemannian optimization algorithm RSGD updates the solution at each $k$-th iteration by using an $f_i$ instance, as
\begin{equation}
	\mathbf{x}_{k+1}=R_{\mathbf{x}_k}\left(-\beta_k\text{grad}f_i\left(\mathbf{x}_k\right)\right),
	\label{r_sgd_update}
\end{equation}
where $\beta_k$ is the step size. 
Assume that the algorithm runs for multiple epochs  referred to as the outer iterations. 
Each epoch contains multiple inner iterations, each of which corresponds to a randomly selected $f_i$ for calculating the update. 
Letting $\mathbf{x}_k^t$ be the solution at the $t$-th inner iteration of the $k$-th outer iteration and $\tilde{\mathbf{x}}_k$ be the solution at the last inner iteration of the $k$-th outer iteration, RSVRG employs a variance reduced extension \cite{zhang2016riemannian} of the update defined in Eq. (\ref{r_sgd_update}), given as
\begin{equation}
	\mathbf{x}_{k}^{t+1}=R_{\mathbf{x}_k^t}\left(-\beta_k\mathbf{v}_k^t\right),  
	\label{rsvrg_update}
\end{equation}
where
\begin{equation}
	\mathbf{v}_k^t=\text{grad}f_{i}\left(\mathbf{x}_k^t\right)-\mathcal{T}_{\tilde{\mathbf{x}}_{k-1}}^{\mathbf{x}_k^t}\left(\text{grad}f_{i}\left(\tilde{\mathbf{x}}_{k-1}\right)-\text{grad}f\left(\tilde{\mathbf{x}}_{k-1}\right)\right).
\end{equation}
Here, the full gradient  information ${\rm grad}f\left(\tilde{\mathbf{x}}_{k-1}\right)$ is used to reduce the variance in the stochastic  gradient $\mathbf{v}_k^t$. 
As a later development, RSRG \cite{kasai2018riemannian}   suggests a recursive  formulation  to improve the variance-reduced gradient $\mathbf{v}_k^t$. Starting from $\mathbf{v}_k^0={\rm grad} f\left(\tilde{\mathbf{x}}_{k-1}\right)$, it updates by 
\begin{equation}
	\mathbf{v}_k^t=\text{grad}f_{i}\left(\mathbf{x}_k^t\right)-\mathcal{T}_{\mathbf{x}_{k}^{t-1}}^{\mathbf{x}_k^t}\left(\text{grad}f_{i}\left(\mathbf{x}_{k}^{t-1}\right)\right)+ \mathcal{T}_{\mathbf{x}_{k}^{t-1}}^{\mathbf{x}_k^t}\left(\mathbf{v}_{k}^{t-1}\right).
	\label{r_sgd_update2}
\end{equation}
This formulation is designed to avoid the accumulated error caused by a distant vector transport.

\subsection{Inexact RTR} \label{inexact_rtr}

For second-order Riemannian optimization, the Inexact RTR \cite{kasai2018inexact} improves the standard RTR \cite{absil2007trust} through subsampling. 
It optimizes an approximation of the objective function formulated  using the second-order Taylor expansion within a trust region $\Delta_k$ around $\mathbf{x}_k$ at iteration $k$. 
A  moving direction $\bm{\eta}_k$ within the trust region is found by solving the subproblem at iteration $k$:
\begin{align}
	\label{eq_sub_problem_inexact}
	\bm{\eta}_k  = \mathop{\arg\min}_{\bm{\eta} \in T_{\mathbf{x}_k}\mathcal{M}}  \;\;\;& f (\mathbf x_k) +\langle\mathbf{G}_k,\bm{\eta} \rangle_{\mathbf{x}_k}+\frac{1}{2}\langle\bm{\eta},\mathbf{H}_k[\bm{\eta}]\rangle_{\mathbf{x}_k},  \\
	\nonumber
	\text{subject to} &\quad \left\|\bm{\eta} \right\|_{\mathbf{x}_k}\le\Delta_k, 
\end{align}
where $\mathbf{G}_k$ and $\mathbf{H}_k[\bm{\eta}]$ are the approximated Riemannian gradient and Hessian calculated by using the subsampling technique. 
The approximation is based on the current solution $\mathbf{x}_k$ and the moving direction $\bm\eta$, calculated as
\begin{align}
	\label{eq_subsample_gradient}
	\mathbf{G}_k=&\;\frac{1}{\left|\mathcal{S}_g\right|}\sum_{i\in\mathcal{S}_g}\text{grad}f_i(\mathbf{x}_k), \\ 
	\label{eq_subsample_hessian}
	\mathbf{H}_k[\bm{\eta}]=&\; \frac{1}{\left|\mathcal{S}_H\right|}\sum_{i\in\mathcal{S}_H}\text{Hess}f_i(\mathbf{x}_k)[\bm{\eta}],
\end{align}
where  $\mathcal{S}_g$, $\mathcal{S}_H\subset\left\{1,...,n\right\}$ are the sets of the subsampled indices used for estimating the Riemannian gradient and Hessian. 
The updated solution $\mathbf{x}_{k+1}=R_{\mathbf{x}_k}(\bm{\eta}_k)$ will be accepted and $\Delta_k$ will be increased, if the decrease of the true objective function $f$ is sufficiently large as compared to that of the approximate objective used in Eq. (\ref{eq_sub_problem_inexact}). 
Otherwise, $\Delta_k$ will be decreased because of  the poor agreement between the approximate and true objectives.

\section{Proposed Method}

\subsection{Inexact Sub-RN-CR Algorithm}
\label{sec:Inexact Sub-RN-CR}

We propose  to improve the subsampling-based construction of the RTR subproblem in Eq. (\ref{eq_sub_problem_inexact}) by cubic regularization. 
This gives rise to the minimization
\begin{equation}
	\label{eq_sub_problem}
	\bm{\eta}_k = \mathop{\arg\min}_{\bm{\eta}  \in T_{\mathbf{x}_k}\mathcal{M}} \hat m_k(\bm{\eta}),
\end{equation}
where
\begin{equation}
	\centering
	\begin{split}
		\hat{m}_k(\bm{\eta})  = \left\{ 
		\begin{aligned}
			\label{eq_m}
			&h_{\mathbf x_k}(\bm{\eta}) +\langle\mathbf{G}_k, \bm{\eta}\rangle_{\mathbf{x}_k},
			&&\text{ if }  \|\mathbf{G}_k\|_{\mathbf{x}_k}\ge\epsilon_g. \\
			&h_{\mathbf x_k}(\bm{\eta}) ,
			&&\text{ otherwise}.
		\end{aligned}
		\right.
	\end{split}
\end{equation}
Here,  $0<\epsilon_g <1$ is a user-specified parameter that plays a role in convergence analysis, which we will explain later.
The core objective component  $h_{\mathbf x_k}(\bm{\eta})$ is formulated by extending the adaptive cubic regularization technique \cite{cartis2011adaptive}, given as
\begin{equation}
	h_{\mathbf x_k}(\bm{\eta})  = f(\mathbf{x}_k)+\frac{1}{2}\langle\bm{\eta},\mathbf{H}_k[\bm{\eta}]\rangle_{\mathbf{x}_k} +\frac{\sigma_k}{3}\|\bm{\eta}\|_{\mathbf{x}_k}^3,  
	\label{eq_hx}
\end{equation}
with
\begin{equation}
	\label{eq_sigma}
	\sigma_{k}=\left\{\begin{array}{lr}\max(\frac{\sigma_{k-1}}{\gamma}, \epsilon_\sigma),\quad \text{if}\ \rho_{k-1}\ge\tau, \\ \gamma\sigma_{k-1},\quad\ \ \text{otherwise,}\end{array}\right.
\end{equation}
and
\begin{equation}
	\label{eq_success}
	\rho_k=\frac{ \hat{f}_k(\mathbf{0}_{\mathbf{x}_k})   -  \hat{f}_k(\bm{\eta}_k) }{\hat{m}_k(\mathbf{0}_{\mathbf{x}_k})-\hat{m}_k(\bm{\eta}_k)},
\end{equation}
where the subscript $k$ is used to highlight the pullback of $f$ at $\mathbf x_k$ as  $ \hat{f}_k(\cdot)$. 
Overall, there are four hyper-parameters to be set by the user, including the gradient norm threshold $0<\epsilon_{\sigma} <1$, the dynamic control parameter $\gamma>1$ to adjust the cubic regularization weight,  the model validity threshold $0 <\tau <1$, and the initial trust parameter $\sigma_0$. We will discuss the setup of  the algorithm in more detail.

\begin{algorithm}[t]
	\caption{Main Inexact Sub-RN-CR Solver  }
	\label{alg_inexact_rtr_arc}
	\begin{algorithmic}[1]
		\REQUIRE  $\epsilon_\sigma\in(0,1)$, $\gamma>1$, $0 <\tau <1$, $\sigma_0>0$, $0<\epsilon_g,\epsilon_H<1$.
		\FOR {$k = 1,2,...$}
		\STATE Sample the index sets $\mathcal{S}_g$ and $\mathcal{S}_H$
		\STATE Compute the subsampled gradient $\mathbf{G}_k$ and $\lambda_{min}(\mathbf{H}_k)$ based on Eqs. (\ref{eq_subsample_gradient})-(\ref{eq_subsample_hessian}) and (\ref{eq_minimum_eigenvalue_hessian})
		\IF {$\|\mathbf{G}_k\|_{\mathbf{x}_k}\le\epsilon_g$ and $\lambda_{min}(\mathbf{H}_k)\ge-\epsilon_H$}
		\STATE Return $\mathbf{x}_k$
		\ELSIF {$\|\mathbf{G}_k\|_{\mathbf{x}_k}\le\epsilon_g$}
		\STATE $\mathbf{G}_k=\mathbf{0}_{\mathbf{x}_k}$
		\ENDIF
		\STATE Inexactly solve $\bm\eta_k^* = \mathop{\arg\min}_{\bm\eta\in T_{\mathbf{x}_k}\mathcal{M}} \hat m_k(\bm\eta)$ by Algorithm \ref{alg_non_linear_Lan} or  Algorithm \ref{alg_non_linear_tCG} \label{sub_prob_code}
		\STATE Calculate $\rho_k=\frac{ \hat{f}_k(\mathbf{0}_{\mathbf{x}_k})   -  \hat{f}_k(\bm\eta_k^*) }{\hat{m}_k(\mathbf{0}_{\mathbf{x}_k})-\hat{m}_k(\bm\eta_k^*)}$
		\STATE Set $\mathbf{x}_{k+1}=\left\{\begin{array}{lr}R_{\mathbf{x}_k}(\bm\eta_k^*)\quad \text{if}\ \rho_k\ge\tau \\ \mathbf{x}_k\qquad\quad\ \ \text{otherwise}\end{array}\right.$
		\STATE \label{step_src_0} Set $\sigma_{k+1}=\left\{\begin{array}{lr}\max(\sigma_k/\gamma, \epsilon_\sigma)\quad \text{if}\ \rho_k\ge\tau \\ \gamma\sigma_k\quad\ \ \text{otherwise}\end{array}\right.$
		\ENDFOR
		\ENSURE $\mathbf{x}_k$
	\end{algorithmic}
\end{algorithm}

We expect the norm of the approximate gradient to approach $\epsilon_g$ with  $0<\epsilon_g <1$.
Following a similar treatment  in \cite{kasai2018inexact}, when the gradient norm is smaller than  $\epsilon_g$, the gradient-based term is ignored. This is important to the convergence analysis  shown in the next section.

The trust region radius $\Delta_k$ is no longer explicitly defined, but  replaced by the cubic regularization term $\frac{\sigma_k}{3}\|\bm{\eta}\|_{\mathbf{x}_k}^3$, where $\sigma_k$ is related to a Lagrange multiplier on a cubic trust-region constraint.
Naturally, the smaller $\sigma_k$ is, the larger a moving step is allowed. Benefits of cubic regularization have been shown in \cite{griewank1981modification,kohler2017sub}. 
It can not only accelerate the local convergence especially when the Hessian is singular, but also help escape better  strict saddle points than the TR methods, providing stronger convergence properties.

The cubic  term $\frac{\sigma_k}{3}\|\bm{\eta}\|_{\mathbf{x}_k}^3$ is equipped with a dynamic penalization control through the adaptive trust quantity $\sigma_k\geq 0$. 
The value of $\sigma_k$  is determined by examining how successful each iteration $k$ is. An iteration $k$ is considered successful if $\rho_k\ge \tau$, and unsuccessful otherwise, where
the value of $\rho_k$  quantifies  the agreement between the changes of the approximate objective $\hat{m}_k(\bm\eta)$ and the true objective $f(\mathbf x)$. 
The larger $\rho_k$ is, the more effective the approximate model is. 
Given $\gamma >1$, in an unsuccessful iteration, $\sigma_k$ is increased to $\gamma\sigma_k$ hoping to obtain a more accurate approximation in the next iteration. 
On the opposite, $\sigma_k$ is decreased to $\frac{\sigma_k}{\gamma}$, relaxing the approximation in a successful iteration, but it is still restricted within the lower bound $\epsilon_\sigma$. This bound $\epsilon_\sigma$ helps avoid solution candidates with overly large norms $\|\bm{\eta}_k\|_{\mathbf{x}_k}$ that can  cause an unstable optimization. 
Below  we formally define what an (un)successful iteration is, which will be used in our later analysis.
\begin{defn}[Successful and Unsuccessful Iterations]
	An iteration $k$ in Algorithm \ref{alg_inexact_rtr_arc} is considered successful if the agreement $\rho_k\ge \tau$, and unsuccessful if $\rho_k< \tau$. In addition, based on Step (\ref{step_src_0}) of Algorithm \ref{alg_inexact_rtr_arc}, a successful iteration has $\sigma_{k+1}\le\sigma_k$, while  an unsuccessful  one has $\sigma_{k+1}>\sigma_k$.
\end{defn}

\subsection{Optimality Examination}

The stopping condition of the algorithm follows  the  definition of $\left(\epsilon_g,\epsilon_H\right)$-optimality \cite{nocedal2006numerical},  stated as below.
\begin{defn}[$\left(\epsilon_g,\epsilon_H\right)$-optimality]
	Given $0<\epsilon_g,\epsilon_H<1$,   a solution  $\mathbf x$ satisfies   $\left(\epsilon_g,\epsilon_H\right)$-optimality if 
	\begin{equation}
		\left\|{\rm grad}f(\mathbf{x})\right\|_{\mathbf{x}}\le\epsilon_g\quad {\rm and}\quad {\rm Hess}f(\mathbf{x})[\bm{\eta}]\succeq-\epsilon_H\mathbf{I},
	\end{equation}
	for all $\bm{\eta}\in T_{\mathbf{x}}\mathcal{M}$, where $\mathbf{I}$ is  an identity matrix.
	\label{ass_optimality}
\end{defn}
This is a relaxation and a manifold extension of  the standard second-order optimality conditions $\left\|{\rm grad}f(\mathbf{x})\right\|_{\mathbf{x}}=0$ and $\text{Hess}f(\mathbf{x})\succeq0$  in Euclidean spaces \cite{mokhtari2018escaping}. 
The algorithm stops  (1) when the gradient norm is sufficiently small and  (2) when the Hessian is sufficiently close to being positive semidefinite.

To examine the Hessian, we follow  a similar approach as in \cite{han2021riemannian} by assessing the solution of the following minimization problem:
\begin{align}
	\label{eq_minimum_eigenvalue_hessian}
	\lambda_{min}(\mathbf{H}_k) := &\min_{\bm\eta\in T_\mathbf{x}\mathcal{M},\ \left\|\bm\eta\right\|_\mathbf{x}=1}\langle\bm{\eta} ,\mathbf{H}_k[\bm{\eta} ]\rangle_{\mathbf{x}_k},
\end{align}
which resembles the smallest eigenvalue of the Riemannian Hessian.
As a result, the algorithm stops when $\left\|{\rm grad}f(\mathbf{x})\right\|_{\mathbf{x}}\le\epsilon_g$ (referred to as the gradient condition)  and  when $\lambda_{min}(\mathbf{H}_k) \ge-\epsilon_H$ (referred to as the Hessian condition), where $\epsilon_g, \epsilon_H  \in(0,1)$ are the user-set stopping parameters. Note, we use the same $\epsilon_g$  for thresholding as in Eq. (\ref{eq_m}). 
Pseudo  code of the complete Inexact Sub-RN-CR  is provided in Algorithm \ref{alg_inexact_rtr_arc}.

\subsection{Subproblem Solvers}

The step (\ref{sub_prob_code}) of Algorithm \ref{alg_inexact_rtr_arc} requires to solve the subproblem in Eq. (\ref{eq_sub_problem}). 
We rewrite its objective function  $\hat{m}_k(\bm\eta)$ as below  for the convenience of explanation:
\begin{equation}
	\label{sub_obj}
	\bm\eta_{k}^*= \arg\min_{\bm\eta\in T_{\mathbf{x}_k}\mathcal{M}}    f(\mathbf{x}_k)+ \delta\langle\mathbf{G}_k, \bm{\eta}\rangle_{\mathbf{x}_k}+ \frac{1}{2}\langle\bm{\eta},\mathbf{H}_k[\bm{\eta}]\rangle_{\mathbf{x}_k} +\frac{\sigma_k}{3}\|\bm{\eta}\|_{\mathbf{x}_k}^3 ,  
\end{equation}
where $\delta=1$ if $\|\mathbf{G}_k\|_{\mathbf{x}_k}\ge\epsilon_g$, otherwise $\delta=0$. 
We demonstrate two solvers commonly used in practice.  

\subsubsection{The Lanczos Method} \label{sec_lanczos_method}

The Lanczos method \cite{1999Solving}  has been widely used to solve the cubic regularization problem in  Euclidean spaces \cite{xu2020newton,kohler2017sub,cartis2011adaptive,jia2021solving} and been recently adapted to  Riemannian spaces   \cite{agarwal2021adaptive}.  
Let $D$ denote the manifold dimension. 
Its core operation is to construct a Krylov subspace $\mathcal{K}_D$, of which the basis $\{\mathbf{q}_i\}_{i=1}^D$ spans the tangent space $\bm\eta\in T_{\mathbf{x}_k}\mathcal{M}$. 
After expressing the solution  as an element in $\mathcal{K}_D$, i.e., $ \bm\eta :=\sum_{i=1}^D y_i\mathbf{q}_i$,  the minimization problem in Eq. (\ref{sub_obj}) can be converted to one in Euclidean spaces $\mathbb{R}^D$, as
\begin{equation}
	\label{eq_lanczos_obj}
	\mathbf{y}^* =  \arg\min_{\mathbf{y}\in \mathbb{R}^D}\ y_1\delta\left\|\mathbf{G}_k\right\|_{\mathbf{x}_k}+ \frac{1}{2}\mathbf{y}^T\mathbf{T}_D \mathbf{y} + \frac{\sigma_k}{3}\left\|\mathbf{y}\right\|_2^3,
\end{equation}
where $\mathbf{T}_D \in \mathbb{R}^{D\times D}$ is a symmetric tridiagonal matrix determined by the basic construction process, e.g., Algorithm 1 of \cite{jia2021solving}. 
We provide a  detailed derivation of Eq. (\ref{eq_lanczos_obj})  in Appendix \ref{app_lan}.
The global solution of this converted problem, i.e., $\mathbf{y}^*=\left[y_1^{*},y_2^{*},\ldots,y_D^{*}\right]$,  can be found by many existing techniques, see \cite{press2007chapter}. 
We employ the Newton root finding method adopted by Section 6 of \cite{cartis2011adaptive}, which was originally proposed by \cite{agarwal2021adaptive}. 
It reduces the problem to a univariate root finding problem.
After this, the global solution of the subproblem  is computed by  $ \bm\eta_{k}^* =\sum_{i=1}^D y_i^*\mathbf{q}_i$. 

\begin{algorithm}[t]
	\caption{Subproblem Solver by Lanczos \cite{agarwal2021adaptive}}
	\label{alg_non_linear_Lan}
	\begin{algorithmic}[1]
		\REQUIRE $ \mathbf{G}_k$ and $\mathbf{H}_k[\bm\eta]$,  $\kappa_\theta\in (0,1/6]$, $\sigma_k$.
		\STATE $\mathbf{q}_1=\frac{\mathbf{G}_k}{\left\|\mathbf{G}_k\right\|_{\mathbf{x}_k}}$, $\mathbf{T}=\mathbf{0}\in \mathbb{R}^{D\times D}$, $\alpha = \langle\mathbf{q}_1, \mathbf{H}_k[\mathbf{q}_1]\rangle_{\mathbf{x}_k}$, $T_{1,1}=\alpha$
		\STATE $\mathbf{r} = \mathbf{H}_{k}[\mathbf{q}_1]-\alpha\mathbf{q}_1$
		\FOR {$l = 1,2, \ldots,D$}
		\STATE Obtain $\mathbf{y}^*$ by optimizing Eq. (\ref{eq_lanczos_obj}) with $D=l$ using Newton root finding 
		\STATE $\beta = \left\|\mathbf{r}\right\|_{\mathbf{x}_k}$
		\STATE $\mathbf{q}_{l+1}=-\frac{\mathbf{r}}{\beta}$
		\STATE $\alpha = \langle\mathbf{q}_{l+1},\mathbf{H}_{k}[\mathbf{q}_{l+1}]-\beta\mathbf{q}_l\rangle_{\mathbf{x}_k}$
		\STATE $\mathbf{r}=\mathbf{H}_{k}[\mathbf{q}_{l+1}]-\beta\mathbf{q}_l-\alpha\mathbf{q}_{l+1}$
		\STATE $T_{l,l+1}=T_{l+1,l}=\beta$, $T_{l+1,l+1}=\alpha$
		\IF { Eq. (\ref{eq_stepsize_stop_additional}) is satisfied} 
		\STATE Return $\sum_{i=1}^l y_i^*\mathbf{q}_i$
		\ENDIF
		\ENDFOR
	\end{algorithmic}
\end{algorithm}

In practice, when the manifold dimension $D$  is  large, it is more practical to find a  good solution rather than the global solution.  
By  working with a lower-dimensional Krylov subspace $\mathcal{K}_l$ with $l<D$, one can derive Eq. (\ref{eq_lanczos_obj}) in $\mathbb{R}^l$, and its  solution  $\mathbf{y}^{*l} $ results in a subproblem solution  $\bm\eta_{k}^{*l} =  \sum_{i=1}^l y_i^{*l}\mathbf{q}_i$. 
Both the global  solution $\bm\eta_{k}^* $  and the approximate solution $\bm\eta_{k}^{*l}$ are always guaranteed to be at least as good as the solution obtained by performing a line search along the gradient direction, i.e.,
\begin{equation}
	\label{eq_lanczos_cauchy}
	\hat{m}_k\left(\bm\eta_{k}^* \right) \leq  \hat{m}_k\left(\bm\eta_{k}^{*l}\right) \le	\min_{\alpha\in\mathbb{R}} \hat{m}_k(\alpha \mathbf{G}_k), 
\end{equation}
because  $\alpha \mathbf{G}_k$ is a common basis vector shared by all the constructed Krylov subspaces $\{\mathcal{K}_l\}_{i=1}^D$.  
We provide pseudo  code for the  Lanczos  subproblem solver   in Algorithm \ref{alg_non_linear_Lan}.

To benefit practitioners and improve understanding of the Lanczos solver, we analyse the gap between a practical solution $\bm\eta_{k}^{*l}$ and the global minimizer $\bm\eta_{k}^* $.
Firstly, we define $\lambda_{max}(\mathbf{H}_k) $ in a similar manner to $\lambda_{min}(\mathbf{H}_k) $  as in Eq. (\ref{eq_minimum_eigenvalue_hessian}). 
It resembles the largest eigenvalue of the Riemannian Hessian,  as 
\begin{align}
	\label{eq_maximum_eigenvalue_hessian}
	\lambda_{max}(\mathbf{H}_k) := &\max_{\bm\eta\in T_\mathbf{x}\mathcal{M},\ \left\|\bm\eta\right\|_\mathbf{x}=1}\langle\bm{\eta} ,\mathbf{H}_k[\bm{\eta} ]\rangle_{\mathbf{x}_k}.
\end{align}
We denote a degree-$l$ polynomial  evaluated at $\mathbf{H}_k[\bm\eta]$ by $p_{l}\left(\mathbf{H}_k\right)[\bm\eta]$, such that
\begin{equation}
	p_{l}\left(\mathbf{H}_k\right)[\bm\eta]:=c_l\mathbf{H}_k^l[\bm\eta]+c_{l-1}\mathbf{H}_k^{l-1}[\bm\eta]+\cdots+c_1 \mathbf{H}_k[\bm\eta]+c_0\bm\eta,
\end{equation}
for some coefficients  $c_0, \; c_1, \;\ldots, \; c_l \in \mathbb{R}$. The quantity $\mathbf{H}_k^l[\bm\eta]$ is recursively defined by $\mathbf{H}_k\left[\mathbf{H}_k^{l-1}\left[\bm\eta\right]\right]$ for $l=2,3,\ldots$
We define below an induced norm, as 
\begin{equation}
	\label{eq_induced_pH_norm}
	\left\|p_{l+1}\left(\tilde{\mathbf{H}}_k\right)-{\rm Id}\right\|_{\mathbf{x}_k} =\sup_{\substack{\bm{\eta}\in T_{\mathbf{x}_k}\mathcal{M} \\ \|\bm{\eta}\|_{\mathbf{x}_k}\ne0}}\frac{\left\|\left(p_{l+1}\left(\tilde{\mathbf{H}}_k\right)-{\rm Id}\right)[\bm{\eta}]\right\|_{\mathbf{x}_k}}{\left\|\bm\eta\right\|_{\mathbf{x}_k}},
\end{equation}
where the identity mapping operator works as ${\rm Id}[\bm{\eta}] = \bm\eta$.
Now we are ready to present our result in the following lemma.
\begin{lem}[Lanczos Solution Gap]
	Let  $\bm\eta_{k}^* $ be the global minimizer  of the subproblem $\hat{m}_k$ in Eq. (\ref{eq_sub_problem}). Denote the  subproblem without cubic regularization in Eq. (\ref{eq_sub_problem_inexact}) by $\bar{m}_k$ and  let $\bar{\bm\eta}_k^*$ be its global minimizer.
	For each $l>0$,  the solution $\bm\eta_{k}^{*l}$ returned by Algorithm  \ref{alg_non_linear_Lan} satisfies 
	\begin{equation}
		\hat{m}_k\left(\bm\eta_k^{*l}\right)-\hat{m}_k\left(\bm\eta_k^*\right)  \le \frac{4\lambda_{max}
			\left(\tilde{\mathbf{H}}_k\right)}{\lambda_{min}\left(\tilde{\mathbf{H}}_k\right)}\left(\bar{m}_k\left(\mathbf{0}_{\mathbf{x}_k}\right)-\bar{m}_k\left(\bar{\bm\eta}_k^*\right)\right)\phi_l\left(\tilde{\mathbf{H}}_k\right)^2,
	\end{equation}
	where $\tilde{\mathbf{H}}_k[\bm\eta]:=(\mathbf{H}_k + \sigma_k\left\|\bm\eta_{k}^*\right\|_{\mathbf{x}_k}{\rm Id})[\bm\eta]$ for a moving direction $\bm\eta$, and $\phi_l\left(\tilde{\mathbf{H}}_k\right)$ is an upper bound of the induced norm $\left\|
	p_{l+1}\left(\tilde{\mathbf{H}}_k\right)-{\rm Id}\right\|_{\mathbf{x}_k}$.  
	\label{lem_krylov}	
\end{lem}
We provide its proof  in Appendix \ref{app_lan_gap}.  In Euclidean spaces, \cite{carmon2018analysis} has shown that $\phi_l(\mathbf{H}_k)=2e^{\frac{-2(l+1)}{\sqrt{ \lambda_{max}(\mathbf{H}_k)\lambda_{min}^{-1}(\mathbf{H}_k)} }}$. 
With the help of Lemma \ref{lem_krylov}, this could serve as a reference to gain an understanding of the solution quality for the Lanczos method in  Riemannian spaces.

\subsubsection{The Conjugate Gradient Method}

We experiment with an alternative subproblem solver by adapting the non-linear conjugate gradient   technique in Riemannian spaces.
It starts from the initialization of $\bm\eta_{k}^0=\mathbf{0}_{\mathbf{x}_k}$ and the first  conjugate direction  $\mathbf{p}_1 = -\mathbf{G}_k$ (negative gradient direction). 
At each inner iteration $i$ (as opposed to the outer iteration $k$ in the main algorithm), it solves the minimization problem with one input variable:
\begin{equation}
	\label{step_search0}
	\alpha_i^* = \arg\min_{\alpha \geq 0} \hat{m}_k \left(\bm\eta_k^{i-1}+\alpha\mathbf{p}_i\right).
\end{equation}
The global solution of this  one-input minimization problem  can be computed by zeroing the derivative of $\hat{m}_k$ with respect to $\alpha$, resulting in a polynomial equation of $\alpha$, which can then be solved by eigen-decomposition \cite{edelman1995polynomial}.  
Its  root that possesses the minimal value of $\hat{m}_k$ is retrieved. The algorithm updates the next conjugate direction $\mathbf{p}_{i+1}$ using the returned $\alpha_i^*$ and $\mathbf{p}_i$. 
Pseudo  code for  the conjugate gradient  subproblem solver  is provided in Algorithm \ref{alg_non_linear_tCG}.

\begin{algorithm}[t]
	\caption{subproblem Solver by Non-linear Conjugate Gradient }
	\label{alg_non_linear_tCG}
	\begin{algorithmic}[1]
		\REQUIRE subproblem $\hat{m}_k(\bm\eta)$, $ \mathbf{G}_k$,  $m$, $\kappa, \theta> 0$,  $\kappa_\theta\in (0,1/6]$.
		\STATE $\bm\eta_{k}^0=\mathbf{0}_{\mathbf{x}_k}$, $\mathbf{r}_0=\mathbf{G}_k$, $\mathbf{p}_1=-\mathbf{r}_0$, $\mathbf{x}_k^0 = \mathbf{x}_k$
		\FOR {$i = 1,2, \ldots,m$}
		\STATE \label{CG_opt_alpha} Solve Eq. (\ref{step_search0}) by zeroing its derivative and solving the resulting polynomial equation  \label{CG_minstep}
		\IF { $\alpha_i^* \leq  10^{-10}$} \label{CG_stopping1}
		\label{CG_return_1}
		\STATE Return $\bm\eta_{k}^*  =\bm\eta_{k}^{i-1}$
		\ENDIF		
		\STATE $\bm\eta_{k}^{i}=\bm\eta_{k}^{i-1}+\alpha_i^*\mathbf{p}_i$
		\IF { Eq. (\ref{eq_stepsize_stop_additional}) is satisfied} 
		\STATE Return $\bm\eta_{k}^*  =\bm\eta_{k}^{i}$
		\ENDIF
		\STATE \label{CG_update_r} $\mathbf{r}_{i}=\mathbf{r}_{i-1}+\alpha_i^*\mathbf{H}_k^{i-1}[\mathbf{p}_i]$
		\STATE $\mathbf{x}_k^i = R_{\mathbf{x}_k^{i-1}}(\alpha_i^*\mathbf{p}_i)$
		\IF {  Eq. (\ref{CG_ref}) is met} \label{CG_stopping2}
		\STATE Return $\bm\eta_{k}^*  =\bm\eta_{k}^{i}$
		\ENDIF 
		\STATE   \label{CG_beta}  Compute $\beta_i$ by Eq. (\ref{beta_PRP})  
		\STATE \label{CG_p_update} $\mathbf{p}_{i+1}=-\mathbf{r}_{i}+\beta_i\mathcal{P}_{\alpha_i^*\mathbf{p}_i}\mathbf{p}_{i}$
		\ENDFOR
		\ENSURE $\bm\eta_{k}^*  = \bm\eta_{k}^{m}$
	\end{algorithmic}
\end{algorithm}

Convergence of a conjugate gradient  method largely depends on how  its conjugate direction is updated. 
This is controlled by the setting of $\beta_i$ for calculating $\mathbf{p}_{i +1}= -\mathbf{r}_{i}+\beta_i\mathcal{P}_{\alpha_i^*\mathbf{p}_i}\mathbf{p}_{i}$  in Step  (\ref{CG_p_update}) of Algorithm \ref{alg_non_linear_tCG}.  
Working in Riemannian spaces under the subsampling setting,  it has been proven by \cite{sakai2021sufficient}  that, when  the Fletcher-Reeves formula  \cite{fletcher1964function}  is used, i.e.,
\begin{equation}
	\beta_i=\frac{\left\|\mathbf{G}_k^i\right\|_{\mathbf{x}_k^i}^2}{\left\|\mathbf{G}_k^{i-1}\right\|_{\mathbf{x}_{k}^{i-1}}^2},
\end{equation}
where
\begin{equation}
	\mathbf{G}_k^i=\frac{1}{|\mathcal{S}_g|}\sum_{j\in\mathcal{S}_g}{\rm grad} f_j\left(\mathbf{x}_k^i\right),
\end{equation}
a conjugate gradient  method can converge to a stationary point with $\lim_{i\to\infty}$ $\left\|{\rm grad} f\left(\mathbf{x}_k^i\right)\right\|_{\mathbf{x}_k^i}=0$.  Working in Euclidean spaces, \cite{wei2006convergence} has shown that the Polak–Ribiere–Polyak  formula, i.e.,
\begin{equation}
	\beta_i=\frac{\left\langle\nabla f \left(\mathbf{x}_k^i \right), \nabla f\left(\mathbf{x}_k^i\right) - \nabla f\left(\mathbf{x}_k^{i-1}\right)\right\rangle}{\left\|\nabla f\left(\mathbf{x}_k^{i-1}\right)\right\|^2},
\end{equation}
performs better than the Fletcher-Reeves formula. 
Building upon these, we propose to compute $\beta_i$ by a modified Polak–Ribiere–Polyak  formula in Riemannian spaces in Step (\ref{CG_beta}) of Algorithm \ref{alg_non_linear_tCG}, given as
\begin{equation}
	\label{beta_PRP}
	\beta_i=\frac{\left\langle\mathbf{r}_{i},\mathbf{r}_{i}-\frac{\left\|\mathbf{r}_{i}\right\|_{\mathbf{x}_k^i}}{\left\|\mathbf{r}_{i-1}\right\|_{\mathbf{x}_k^{i-1}}}\mathcal{P}_{\alpha_i^*\mathbf{p}_i}\mathbf{r}_{i-1}\right\rangle_{\mathbf{x}_k^i}}{2\left\langle\mathbf{r}_{i-1},\mathbf{r}_{i-1}\right\rangle_{\mathbf{x}_k^{i-1}}}.
\end{equation}
We prove that  the resulting algorithm converges to a stationary point, and present the convergence result in Theorem \ref{thm_tCG_converge}, with its proof deferred to  Appendix \ref{app_CGconvergence}.
\begin{thm}[Convergence of the Conjugate Gradient Solver]
	Assume that the step size $\alpha_i^*$ in Algorithm \ref{alg_non_linear_tCG} satisfies the strong Wolfe conditions \cite{hosseini2018line}, i.e., given a smooth function $f:\mathcal{M}\to \mathbb{R}$, it has
	\begin{align}
		\label{eq_wolfe_1}
		f\left(R_{\mathbf{x}_k^{i-1}}(\alpha_i^*\mathbf{p}_i)\right) \le&\; f \left(\mathbf{x}_k^{i-1}\right) + c_1\alpha_i^* \left\langle\mathbf{G}_k^{i-1},\mathbf{p}_i\right\rangle_{\mathbf{x}_k^{i-1}},\\
		\label{eq_wolfe_2}
		\left|\left\langle\mathbf{G}_k^{i}, \mathcal{P}_{\alpha_i^*\mathbf{p}_i}(\mathbf{p}_i)\right\rangle_{\mathbf{x}_k^i}\right| \le&\;  -c_2\left\langle\mathbf{G}_k^{i-1},\mathbf{p}_i\right\rangle_{\mathbf{x}_k^{i-1}},
	\end{align}
	with $0<c_1<c_2<1$. When Step (\ref{CG_beta})  of Algorithm \ref{alg_non_linear_tCG} computes $\beta_i$ by Eq. (\ref{beta_PRP}), 
	Algorithm \ref{alg_non_linear_tCG} converges to a stationary point, i.e., $\lim_{i\to\infty} \left\|\mathbf{G}_k^i\right\|_{\mathbf{x}_k^i}=0$.
	\label{thm_tCG_converge}
\end{thm}

In practice, Algorithm \ref{alg_non_linear_tCG} terminates when there is no obvious change in the solution, which is examined in Step (\ref{CG_stopping1}) by checking whether the step size is sufficiently small, i.e., whether $\alpha_i \leq10^{-10}$ (Section 9 in \cite{agarwal2021adaptive}). 
To improve the convergence rate, the algorithm also terminates when  $\mathbf{r}_{i}$ in Step  (\ref{CG_stopping2}) is sufficiently small, i.e., following a  classical criterion \cite{absil2007trust},  to check whether  
\begin{equation}
	\label{CG_ref}
	\left\|\mathbf{r}_{i}\right\|_{\mathbf{x}_k^i}\le\left\|\mathbf{r}_0\right\|_{\mathbf{x}_k}\min\left(\left\|\mathbf{r}_0\right\|_{\mathbf{x}_k}^\theta, \kappa\right),
\end{equation}
for some $\theta,\kappa>0$.

\subsubsection{Properties of the Subproblem Solutions}
\label{sec:sub_solution}

In Algorithm \ref{alg_non_linear_Lan},  the basis $\{\mathbf{q}_i\}_{i=1}^D$ is constructed successively starting from $q_1 = \mathbf{G}_k$, while the converted problem in Eq. (\ref{eq_lanczos_obj}) is solved for each $\mathcal{K}_l$ starting from $l=1$.   
This process allows up to $D$ inner iterations. The solution $ \bm\eta_{k}^{*}$ obtained in the last inner iteration where $l=D$ is the global minimizer over $\mathbb{R}^D$. 
Differently, Algorithm \ref{alg_non_linear_tCG} converges to a stationary point as proved in Theorem \ref{thm_tCG_converge}. 
In practice, a maximum inner iteration number $m$ is set in advance. Algorithm \ref{alg_non_linear_tCG}  stops when it reaches the maximum iteration number or converges to a status where the change in either the solution or the conjugate direction is very small. 

The convergence property of the main algorithm presented in Algorithm \ref{alg_inexact_rtr_arc} relies on the quality of the subproblem solution.  
Before discussing it, we first familiarize the reader with the classical TR concepts of Cauchy steps and eigensteps but defined for the Inexact  RTR problem introduced in Section \ref{inexact_rtr}.
According to  Section 3.3 of \cite{boumal2019global}, when $\hat{m}_k$ is the RTR subproblem, the closed-form Cauchy step $\hat{\bm\eta}_k^C$ is an improving direction defined by
\begin{equation}
	\hat{\bm\eta}_k^C := \min \Bigg(\frac{\left\|\mathbf{G}_k\right\|_{\mathbf{x}_k}^2}{\langle\mathbf{G}_k, \mathbf{H}_k[\mathbf{G}_k]\rangle}_{\mathbf{x}_k}, \frac{\Delta_k}{\left\|\mathbf{G}_k\right\|_{\mathbf{x}_k}}\Bigg)\mathbf{G}_k.
\end{equation}
It points towards the gradient direction with an optimal step size computed by the $\min(\cdot,\cdot)$ operation, and follows the form of the general Cauchy step defined by 
\begin{equation}
	\label{eq_cauchy_general_definition}
	\bm\eta_k^C:=\arg\min_{\alpha\in\mathbb{R}}(\hat{m}_k(\alpha\mathbf{G}_k))\mathbf{G}_k.
\end{equation}
According to Section 2.2 of \cite{kasai2018inexact}, for some $\nu\in(0,1]$, the eigenstep $\bm \eta_k^E$ satisfies
\begin{equation}
	\left \langle \bm \eta_k^E, \mathbf{H}_k\left[\bm \eta_k^E \right] \right\rangle_{\mathbf{x}_k}\le\nu\lambda_{min}(\mathbf{H}_k)\left\| \bm \eta_k^E\right\|^2_{\mathbf{x}_k}<0.
	\label{eq_eigen_condition}
\end{equation}
It is an approximation of the negative curvature direction by an eigenvector associated with the smallest negative eigenvalue.

The following three assumptions on the subproblem solution are needed by the convergence analysis later.  
We  define the induced norm for the Hessian as below:
\begin{equation}
	\label{eq_induced_hessian_norm}
	\|\mathbf{H}_k\|_{\mathbf{x}_k} =\sup_{\substack{\bm{\eta}\in T_{\mathbf{x}_k}\mathcal{M} \\ \|\bm{\eta}\|_{\mathbf{x}_k}\ne0}}\frac{\left\|\mathbf{H}_k[\bm{\eta}]\right\|_{\mathbf{x}_k}}{\left\|\bm\eta\right\|_{\mathbf{x}_k}}.
\end{equation}
\begin{assu}[Sufficient Descent Step] Given the Cauchy step $\bm\eta _k^C$ and the eigenstep $\bm\eta _k^E$ for $\nu\in(0,1]$,  assume the subproblem solution $\bm\eta_{k}^{*}$ satisfies the Cauchy condition 
	\begin{equation}
		\label{eq_cauchy_point}
		\hat{m}_k \left(\bm\eta_{k}^{*} \right)\le \hat{m}_k\left(\bm\eta _k^C\right)\le \hat{m} _k\left(\mathbf{0}_{\mathbf{x}_k}\right)- \max(a_k,b_k),
	\end{equation}
	and the eigenstep condition 
	\begin{equation}
		\label{eq_eigen_point}
		\hat{m}_k \left(\bm\eta_{k}^{*} \right)\le \hat{m}_k\left(\bm\eta _k^E\right)\le \hat{m} _k\left(\mathbf{0}_{\mathbf{x}_k}\right) -c_k,
	\end{equation}
	where
	\begin{align}
		\label{eq_a} a_k=\; &     \frac{\|\mathbf{G}_k\|_{\mathbf{x}_k}}{2\sqrt{3}}\min\Bigg (\frac{\|\mathbf{G}_k\|_{\mathbf{x}_k}}{\|\mathbf{H}_k\|_{\mathbf{x}_k}},\sqrt{\frac{\|\mathbf{G}_k\|_{\mathbf{x}_k}}{\sigma_k}}\Bigg ), \\
		\label{eq_b} b_k=\; & \frac{\left\|\bm\eta _k^C\right\|_{\mathbf{x}_k}^2}{12}\left(\sqrt{\|\mathbf{H}_k\|_{\mathbf{x}_k}^2+4\sigma_k\|\mathbf{G}_k\|_{\mathbf{x}_k}}-\|\mathbf{H}_k\|_{\mathbf{x}_k}\right) ,\\
		\label{eq_c} c_k=\; &  \frac{\nu \left |\lambda_{min}(\mathbf{H}_k) \right|}{6}\max\Bigg (\left\|\bm\eta _k^E\right\|_{\mathbf{x}_k}^2,\frac{\nu^2|\lambda_{min}(\mathbf{H}_k)|^2}{\sigma_k^2}\Bigg ).
	\end{align}
	\label{assu_cauchy_eigen_point}
\end{assu}
The two last  inequalities in Eqs. (\ref{eq_cauchy_point}) and (\ref{eq_eigen_point})  concerning the  Cauchy step and eigenstep are derived in Lemma 6 and Lemma 7 of \cite{xu2020newton}.
Assumption \ref{assu_cauchy_eigen_point} generalizes Condition 3 in \cite{xu2020newton} to the Riemannian case.
It assumes that the  subproblem solution $\bm\eta_{k}^{*}$ is better than the Cauchy step and eigenstep, decreasing  more the value of the subproblem objective function.  
The following two assumptions enable a stronger convergence result for  Algorithm \ref{alg_inexact_rtr_arc}, which will be used in the proof of Theorem \ref{theorem3}.
\begin{assu}[Sub-model Gradient Norm \cite{cartis2011adaptive,kohler2017sub}]  Assume the subproblem solution $\bm\eta_{k}^{*}$ satisfies 
	\begin{equation}
		\label{eq_assu_g}
		\left\|\nabla_{\bm\eta}\hat{m}_k\left(  \bm\eta_{k}^{*} \right)\right \|_{\mathbf{x}_k}\le \kappa_\theta\min\left(1, \left\| \bm\eta_{k}^{*}  \right\|_{\mathbf{x}_k} \right) \|\mathbf{G}_k\|_{\mathbf{x}_k},
	\end{equation}
	where $\kappa_\theta\in (0,1/6]$.
	\label{assu_g}
\end{assu}

\begin{assu}[Approximate Global Minimizer \cite{cartis2011adaptive,yao2021inexact}] Assume the subproblem solution $\bm\eta_{k}^{*}$  satisfies 
	\begin{align}
		\label{eq_sufficient_step0}
		\langle \mathbf{G}_k, \bm\eta_{k}^{*}   \rangle_{\mathbf{x}_k}+	\langle \bm\eta_{k}^{*} ,\mathbf{H}_k[  \bm\eta_{k}^{*} ]\rangle_{\mathbf{x}_k}+\sigma_k\|\bm\eta_{k}^{*} \|_{\mathbf{x}_k}^3 = 0, \\
		\langle\bm\eta_{k}^{*} ,\mathbf{H}_k[ \bm\eta_{k}^{*}  ]\rangle_{\mathbf{x}_k}+ \sigma_k\| \bm\eta_{k}^{*} \| _{\mathbf{x}_k}^3 \ge 0.
		\label{eq_sufficient_step}
	\end{align}
	\label{assu_cg}
\end{assu}

Driven by these assumptions, we characterize the subproblem solutions and present the results  in the following lemmas. Their proofs  are deferred to Appendix \ref{app_sub_solution}.

\begin{lem}[Lanczos Solution]
	The subproblem solution obtained by Algorithm \ref{alg_non_linear_Lan} when being executed $D$ (the dimension of $\mathcal{M}$) iterations satisfies Assumptions \ref{assu_cauchy_eigen_point}, \ref{assu_g}, \ref{assu_cg}. 
	When being executed $l<D$ times, the solution satisfies  the Cauchy condition in Assumption \ref{assu_cauchy_eigen_point}, also Assumptions \ref{assu_g} and \ref{assu_cg}.
	\label{lem_lanczos_sol_property}
\end{lem}

\begin{lem}[Conjugate Gradient Solution]
	\label{lem_tcg_sol_property}
	The subproblem solution  obtained by Algorithm \ref{alg_non_linear_tCG} satisfies the Cauchy condition in Assumption \ref{assu_cauchy_eigen_point}. Assuming $\hat{m}_k(\bm\eta)\approx f(R_{\mathbf{x}_k}(\bm\eta))$, it also satisfies  
	\begin{equation}
		\left\|\nabla_{\bm\eta}\hat{m}_k\left(  \bm\eta_{k}^{*} \right)\right \|_{\mathbf{x}_k}\approx 0,
	\end{equation}
	and approximately the first condition of Assumption \ref{assu_cg}, as
	\begin{equation}
		\label{assump3_app}
		\langle\mathbf{G}_k, \bm\eta_k^*\rangle_{\mathbf{x}_k}+\langle\bm\eta_k^*, \mathbf{H}_k[\bm\eta_k^*]\rangle_{\mathbf{x}_k}+\sigma_k\left\|\bm\eta_{k}^{*}\right\|_{\mathbf{x}_k}^3\approx 0.
	\end{equation}
\end{lem}
In practice, Algorithm \ref{alg_non_linear_Lan} based on lower-dimensional Krylov subspaces with $l<D$ returns a less optimal solution, while Algorithm \ref{alg_non_linear_tCG}  returns at most a local minimum. 
They are not guaranteed to satisfy the eigenstep condition in Assumption \ref{assu_cauchy_eigen_point}. But the early-returned solutions from Algorithm \ref{alg_non_linear_Lan} still satisfy Assumptions \ref{assu_g} and \ref{assu_cg}. However, solutions from Algorithm \ref{alg_non_linear_tCG} do not satisfy these two Assumptions exactly, but they could get close in an approximate manner. 
For instance, according to Lemma \ref{lem_tcg_sol_property}, $\left\|\nabla_{\bm\eta}\hat{m}_k\left(  \bm\eta_{k}^{*} \right)\right \|_{\mathbf{x}_k}\approx 0$, and we know that $0 \leq \kappa_\theta\min\left(1, \left\| \bm\eta_k^*  \right\|_{\mathbf{x}_k} \right) \|\mathbf{G}_k\|_{\mathbf{x}_k}$; thus, there is a fair chance for Eq. (\ref{eq_assu_g}) in Assumption \ref{assu_g}  to be met by the solution from Algorithm \ref{alg_non_linear_tCG}.  
Also, given that $\bm\eta_k^*$ is a descent direction, it has $\langle\mathbf{G}_k, \bm\eta_k^*\rangle_{\mathbf{x}_k}\le 0$. Combining this with Eq. (\ref{assump3_app}) in Lemma \ref{lem_tcg_sol_property},  there is a fair chance for  Eq. (\ref{eq_sufficient_step0}) in Assumption \ref{assu_g}  to be met. 
We present experimental results in Section \ref{sec_assum_satisfaction}, showing empirically to what extent  the different solutions  satisfy or are close to the eigenstep condition in Assumption \ref{assu_cauchy_eigen_point}, Assumptions \ref{assu_g} and \ref{assu_cg}.

\subsection{Practical Early Stopping} 
\label{sec:early}

In practice, it is often more efficient to stop the optimization earlier before meeting the optimality conditions and obtain a reasonably good solution much faster. 
We employ  a practical and simple early stopping mechanism to accommodate this need. 
Algorithm \ref{alg_inexact_rtr_arc} is allowed to terminate  earlier when: (1) the norm of the approximate gradient continually fails to decrease for $K$ times, and (2) when the percentage of the function decrement is lower than a given threshold, i.e., 
\begin{equation}
	\label{eq_early_stopping_1}
	\frac{f(\mathbf{x}_{k-1})-f(\mathbf{x}_k)}{|f(\mathbf{x}_{k-1})|} \le\tau_f,
\end{equation}
for a consecutive number of $K$ times, with $K$ and $\tau_f>0$ being user-defined.

For the subproblem, both Algorithms \ref{alg_non_linear_Lan} and \ref{alg_non_linear_tCG} are allowed to terminate when the current solution $\bm\eta_k$, i.e., $\bm\eta_k = \bm\eta_{k}^{*l}$ for Algorithm \ref{alg_non_linear_Lan} and $\bm\eta_k = \bm\eta_{k}^{i}$ for  Algorithm \ref{alg_non_linear_tCG}, satisfies
\begin{equation}
	\left\|\nabla_{\bm\eta}\hat{m}_k\left(  \bm\eta_k \right)\right \|_{\mathbf{x}_k}\le \kappa_\theta\min\left(1, \left\| \bm\eta_k  \right\|_{\mathbf{x}_k} \right) \|\mathbf{G}_k\|_{\mathbf{x}_k}.
	\label{eq_stepsize_stop_additional}
\end{equation}
This implements an examination of Assumption \ref{assu_g}.
Regarding  Assumption \ref{assu_cauchy_eigen_point}, both Algorithms \ref{alg_non_linear_Lan} and \ref{alg_non_linear_tCG} optimize along the direction of the Cauchy step in their first iteration and thus satisfy the Cauchy condition. Therefore there is no need to examine it. As for the eigenstep condition,  it is costly to compute and compare with the eigenstep in each inner iteration, so we do not use it as a stopping criterion in practice. 
Regarding Assumption \ref{assu_cg},  according to Lemma \ref{lem_lanczos_sol_property}, it is always satisfied by the solution from Algorithm \ref{alg_non_linear_Lan}. 
Therefore, there is no need to  examine it in Algorithm \ref{alg_non_linear_Lan}.
As for Algorithm \ref{alg_non_linear_tCG}, the examination by Eq. (\ref{eq_stepsize_stop_additional}) also plays a role in checking  Assumption \ref{assu_cg}. 
For the first condition in  Assumption \ref{assu_cg}, Eq. (\ref{eq_sufficient_step0}) is equivalent to $\langle\nabla_{\bm\eta}\hat{m}_k\left(  \bm\eta_{k}^{*} \right),\bm\eta_{k}^{*}\rangle_{\mathbf{x}_k}=0$ (this results from  Eq. (\ref{appD_eq1}) in Appendix \ref{cite_proof}). 
In practice, when Eq. (\ref{eq_stepsize_stop_additional}) is satisfied with a small value of $\left\|\nabla_{\bm\eta}\hat{m}_k\left(  \bm\eta_k^* \right)\right \|_{\mathbf{x}_k}$,  it has $\left\langle\nabla_{\bm\eta}\hat{m}_k\left(  \bm\eta_{k}^* \right),\bm\eta_{k}^*\right\rangle_{\mathbf{x}_k}\approx 0$, indicating that the first  condition of Assumption \ref{assu_cg} is met approximately.  
Also, since  $\langle\mathbf{G}_k,\bm\eta_k^*\rangle_{\mathbf{x}_k}\le0$ due to the descent direction $\bm\eta_k^*$, the second condition of  Assumption \ref{assu_cg} has a fairly high chance to be met.

\section{Convergence Analysis}

\subsection{Preliminaries}
We start from listing those assumptions and conditions from existing literature that are adopted to support  our analysis. 
Given a function $f$, the Hessian of its pullback  $\nabla^2\hat f\left(\mathbf{x}\right)[\mathbf{\bm\eta}]$ and its Riemannian Hessian ${\rm Hess} f\left(\mathbf{x}\right)[\bm{\eta}]$   are identical when a second-order  retraction is  used \cite{boumal2019global}, and this serves as an assumption to ease the analysis.
\begin{assu}[Second-order Retraction \cite{boumal2020introduction}]
	The retraction mapping is assumed to be a second-order retraction. That is, for all $\mathbf{x}\in\mathcal{M}$ and all $\bm\eta\in T_\mathbf{x}\mathcal{M}$, the curve $\gamma(t):=R_\mathbf{x}(t\bm\eta)$ has zero acceleration at $t=0$, i.e., $\gamma''(0)=\frac{\mathcal{D}^2}{dt^2}R_\mathbf{x}(t\bm\eta)\big|_{t=0}=0$.   		
	\label{assu_second_order_retr}
\end{assu}

The following  two assumptions originate from the assumptions required by the convergence analysis of the standard RTR algorithm \cite{boumal2019global,ferreira2002kantorovich}, and are adopted here to support the inexact analysis.
\begin{assu}[Restricted Lipschitz Hessian] There exists $L_H>0$ such that for all $\mathbf{x}_k$ generated by Algorithm \ref{alg_inexact_rtr_arc} and all $\bm{\eta}_k\in T_{\mathbf{x}_k}\mathcal{M}$, $\hat{f}_k$ satisfies
	\begin{equation}
		\left|\hat{f}_k(\bm{\eta}_k)-f(\mathbf{x}_k)-\langle{\rm grad} f(\mathbf{x}_k),\bm{\eta}_k\rangle_{\mathbf{x}_k}- \frac{1}{2}\left\langle\bm{\eta}_k,\nabla^2\hat{f}_k(\mathbf{0}_{\mathbf{x}_k})[\bm{\eta}_k]\right\rangle_{\mathbf{x}_k}\right|\le\frac{L_H}{6}\|\bm{\eta}_k\|_{\mathbf{x}_k}^3,
		\label{eq_restricted_lipschitz_hessian_1}
	\end{equation}
	and	
	\begin{equation}
		\left\|\mathcal{P}_{\bm{\eta}_k}^{-1}\left({\rm grad}\hat{f}_k(\bm{\eta}_k) \right)-{\rm grad}f(\mathbf{x}_k)- \nabla^2\hat{f}_k(\mathbf{0}_{\mathbf{x}_k})[\bm{\eta}_k] \right\|_{\mathbf{x}_k} \le\frac{L_H}{2}\|\bm{\eta}_k\|_{\mathbf{x}_k}^2,
		\label{eq_restricted_lipschitz_hessian_2}
	\end{equation}
	where $\mathcal{P}^{-1}$ denotes the inverse process  of the parallel transport operator.
	\label{ass_restricted_lipschitz_hessian}
\end{assu}
\begin{assu}[Norm Bound on Hessian] \textit{For all $\mathbf{x}_k$, there exists $K_H\ge0$ so that the inexact Hessian $\mathbf{H}_k$  satisfies}
	\begin{equation}
		\begin{split}
			\|\mathbf{H}_k\|_{\mathbf{x}_k}\le K_H.
		\end{split}
		\label{hessian_bound}
	\end{equation}
	\label{assu_hessian_norm_bound}
\end{assu}

The following key conditions on the inexact gradient and Hessian approximations are developed in Euclidean spaces by \cite{roosta2019sub} (Section 2.2) and \cite{xu2020newton} (Section 1.3), respectively. We make use of these in Riemannian spaces.
\begin{cond}[Approximation Error Bounds] \textit{For all $\mathbf{x}_k$ and $\bm{\eta}_k\in T_{\mathbf{x}_k}\mathcal{M}$, suppose that there exist $\delta_g,\delta_H>0$,  such that the approximate gradient and Hessian satisfy}
	\begin{align}
		\label{eq_approximate_grad_hess_bound1}
		\|\mathbf{G}_k-{\rm grad}f(\mathbf{x}_k)\|_{\mathbf{x}_k}&\le\delta_g, \\
		\label{eq_approximate_grad_hess_bound2}
		\|\mathbf{H}_k[\bm{\eta}_k]-\nabla^2\hat{f}_k(\mathbf{0}_{\mathbf{x}_k}) [\bm{\eta}_k]\|_{\mathbf{x}_k}&\le\delta_H\|\bm{\eta}_k\|_{\mathbf{x}_k}.
	\end{align}
	\label{cond_approximate_grad_hess_bound}
\end{cond}
As will be shown in Theorem \ref{theorem1} later, these allow the used sampling size in the gradient and Hessian approximations  to be fixed throughout the training process.
As a result, it can serve as a guarantee of the algorithmic efficiency when dealing with large-scale problems.

\subsection{Supporting Theorem and Assumption}

In this section, we prove a preparation theorem and present new conditions required by our results. 
Below, we re-develop Theorem 4.1 in \cite{kasai2018inexact} using the matrix Bernstein inequality \cite{gross2011recovering}. 
It provides lower bounds on the required subsampling size for approximating the gradient and Hessian  in order for Condition \ref{cond_approximate_grad_hess_bound} to hold. 
The proof is provided in Appendix \ref{appdendix_E}.
\begin{thm} [Gradient and Hessian Sampling Size] Define the suprema of the Riemannian gradient and Hessian
	\begin{align}
		\label{eq_Kgmax}
		K_{g_{\max}} := &\max_{i\in[n]} \sup_{\mathbf{x}\in\mathcal{M}} \left \| {\rm grad} f_i(\mathbf{x})\right \|_{\mathbf{x}}, \\
		\label{eq_KHmax}
		K_{H_{\max}} := & \max_{i\in[n]} \sup_{\mathbf{x}\in\mathcal{M}}  \sup_{\substack{\bm{\eta}\in T_{\mathbf{x}}\mathcal{M}\\ \|\bm{\eta}\|_{\mathbf{x}}\ne 0}}  \frac{\left\| {\rm Hess} f_i(\mathbf{x})[\mathbf{\bm\eta}]\right\|_{\mathbf{x}}}{\left\|\bm\eta\right\|_{\mathbf{x}}}. 
	\end{align}
	Given $0<\delta<1$,  Condition \ref{cond_approximate_grad_hess_bound} is satisfied with probability at least $\left(1-\delta\right)$ if
	\begin{align}
		\label{eq_restriction_sg}
		|\mathcal{S}_g| &\ge\frac{8\left(K_{g_{max}}^2+K_{g_{max}}\right)\ln\left(\frac{d+r}{\delta}\right)}{\delta_g^2}, \\
		|\mathcal{S}_H| & \ge\frac{8\left(K_{H_{max}}^2+\frac{K_{H_{max}} }{\left\| \bm\eta\right\|_{\mathbf{x}}}\right)\ln\left( \frac{d+r}{\delta}\right)}{\delta_H^2}.
		\label{eq_restriction_sH}
	\end{align}
	where $|\mathcal{S}_g|$ and $|\mathcal{S}_H|$ denote the sampling sizes,  while $d$ and $r$ are the dimensions of $\mathbf{x}$.
	\label{theorem1}
\end{thm}

The two quantities   $\delta_g$ and $\delta_H$ in Condition \ref{cond_approximate_grad_hess_bound} are the upper bounds of the gradient and Hessian approximation errors, respectively. 
The following assumption bounds $\delta_g$ and $\delta_H$. 
\begin{assu}[Restrictions on $\delta_g$ and $\delta_H$] 
	\label{assu_restriction_on_eg_eh}
	Given $\nu\in(0,1]$, $K_H\geq 0$, $L_H>0$,  $0<\tau, \epsilon_g<1$, we assume that $\delta_g$ and $\delta_H$ satisfy
	\begin{align}
		\label{eq_restrictions_delta_g}
		\delta_g\le \; &   \frac{(1-\tau)\left(\sqrt{K_H^2+4L_H\epsilon_g}-K_H\right)^2}{48L_H} , \\
		\label{eq_restrictions_delta_h}
		\delta_H\le\;  &\min \Bigg (\frac{1-\tau}{12}\left(\sqrt{K_H^2+4L_H\epsilon_g}-K_H\right),\frac{1-\tau}{3}\nu\epsilon_H\Bigg).
	\end{align}	
	\label{assu_restrictions_delta_g_h}
\end{assu}
As seen in Eqs. (\ref{eq_restriction_sg}) and (\ref{eq_restriction_sH}), sampling sizes $|\mathcal{S}_g|$ and $|\mathcal{S}_H|$ are directly proportional to the probability $(1-\delta)$ but inversely proportional to the error tolerances $\delta_g$ and $\delta_H$, respectively. Hence, a higher $(1-\delta)$ and smaller $\delta_g$ and $\delta_H$ (affected by $K_H$ and $L_H$) require larger $|\mathcal{S}_g|$ and $|\mathcal{S}_H|$ for estimating the inexact Riemannian gradient and Hessian.

\subsection{Main Results} \label{main_theorem}

Now we are ready to present our main convergence results in two main theorems  for  Algorithm \ref{alg_inexact_rtr_arc}. 
Different from \cite{sun2019escaping} which explores the escape rate from a saddle point to a local minimum, we study the convergence rate from a random point.
\begin{thm}[Convergence Complexity of  Algorithm \ref{alg_inexact_rtr_arc}] Consider $0<\epsilon_g,\epsilon_H<1$ and $\delta_g,\delta_H>0$. Suppose that Assumptions \ref{assu_second_order_retr}, \ref{ass_restricted_lipschitz_hessian}, \ref{assu_hessian_norm_bound} and \ref{assu_restrictions_delta_g_h} hold and the solution of the subproblem  in Eq. (\ref{eq_sub_problem}) satisfies Assumption \ref{assu_cauchy_eigen_point}. Then, if the inexact gradient $\mathbf{G}_k$ and Hessian $\mathbf{H}_k$ satisfy Condition \ref{cond_approximate_grad_hess_bound}, Algorithm \ref{alg_inexact_rtr_arc} returns an $\left( \epsilon_g,\epsilon_H\right)$-optimal solution in $ \mathcal{O}\left(\max\left(\epsilon_g^{-2},\epsilon_H^{-3}\right)\right)$ iterations.
	\label{theorem2}
\end{thm}
The proof along with the supporting lemmas is provided in Appendices \ref{app_theorem2A} and \ref{app_theorem2B}.
When the Hessian at the solution is close to positive semi-definite which indicates a small $\epsilon_H$, the Inexact Sub-RN-CR finds an  $\left( \epsilon_g,\epsilon_H\right)$-optimal solution in fewer iterations than the Inexact RTR, i.e., $\mathcal{O}\left(\max\left(\epsilon_g^{-2},\epsilon_H^{-3}\right)\right)$ iterations for the Inexact Sub-RN-CR as compared to  $ \mathcal{O}\left(\max\left(\epsilon_g^{-2}\epsilon_H^{-1},\epsilon_H^{-3}\right)\right)$ for the Inexact RTR \cite{kasai2018inexact}. 
Such a  result  is satisfactory. 
Combining Theorems \ref{theorem1} and \ref{theorem2}, it leads to the following corollary.
\begin{cor}
	Consider $0<\epsilon_g,\epsilon_H<1$ and $\delta_g,\delta_H>0$. Suppose that Assumptions \ref{assu_second_order_retr}, \ref{ass_restricted_lipschitz_hessian}, \ref{assu_hessian_norm_bound} and \ref{assu_restrictions_delta_g_h} hold and the solution of the subproblem  in Eq. (\ref{eq_sub_problem}) satisfies Assumption \ref{assu_cauchy_eigen_point}. For any $0<\delta<1$,
	suppose Eqs. (\ref{eq_restriction_sg}) and (\ref{eq_restriction_sH}) are satisfied at each iteration.
	Then, Algorithm \ref{alg_inexact_rtr_arc} returns an $\left( \epsilon_g,\epsilon_H\right)$-optimal solution in $ \mathcal{O}\left(\max\left(\epsilon_g^{-2},\epsilon_H^{-3}\right)\right)$ iterations  with a probability at least $p=(1-\delta)^{\mathcal{O}\left(\max\left(\epsilon_g^{-2},\epsilon_H^{-3}\right)\right)}$.
	\label{coro_1}
\end{cor}
\noindent
The proof is provided in Appendix \ref{app_theorem2C}.
We use an example to illustrate the effect of $\delta$ on the approximate gradient sample size $|\mathcal{S}_g|$. Suppose $\epsilon^{-2}_g > \epsilon^{-3}_H$, then $p=(1-\delta)^{\mathcal{O}\left(\epsilon_g^{-2}\right)}$. In addition, when $\delta= \mathcal{O}\left(\epsilon^2_g/10\right)$, $p\approx0.9$. Replacing $\delta$ with $ \mathcal{O}\left(\epsilon^2_g/10\right)$ in Eqs. (\ref{eq_restriction_sg}) and (\ref{eq_restriction_sH}), it can be obtained that the lower bound of $|\mathcal{S}_g|$ is proportional to $\ln\left(10\epsilon^{-2}_g\right)$.

Combining Assumption \ref{assu_cg} and the stopping condition in Eq. (\ref{eq_stepsize_stop_additional}) for the inexact solver,  a stronger convergence result can be obtained for Algorithm \ref{alg_inexact_rtr_arc}, which is presented in the following theorem and corollary.
\begin{thm}[\label{thm_opt_complex}Optimal Convergence Complexity of Algorithm \ref{alg_inexact_rtr_arc}] \hfill
	Consider $0<\epsilon_g,\epsilon_H<1$ and $\delta_g,\delta_H>0$. Suppose that  Assumptions \ref{assu_second_order_retr}, \ref{ass_restricted_lipschitz_hessian}, \ref{assu_hessian_norm_bound} and \ref{assu_restrictions_delta_g_h} hold and the solution of the subproblem satisfies  Assumptions \ref{assu_cauchy_eigen_point}, \ref{assu_g} and \ref{assu_cg}. Then, if the inexact gradient $\mathbf{G}_k$ and Hessian $\mathbf{H}_k$ satisfy Condition \ref{cond_approximate_grad_hess_bound}  and  $\delta_g\le\delta_H\le\kappa_\theta\epsilon_g$, Algorithm \ref{alg_inexact_rtr_arc} returns an $\left( \epsilon_g,\epsilon_H\right)$-optimal solution in $ \mathcal{O}\left(\max\left(\epsilon_g^{-\frac{3}{2}},\epsilon_H^{-3}\right)\right)$ iterations. 
	\label{theorem3}
\end{thm}
\begin{cor}
	Consider $0<\epsilon_g,\epsilon_H<1$ and $\delta_g,\delta_H>0$. Suppose that  Assumptions \ref{assu_second_order_retr}, \ref{ass_restricted_lipschitz_hessian}, \ref{assu_hessian_norm_bound} and \ref{assu_restrictions_delta_g_h} hold, and the solution of the subproblem satisfies Assumptions \ref{assu_cauchy_eigen_point}, \ref{assu_g} and \ref{assu_cg}.  For any $0<\delta<1$,
	suppose Eqs. (\ref{eq_restriction_sg}) and (\ref{eq_restriction_sH}) are satisfied at each iteration. 
	Then, Algorithm \ref{alg_inexact_rtr_arc} returns an $\left( \epsilon_g,\epsilon_H\right)$-optimal solution in $ \mathcal{O}\left(\max\left(\epsilon_g^{-\frac{3}{2}},\epsilon_H^{-3}\right)\right)$ iterations with a probability at least $p=(1-\delta)^{\mathcal{O}\left(\max\left(\epsilon_g^{-\frac{3}{2}},\epsilon_H^{-3}\right)\right)}$.
	\label{coro_2}
\end{cor}
\noindent
The proof of Theorem \ref{theorem3} along with its supporting lemmas is provided in Appendices \ref{app_theorem3A} and \ref{app_theorem3B}. The proof of Corollary \ref{coro_2} follows Corollary \ref{coro_1} and is provided in Appendix \ref{app_theorem3C}.

\subsection{Computational Complexity Analysis}

We analyse the number of main operations required by the proposed algorithm.
Taking the PCA  task as an example, it optimizes over the Grassmann manifold ${\rm Gr}\left(r,d\right)$. 
Denote by $m$ the number of inexact iterations and $D$ the manifold dimension, i.e., $D=d\times(d-r)$ for the Grassmann manifold. 
Starting from the gradient and Hessian computation, the full case requires $\mathcal{O}(ndr)$  operations for both in the PCA task.
By using the  subsampling technique, these  can be  reduced to $\mathcal{O}(|\mathcal{S}_g|dr)$ and $\mathcal{O}(|\mathcal{S}_H|dr)$ by gradient and Hessian approximation. 
Following an existing setup for cost computation, i.e., Inexact RTR method \cite{kasai2018inexact}, the full function cost evaluation takes  $n$ operations, while the approximate cost evaluation after subsampling becomes $\mathcal{O}(|\mathcal{S}_n|dr)$, where $\mathcal{S}_n$ is the subsampled set of data points used to compute the function cost.
These show that, for large-scale practices with $n\gg \max\left(|\mathcal{S}_g|, |\mathcal{S}_H|,|\mathcal{S}_n| \right)$,  the computational cost reduction gained from the subsampling technique is significant.

For the subproblem solver by Algorithm \ref{alg_non_linear_Lan} or \ref{alg_non_linear_tCG}, the dominant computation within each iteration is the Hessian computation, which as mentioned above requires $\mathcal{O}(|\mathcal{S}_H|dr)$ operations after using the subsampling technique. Taking this into account to analyze Algorithm \ref{alg_inexact_rtr_arc}, its  overall computational complexity  becomes $\mathcal{O}\left(\max \left(\epsilon_g^{-2},   \epsilon_H^{-3} \right)\right)\times \left[\mathcal{O}(n + |\mathcal{S}_g|dr )+
\mathcal{O}\left(m| \mathcal{S}_H|d^2(d-r)r\right)\right]$ based on Theorem \ref{theorem2}, where $\mathcal{O}(n + |\mathcal{S}_g|dr)$ corresponds to the operations for computing the full function cost and the approximate gradient in an outer iteration.  This overall complexity can be simplified to $\mathcal{O}\left(\max \left(\epsilon_g^{-2},   \epsilon_H^{-3}\right)\right)\times \mathcal{O}(n + |\mathcal{S}_g|dr + m|\mathcal{S}_H|d^2(d-r)r)$, where $\mathcal{O}(m|\mathcal{S}_H|d^2(d-r)r)$ is the cost of the subproblem solver by either Algorithm \ref{alg_non_linear_Lan} or Algorithm \ref{alg_non_linear_tCG}.
Algorithm \ref{alg_non_linear_Lan} is guaranteed to return the optimal subproblem solution within at most $m=D=d\times(d-r)$ inner iterations, of which the complexity is at most $\mathcal{O}(|\mathcal{S}_H|d^2(d-r)^2r^2)$. 
Such a polynomial complexity is at least as good as the ST-tCG solver used in the Inexact RTR algorithm.
For Algorithm \ref{alg_non_linear_tCG}, although $m$ is not guaranteed to be bounded and polynomial, we have empirically observed that $m$ is generally smaller than $D$ in practice, presenting a similar complexity to Algorithm \ref{alg_non_linear_Lan}.

\section{Experiments and Result Analysis} \label{sec_experiment}

\begin{figure}[t]
	\centering
	\subfloat[\label{subfig_P_T1}Synthetic dataset P1]{%\label{subfig_P1_b}
		\includegraphics[width=0.5\textwidth]{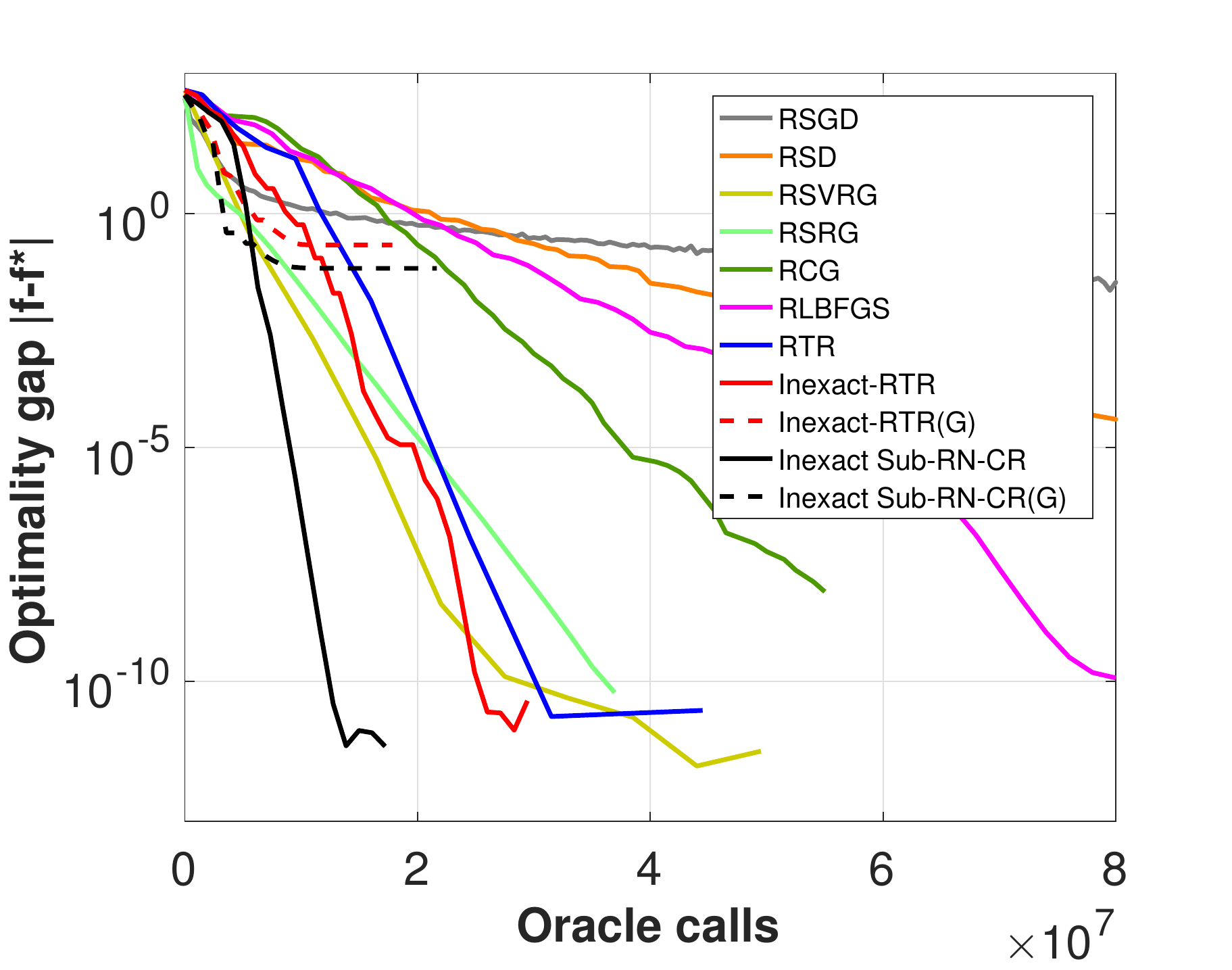}
		\includegraphics[width=0.5\textwidth]{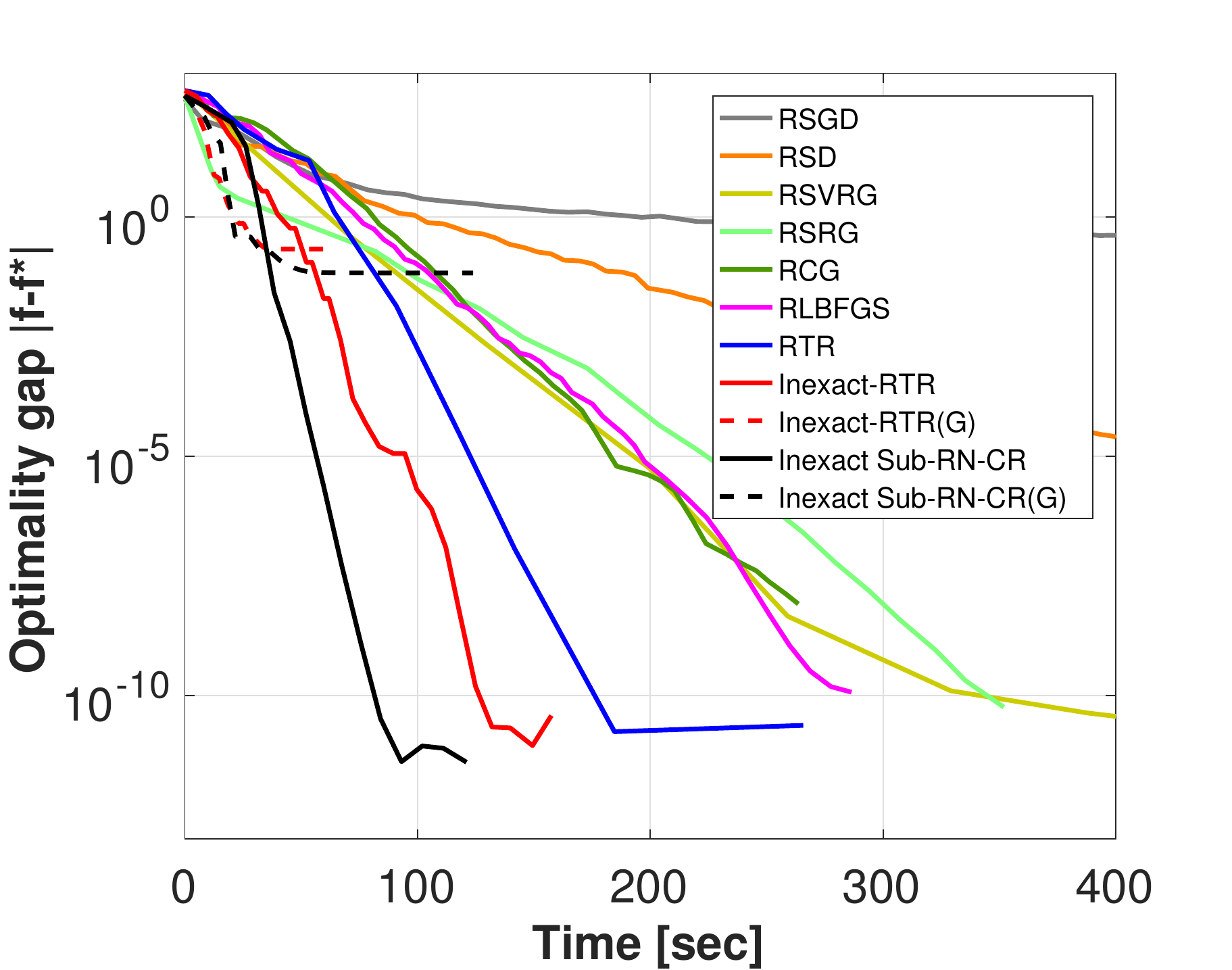}
	}
	\\
	\subfloat[\label{subfig_P_T2} MNIST dataset]{%\label{subfig_P1_b}
		\includegraphics[width=0.5\textwidth]{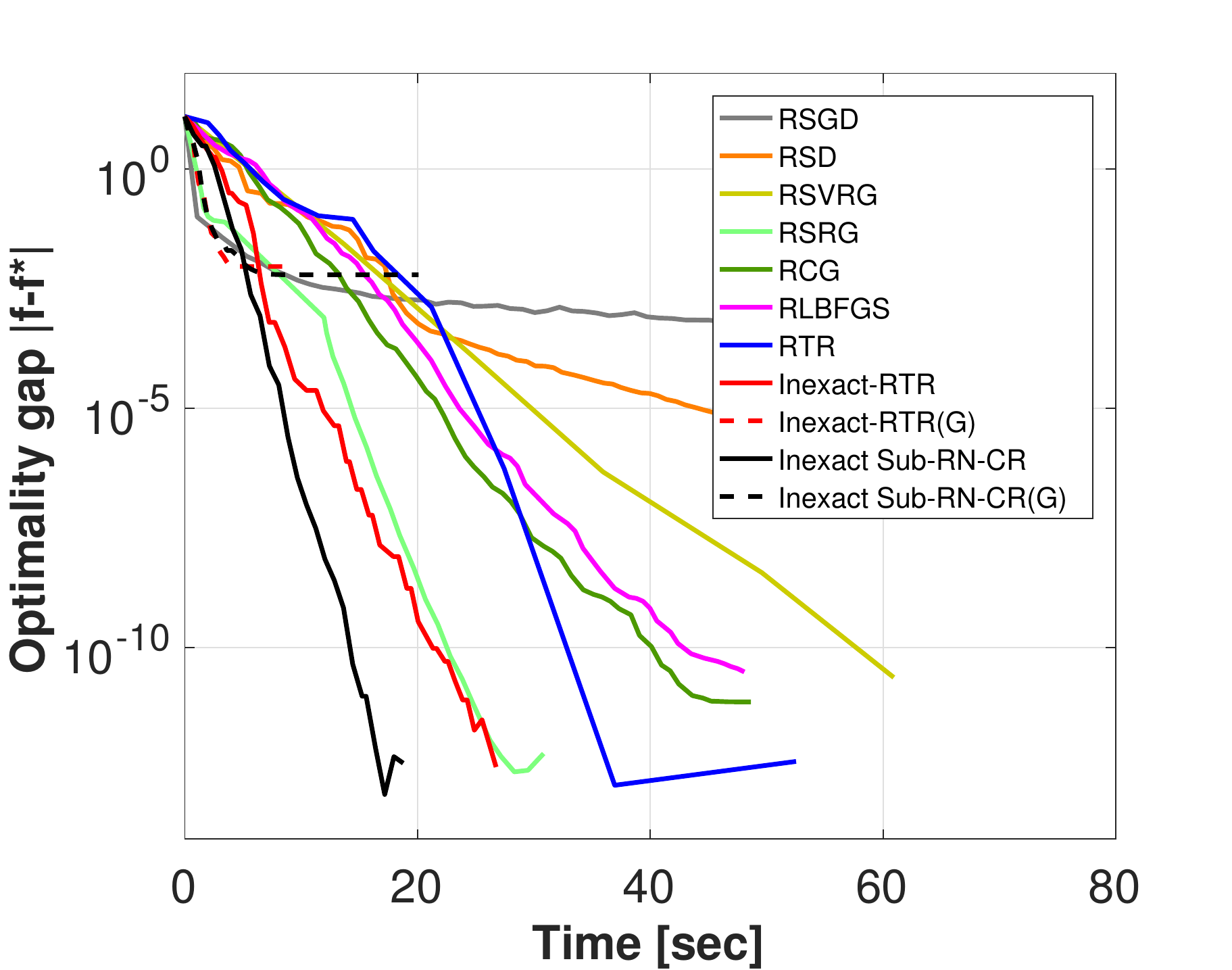}
	}
	\subfloat[\label{subfig_P_T3}Covertype dataset]{%\label{subfig_P1_b}
		\includegraphics[width=0.5\textwidth]{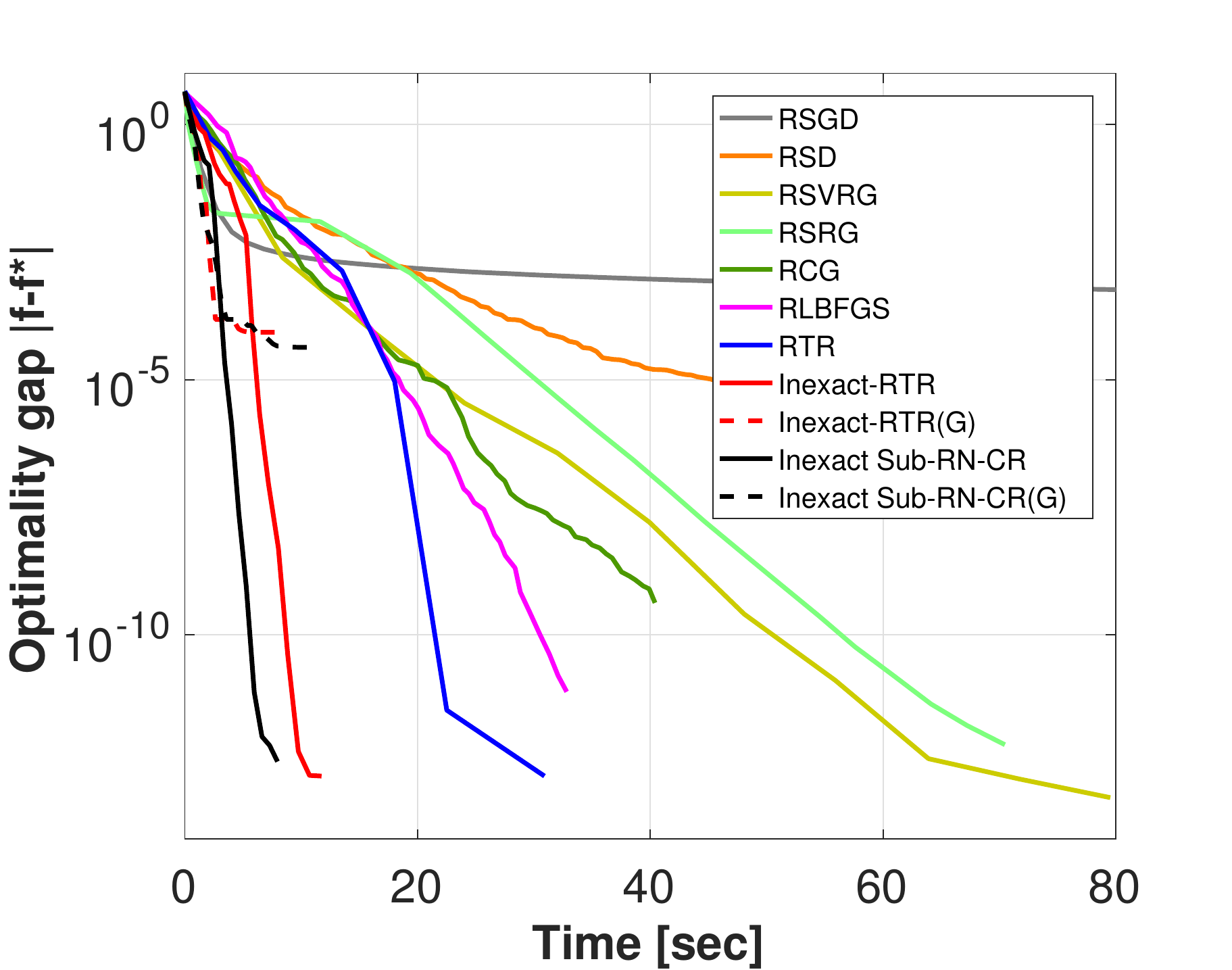}
	}
	\caption{\label{fig_PCA_all} Performance comparison by optimality gap for the PCA task.}
\end{figure}

\begin{figure}[t]
	\centering
	\includegraphics[width=1.0\textwidth]{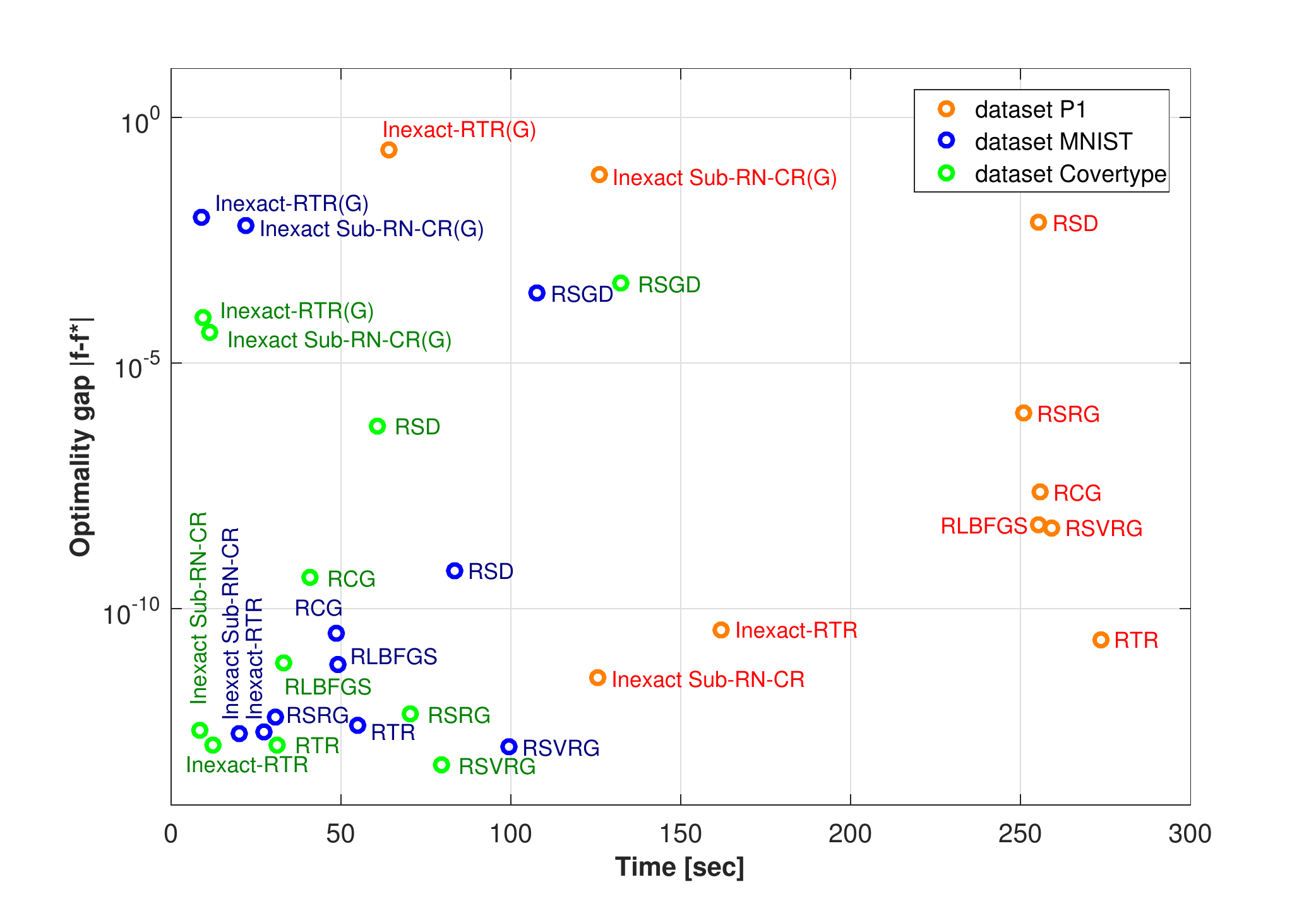}
	\caption{\label{fig_pca_dots} PCA  performance summary for the participating methods.}
\end{figure}

We compare the proposed Inexact Sub-RN-CR  algorithm with state-of-the-art and popular Riemannian  optimization algorithms. 
These  include the  Riemannian stochastic gradient descent (RSGD) \cite{bonnabel2013stochastic}, Riemannian steepest descent (RSD) \cite{absil2009optimization}, Riemannian conjugate gradient (RCG) \cite{absil2009optimization}, Riemannian limited memory BFGS algorithm (RLBFGS) \cite{yuan2016riemannian}, Riemannian stochastic variance-reduced gradient (RSVRG) \cite{zhang2016riemannian}, Riemannian stochastic recursive gradient (RSRG) \cite{kasai2018riemannian}, RTR \cite{absil2007trust}, Inexact RTR \cite{kasai2018inexact} and RTRMC \cite{boumal2011rtrmc}. 
Existing implementations of these algorithms are available in either Manopt \cite{boumal2014manopt} or Pymanopt \cite{townsend2016pymanopt} library. 
They are often used for algorithm comparison in existing literature, e.g., by Inexact RTR \cite{kasai2018inexact}.  
Particularly, RSGD, RSD, RCG, RLBFGS, RTR and RTRMC algorithms have been encapsulated into Manopt, and RSD, RCG and RTR also into Pymanopt. RSVRG, RSRG and Inexact RTR are implemented by \cite{kasai2018inexact} based on Manopt.
We use  existing implementations to reproduce their methods.
Our Inexact Sub-RN-CR implementation  builds on Manopt.

For the competing methods, we follow  the same parameter settings from the existing implementations, including  the batch size (i.e. sampling size), step size (i.e. learning rate) and the inner iteration number to ensure the same results as the reported ones.
For our method, we first keep the common algorithm parameters the same as the competing methods, including $\gamma$, $\tau$, $\epsilon_g$ and $\epsilon_H$. Then, we use a grid search to find appropriate values of $\theta$ and $\kappa_\theta$ for both Algorithms \ref{alg_non_linear_Lan} and \ref{alg_non_linear_tCG}. Specifically, the searching grid for $\theta$ is $(0.02, 0.05, 0.1, 0.2, 0.5, 1)$, and the searching grid for $\kappa_\theta$ is $(0.005, 0.01, 0.02, 0.04, 0.08, 0.16)$. For the parameter $\kappa$ in Algorithm \ref{alg_non_linear_tCG}, we keep it the same as the other conjugate gradient solvers.
The early stopping approach as described in Section \ref{sec:early} is applied to all the compared algorithms.

Regarding the batch setting, which is also the sample size setting for  approximating the gradient and Hessian, we adopt the same value as used in  existing subsampling implementations to  keep consistency.  Also, the same settings are used for both the PCA and matrix completion tasks. 
Specifically, the batch size  $\left|\mathcal{S}_g\right| = n/100$ is used for RSGD, RSVRG and RSRG where $\mathcal{S}_H$ is not considered as these are first-order methods.
For both the Inexact RTR and the proposed Inexact Sub-RN-CR, $\left|\mathcal{S}_H\right|=n/100$ and $\left|\mathcal{S}_g \right| = n$ is used. This is to follow  the existing setting in \cite{kasai2018inexact}  for benchmark purposes, which exploits the approximate Hessian but the full gradient.
In addition to these, we experiment with another batch setting of $\left\{\left|\mathcal{S}_H\right|=n/100,\left|\mathcal{S}_g\right|=n/10\right\}$ for both the Inexact RTR and Inexact Sub-RN-CR. This is flagged by  $(G)$  in the algorithm name, meaning that the algorithm uses the approximate gradient in addition to the approximate Hessian. 
Its purpose  is to evaluate the effect of $\mathcal{S}_g$ in the optimization.

Evaluation is conducted based on two machine learning tasks of PCA and low-rank matrix completion using both synthetic and real-world datasets with $n\gg d\gg 1$. 
Both tasks can be formulated as non-convex optimization problems on the Grassmann manifold Gr$\left(r,d\right)$.
The algorithm performance is evaluated by oracle calls and the run time. 
The former counts the number of function, gradient, and Hessian-vector product computations. 
For instance, Algorithm \ref{alg_inexact_rtr_arc} requires $n+|\mathcal{S}_g|+m|\mathcal{S}_H|$ oracle calls each iteration, where $m$ is the number of iterations of the subproblem solver.
Regarding the user-defined parameters in Algorithm \ref{alg_inexact_rtr_arc}, we  use $\epsilon_\sigma=10^{-18}$. 
Empirical observations suggest that the magnitude of the data entries  affects the optimization in its early stage, and hence these factors are taken into account in the setting of $\sigma_0$.
Let $\mathbf{S}=[s_{ij}]$ denote the input data matrix containing $L$ rows and $H$ columns. We compute $\sigma_0$ by considering the data dimension, also the averaged data magnitude normalized by its standard deviation, given as
\begin{equation}
	\sigma_0= \left(\sum_{i\in[L], j\in[H]} \frac{|s_{ij}|}{LH}\right)^2\left(\frac{dim(M)*d}{\sqrt{\frac{1}{LH}\sum_{i\in[L], j\in[H]}(s_{ij}-\mu_S)^2}}\right)^{\frac{1}{2}},
\end{equation}
where  $\mu_S =  \frac{1}{LH}\sum_{i\in[L], j\in[H]}s_{i,j}$ and $dim(M)$ is the manifold dimension.

Regarding the early stopping setting in Eq. (\ref{eq_early_stopping_1}), $K=5$ is used for both tasks, and we use $\tau_f = 10^{-12}$ for MNIST and $\tau_f = 10^{-10}$ for the remaining datasets in the PCA task. In the matrix completion task, we set $\tau_f = 10^{-10}$ for the synthetic datasets and $\tau_f = 10^{-3}$ for the real-world datasets.   
For the early stopping settings in Eq. (\ref{eq_stepsize_stop_additional}) and in Step (\ref{CG_stopping2}) of Algorithm \ref{alg_non_linear_tCG}, we adopt $\kappa_\theta=0.08$ and $\theta=0.1$.
Between the two subproblem solvers, we observe that Algorithm \ref{alg_non_linear_Lan} by the Lanczos method and  Algorithm \ref{alg_non_linear_tCG} by the conjugate gradient perform similarly. Therefore, we report the main results using  Algorithm \ref{alg_non_linear_Lan}, and provide supplementary results for Algorithm \ref{alg_non_linear_tCG} in a separate Section \ref{exp_cg}.

\begin{figure}[t]
	\centering
	\subfloat[\label{subfig_P_T1_2} P1, no early stopping]{\label{subfig_P1_b}
		\includegraphics[width=0.5\textwidth]{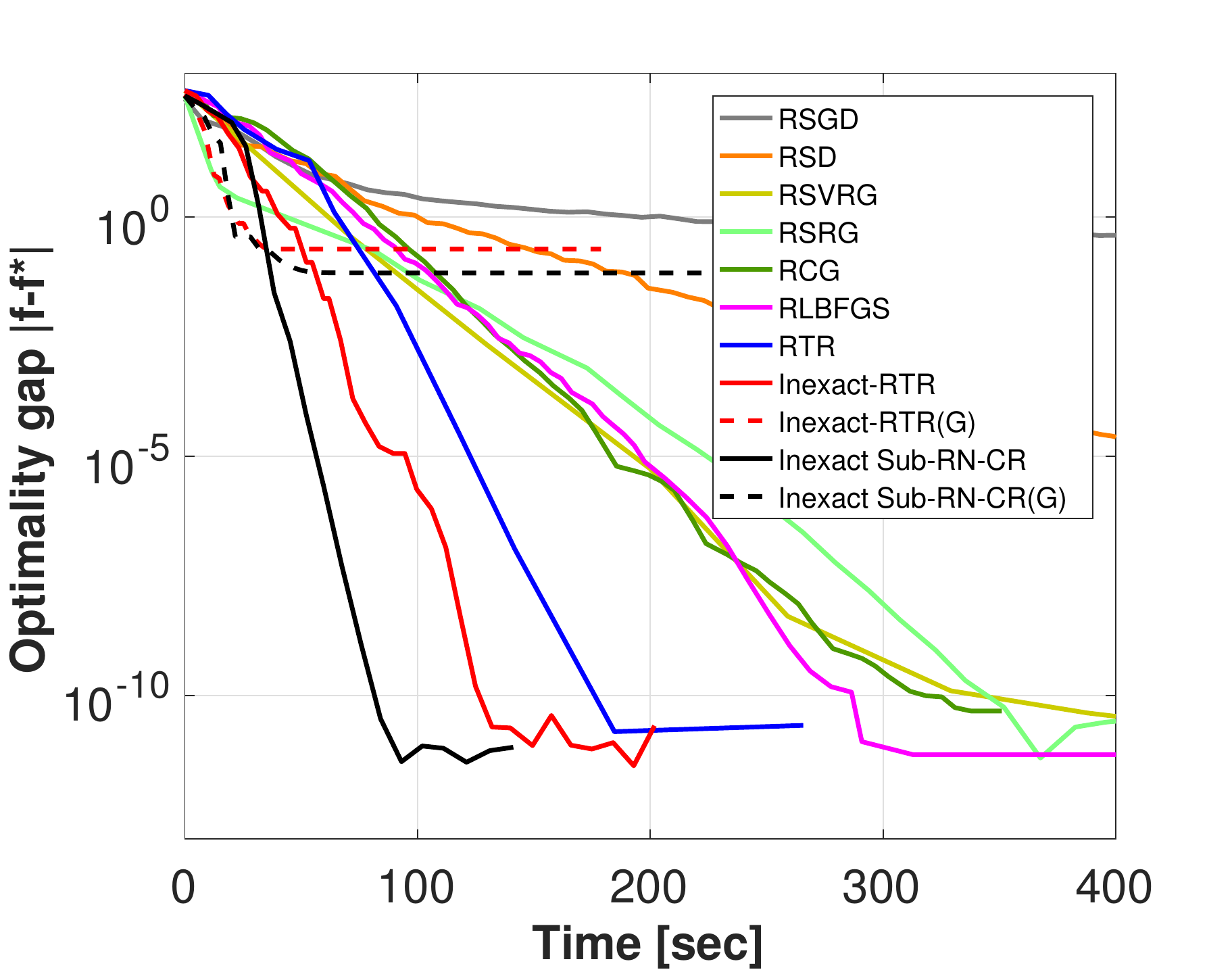}
	}
	\subfloat[\label{subfig_PCA_diff_sg1} P1,  varying $|\mathcal{S}_g|$ settings]{%
		\includegraphics[width=0.5\textwidth]{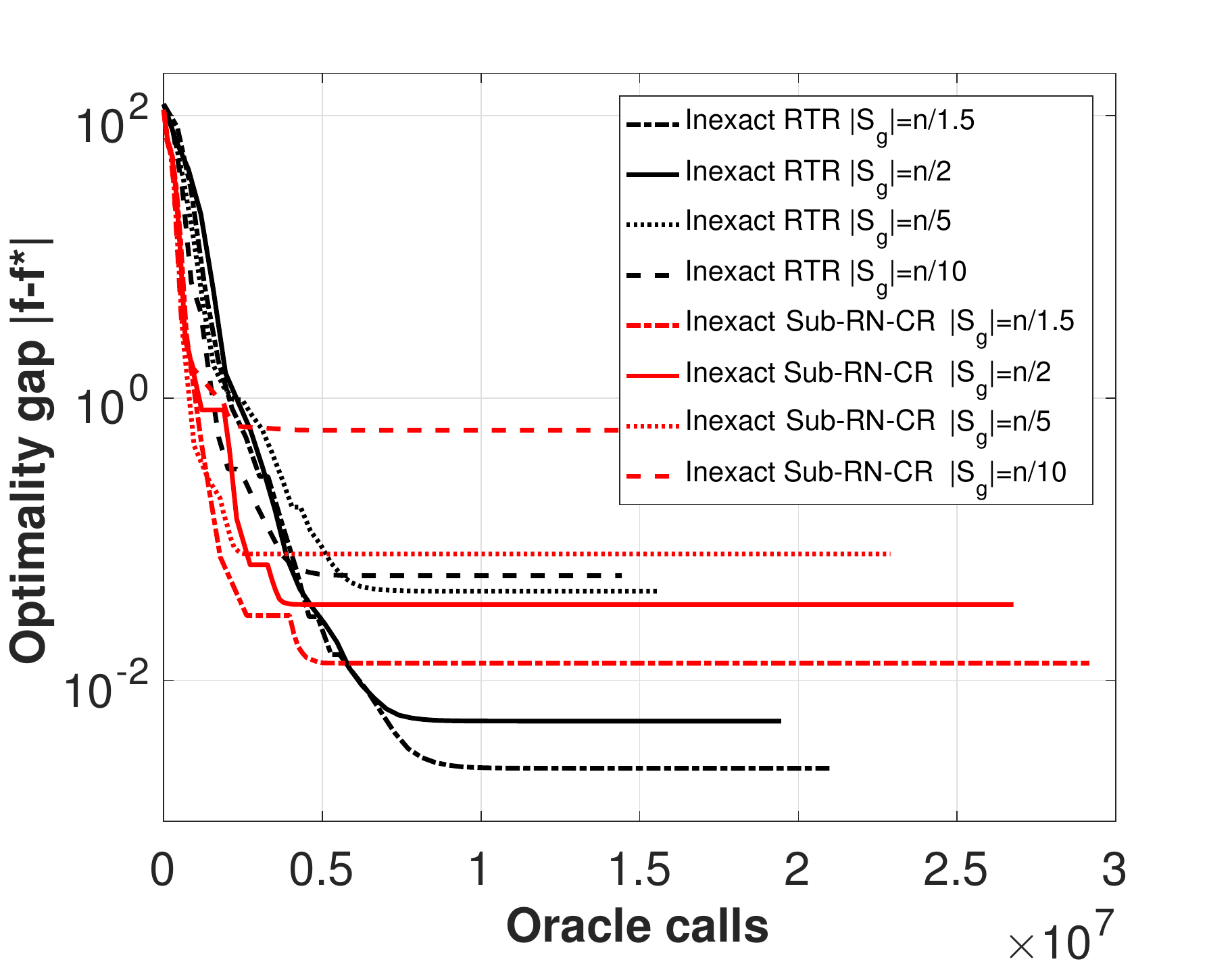}
	} \\
	\subfloat[\label{subfig_PCA_diff_sh1}  P1, varying $|\mathcal{S}_H|$ settings]{%
		\includegraphics[width=0.5\textwidth]{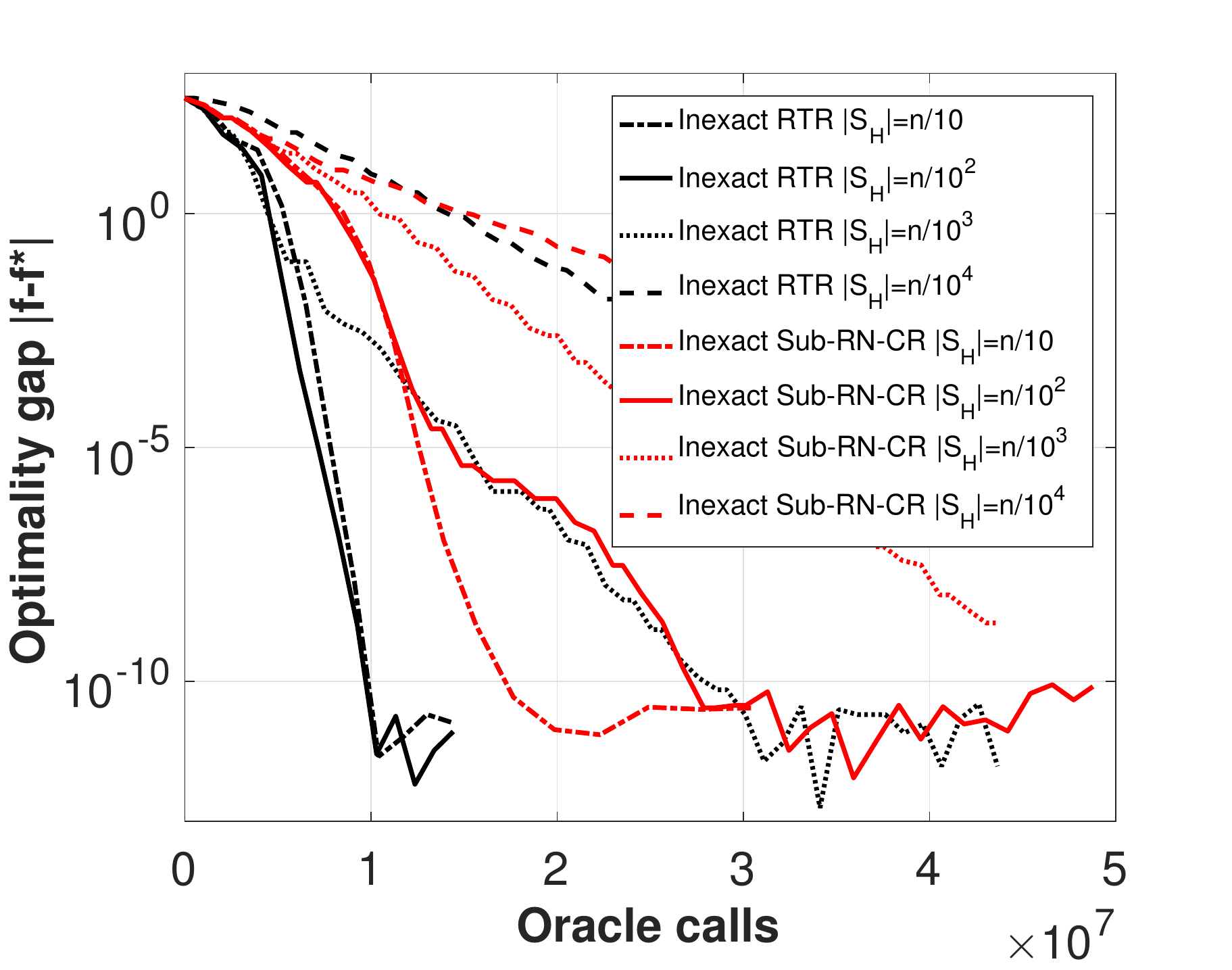}
	}
	%\hspace{2pt}
	\subfloat[\label{subfig_PCA_diff_sh2}  MNIST, varying $|\mathcal{S}_H|$ settings]{%
		\includegraphics[width=0.5\textwidth]{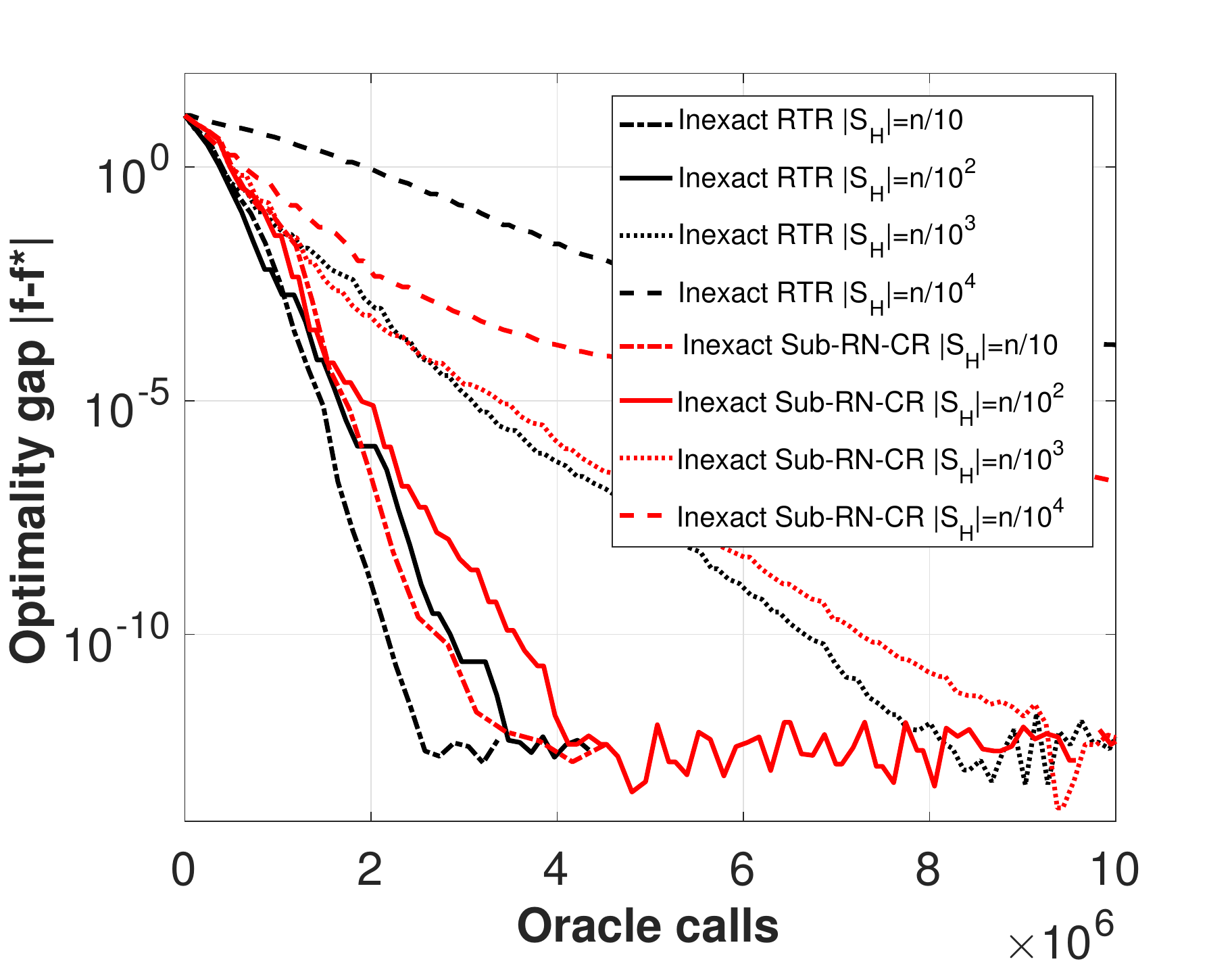}
	}\\
	\caption{\label{fig_PCA_diff_sh} Additional comparisons for the PCA task.}
\end{figure}

\begin{figure}[!b]
	\centering
	\subfloat[\label{fig_MC_M1} Synthetic  Dataset M1]{
		\includegraphics[width=0.5\textwidth]{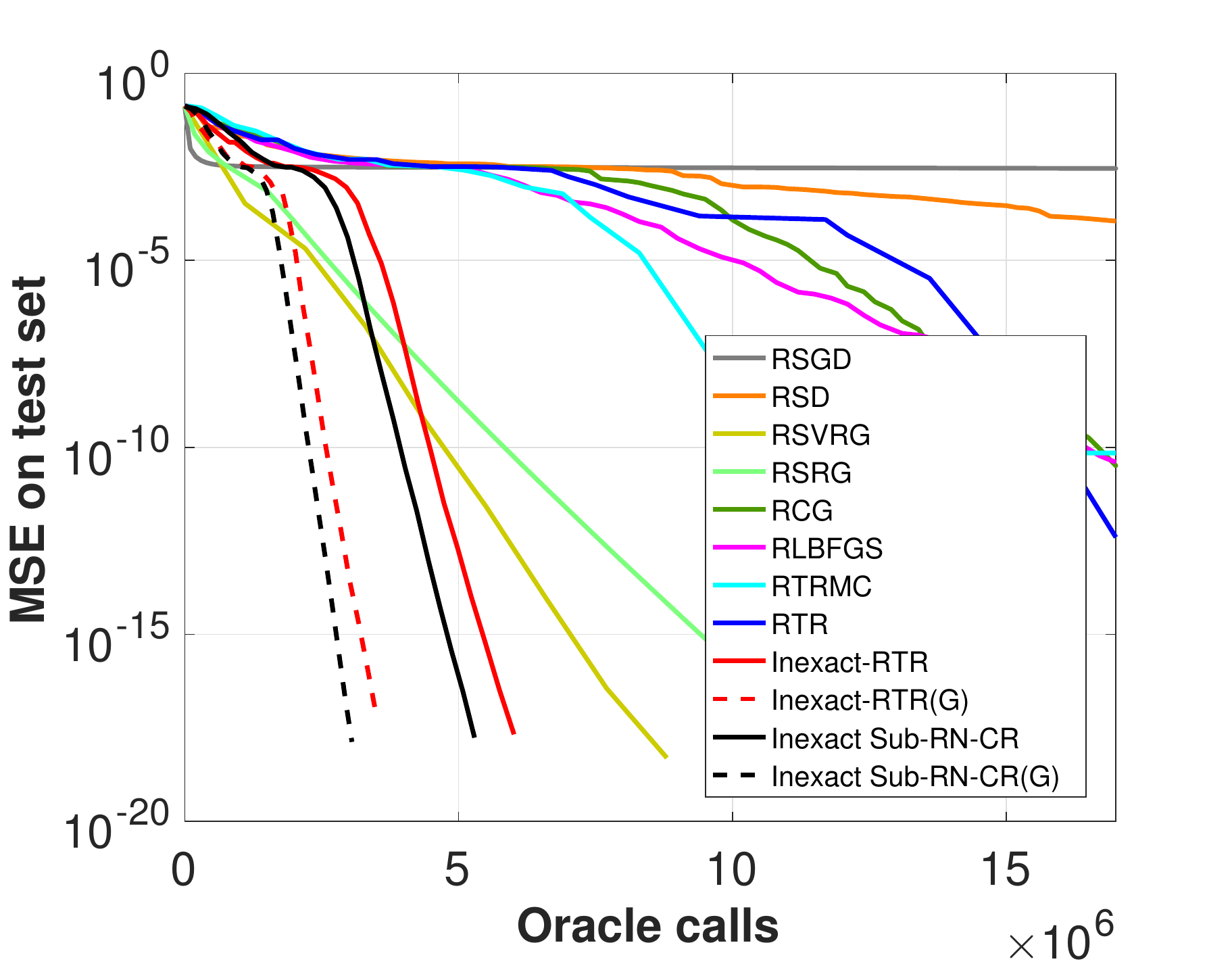}
		\includegraphics[width=0.5\textwidth]{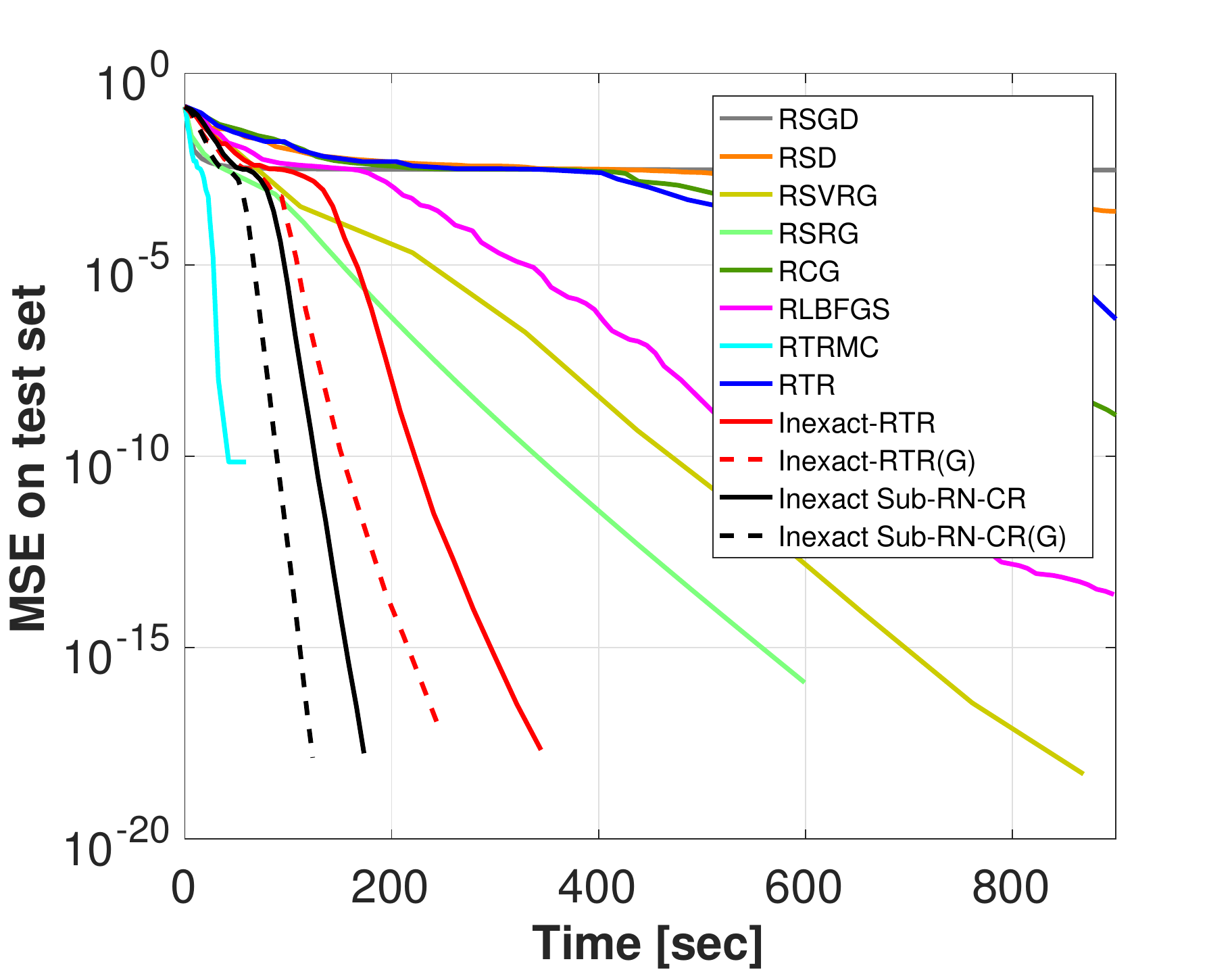}
	}
	\\
	\subfloat[\label{fig_MC_M2} Synthetic Dataset M2]{
		\includegraphics[width=0.5\textwidth]{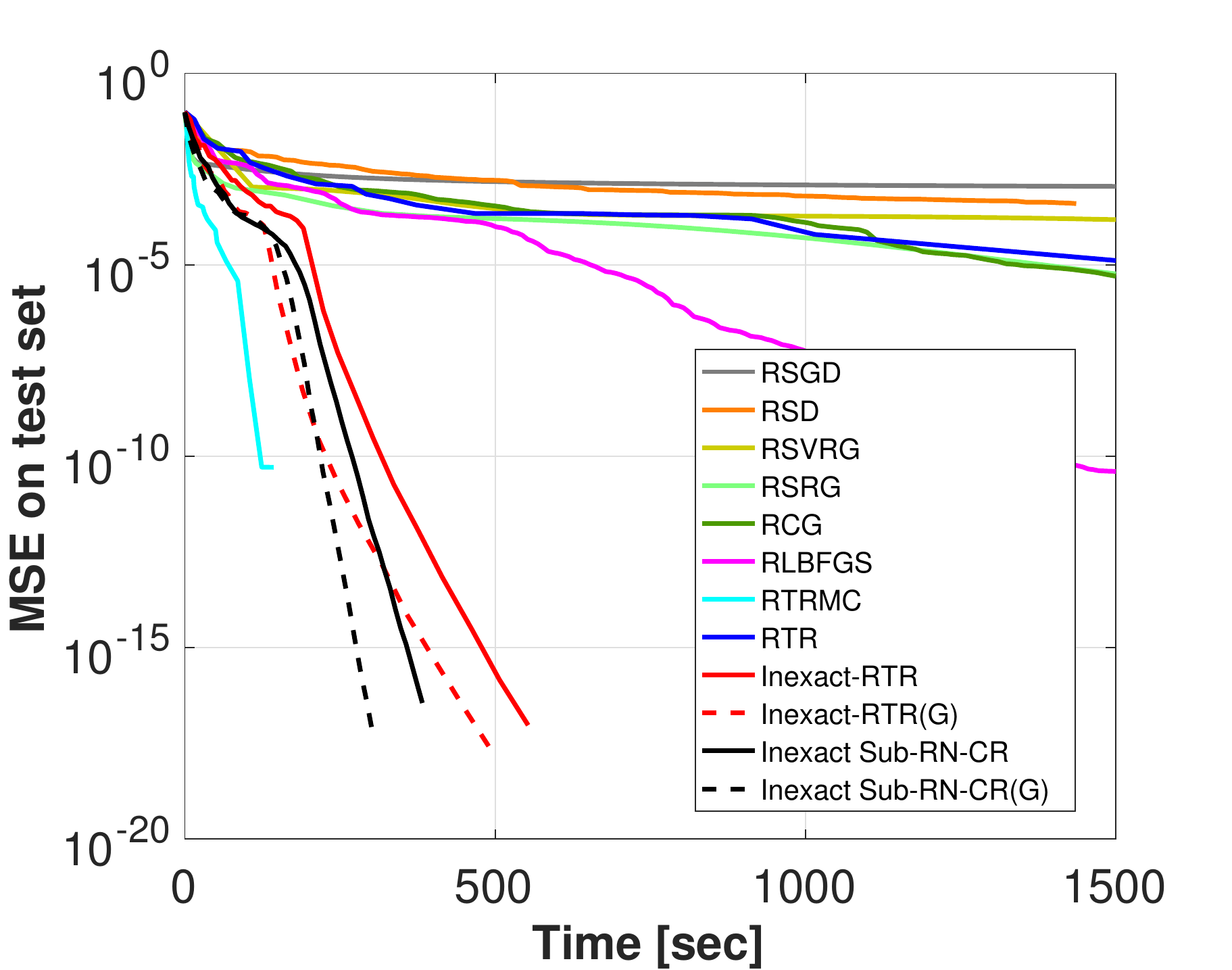}
	}
	%\\
	\subfloat[\label{fig_MC_M3} Synthetic Dataset M3]{
		\includegraphics[width=0.5\textwidth]{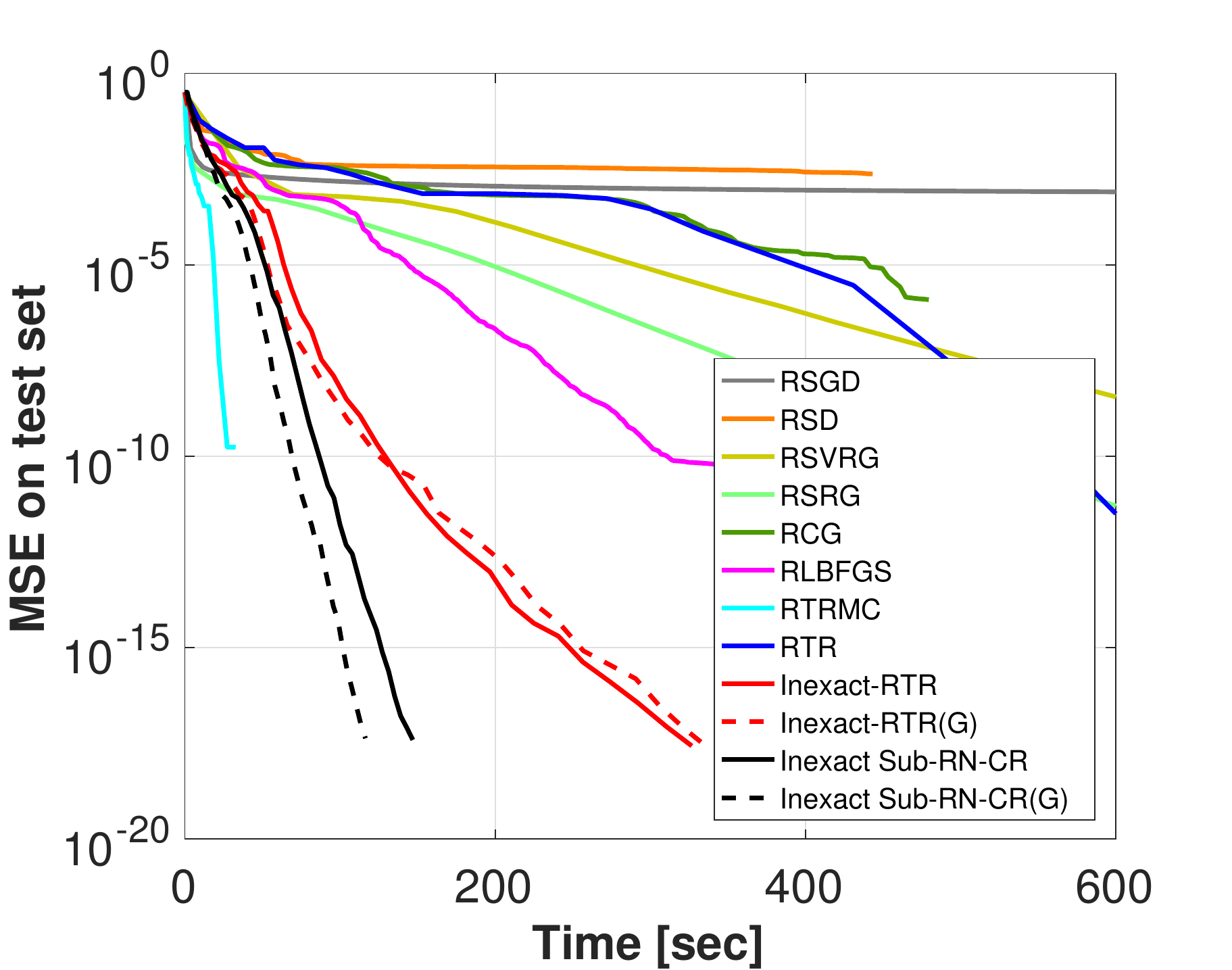}
	}
	\\
	\subfloat[\label{fig_MC_all2} Jester Dataset ]{
		\includegraphics[width=0.5\textwidth]{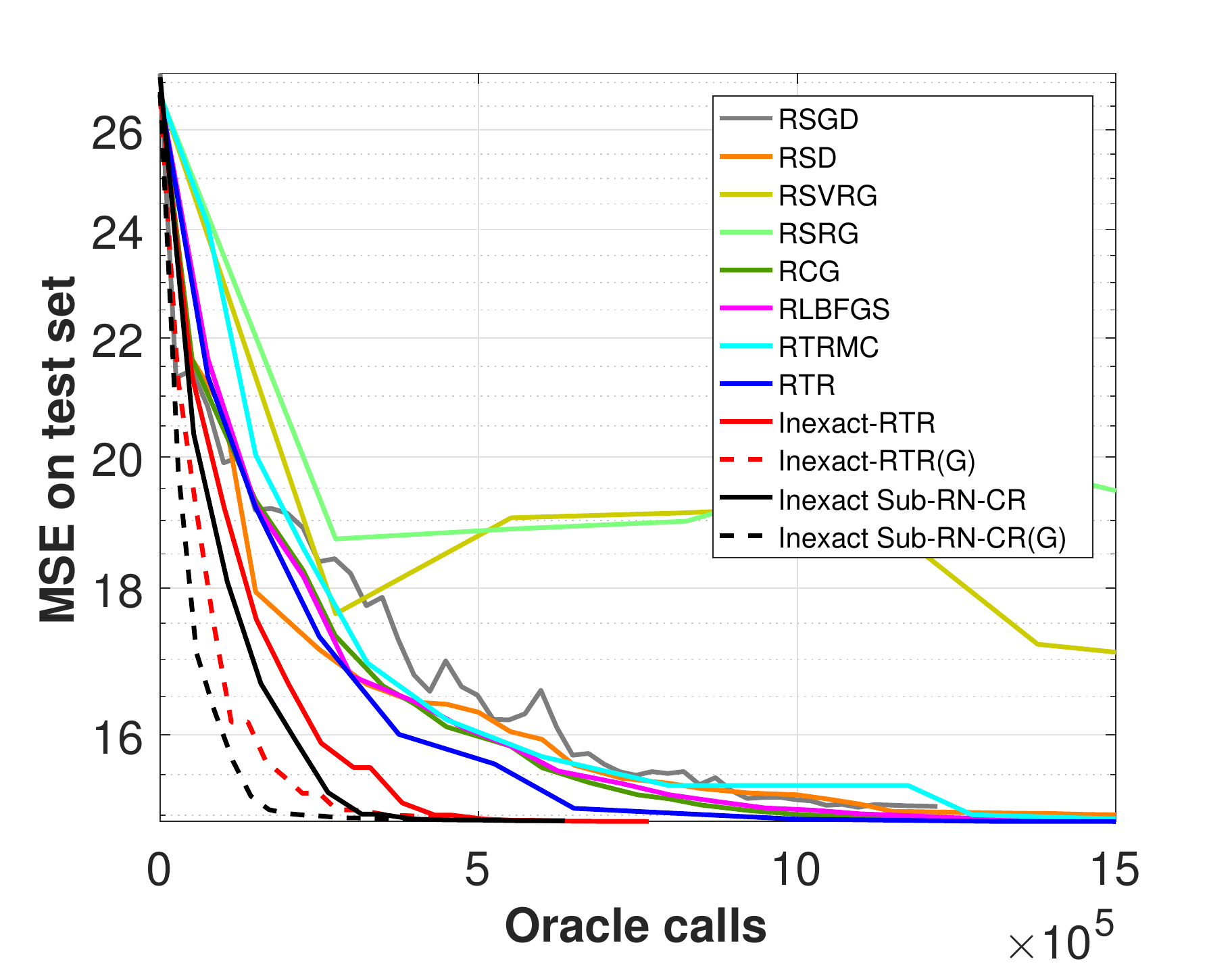}
		\includegraphics[width=0.5\textwidth]{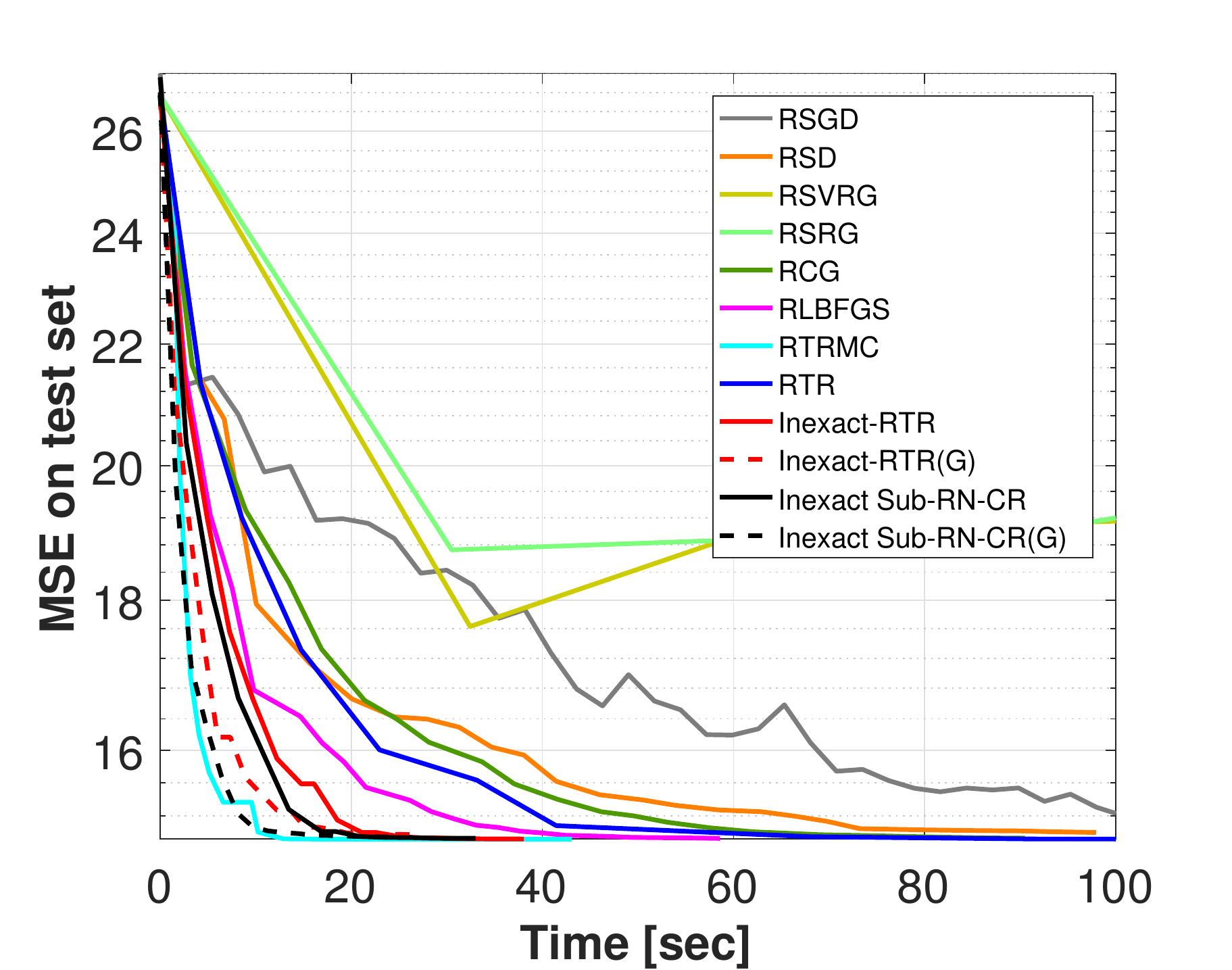}
	}
	\caption{\label{fig_MC_all} Performance comparison by MSE for the matrix completion task.}
\end{figure}

\subsection{PCA Experiments}

PCA can be interpreted as a minimization of  the sum of squared residual errors between the projected  and the original data points, formulated as
\begin{equation}
	\min_{\mathbf{U}\in {\rm Gr}\left(r,d\right)}\frac{1}{n}\sum_{i=1}^n\left\|\mathbf{z}_i-\mathbf{UU^T}\mathbf{z}_i\right\|_2^2,
	\label{eq_pca_formula_0}
\end{equation}
where $\mathbf{z}_i \in \mathbb{R}^{d}$. The objective function can be re-expressed as one on  the Grassmann manifold via
\begin{equation}
	\min_{\mathbf{U}\in {\rm Gr}\left(r,d\right)}-\frac{1}{n}\sum_{i=1}^n\mathbf{z}_i^T\mathbf{UU^T}\mathbf{z}_i.
	\label{eq_pca_formula}
\end{equation}
One synthetic  dataset P1 and two real-world  datasets including MNIST \cite{lecun1998gradient} and  Covertype \cite{blackard1999comparative} are used in the evaluation. 
The P1 dataset is firstly generated by randomly sampling each element of a matrix $\mathbf{A}\in \mathbb{R}^{n\times d}$ from a  normal distribution $\mathcal{N}(0,1)$. 
This is then followed by a multiplication with a diagonal matrix $\mathbf{S}\in \mathbb{R}^{d\times d}$ with each diagonal element randomly sampled from an exponential distribution $\textmd{Exp}(2)$, which increases the difference between the feature variances.
After that, a mean-subtraction preprocessing is applied to $\mathbf{A}\mathbf{S}$ to obtain the final P1 dataset.
The $\left(n,d,r\right)$  values are: $\left(5\times 10^5,10^3,5\right)$ for P1,   $\left(6\times 10^4,784,10\right)$ for  MNIST, and $\left(581012,54,10\right)$  for Covertype.  
Algorithm accuracy is assessed by optimality gap, defined as the absolute difference $|f-f^*|$, where $f=-\frac{1}{n}\sum_{i=1}^n\mathbf{z}_i^T\hat{\mathbf{U}}\hat{\mathbf{U}}^T\mathbf{z}_i$ with $\hat{\mathbf{U}}$ as the optimal solution returned by Algorithm \ref{alg_inexact_rtr_arc}. 
The optimal function value $f^*=-\frac{1}{n}\sum_{i=1}^n\mathbf{z}_i^T\tilde{\mathbf{U}}\tilde{\mathbf{U}}^T\mathbf{z}_i$ is computed by using the eigen-decomposition solution $\tilde{\mathbf{U}}$, which is a classical way to obtain PCA result without going through the optimization program.

\begin{figure}[t]
	\centering
	\includegraphics[width=1.0\textwidth]{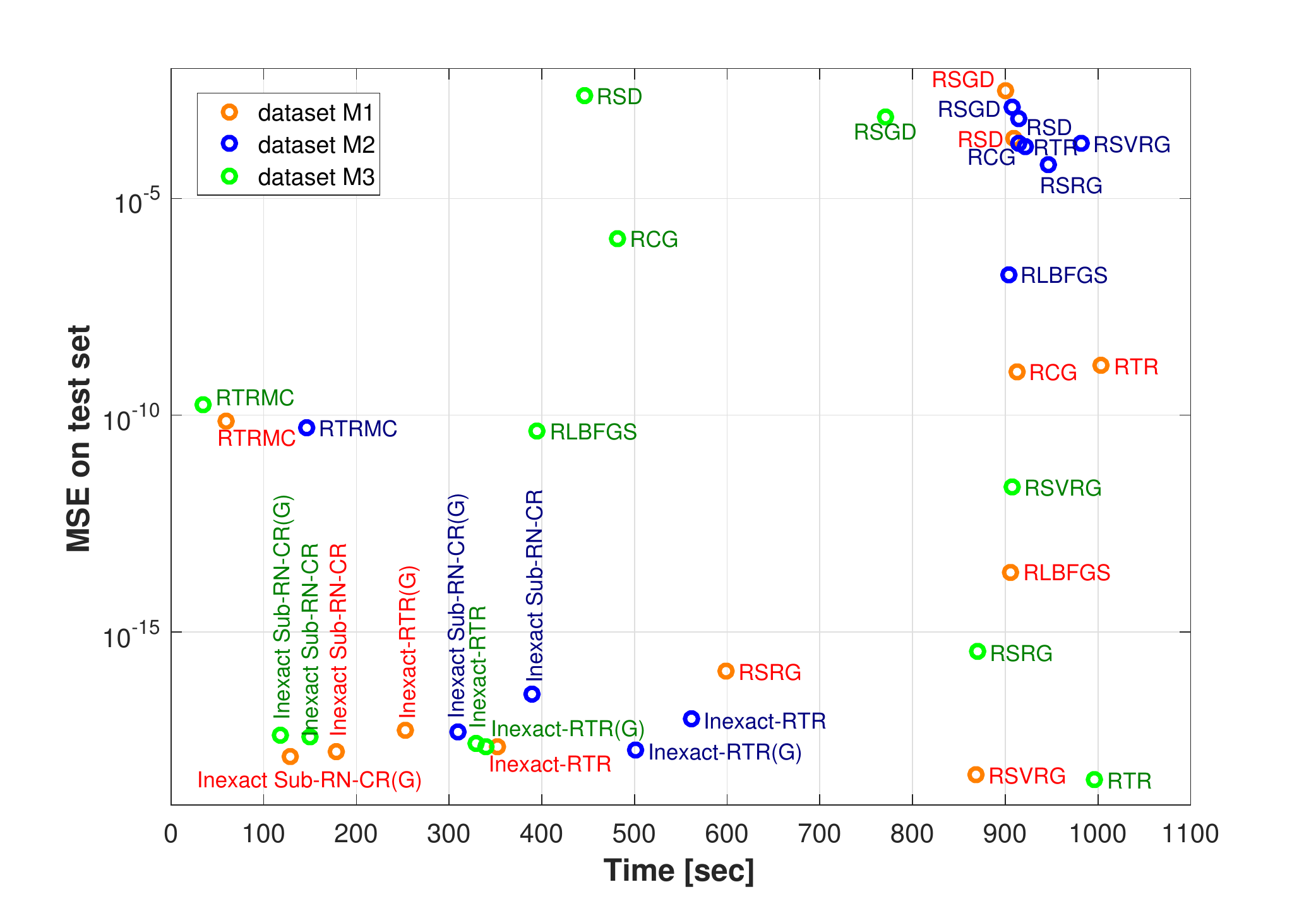}
	\caption{\label{fig_dots_mc} Matrix completion  performance summary on synthetic datasets.} %optimization results
\end{figure}

\begin{figure}[t]
	\centering
	\includegraphics[width=0.8\textwidth]{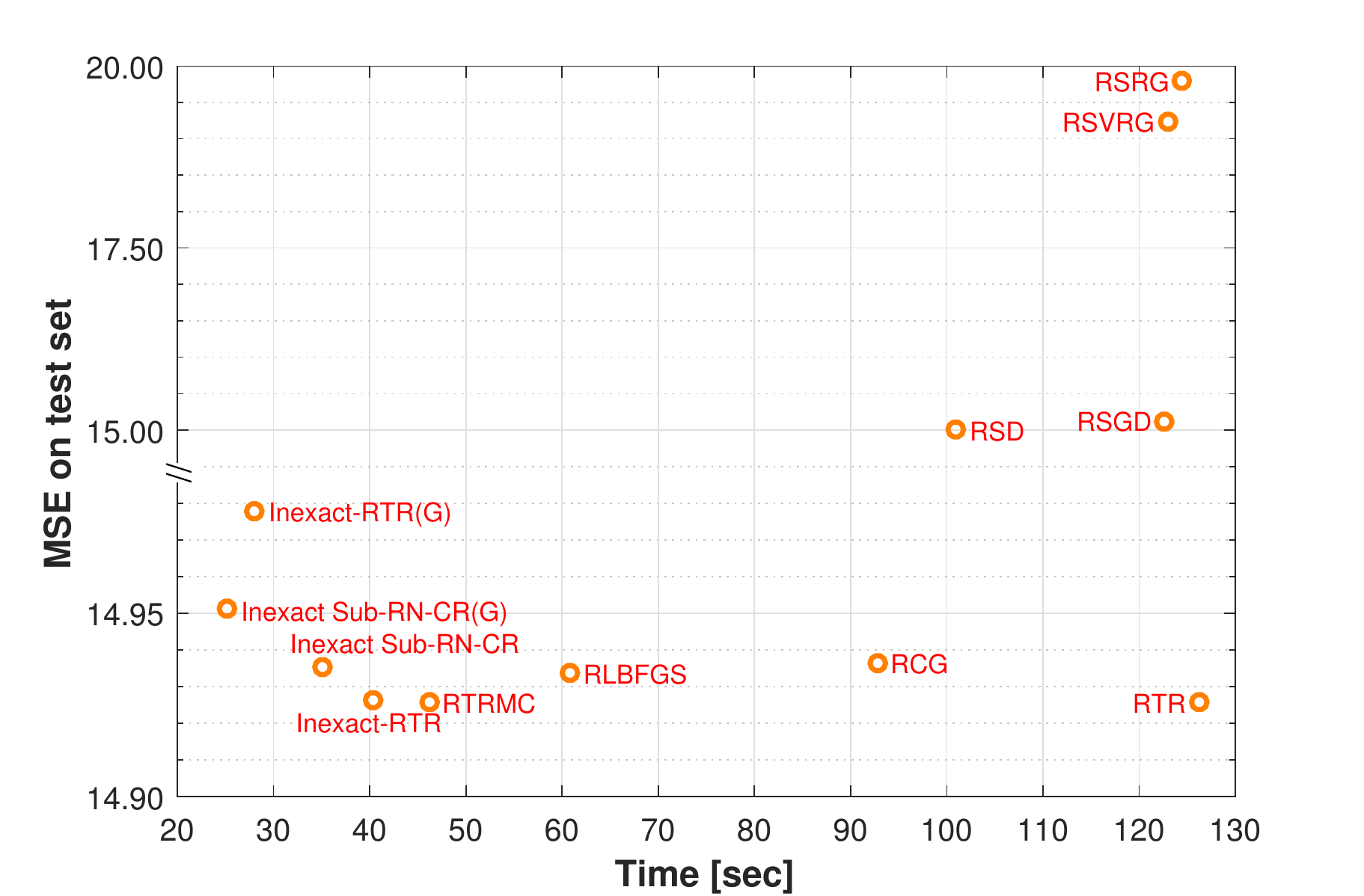}
	\caption{\label{fig_MC_dots_jester} Matrix completion  performance summary on  the Jester dataset.}
\end{figure}

Fig. \ref{fig_PCA_all} compares the     optimality gap  changes over iterations for all the  competing algorithms. 
Additionally, Fig. \ref{fig_pca_dots} summarizes their accuracy and convergence performance in optimality gap and run time.  
Fig. \ref{subfig_P1_b} reports the performance without using early stopping for the P1 dataset. 
It can be seen that the Inexact Sub-RN-CR reaches the minima with the smallest iteration number for both the synthetic and real-world datasets. 
In particular, the larger the scale of a problem is, the more obvious the advantage of our Inexact Sub-RN-CR  is,  evidenced by  the performance difference.

However,  both the Inexact RTR and Inexact Sub-RN-CR achieve their best PCA performance when using a full gradient calculation accompanied by a subsampled Hessian. The subsampled gradient does not seem to result in a satisfactory solution as shown in Fig. \ref{fig_PCA_all}  with $\left|\mathcal{S}_g\right|=n/10$. Additionally, we report more results for the Inexact RTR and the proposed Inexact Sub-RN-CR in Fig.  \ref{subfig_PCA_diff_sg1} on the P1 dataset with different gradient batch sizes, including  $\left|\mathcal{S}_g\right|\in\left\{n/1.5,n/2,n/5,n/10\right\}$. They all perform less well than $\left|\mathcal{S}_g\right| =n$. 
More accurate gradient information is required to produce a high-precision solution in these tested cases.
A hypothesis on the cause of  this phenomenon might be that the variance of the approximate gradients across samples is larger than that of the approximate Hessians. 
Hence, a sufficiently large sample size is needed for a stable approximation of the gradient information. 
Errors in approximate gradients may cause the algorithm to converge to a sub-optimal point with a higher cost, thus performing less well. 
Another hypothesis might be that the quadratic term $\mathbf{U}\mathbf{U}^T$ in  Eq. (\ref{eq_pca_formula}) would square the approximation error from the approximate gradient, which could significantly increase the PCA reconstruction error.

By fixing the gradient batch size $\left|\mathcal{S}_g\right| =n$ for both the Inexact RTR and Inexact Sub-RN-CR,  we compare in Figs. \ref{subfig_PCA_diff_sh1} and \ref{subfig_PCA_diff_sh2}  their sensitivity to the used batch size for Hessian approximation. 
We   experiment  with $\left|\mathcal{S}_H\right|\in\{n/10,n/10^2$, $n/10^3,n/10^4\}$. It can be seen that the Inexact Sub-RN-CR outperforms the Inexact RTR in almost all cases except for $\left|\mathcal{S}_H\right|=n/10^4$ for the MNIST dataset. The Inexact Sub-RN-CR possesses a rate of increase in oracle calls  significantly smaller for  large sample sizes. 
This implies that the Inexact Sub-RN-CR is more robust than the Inexact RTR  to batch-size change for inexact Hessian approximation.

\begin{figure}[t]
	\centering
	\subfloat[\label{fig_MC_T4} Synthetic M3 (left) and Jester  (right) with varying $|\mathcal{S}_g|$  settings.]{
		\includegraphics[width=0.5\textwidth]{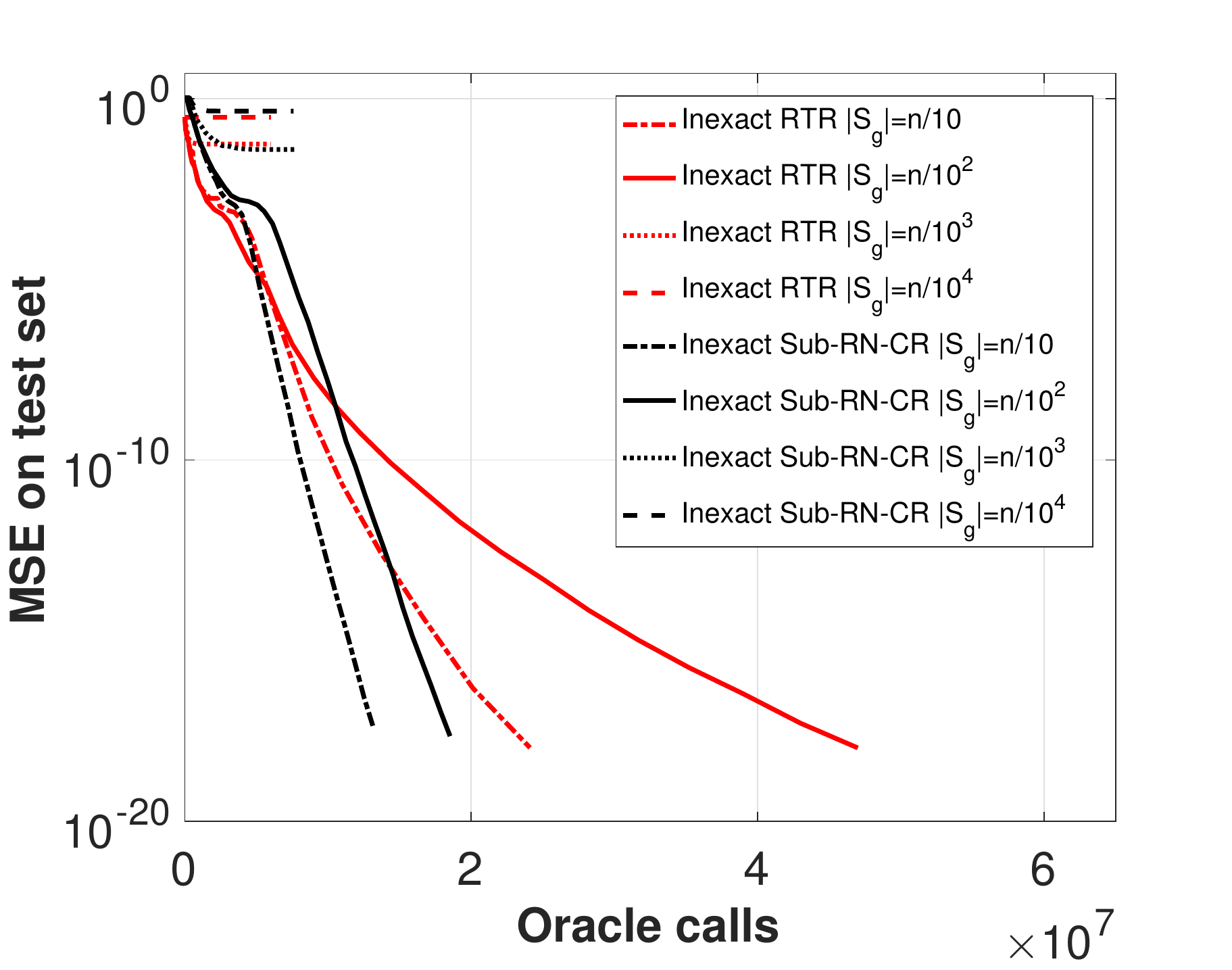}
		\includegraphics[width=0.5\textwidth]{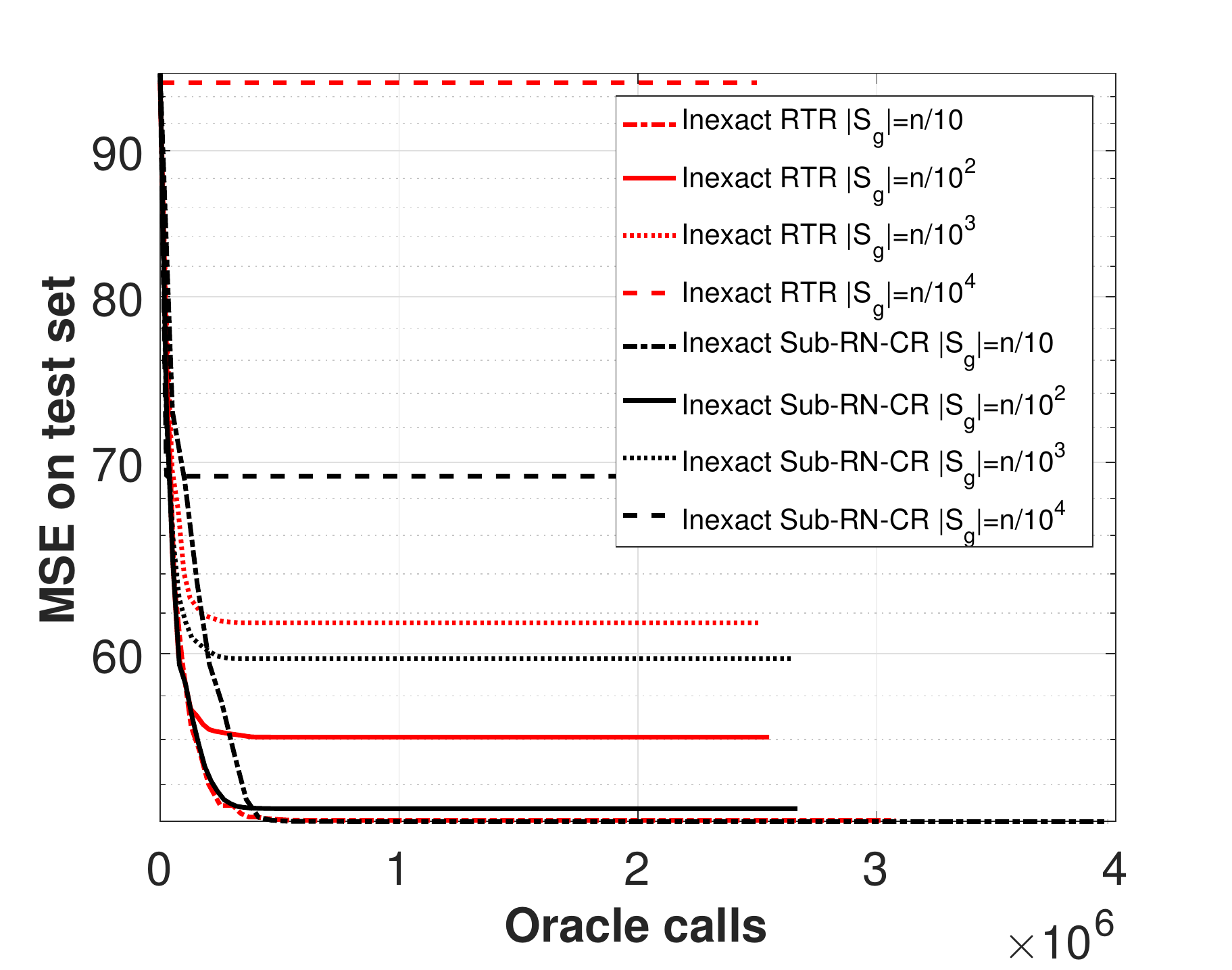}
	}
	\\
	\subfloat[\label{fig_MC_T5} Synthetic  M3 (left) and Jester  (right) with varying $|\mathcal{S}_H|$ settings. ]{%\label{subfig_P1_b}
		\includegraphics[width=0.5\textwidth]{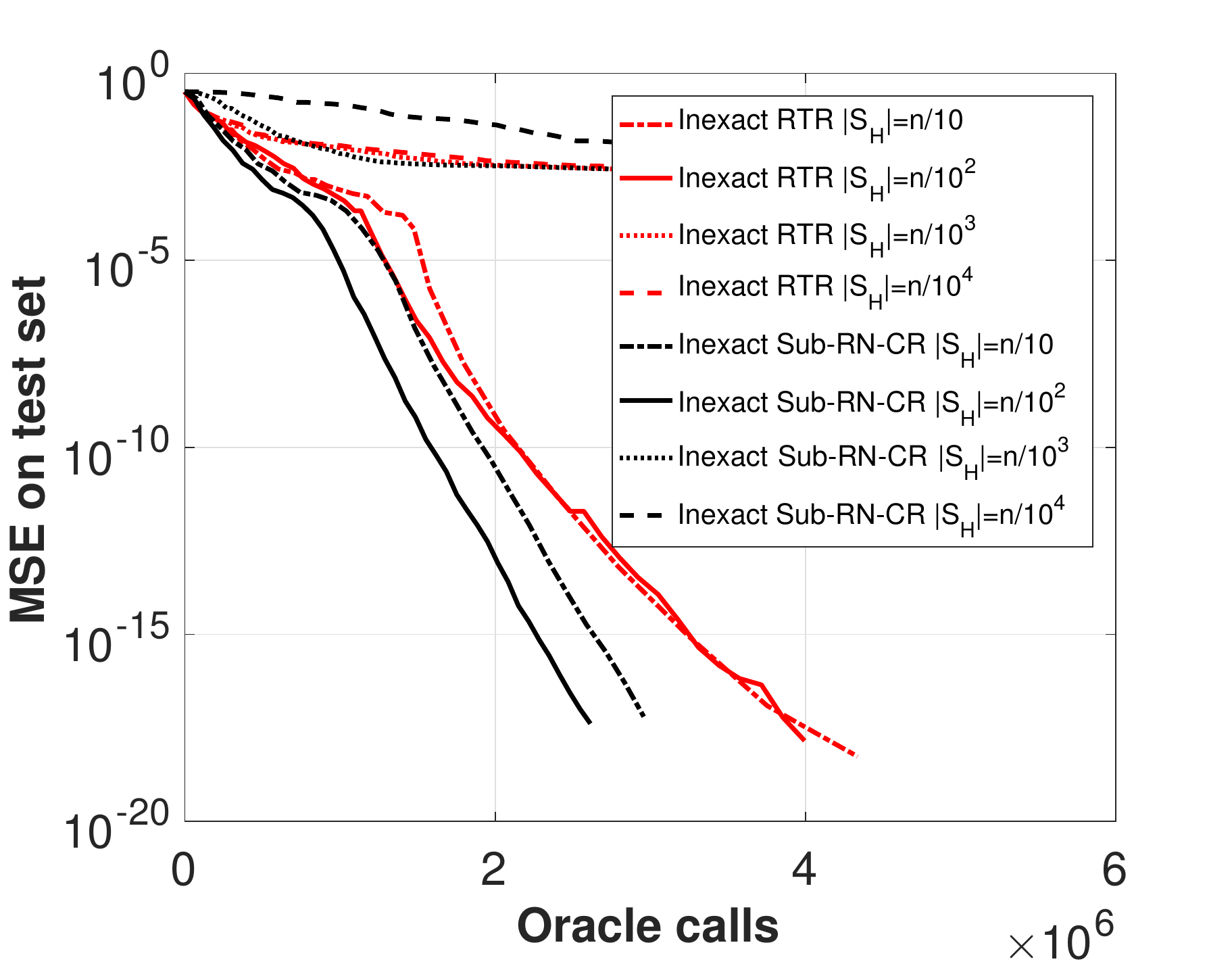}
		\includegraphics[width=0.5\textwidth]{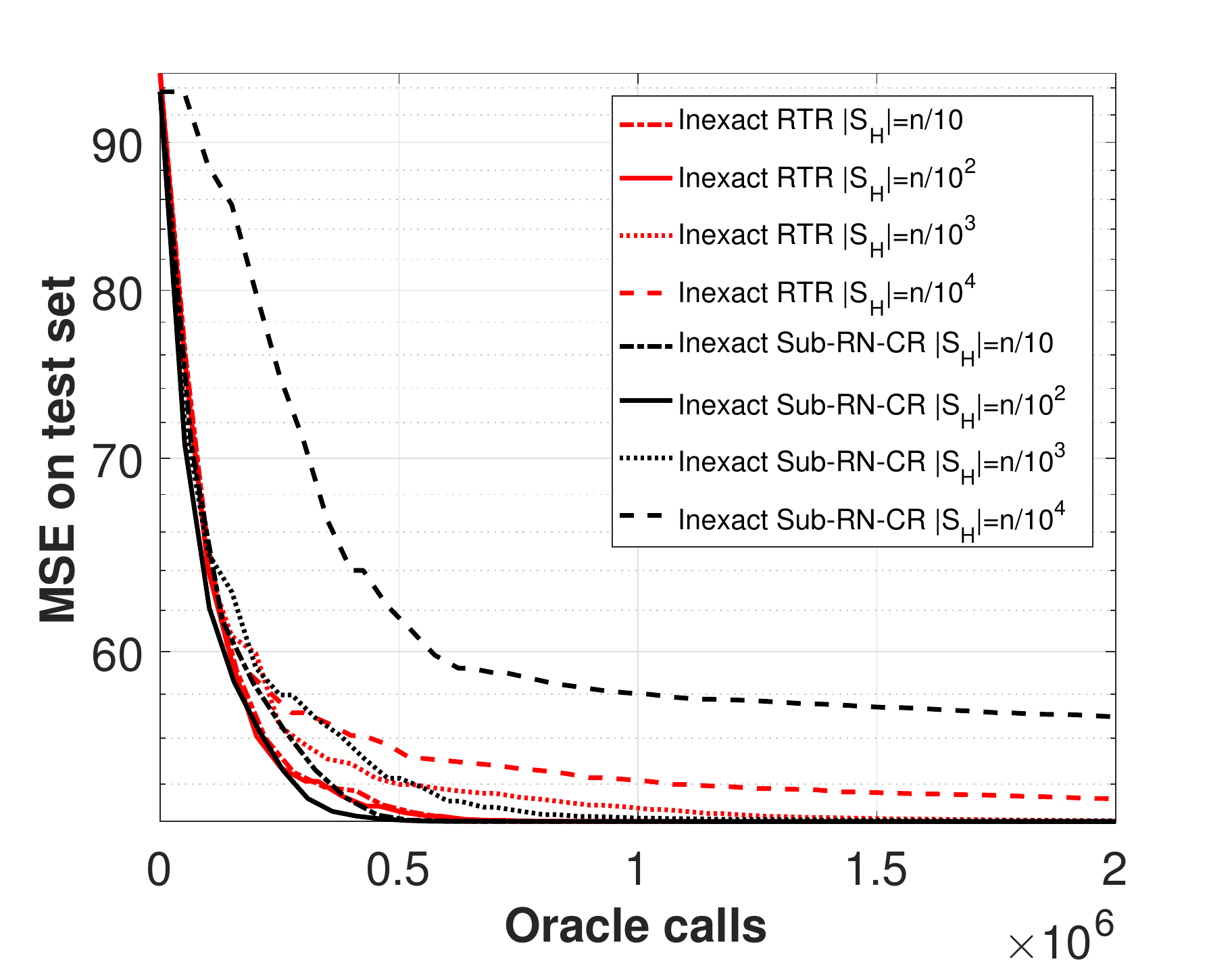}
	}
	\caption{\label{fig_MC_diff_sg_sh}  Inexact RTR vs. Inexact Sub-RN-CR for matrix completion with varying subsampling sizes for gradient and Hessian approximation.}
\end{figure}

\subsection{Low-rank Matrix Completion Experiments}

Low-rank matrix completion aims at completing a partial matrix $\mathbf{Z}$ with only a small number of  entries observed, under the assumption that the matrix has a low rank. 
One way to formulate the problem is shown as below
\begin{equation}
	\min_{\mathbf{U}\in {\rm Gr}\left(r,d\right), \mathbf{A}\in\mathbb{R}^{r\times n}}\frac{1}{|\Omega|}\left\|\mathcal{P}_\Omega\left(\mathbf{UA}\right)-\mathcal{P}_\Omega\left(\mathbf{Z}\right)\right\|_F^2,
	\label{eq_mc_formula}
\end{equation}
where $\Omega$ denotes  the index set of the observed matrix entries. 
The operator $\mathcal{P}_\Omega: \mathbb{R}^{d\times n}\rightarrow \mathbb{R}^{d\times n}$ is defined as $\mathcal{P}_\Omega\left(\mathbf{Z}\right)_{ij}=\mathbf{Z}_{ij}$ if $\left(i,j \right)\in\Omega$, while $\mathcal{P}_\Omega\left(\mathbf{Z}\right)_{ij}=0$ otherwise. 
We generate it by uniformly sampling a set of $|\Omega|=4r(n+d-r)$ elements from the  $dn$ entries. 
Let $\mathbf{a}_i$ be the $i$-th column of $\mathbf{A}$, $\mathbf{z}_i$ be the $i$-th column of $\mathbf{Z}$, and ${\Omega}_i$ be the subset of ${\Omega}$ that contains sampled indices for the $i$-th column of $\mathbf{Z}$. 
Then, $\mathbf{a}_i$ has a closed-form solution $\mathbf{a}_i=(\mathbf{U}_{{\Omega}_i})^{\dagger}\mathbf{z}_{\Omega_i}$ \cite{kasai2018inexact}, where $\mathbf{U}_{{\Omega}_i}$ contains the selected rows of $\mathbf{U}$, and $\mathbf{z}_{\Omega_i}$  the selected elements of  $\mathbf{z}_i$ according to the indices in $\Omega_i$, and $\dagger$ denotes the pseudo inverse operation.
To evaluate a solution $\mathbf{U}$, we generate another index set $\tilde{\Omega}$, which is used as the test set and satisfies $\tilde{\Omega}\cap\Omega = \emptyset $, following the same way as generating $\Omega$. We compute  the mean squared error (MSE)  by
\begin{equation}
	\textmd{MSE} = \frac{1}{\left|\tilde{\Omega} \right|}\left\|\mathcal{P}_{\tilde{\Omega}}(\mathbf{U}\mathbf{A})-\mathcal{P}_{\tilde{\Omega}}(\mathbf{Z})\right\|_F^2.
\end{equation}

In evaluation, three synthetic datasets and a real-world dataset Jester \cite{goldberg2001eigentaste} are used where the training and test sets are already predefined by \cite{goldberg2001eigentaste}. 
The synthetic datasets are generated by following a generative model similar to \cite{ngo2012scaled} based on SVD. 
Specifically, to develop a synthetic dataset, we generate two matrices $\mathbf{A}\in\mathbb{R}^{d\times r}$ and $\mathbf{B}\in\mathbb{R}^{n\times r}$ with their elements independently sampled from the normal distribution $\mathcal{N}(0,1)$. 
Then, we generate two orthogonal matrices $\mathbf{Q}_A$ and $\mathbf{Q}_B$ by applying the QR decomposition \cite{trefethen1997numerical} respectively to $\mathbf{A}$ and $\mathbf{B}$. 
After that, we construct a diagonal matrix $\mathbf{S}\in \mathbb{R}^{r\times r}$ of which the diagonal elements are computed by 
$s_{i,i}=10^{3+\frac{(i-r)\log_{10}(c)}{r-1}}$ for $i=1,...,r$, and the final data matrix  is computed by $\mathbf{Z}=\mathbf{Q}_A\mathbf{S}\mathbf{Q}_B^T$.
The reason to construct  $\mathbf{S}$ in this specific way is to have an easy control over the condition number of the data matrix, denoted by $\kappa(\mathbf{Z})$, which is the ratio between the maximum and minimum singular values of $\mathbf{Z}$. 
Because $\kappa(\mathbf{Z}) = \frac{\sigma_{max}(\mathbf{Z})}{\sigma_{min}(\mathbf{Z})}=\frac{10^3}{10^{3-\log_{10}(c)} }=c$, we can adjust the condition number by tuning the parameter $c$.
Following this generative model, each synthetic dataset is generated by randomly sampling two orthogonal matrices and constructing one diagonal matrix subject to the constraint of condition numbers, i.e., $c=\kappa(\mathbf{Z})=5$ for M1 and $c=\kappa(\mathbf{Z})=20$ for M2 and M3.
The $\left(n,d,r\right)$  values  of the used datasets are:  $\left(10^5,10^2,5\right)$ for M1, $\left(10^5,10^2,5\right)$ for M2,  $\left(3\times 10^4,10^2,5\right)$ for M3, and  $\left(24983,10^2,5\right)$ for Jester.

\begin{figure}[t]
	\centering
	\includegraphics[width=0.45\textwidth]{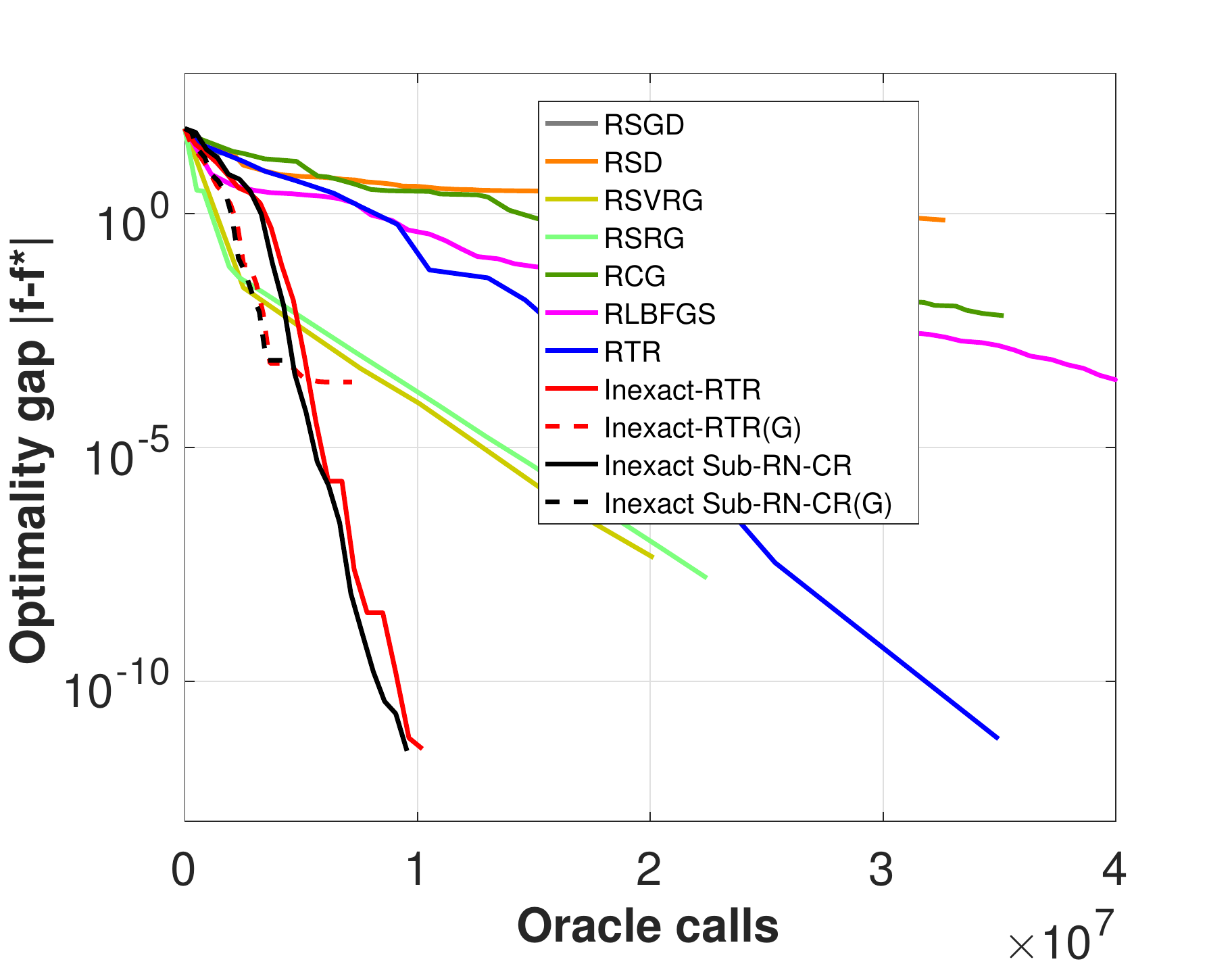}
	\includegraphics[width=0.45\textwidth]{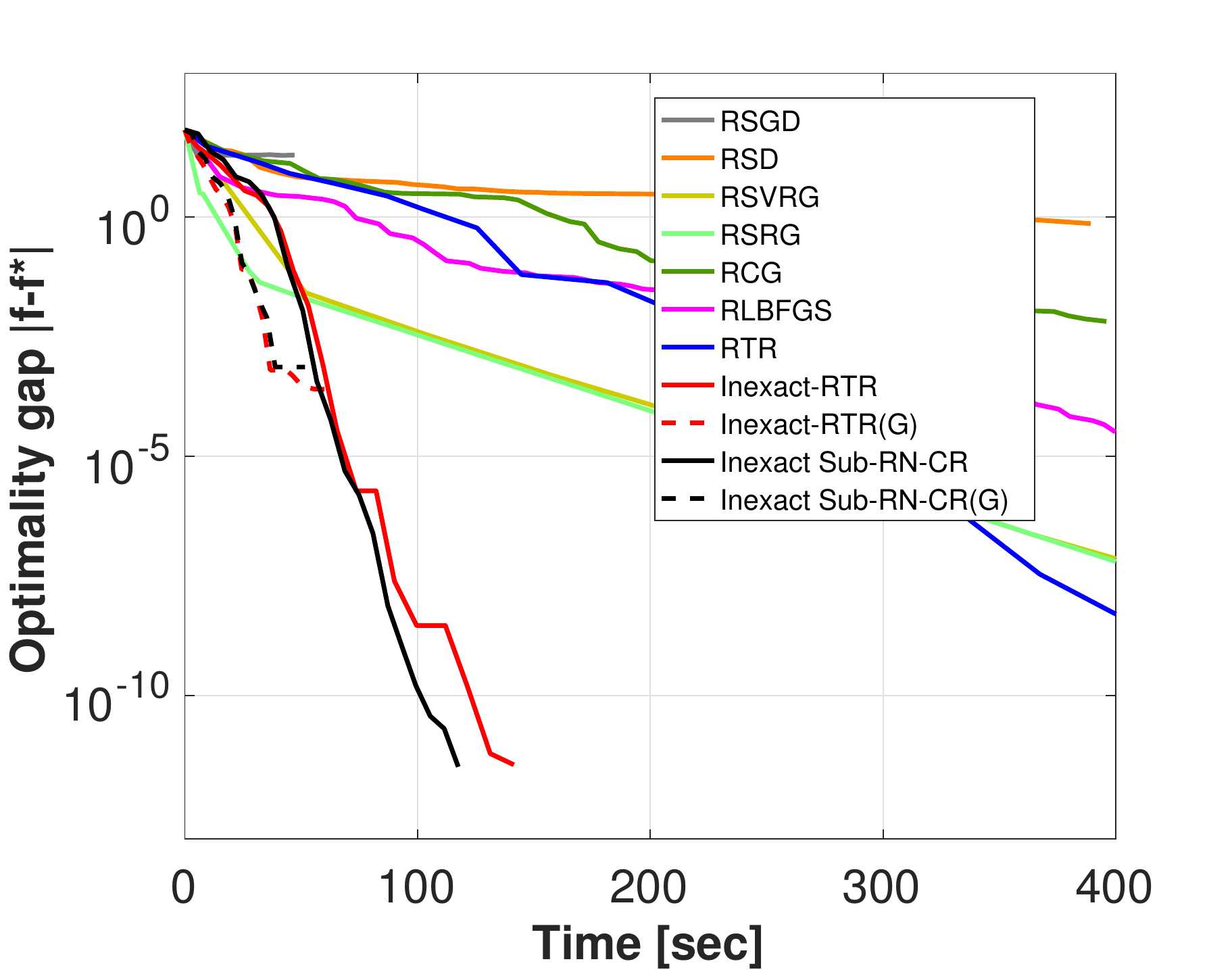}
	\\
	\includegraphics[width=0.8\textwidth]{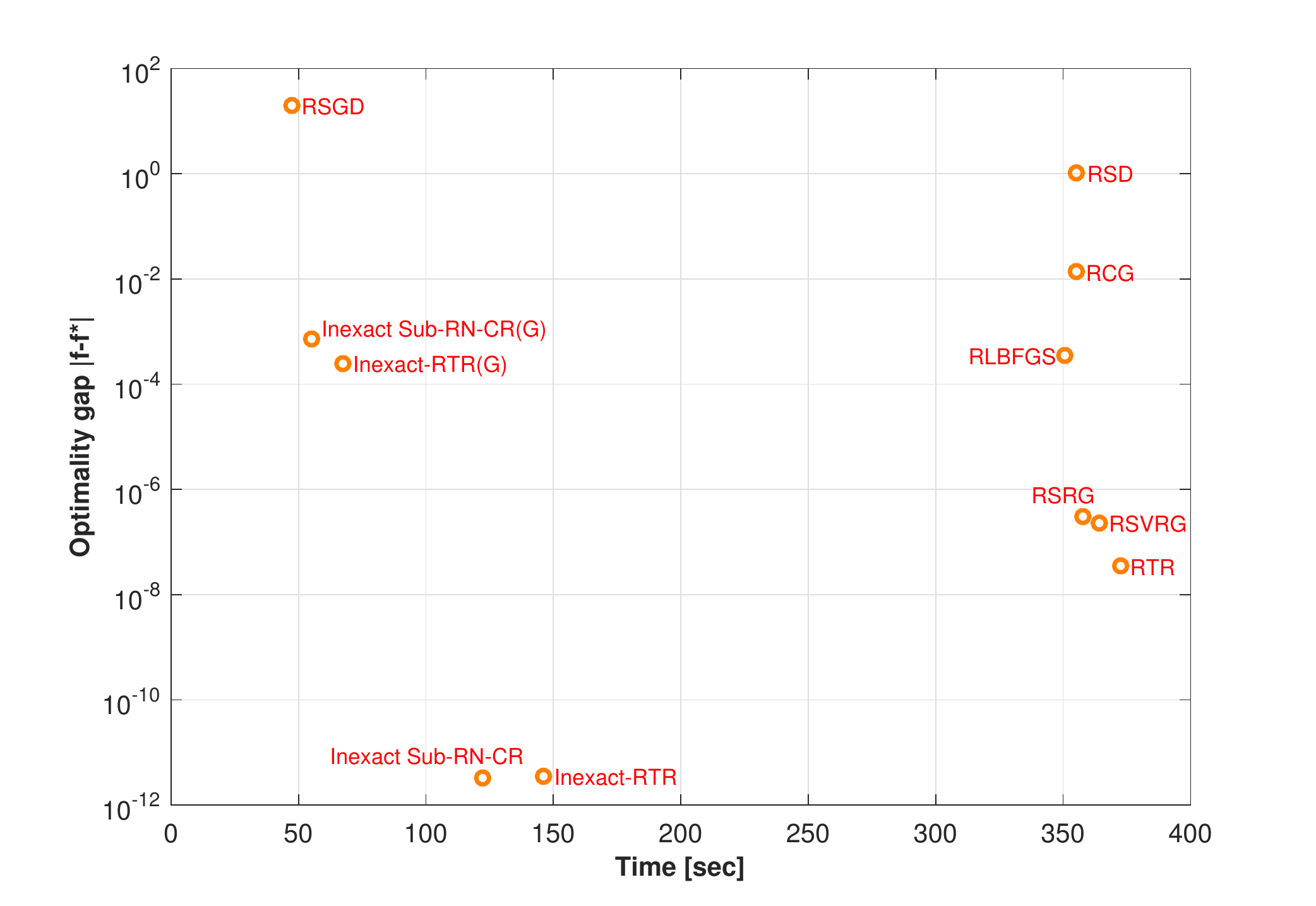}
	\caption{\label{fig_fmri_all}  PCA optimality gap comparison for fMRI analysis.}
\end{figure}

Fig. \ref{fig_MC_all}  compares the MSE changes over iterations, while Fig. \ref{fig_dots_mc} and Fig.  \ref{fig_MC_dots_jester}  summarize both the MSE performance and the  run time in the same plot for different algorithms and datasets.  In Fig. \ref{fig_dots_mc}, the Inexact Sub-RN-CR outperforms the others in most cases, and it can even be nearly twice as fast as the state-of-the-art methods for cases with a larger condition number (see dataset M2 and M3). This  shows that  the proposed algorithm is efficient at handling ill-conditioned problems.   In Fig. \ref{fig_MC_dots_jester}, the Inexact Sub-RN-CR achieves a sufficiently small MSE with the shortest run time, faster than the Inexact RTR and RTRMC. 
Unlike in the PCA task,  the subsampled gradient approximation actually helps to improve the convergence.
A hypothesis for explaining this phenomenon could be that, as compared to the quadratic term $\mathbf{U}\mathbf{U}^T$ in the PCA objective function, the linear term $\mathbf{U}$ in the matrix completion objective function accumulates fewer errors from the inexact gradient, making the optimization more stable.

Additionally, Fig. \ref{fig_MC_T4} compares the Inexact RTR and the Inexact Sub-RN-CR with varying batch sizes for gradient estimation and with fixed $|\mathcal{S}_H| =n$. The M1-M3 results show that our algorithm exhibits stronger robustness to $|\mathcal{S}_g|$, as it converges to the minima with only 50$\%$ additional oracle calls when reducing $|\mathcal{S}_g|$ from $n/10$ to $n/10^2$, whereas Inexact RTR requires twice as many calls. For the Jester dataset, in all settings of gradient sample sizes, our method achieves lower MSE than the Inexact RTR, especially when $|\mathcal{S}_g|=n/10^4$.  
Fig. \ref{fig_MC_T5} compares  sensitivities in Hessian sample sizes $|\mathcal{S}_H|$  with fixed $|\mathcal{S}_g| =n$. Inexact Sub-RN-CR performs better for the synthetic dataset M3 with $\left|\mathcal{S}_H\right|\in\{n/10,n/10^2\}$, showing roughly twice faster convergence. For the Jester dataset, Inexact Sub-RN-CR performs better with $\left|\mathcal{S}_H\right|\in\{n/10,n/10^2,n/10^3\}$ except for the case of $\left|\mathcal{S}_H\right|=n/10^4$, which is possibly because the construction of the Krylov subspace requires a more accurately approximated Hessian.

To summarize, we have observed from both the PCA and matrix completion tasks that the effectiveness of the subsampled gradient in the proposed approach can be sensitive to the choice of the practical problems, while the subsampled Hessian steadily contributes to a faster convergence rate.

\begin{figure}[t]
	\centering
	\subfloat[\label{fig_fmri_comp1}component 1]{%\label{subfig_P1_b}
		\includegraphics[width=1.0\textwidth]{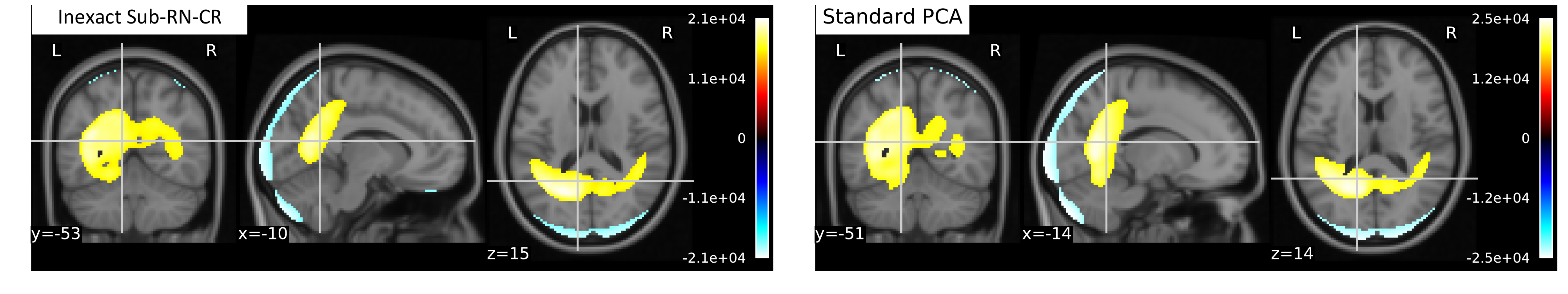}
	}
	\\
	\subfloat[\label{fig_fmri_comp2}component 2]{%\label{subfig_P1_b}
		\includegraphics[width=1.0\textwidth]{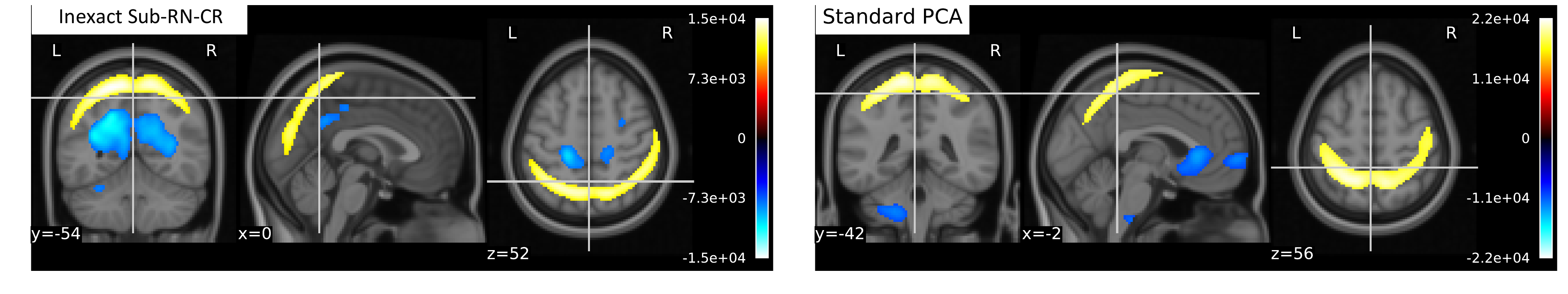}
	}
	\\
	\subfloat[\label{fig_fmri_comp3}component 3]{%\label{subfig_P1_b}
		\includegraphics[width=1.0\textwidth]{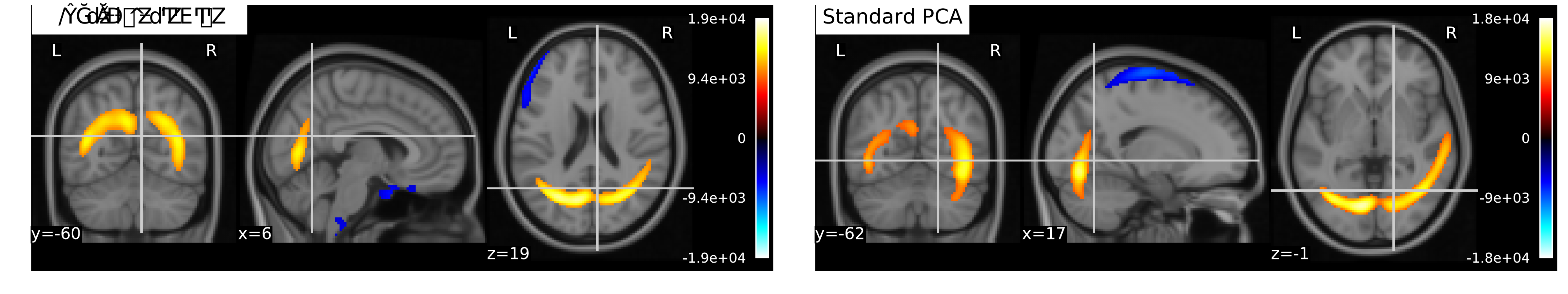}
	}
	\caption{\label{fig_fmri_pca_show} Comparison of the learned fMRI principal components.}
\end{figure}

\begin{figure}[t]
	\centering
	\includegraphics[width=0.45\textwidth]{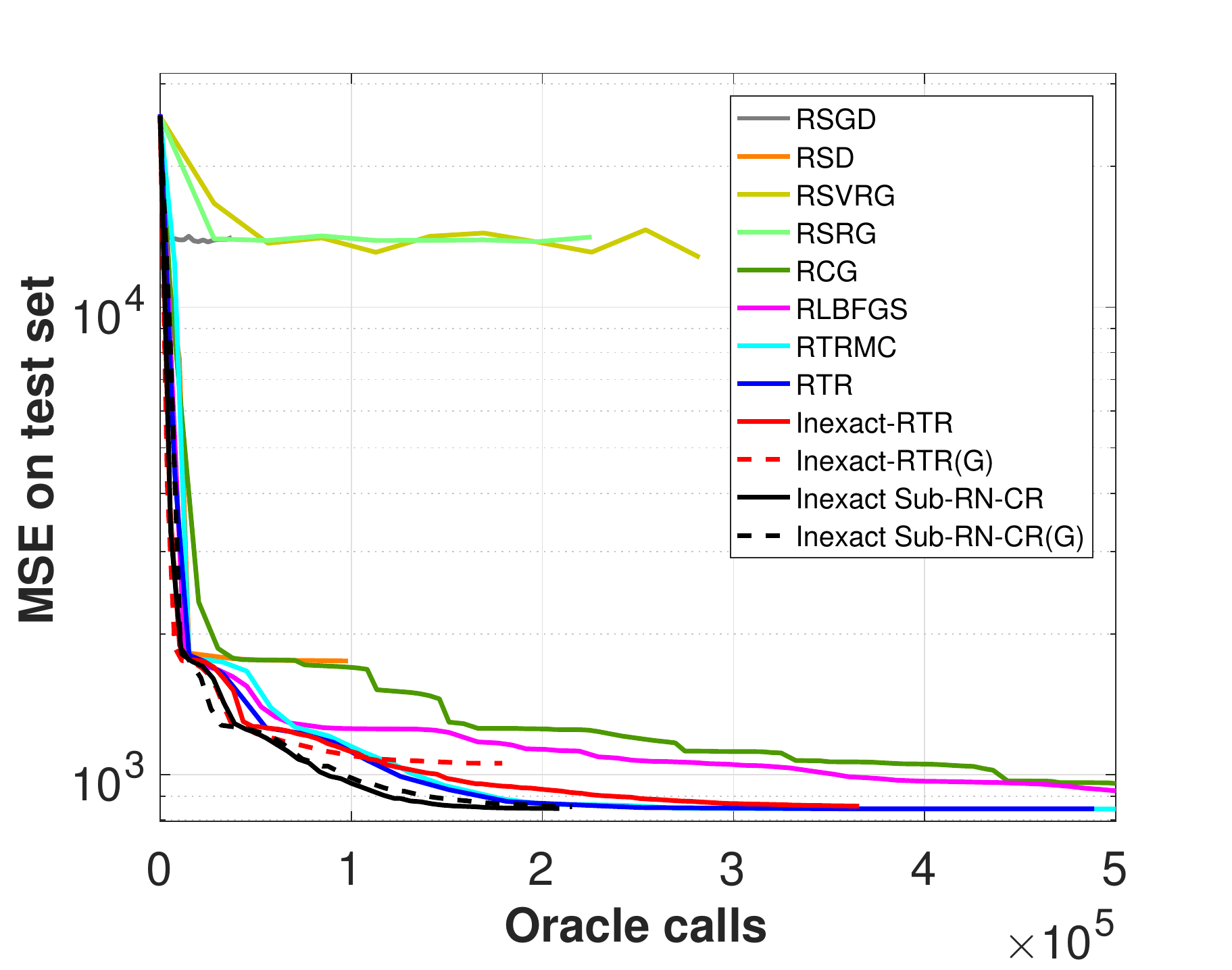}
	\includegraphics[width=0.45\textwidth]{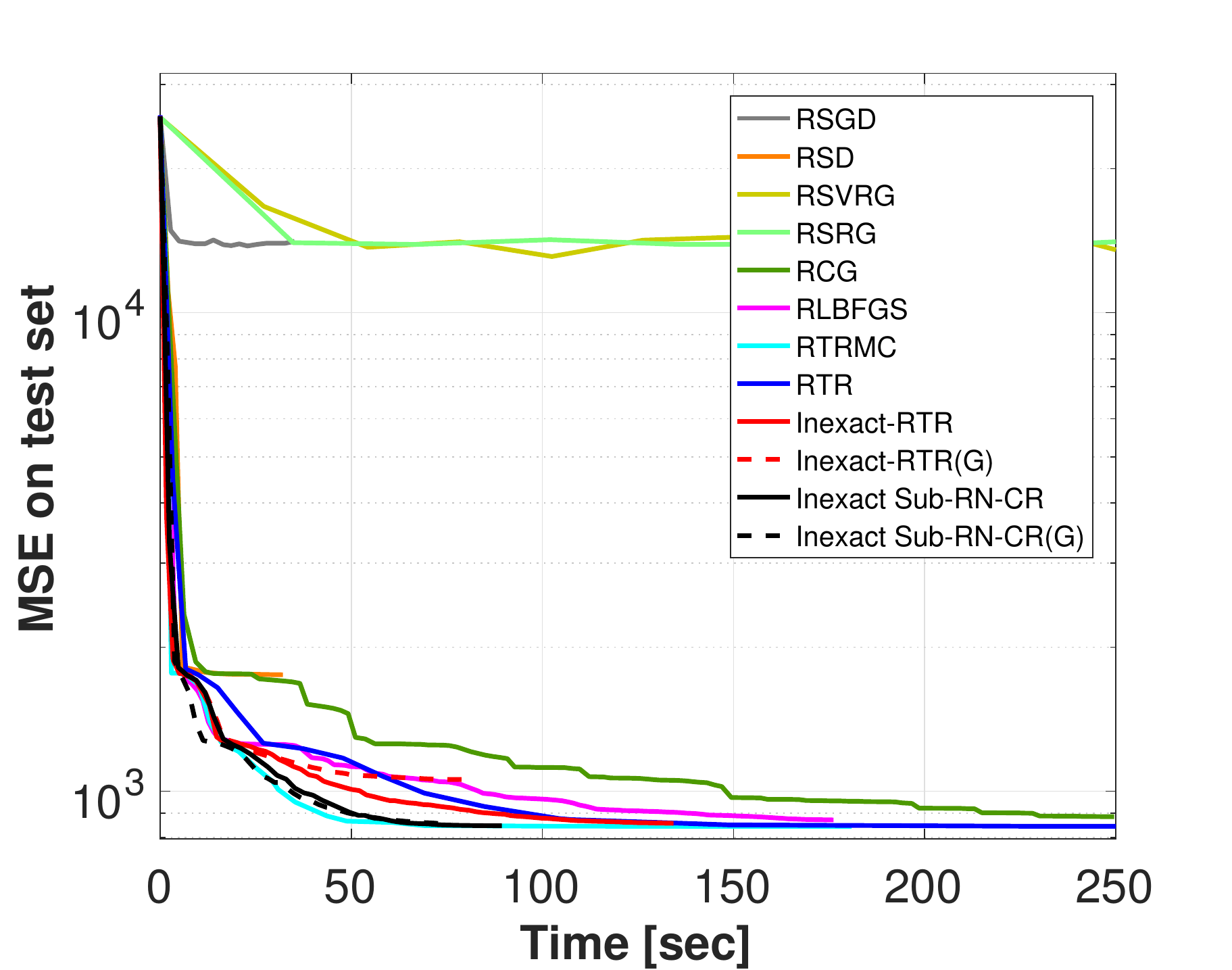}
	\\
	\includegraphics[width=0.7\textwidth]{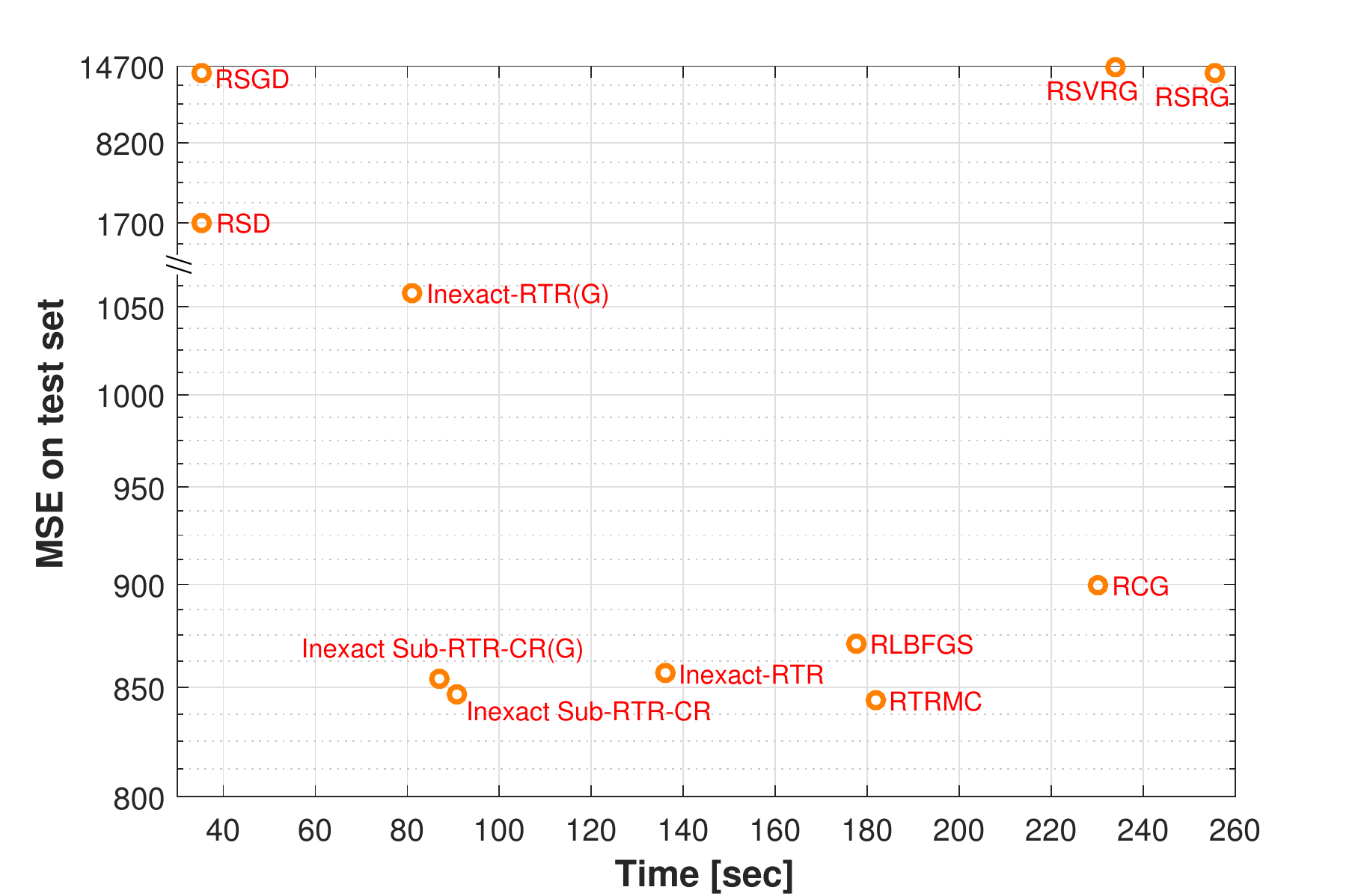}
	\caption{\label{fig_big_all} MSE comparison for scene image recovery by matrix completion.}
\end{figure}

\subsection{Imaging Applications}
In this section, we demonstrate some practical applications of PCA and matrix completion, which are solved by using the proposed optimization algorithm Inexact Sub-RN-CR,  for analyzing medical images and scene images.

\subsubsection{Functional Connectivity in fMRI by PCA}

Functional magnetic-resonance imaging (fMRI) can be used to measure brain activities and PCA is often used to find functional connectivities between brain regions based on the fMRI scans \cite{zhong2009detecting}.  This method is based on the assumption that the activation is independent of other signal variations such as brain motion and physiological signals \cite{zhong2009detecting}. Usually, the fMRI images are represented as a 4D data block subject to observational noise, including 3 spatial dimensions and 1 time dimension.   Following a common preprocessing routine \cite{sidhu2012kernel,kohoutova2020toward}, we denote an fMRI data block by  $\mathbf{D}\in\mathbb{R}^{u\times v\times w\times T}$ and a mask by $\mathbf{M}\in\{0,1\}^{u\times v\times w}$ that contains $d$ nonzero elements marked by $1$. By applying the mask to the data block, we obtain a  feature matrix $\mathbf{f}\in\mathbb{R}^{d\times T}$, where each column stores the features of the brain at a given time stamp.  One can increase the sample size by collecting $k$ fMRI data blocks corresponding to $k$ human subjects, after which the feature matrix is expanded to a larger matrix $\mathbf{F}\in\mathbb{R}^{d\times kT}$.

In this experiment, an fMRI dataset referred to as $ds000030$ provided by the OpenfMRI database \cite{poldrack2017openfmri} is used, where $u=v=64$, $w=34$, and $T=184$. We select $k=13$ human subjects and use the provided brain mask with $d=228483$. The dimension of the final data matrix is $(n,d)=(2392,228483)$, where $n=kT$. 
We  set the rank as $r=5$ which is sufficient to capture over $93\%$ of the variance in the data. After the PCA processing, each computed principal component can be rendered back to the brain reconstruction by using the open-source library Nilearn \cite{abraham2014machine}. Fig. \ref{fig_fmri_all} displays the optimization performance, where the Inexact Sub-RN-CR converges faster in terms of both run time and oracle calls. For our method and the Inexact RTR, adopting the subsampled gradient leads to a suboptimal solution in less time than using the full gradient.  We speculate that imprecise gradients cause an oscillation of the optimization near local optima. Fig. \ref{fig_fmri_pca_show} compares the results obtained by our optimization algorithm with those computed by eigen-decomposition. The highlighted regions denote the main activated regions with positive connectivity (yellow) and negative connectivity (blue). The components learned by the two approaches are similar, with some cases learned by our approach tending to be more connected (see Figs. \ref{fig_fmri_comp1} and \ref{fig_fmri_comp3}).

\subsubsection{Image Recovery by Matrix Completion}
In this application, we demonstrate image recovery with matrix completion using a $(W,H,C)=2519\times1679\times3$ scene image selected from the BIG dataset \cite{cheng2020cascadepsp}. As seen in Fig. \ref{fig_big_ori}, this image contains rich texture information. The values  of $(n,d,r)$  for conducting the matrix completion task are $(1679,2519,20)$ where we use a relatively large rank to allow more accurate recovery. The sampling ratio for observing the pixels is set as ${\rm SR}=\frac{\left|\Omega\right|}{W\times H\times C} =0.6$.
Fig. \ref{fig_big_all} compares the performance of different algorithms, where the Inexact Sub-RN-CR takes the shortest time to obtain a satisfactory solution. The subsampled gradient promotes the optimization speed of the Inexact Sub-RN-CR without sacrificing much  the MSE error. Fig. \ref{fig_big_mc_show} illustrates the recovered image using three representative algorithms, providing  similar visual results.

\begin{figure}[t]
	\centering
	\subfloat[\label{fig_big_ori}Original image]{%\label{subfig_P1_b}
		\includegraphics[width=0.3\textwidth]{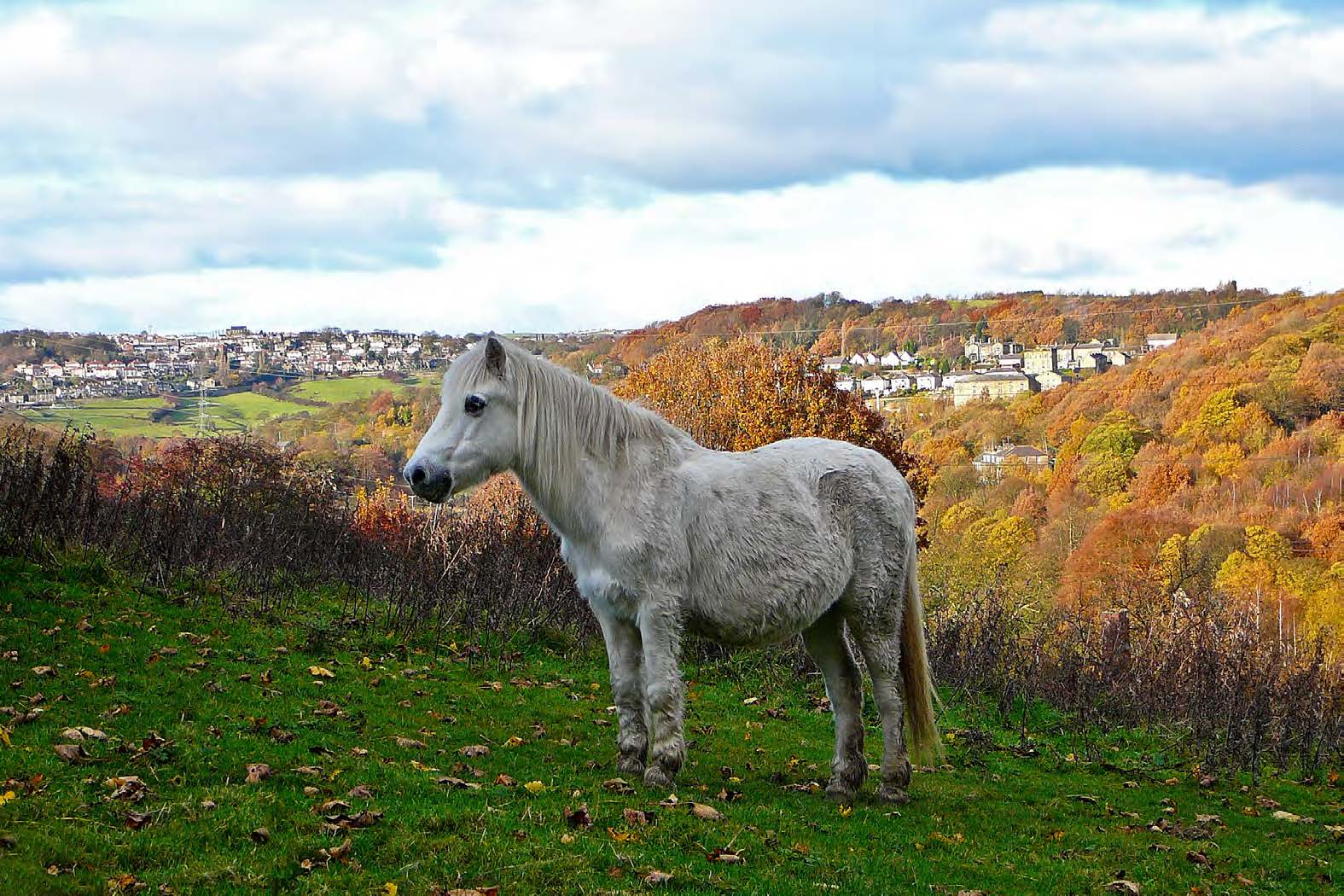}
	}
	\subfloat[\label{fig_big_mask}Observation]{%\label{subfig_P1_b}
		\includegraphics[width=0.3\textwidth]{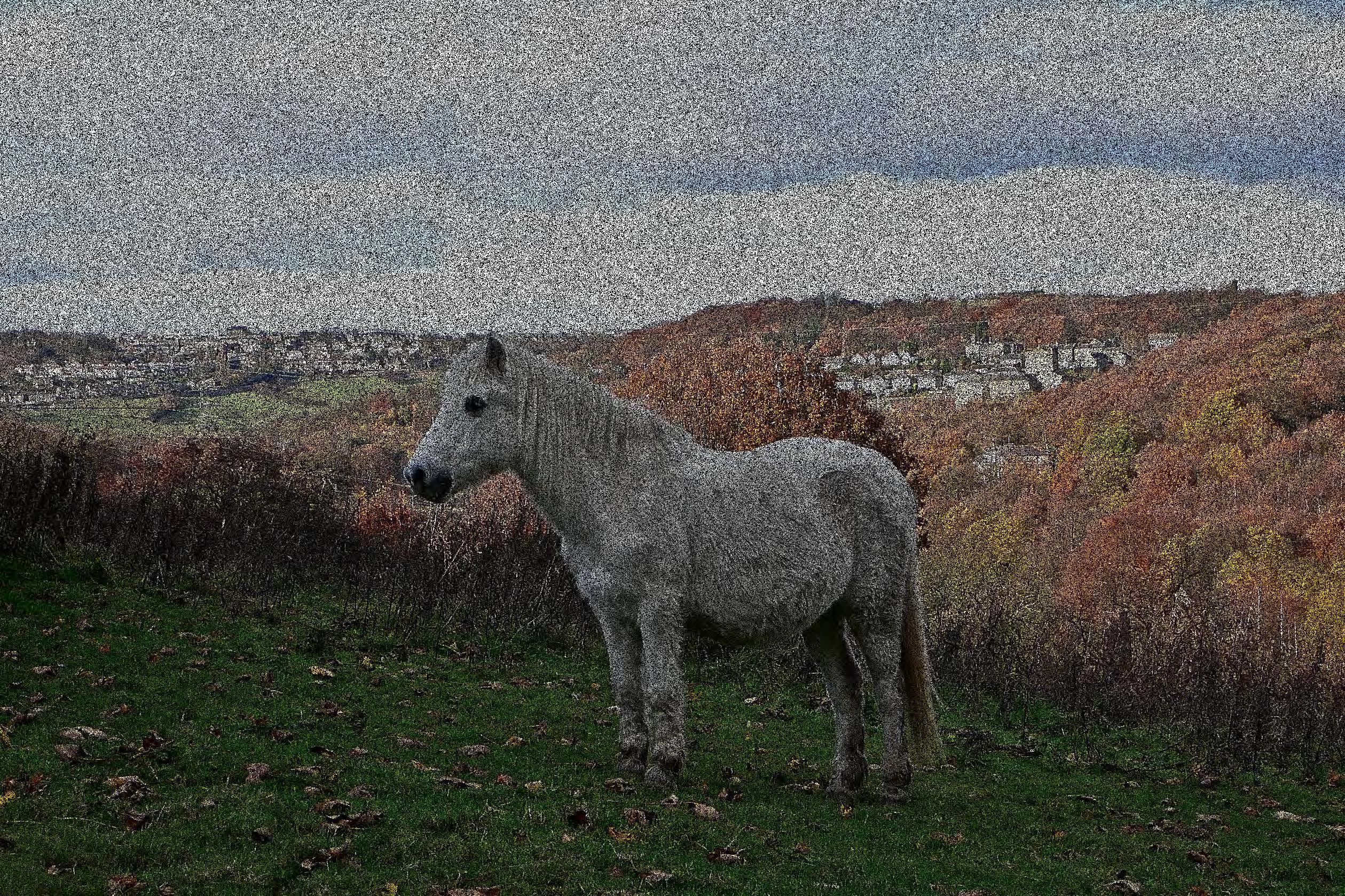}
	}
	\\
	\subfloat[\label{fig_big_rscr}Inexact Sub-RN-CR]{%\label{subfig_P1_b}
		\includegraphics[width=0.3\textwidth]{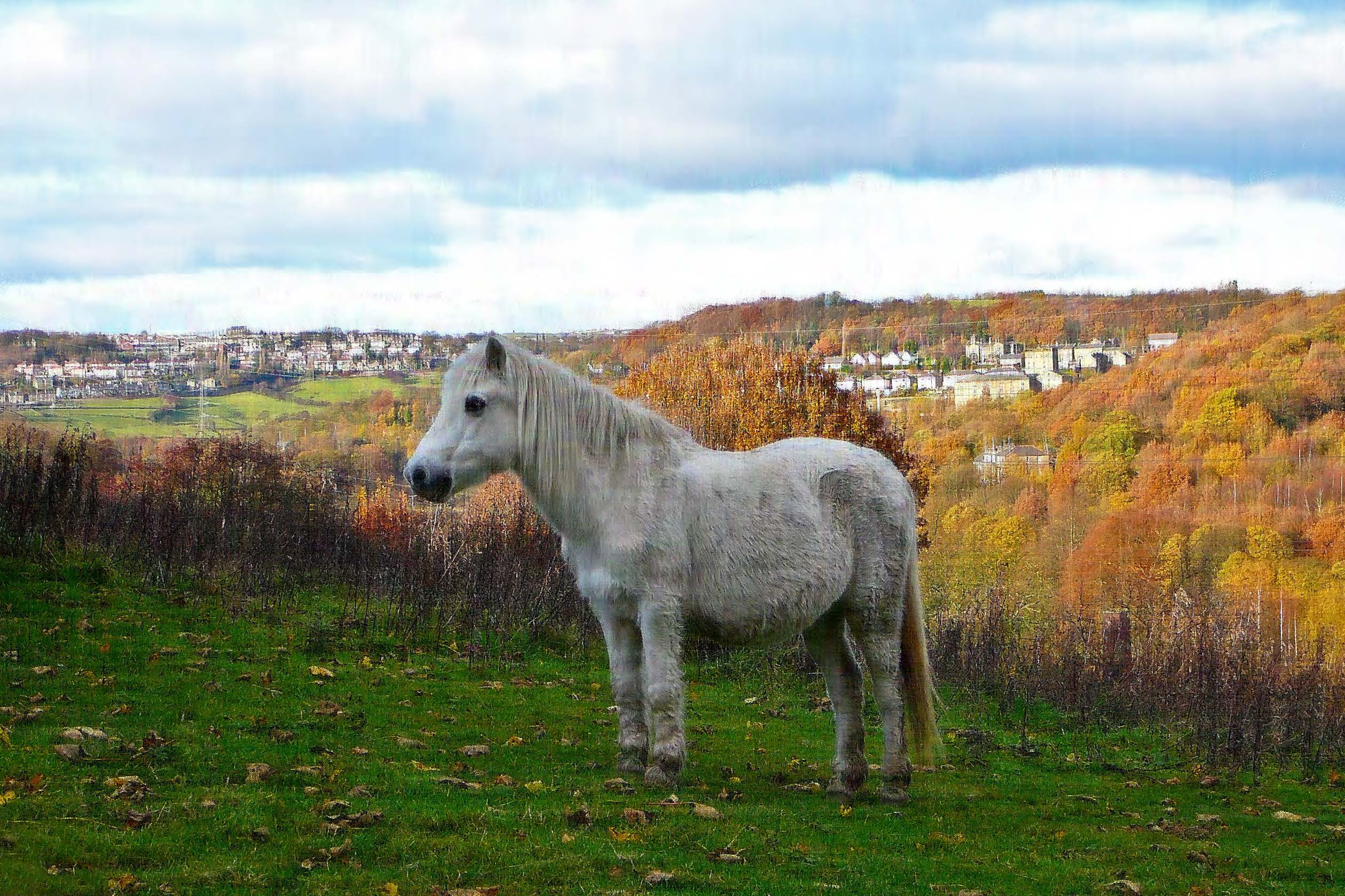}
	}
	\subfloat[\label{fig_big_rtr}Inexact RTR]{%\label{subfig_P1_b}
		\includegraphics[width=0.3\textwidth]{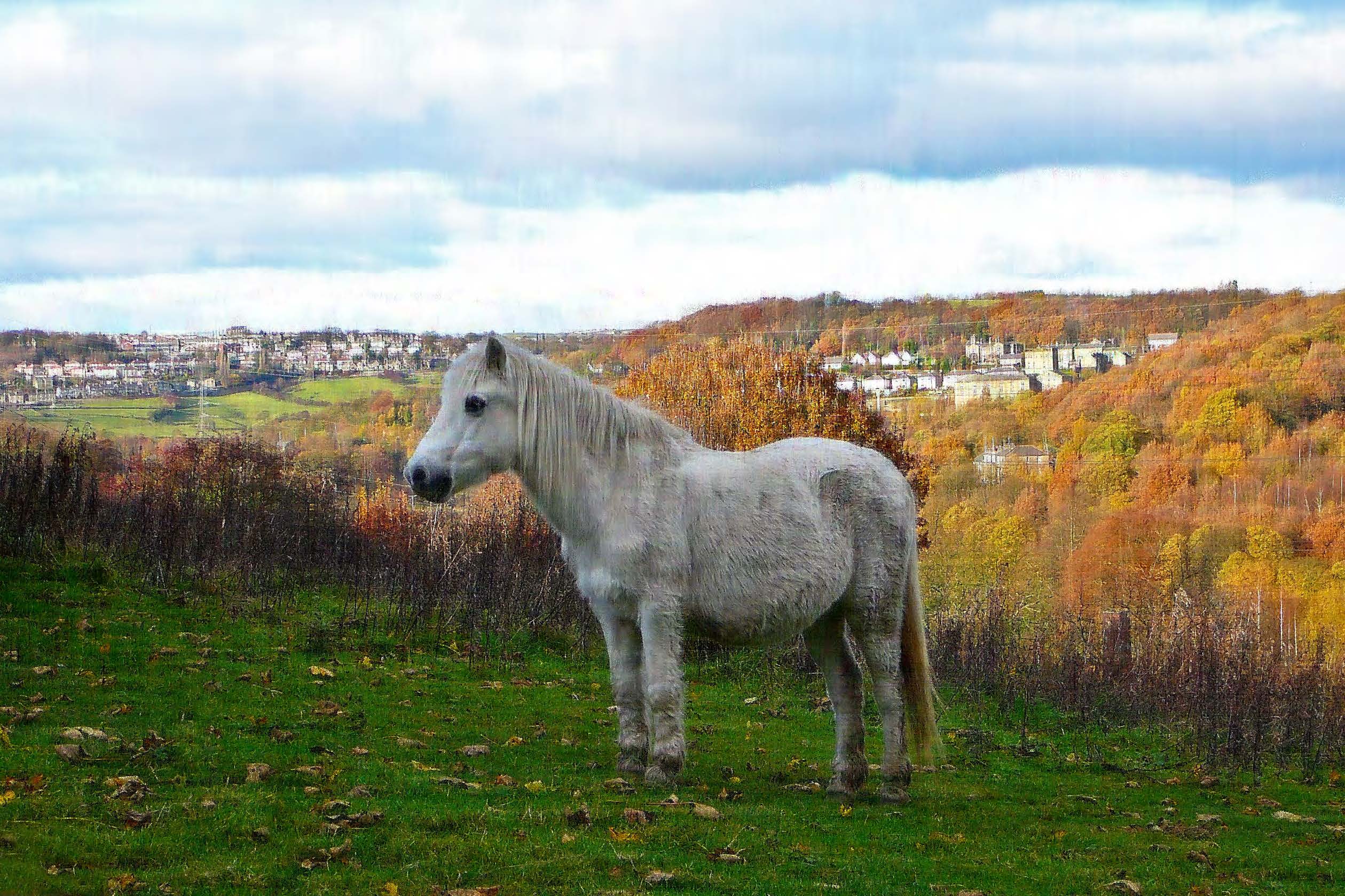}
	}
	\subfloat[\label{fig_big_rtrmc}RTRMC]{%\label{subfig_P1_b}
		\includegraphics[width=0.3\textwidth]{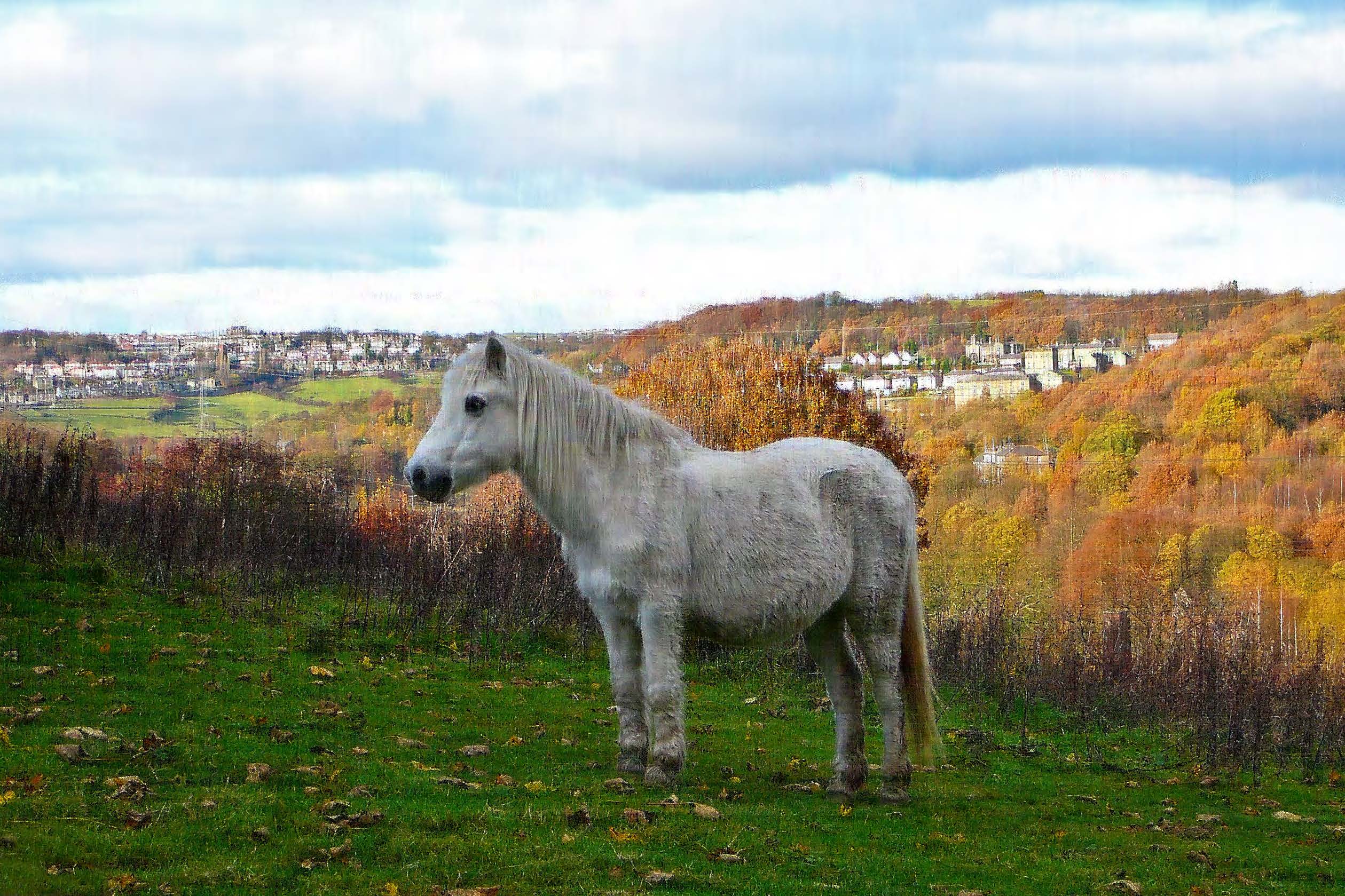}
	}
	\caption{\label{fig_big_mc_show} Comparison of the  scene images recovered by different algorithms.}
\end{figure}

\subsection{Results of Conjugate Gradient Subproblem Solver}
\label{exp_cg}

\begin{figure}[h]
	\centering
	\subfloat[Synthetic dataset P1]{%\label{subfig_P1_b}
		\includegraphics[width=0.47\textwidth]{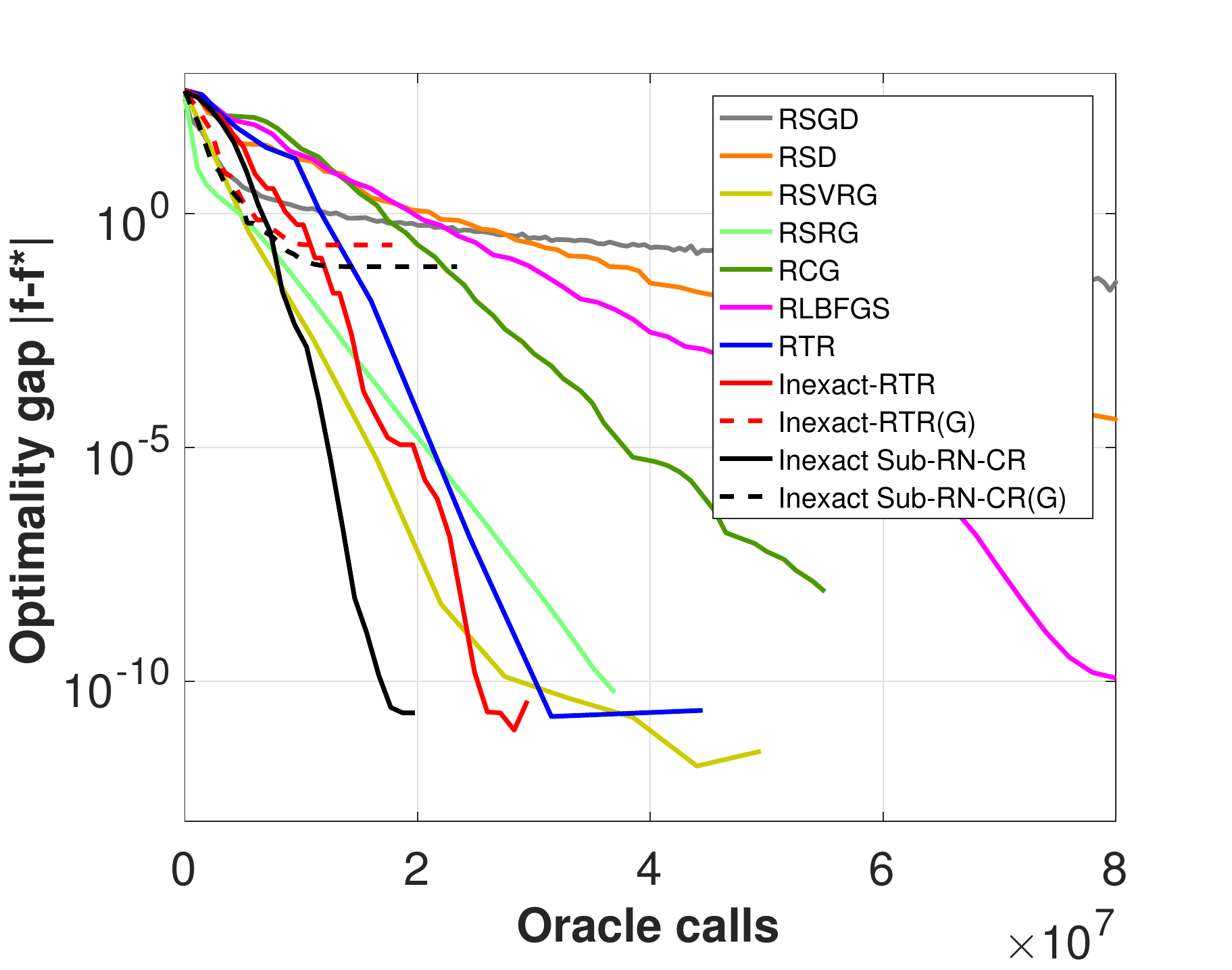}
		\includegraphics[width=0.47\textwidth]{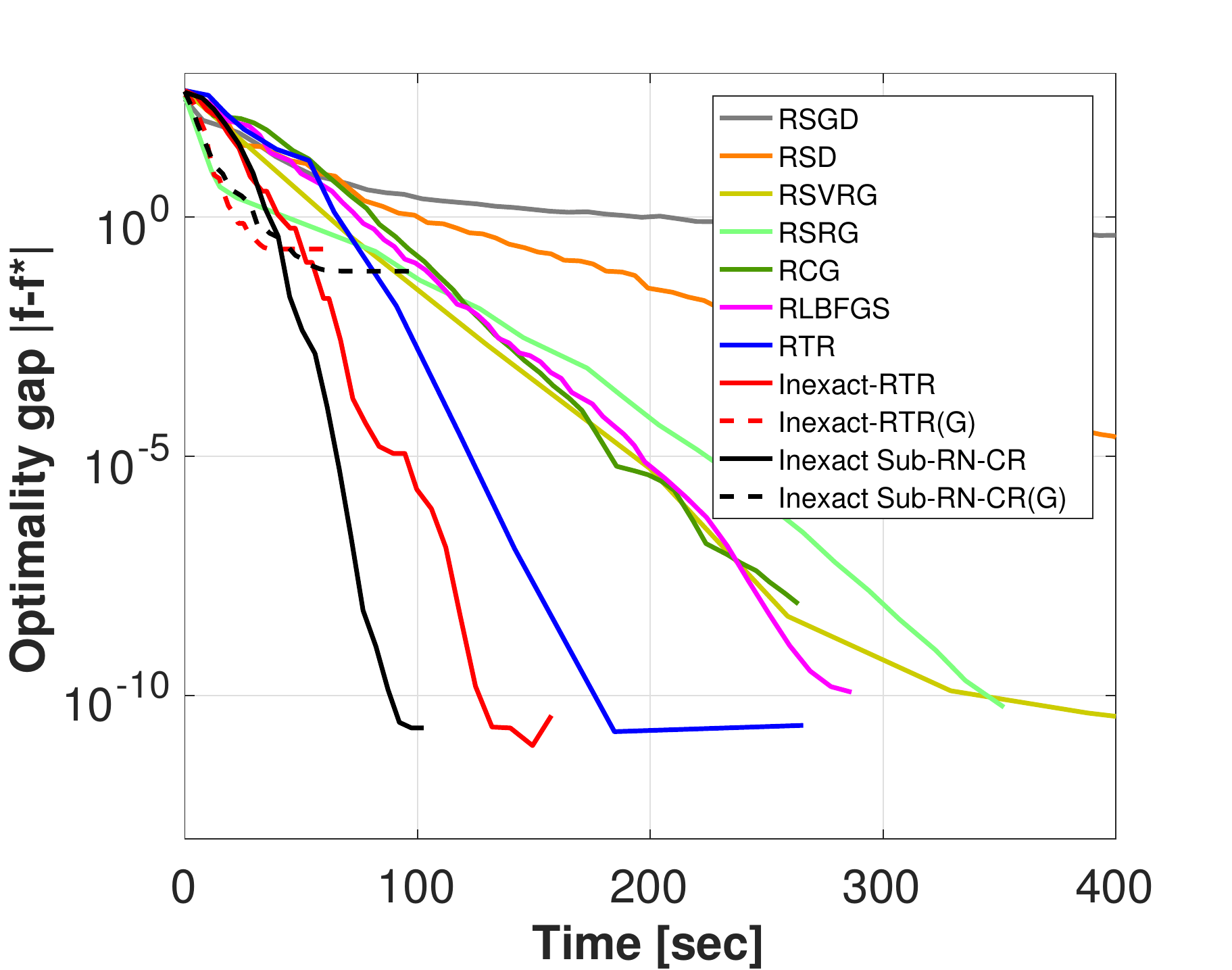}
	}
	\\
	\subfloat[ MNIST dataset]{%\label{subfig_P1_b}
		\includegraphics[width=0.47\textwidth]{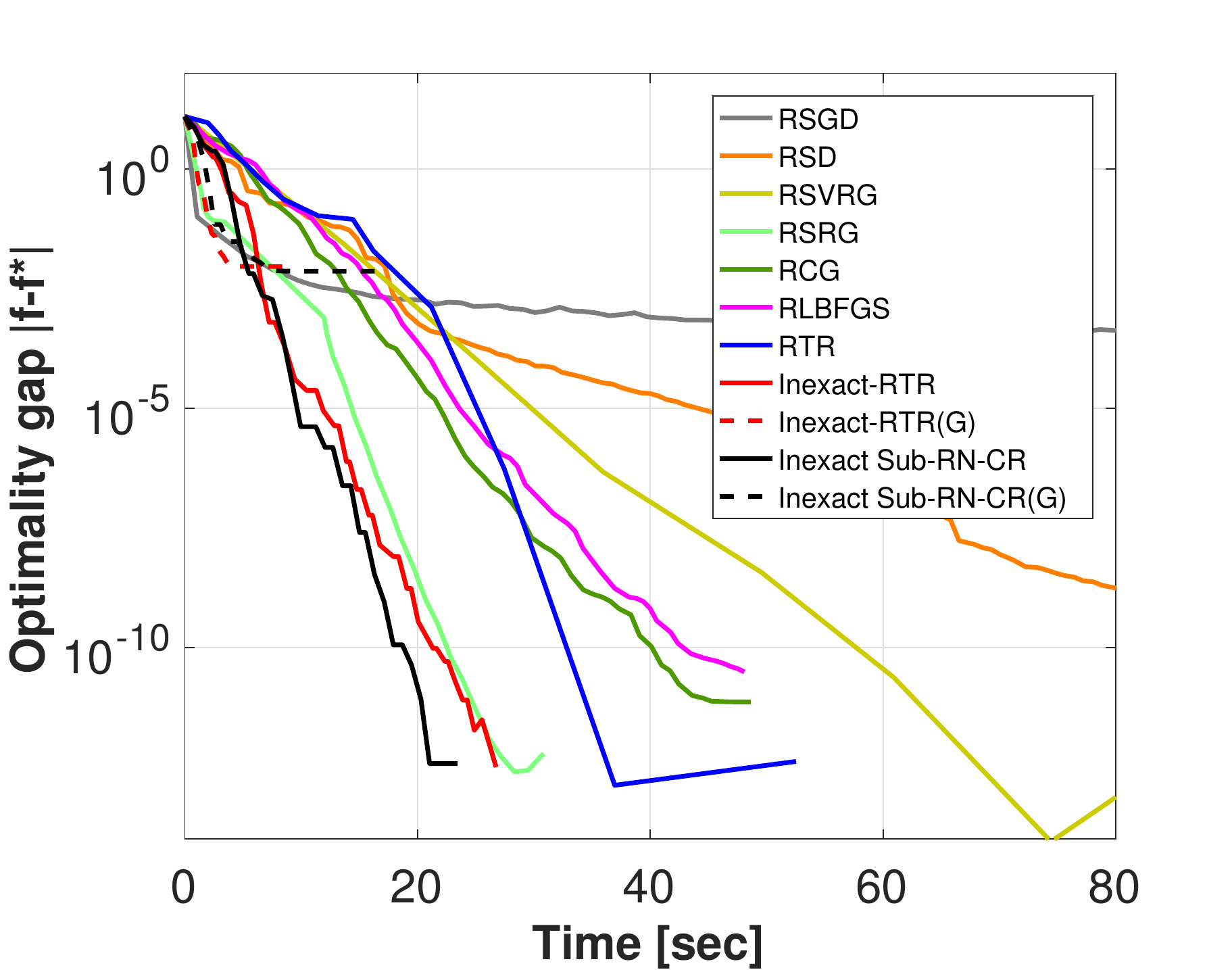}
	}
	\subfloat[Covertype dataset]{%\label{subfig_P1_b}
		\includegraphics[width=0.47\textwidth]{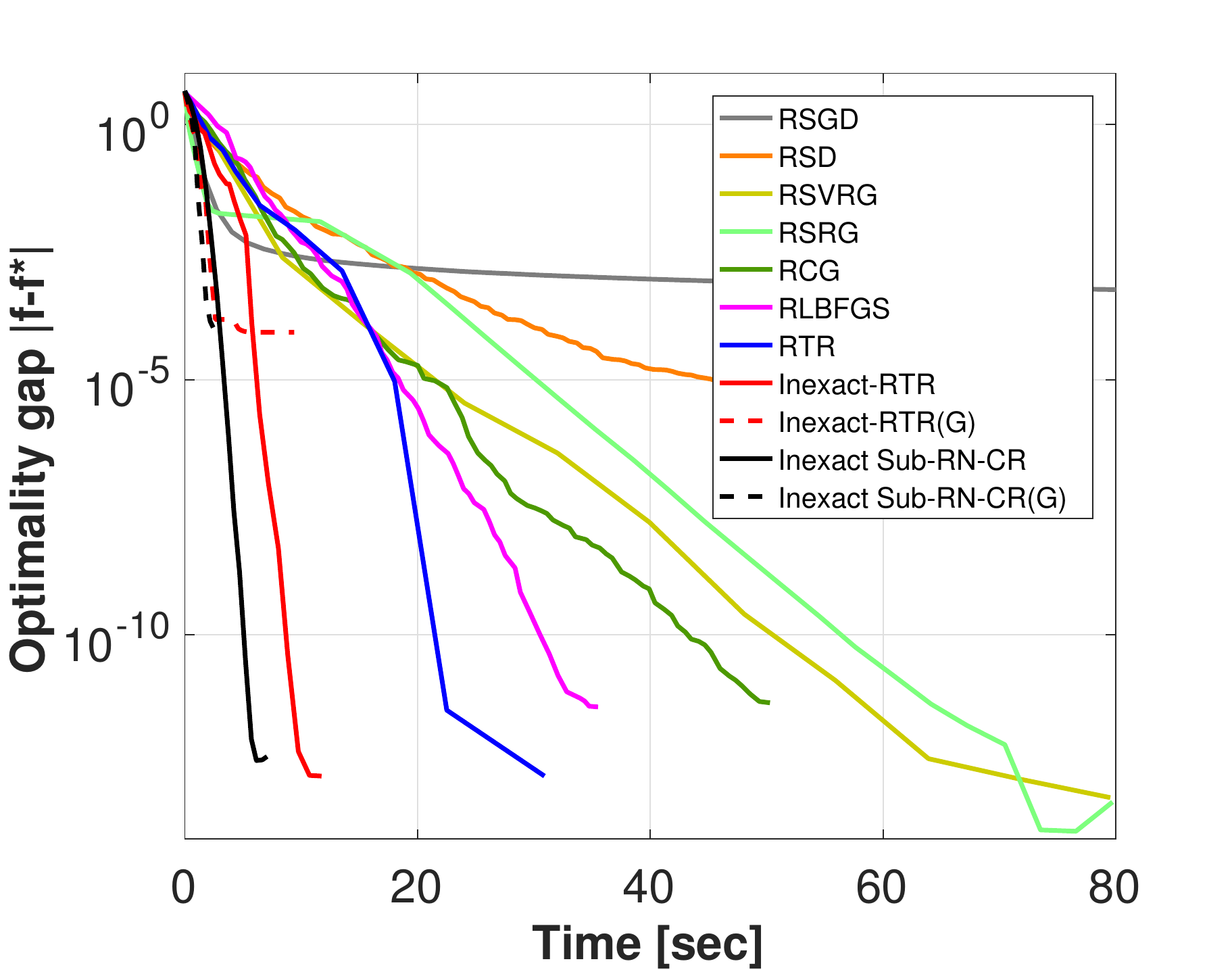}
	}
	\caption{\label{fig_PCA_all_CG} Performance comparison by optimality gap for the PCA task (using the CG solver in Inexact Sub-RN-CR).}
\end{figure}

\begin{figure}[!t]
	\centering
	\subfloat[ Synthetic  Dataset M1]{
		\includegraphics[width=0.5\textwidth]{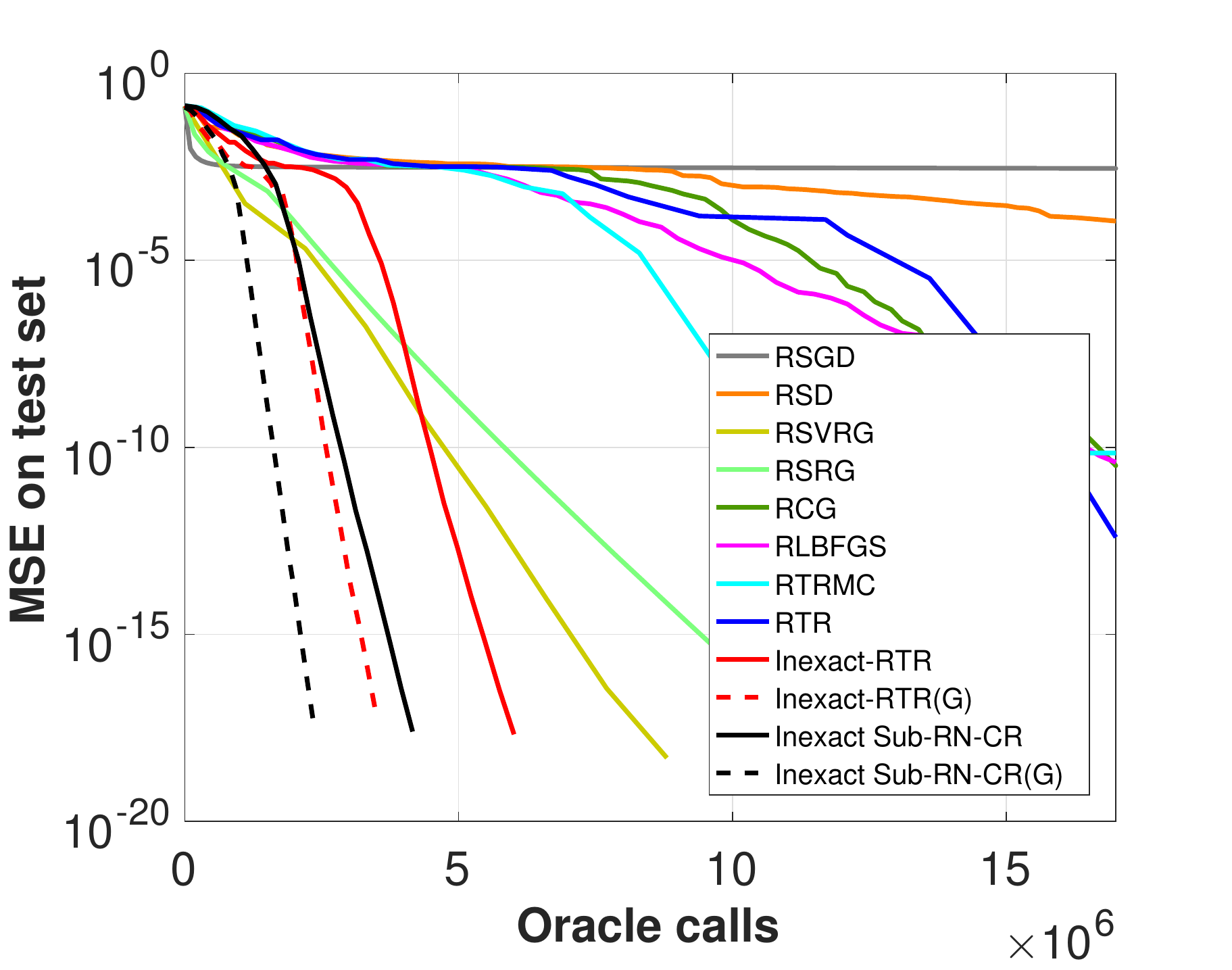}
		\includegraphics[width=0.5\textwidth]{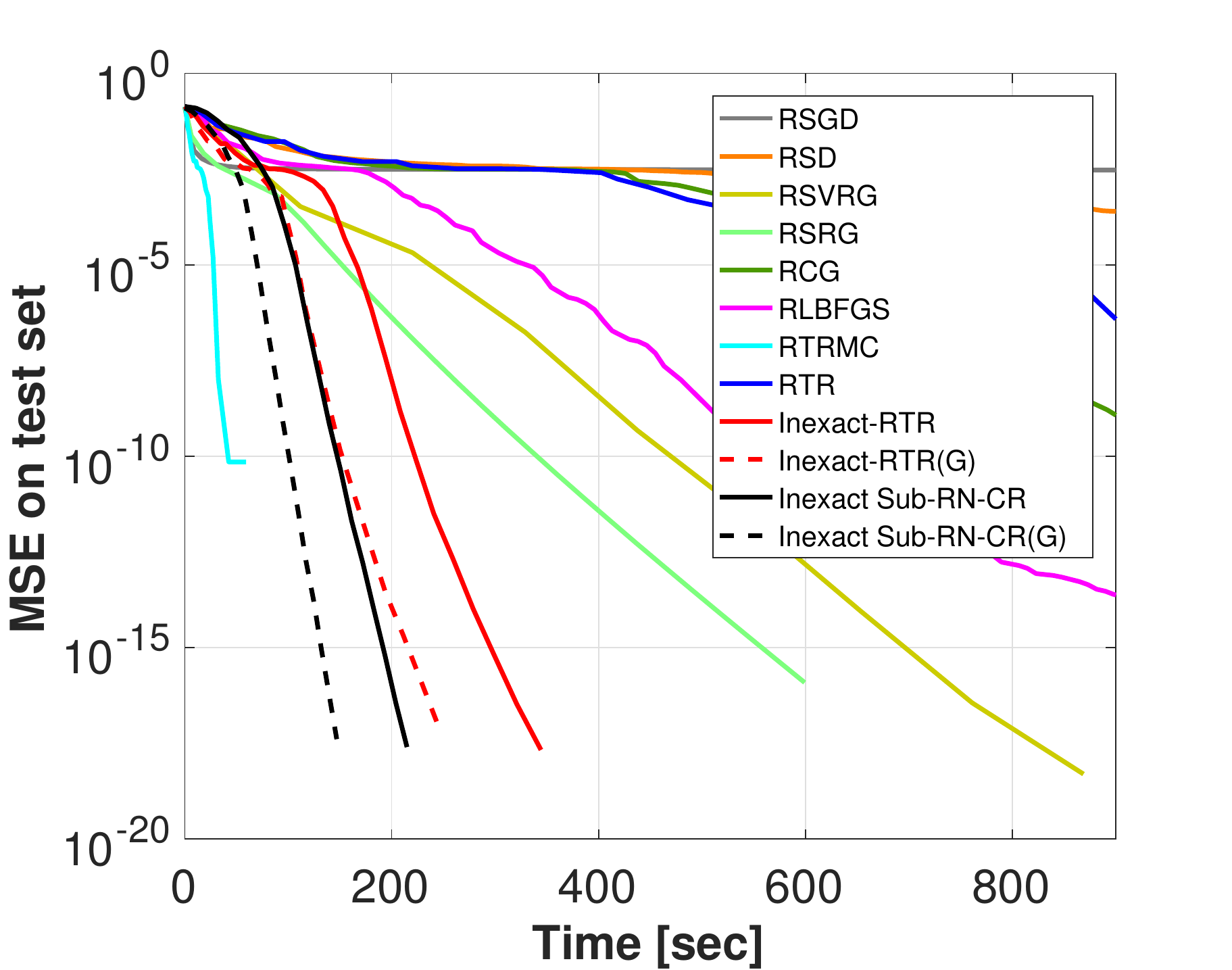}
	}
	\\
	\subfloat[ Synthetic Dataset M2]{
		\includegraphics[width=0.5\textwidth]{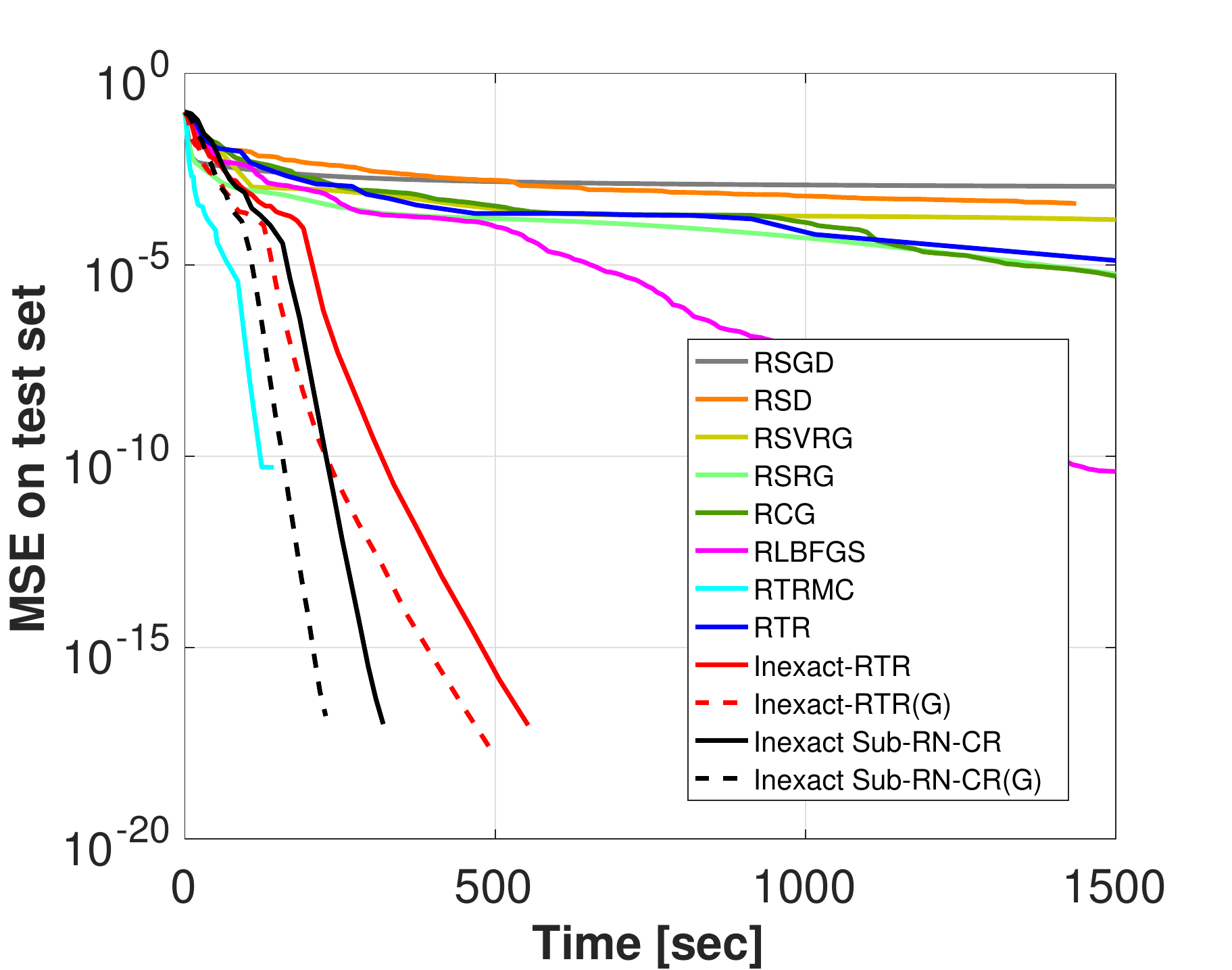}
	}
	%\\
	\subfloat[ Synthetic Dataset M3]{
		\includegraphics[width=0.5\textwidth]{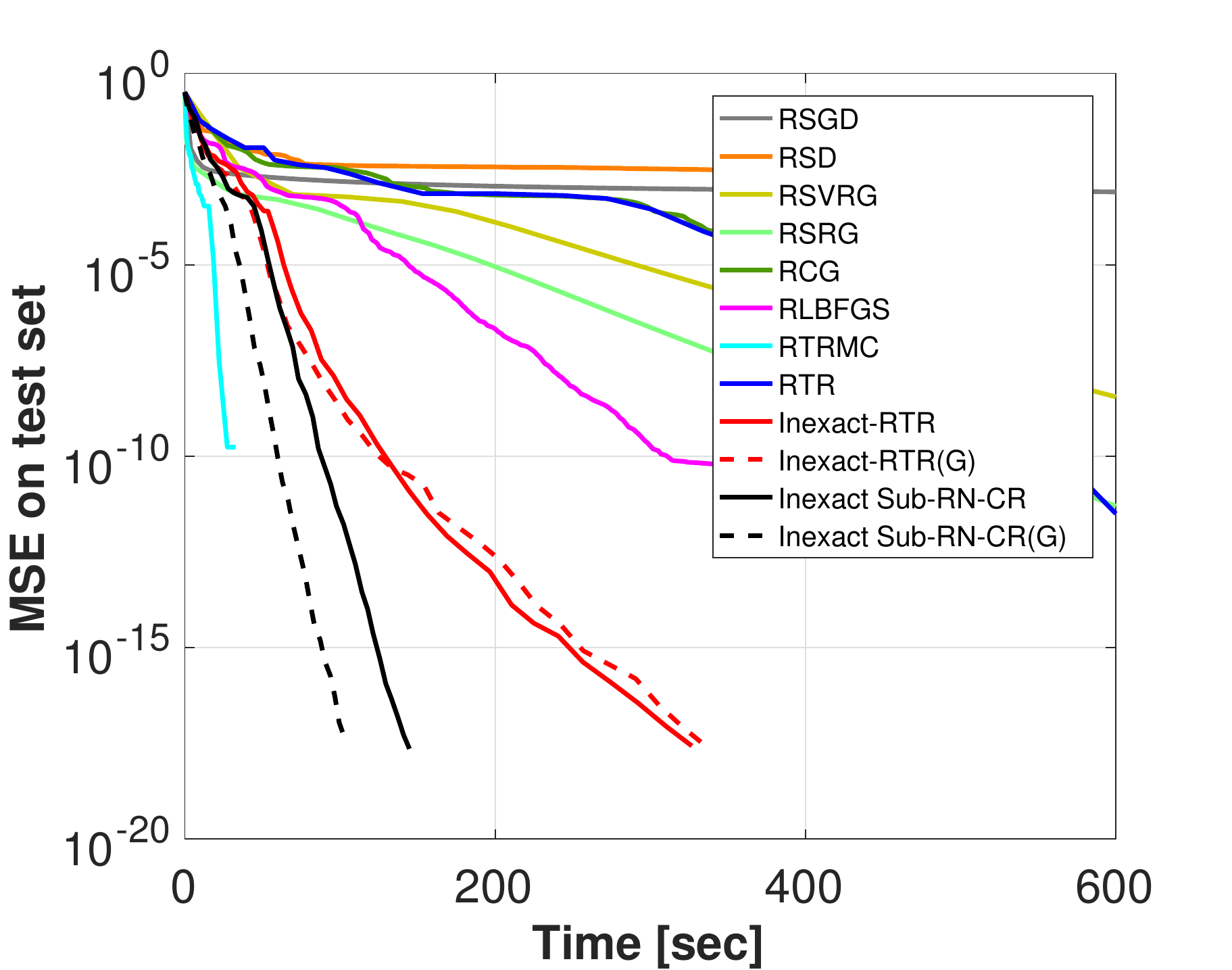}
	}
	\\
	\subfloat[ Jester Dataset ]{
		\includegraphics[width=0.5\textwidth]{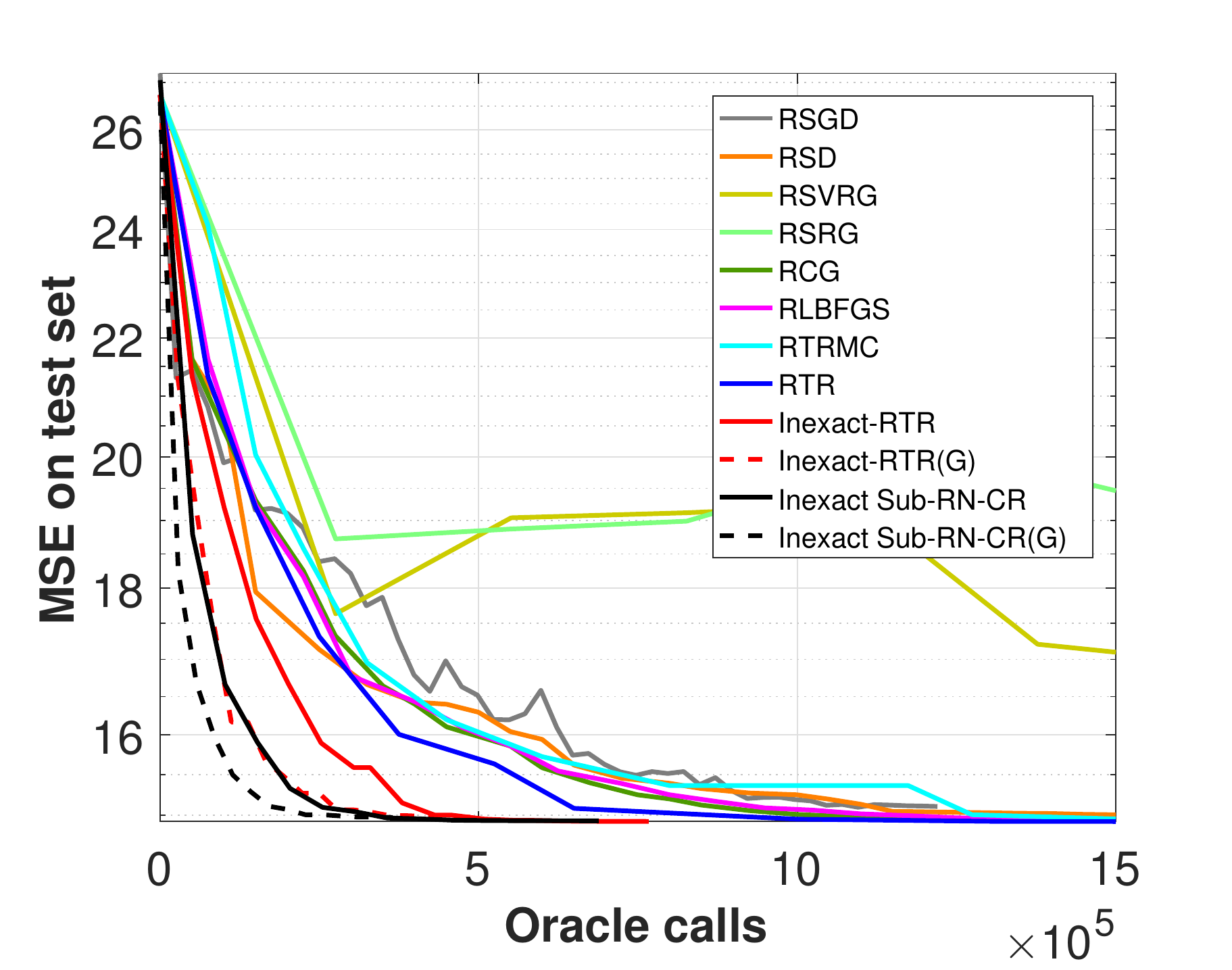}
		\includegraphics[width=0.5\textwidth]{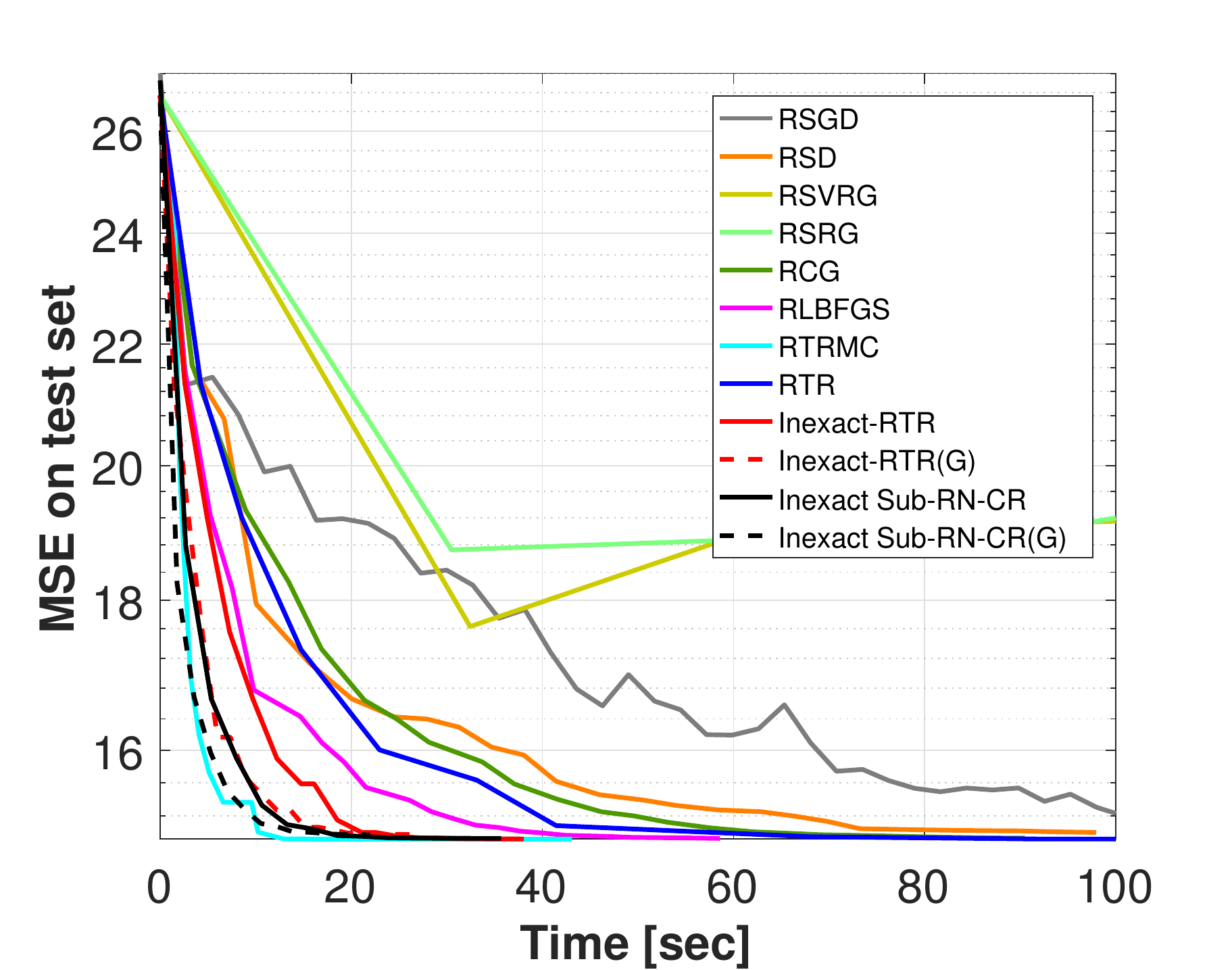}
	}
	\caption{\label{fig_MC_all_CG} Performance comparison by MSE for the matrix completion task (using the CG solver in Inexact Sub-RN-CR).}
\end{figure}

\begin{figure}[!t]
	\centering
	\includegraphics[width=0.47\textwidth]{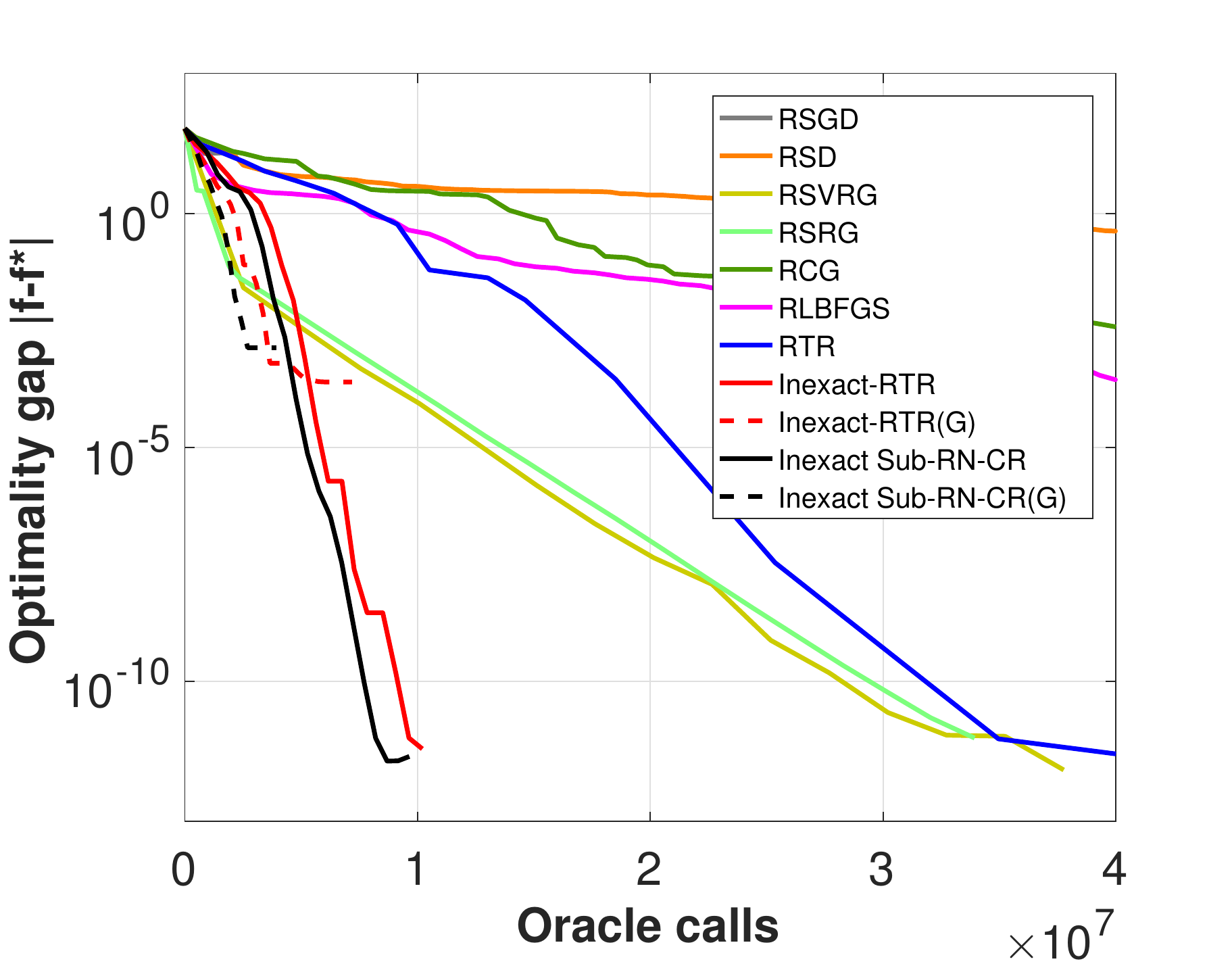}
	\includegraphics[width=0.47\textwidth]{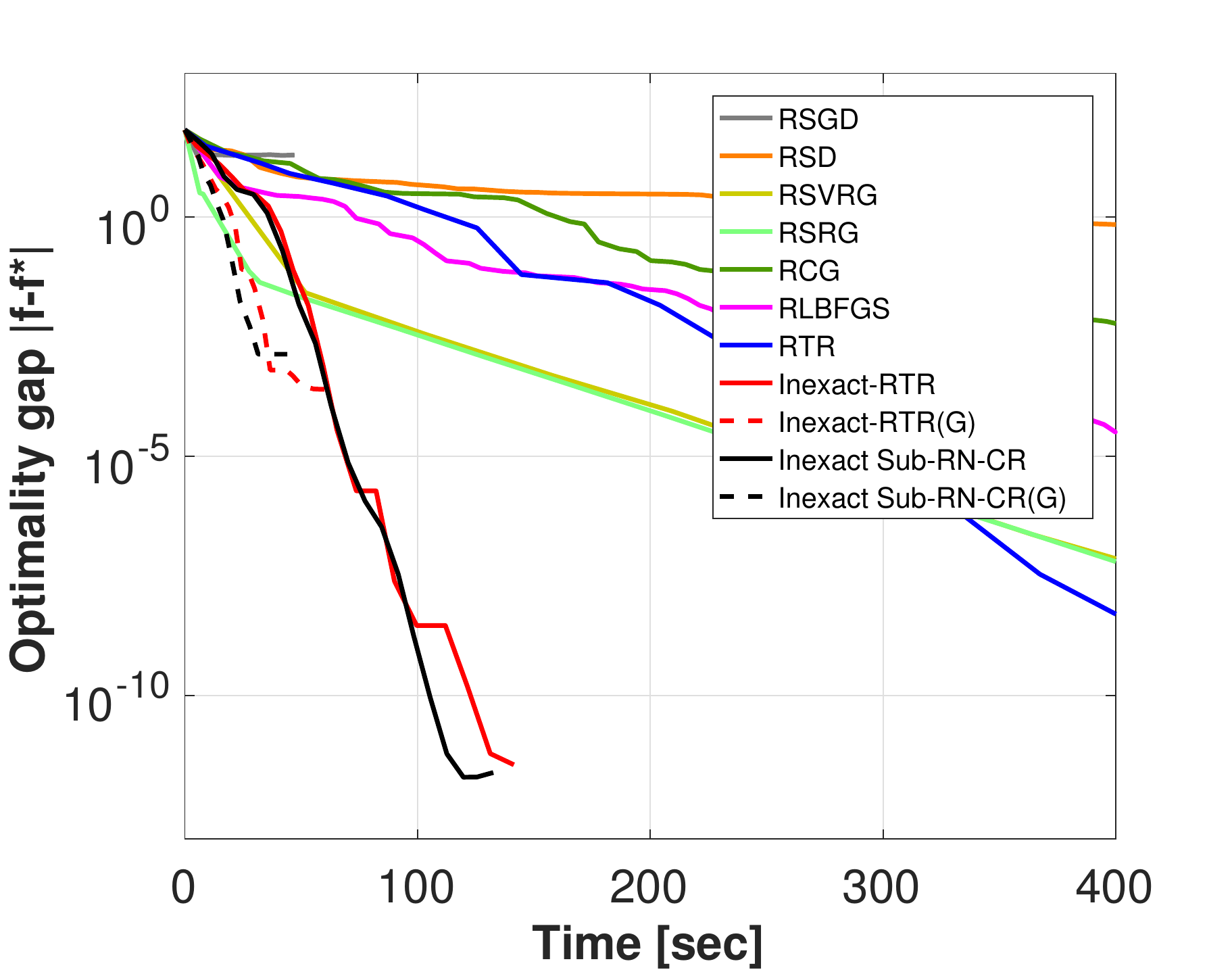}
	\caption{\label{fig_fmri_all_CG}  PCA optimality gap comparison for fMRI analysis (using the CG solver in Inexact Sub-RN-CR).}
\end{figure}

\begin{figure}[!t]
	\centering
	\includegraphics[width=0.47\textwidth]{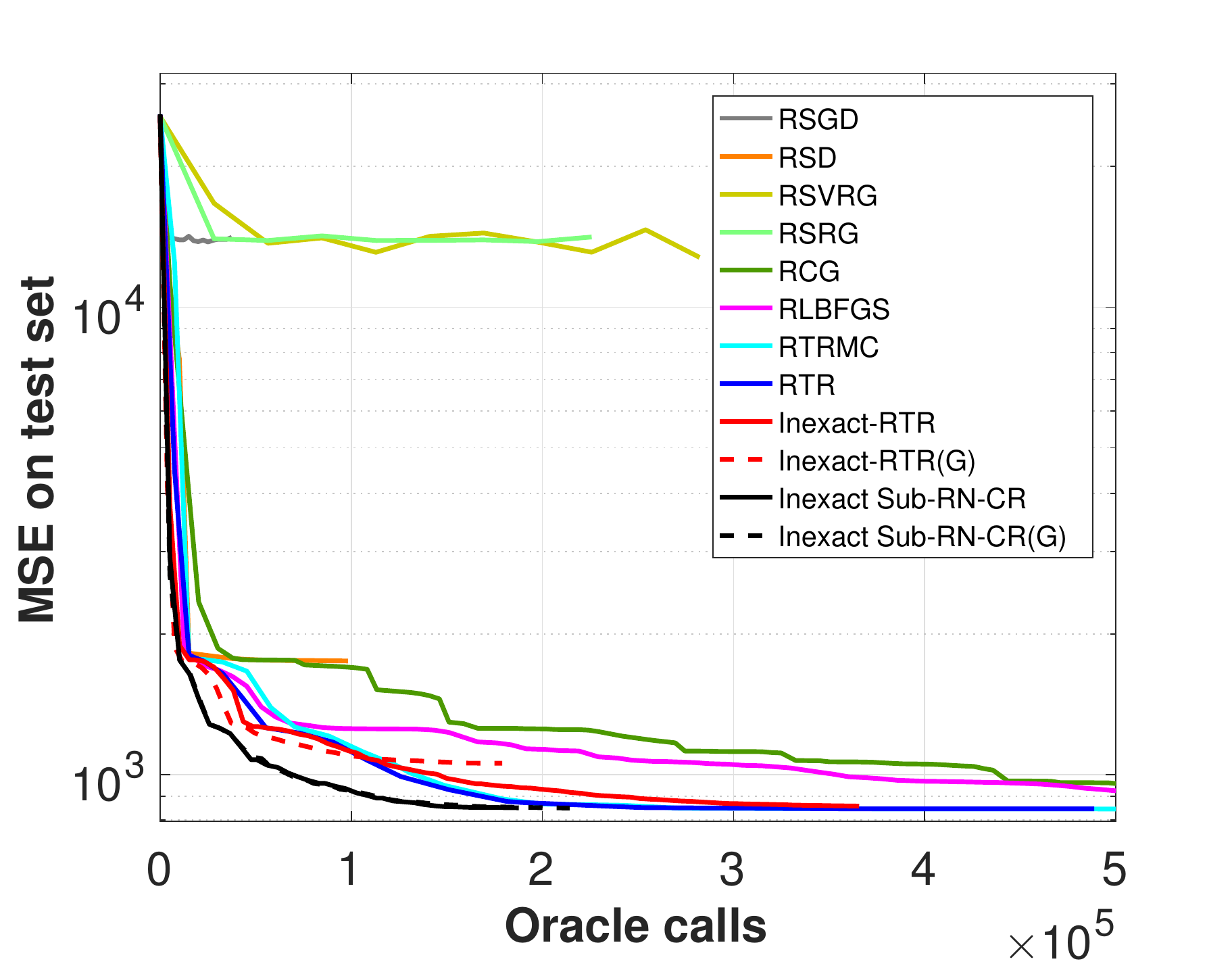}
	\includegraphics[width=0.47\textwidth]{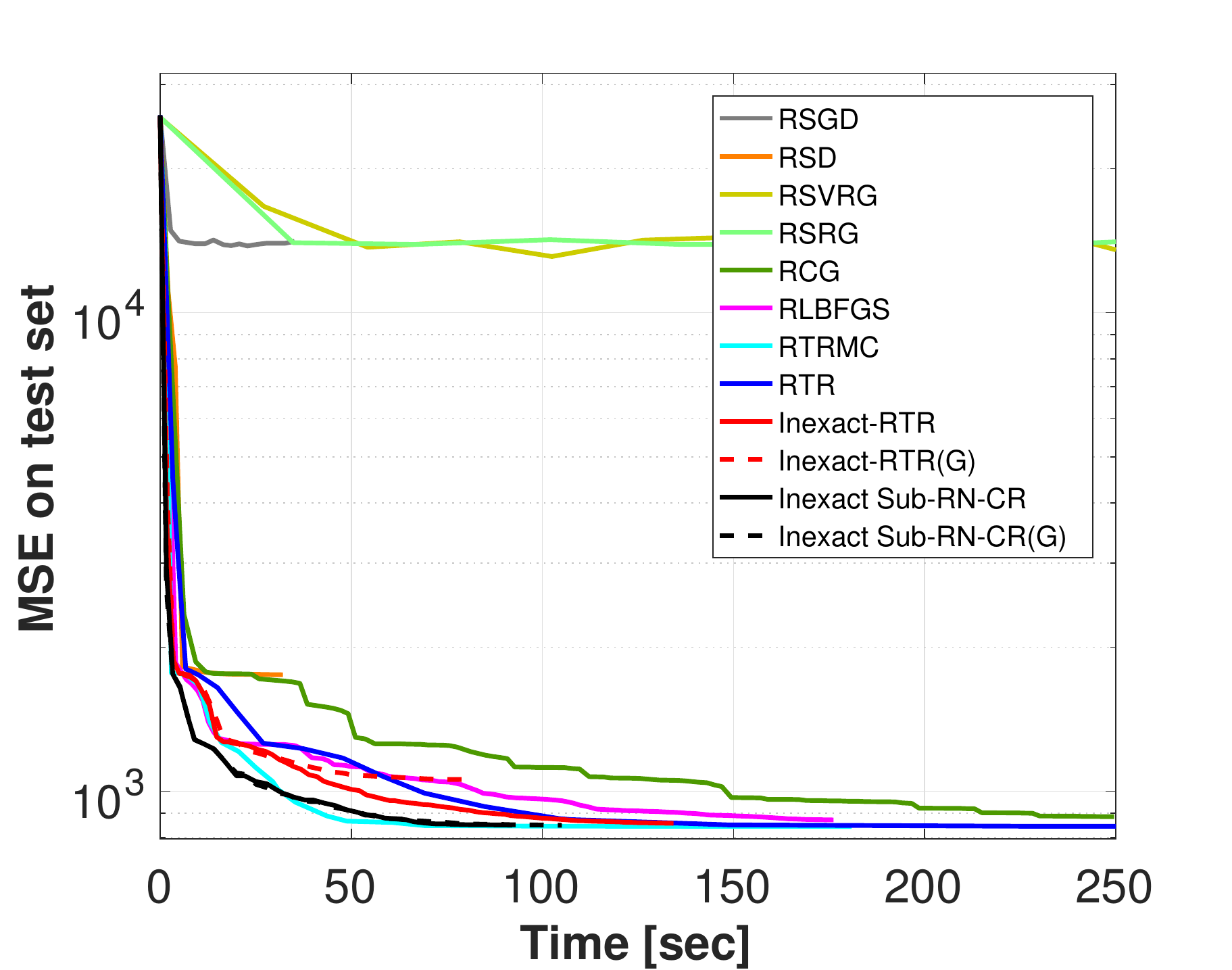}
	\caption{\label{fig_big_all_CG} MSE comparison for scene image recovery by matrix completion (using the CG solver in Inexact Sub-RN-CR).}
\end{figure}

We experiment with Algorithm \ref{alg_non_linear_tCG} for solving the subproblem.  
In Step (\ref{CG_opt_alpha}) of Algorithm \ref{alg_non_linear_tCG}, the eigen-decomposition method \cite{edelman1995polynomial}  used to solve the minimization problem has a complexity $\mathcal{O}(C^3)$ where $C=4$ is the fixed degree of the polynomial.
Figs.  \ref{fig_PCA_all_CG}-\ref{fig_big_all_CG} display the results for both the PCA and matrix completion tasks.
Overall, Algorithm \ref{alg_non_linear_tCG} can obtain the optimal results with the fastest convergence speed, as compared to the opponent approaches.
We have observed that, in general, Algorithm \ref{alg_non_linear_tCG}  provides similar results to Algorithm \ref{alg_non_linear_Lan}, but they differ  in run time. 
For instance,  Algorithm \ref{alg_non_linear_Lan} runs $18\%$ faster for the PCA task with the MNIST dataset and $20\%$ faster for the matrix completion task with the M1 dataset, as compared to Algorithm \ref{alg_non_linear_tCG}. 
But Algorithm \ref{alg_non_linear_tCG} runs $17\%$ faster than Algorithm \ref{alg_non_linear_Lan} for the matrix completion task with the M2 dataset. 
A hypothesis  could be that Algorithm \ref{alg_non_linear_Lan} performs well on well-conditioned data (e.g. MNIST and M1) because of its strength of finding the global solution, while for  ill-conditioned data (e.g. M2), it may not show significant advantages over  Algorithm \ref{alg_non_linear_tCG}. 
Moreover, from the computational aspect,  the Step (\ref{CG_opt_alpha}) in Algorithm \ref{alg_non_linear_tCG} is of $\mathcal{O}(C^3)$ complexity, which tends to be faster than solving Eq. (\ref{eq_lanczos_obj}) as required by Algorithm \ref{alg_non_linear_Lan}. 
Overall this makes Algorithm \ref{alg_non_linear_tCG} probably a better choice  than Algorithm \ref{alg_non_linear_Lan} for processing ill-conditioned data.

\subsection{Examination of Convergence Analysis Assumptions} \label{sec_assum_satisfaction}

As explained in Section \ref{sec:sub_solution} and Section \ref{sec:early}, the eigenstep condition in Assumption \ref{assu_cauchy_eigen_point}, Assumption \ref{assu_g} and Assumption \ref{assu_cg}, although are required by convergence analysis,  are not always satisfied by a subproblem solver in practice.
In this experiment, we attempt to estimate the probability $P$ that an assumption is satisfied in the process of optimization, by counting the number of outer iterations of Algorithm \ref{alg_inexact_rtr_arc} where an assumption holds. 
We repeat this entire process five times ($T=5$) to attain a stable result. 
Let $N_i$  be the number of outer iterations where the assumption is satisfied, and $M_i$ the total number of outer iterations, in the $i$-th repetition ($i=1,\;2,\;\ldots,5$). 
We compute the probability   by $P=\frac{\sum_{i\in[T]}N_i}{\sum_{i\in[T]}M_i}$.
Experiments are conducted for the PCA task using the P1 dataset.

In order to examine  Assumption \ref{assu_g}, which is  the stopping criterion in Eq. (\ref{eq_stepsize_stop_additional}), we temporarily deactivate the other stopping criteria.
We observe that Algorithm \ref{alg_non_linear_Lan} can always produce a solution that satisfies  Assumption \ref{assu_g}. 
However,  Algorithm \ref{alg_non_linear_tCG} has only  $P\approx50\%$ chance to produce a solution satisfying Assumption \ref{assu_g}.
The reason is probably that when computing $\mathbf{r}_i$ in Step (\ref{CG_update_r}) of Algorithm \ref{alg_non_linear_tCG},  the first-order approximation of $\nabla f(R_{\mathbf{x}_k}(\bm\eta_k^*))$ is used rather than the exact $\nabla f(R_{\mathbf{x}_k}(\bm\eta_k^*))$ for the sake of computational efficiency. 
This can result in an approximation error.

Regarding the eigenstep condition in Assumption \ref{assu_cauchy_eigen_point},  it can always be met by Algorithm \ref{alg_non_linear_Lan}  with $P\approx 100\%$. 
This indicates that even a few inner iterations are sufficient for it to find a solution pointing in  the direction of negative curvature.
However, Algorithm \ref{alg_non_linear_tCG} has    a $P\approx 70\%$ chance to meet the eigenstep condition. This might be caused by insufficient inner iterations according to Theorem \ref{thm_tCG_converge}. 
Moreover, the solution obtained by Algorithm \ref{alg_non_linear_tCG}  is only guaranteed to be stationary according to Theorem \ref{thm_tCG_converge}, rather than pointing in the direction of the negative curvature. This could  be a second cause for  Algorithm \ref{alg_non_linear_tCG} not to meet the eigenstep condition in Eq. (\ref{eq_eigen_condition}).

While about Assumption \ref{assu_cg}, according to Lemma \ref{lem_lanczos_sol_property}, Algorithm \ref{alg_non_linear_Lan} always satisfies it. This is verified by our results with $P=100\%$.
Algorithm \ref{alg_non_linear_tCG} has a $P\approx 80\%$ chance to meet  Assumption \ref{assu_cg} empirically.
This empirical result matches the theoretical result indicated by Lemma \ref{lem_tcg_sol_property} where solutions from Algorithm \ref{alg_non_linear_tCG} tend to approximately satisfy Assumption \ref{assu_cg}.

\section{Conclusion}

We have proposed the Inexact Sub-RN-CR algorithm to offer an effective and fast optimization for an important class of non-convex problems whose constraint sets possess manifold structures. The algorithm improves the current state of the art in second-order Riemannian optimization by using cubic regularization and subsampled Hessian and gradient approximations. We have also provided rigorous theoretical results on its convergence, and empirically evaluated  and compared it with state-of-the-art Riemannian optimization techniques for two general machine learning tasks and multiple datasets. Both theoretical and experimental results demonstrate that the Inexact Sub-RN-CR offers improved convergence and computational costs. 
Although the proposed method is promising in solving large-sample problems,  there remains an open and interesting question of whether the proposed algorithm can be effective in training a constrained deep neural network.  This is more demanding in its required computational complexity and convergence characteristics than many other machine learning problems, and it  is more challenging to perform the Hessian approximation. Our future work will pursue this direction.

\bibliographystyle{spmpsci}      % mathematics and physical sciences
\bibliography{ref}   % name your BibTeX data base

\newpage
\appendix

\section{Appendix: Derivation of Lanczos Method}
\label{app_lan}

Instead of solving the subproblem in the tangent space $\bm\eta\in T_{\mathbf{x}_k}\mathcal{M}$ of the manifold dimension $D$, the Lanczos method solves it within a Krylov subspace $\mathcal{K}_l$, where $l$ can range from 1 to  $D$. 
This subspace is  defined by the span of the following elements: 
\begin{equation}
	\label{krylov_span}
	\mathcal{K}_l(\mathbf{H}_k, \mathbf{G}_k):=\left\{\mathbf{G}_k, \mathbf{H}_k[\mathbf{G}_k], \mathbf{H}_k^2[\mathbf{G}_k],...,\mathbf{H}_k^l[\mathbf{G}_k] \right\}, 
\end{equation}
where, for $l \geq 2 $, $\mathbf{H}_k^l[\mathbf{G}_k]$ is recursively defined by $\mathbf{H}_k\left[\mathbf{H}^{l-1}_k\left[\mathbf{G}_k\right]\right]$.  
Its orthonormal basis $\mathbf{Q}_l=\{\mathbf{q}_1,...,\mathbf{q}_l\}$, where $\mathbf{Q}_l^T\mathbf{Q}_l=\mathbf{I}$, is successively constructed to satisfy
\begin{eqnarray}
	\mathbf{q}_1 = & \frac{\mathbf{G}_k}{\left\|\mathbf{G}_k\right\|_{\mathbf{x}_k}}, \\
	\langle\mathbf{q}_i, \mathbf{H}_k[\mathbf{q}_j]\rangle_{\mathbf{x}_k}=&\left(\mathbf{T}_l\right)_{i,j},
\end{eqnarray}
for $i,j\in[n]$, where  $\left(\mathbf{T}_l\right)_{i,j}$ denotes the $ij$-th element of the matrix  $\mathbf{T}_l$.
Each element $\bm\eta\in \mathcal{K}_l$ in the Krylov subspace can be expressed as  $ \bm\eta =\sum_{i=1}^l y_i \mathbf{q}_i$. Store these $\{y_i\}_{i=1}^l$ in the vector $\mathbf{y}\in \mathbb{R}^l$,  the  subproblem objective function $\mathop{\min}_{\bm\eta\in \mathcal{K}_l} \hat{m}(\bm\eta)$ is minimized in the  Krylov subspace   instead. 
By substituting  $ \bm\eta :=\sum_{i=1}^l y_i \mathbf{q}_i$ into $\hat{m}(\bm\eta)$, the objective function becomes
\begin{align}
	\nonumber
	\hat{m}(\bm\eta)&=  f(\mathbf{x}_k)+ \delta \left \langle\mathbf{G}_k, \sum_{i=1}^l y_i \mathbf{q}_i \right\rangle_{\mathbf{x}_k}+ \frac{1}{2} \left\langle\sum_{i=1}^l y_i \mathbf{q}_i,\mathbf{H}_k \left[\sum_{i=1}^l y_i \mathbf{q}_i \right] \right\rangle_{\mathbf{x}_k} +\frac{\sigma_k}{3} \left \|\sum_{i=1}^l y_i \mathbf{q}_i \right\|_{\mathbf{x}_k}^3 \\
	\nonumber
	&=f(\mathbf{x}_k)+ \delta \left\langle\mathbf{G}_k, y_1 \mathbf{q}_1 \right\rangle_{\mathbf{x}_k}+ \frac{1}{2}\sum_{i,j=1}^l y_iy_j \left\langle\mathbf{q}_i, \mathbf{H}_k[\mathbf{q}_j] \right\rangle_{\mathbf{x}_k} + \frac{\sigma_k}{3}\left\|\mathbf{y}\right\|_2^3\\
	& =f(\mathbf{x}_k)+ \ y_1\delta\left\|\mathbf{G}_k\right\|_{\mathbf{x}_k}+ \frac{1}{2}\mathbf{y}^T\mathbf{T}_D \mathbf{y} + \frac{\sigma_k}{3}\left\|\mathbf{y}\right\|_2^3.
\end{align}
The properties $\mathbf{q}_i\perp \mathbf{q}_j$ for $i\ne j$ and $\mathbf{G}_k\perp\mathbf{q}_i$ for $i\ne 1$ are used in the derivation.  Therefore, to solve $\min_{\bm\eta\in T_{\mathbf{x}_k} \mathcal{M}}\ \hat{m}(\bm\eta) $ is equivalent to  
\begin{equation}
	\min_{\mathbf{y}\in \mathbb{R}^l} \ y_1\delta\left\|\mathbf{G}_k\right\|_{\mathbf{x}_k}+ \frac{1}{2}\mathbf{y}^T\mathbf{T}_D \mathbf{y} + \frac{\sigma_k}{3}\left\|\mathbf{y}\right\|_2^3.
\end{equation}

\section{Appendix:  Proof of Lemma \ref{lem_krylov}}
\label{app_lan_gap}

\begin{proof} 
	
	Let  $\lambda_* = \sigma_k\left\|\bm\eta_{k}^*\right\|$.
	The Krylov subspaces are invariant to shifts by scalar matrices, therefore $\mathcal{K}_l(\mathbf{H}_k, \mathbf{G}_k)= \mathcal{K}_l(\mathbf{H}_k + \lambda_*{\rm Id}, \mathbf{G}_k) = \mathcal{K}_l(\tilde{\mathbf{H}}_k, \mathbf{G}_k)$ \cite{carmon2018analysis}, 
	where the definition of  $\mathcal{K}_l(\mathbf{H}_k, \mathbf{G}_k)$ follows Eq. (\ref{krylov_span}).
	Let $\bm\xi_l\in\mathcal{K}_l$ be the solution found in the  Krylov subspace $\mathcal{K}_l(\mathbf{H}_k, \mathbf{G}_k)$, which is  thus an element of $\mathcal{K}_l(\tilde{\mathbf{H}}_k, \mathbf{G}_k)$ expressed by  \begin{equation}
		\bm\xi_l=-p_l\left(\tilde{\mathbf{H}}_k\right)[\mathbf{G}_k]=-c_0\mathbf{G}_k-c_1\tilde{\mathbf{H}}_k[\mathbf{G}_k]\cdots- c_l\tilde{\mathbf{H}}_k^l[\mathbf{G}_k],
	\end{equation}
	for some values of $c_0,\;c_1, \;\ldots,\; c_l \in \mathbb{R}$.
	According to Section 6.2 of \cite{absil2009optimization}, a global minimizer $\bar{\bm\eta}_k^*$ of  the  RTR subproblem without cubic regularization in Eq. (\ref{eq_sub_problem_inexact})  is expected to satisfy the Riemannian quasi-Newton equation:
	\begin{equation}
		{\rm grad} \bar{m}_k({\mathbf{0}_k}) + \left({\rm Hess} \bar{m}_k(\mathbf{0}_k) + \lambda_*{\rm Id} \right)[\bar{\bm\eta}_k^*] = \mathbf{0}_{\mathbf{x}_k},
	\end{equation}
	where  $\lambda_*\ge  \max(-\lambda_{min}(\mathbf{H}_k), 0) $ and  $\lambda_*\left(\Delta_k-\left\|\bar{\bm\eta}_k^*\right\|_{\mathbf{x}_k}\right)=0$ according to Corollary 7.2.2 of \cite{conn2000trust}.
	Using the approximate gradient and Hessian, the inexact minimizer is expected to satisfy 
	\begin{equation}
		\mathbf{G}_k + (\mathbf{H}_k + \lambda_*{\rm Id})[\bar{\bm\eta}_k^*] = \mathbf{G}_k + \tilde{\mathbf{H}}_k [\bar{\bm\eta}_k^*]  = \mathbf{0}_{\mathbf{x}_k}.
		\label{eq_newton}
	\end{equation}
	
	Introduce $\bm\zeta_l=(1-\alpha)\bm\xi_l$ where $\alpha=\frac{\left\|\bm\xi_l\right\|_{\mathbf{x}_k}-\left\|\bar{\bm\eta}_k^*\right\|_{\mathbf{x}_k}}{\max\left(\left\|\bm\xi_l\right\|_{\mathbf{x}_k},\left\|\bar{\bm\eta}_k^*\right\|_{\mathbf{x}_k}\right)}$. 
	When $\left\|\bm\xi_l\right\|_{\mathbf{x}_k}<\left\|\bar{\bm\eta}_k^*\right\|_{\mathbf{x}_k}$, we start from the fact that $\left(\left\|\bar{\bm\eta}_k^*\right\|_{\mathbf{x}_k}-\left\|\bm\xi_l\right\|_{\mathbf{x}_k}\right)^2\ge0$, which results in the following:
	\begin{align}
		\nonumber
		\left(2\left\|\bar{\bm\eta}_k^*\right\|_{\mathbf{x}_k} -\left\|\bm\xi_l\right\|_{\mathbf{x}_k}\right) \left\|\bm\xi_l\right\|_{\mathbf{x}_k} \le \left\|\bar{\bm\eta}_k^*\right\|_{\mathbf{x}_k}^2  & \Longleftrightarrow \left(1 +\frac{\left\|\bar{\bm\eta}_k^*\right\|_{\mathbf{x}_k}-\left\|\bm\xi_l\right\|_{\mathbf{x}_k}}{\left\|\bar{\bm\eta}_k^*\right\|_{\mathbf{x}_k}}\right)\left\|\bm\xi_l\right\|_{\mathbf{x}_k}\le\left\|\bar{\bm\eta}_k^*\right\|_{\mathbf{x}_k}\\
		\nonumber
		&\Longleftrightarrow(1-\alpha)\left\|\bm\xi_l\right\|_{\mathbf{x}_k}\le\left\|\bar{\bm\eta}_k^*\right\|_{\mathbf{x}_k}\\
		&\Longleftrightarrow\left\|\bm\zeta_l\right\|_{\mathbf{x}_k}\le\left\|\bar{\bm\eta}_k^*\right\|_{\mathbf{x}_k}.
	\end{align}
	When $\left\|\bm\xi_l\right\|_{\mathbf{x}_k}\ge\left\|\bar{\bm\eta}_k^*\right\|_{\mathbf{x}_k}
	$, it has 
	\begin{equation}
		\left\|{\bm\zeta}_l\right\|_{\mathbf{x}_k}= \left\|(1-\alpha)\bm\xi_l\right\|_{\mathbf{x}_k}= \left\|\left(1- \frac{\left\|\bm\xi_l\right\|_{\mathbf{x}_k}-\left\|\bar{\bm\eta}_k^*\right\|_{\mathbf{x}_k}}{ \left\|\bm\xi_l\right\|_{\mathbf{x}_k}  }\right)\bm\xi_l\right\|_{\mathbf{x}_k} =\frac{\left\|\bar{\bm\eta}_k^*\right\|_{\mathbf{x}_k} \left\|{\bm\xi}_l\right\|_{\mathbf{x}_k}}{\left\|{\bm\xi}_l\right\|_{\mathbf{x}_k}}=\left\|\bar{\bm\eta}_k^*\right\|_{\mathbf{x}_k}.
	\end{equation}
	This concludes that for any $\bm\xi_l$,  it has  $\left\|\bm\zeta_l\right\|_{\mathbf{x}_k}\le\left\|\bar{\bm\eta}_k^*\right\|_{\mathbf{x}_k}$.

	We introduce the notation $\tilde{m}_k$ to denote the subproblem  in Eq. (\ref{eq_sub_problem_inexact}) using the inexact Hessian $\tilde{\mathbf{H}}_k$. 
	Let $\psi_{k}^*= \left(p_{l+1}\left(\tilde{\mathbf{H}}_k\right)-{\rm Id}\right)[\bar{\bm\eta}_k^*]$ and $\iota\left(\tilde{\mathbf{H}}_k\right)=\frac{\lambda_{max}(\tilde{\mathbf{H}}_k)}{\lambda_{min}(\tilde{\mathbf{H}}_k)}$. 
	Since $\phi_l\left(\tilde{\mathbf{H}}_k\right)$ is the upper bound of $\left\|p_{l+1}\left(\tilde{\mathbf{H}}_k\right)-{\rm Id}\right\|_{\mathbf{x}_k}$, it has $\left\|\psi_{k}^*\right\|_{\mathbf{x}_k} \leq \phi_l\left(\tilde{\mathbf{H}}_k\right)\left\|\bar{\bm\eta}_k^*\right\|_{\mathbf{x}_k}$.
	Then, we have 
	\begin{align}
		\nonumber
		&\bar{m}_k(\bm\zeta_l)-\bar{m}_k(\bar{\bm\eta}_k^*) \\
		\nonumber
		=&\; \tilde{m}_k(\bm\zeta_l)-\tilde{m}_k(\bar{\bm\eta}_k^*) + \frac{\lambda_*}{2}\left(\left\|\bar{\bm\eta}_k^*\right\|^2_{\mathbf{x}_k}-\left\|\bm\zeta_l\right\|^2_{\mathbf{x}_k}\right)	\\
		\nonumber
		= &\; \frac{1}{2} \left\langle\bm\zeta_l-\bar{\bm\eta}_k^*, \tilde{\mathbf{H}}_k[\bm\zeta_l-\bar{\bm\eta}_k^*]\right\rangle_{\mathbf{x}_k} + \frac{\lambda_*}{2}\left(\left\|\bar{\bm\eta}_k^*\right\|^2_{\mathbf{x}_k}-\left\|\bm\zeta_l\right\|^2_{\mathbf{x}_k}\right)\\
		\nonumber
		\le&\;  \frac{1}{2} \left\langle\bm\zeta_l-\bar{\bm\eta}_k^*, \tilde{\mathbf{H}}_k[\bm\zeta_l-\bar{\bm\eta}_k^*]\right\rangle_{\mathbf{x}_k} + \lambda_*\left\|\bar{\bm\eta}_k^*\right\|_{\mathbf{x}_k}\left(\left\|\bar{\bm\eta}_k^*\right\|_{\mathbf{x}_k}-\left\|\bm\zeta_l\right\|_{\mathbf{x}_k}\right)\\
		\label{eq_rtr_krylov0}
		\le &\; \frac{(1-\alpha)^2}{2} \left\langle\left(p_{l+1}\left(\tilde{\mathbf{H}}_k\right)-{\rm Id}\right)[\bar{\bm\eta}_k^*], \tilde{\mathbf{H}}_k\left[\left(p_{l+1}\left(\tilde{\mathbf{H}}_k\right)-{\rm Id}\right)\left[\bar{\bm\eta}_k^*\right]\right]\right\rangle_{\mathbf{x}_k} + \lambda_*\left\|\bar{\bm\eta}_k^*\right\|_{\mathbf{x}_k}^2\alpha^2\\	
		\nonumber
		=&\; \frac{(1-\alpha)^2}{2}\left\|\psi_{k}^*\right\|_{\mathbf{x}_k}^2\left\langle\frac{\psi_{k}^*}{\left\|\psi_{k}^*\right\|_{\mathbf{x}_k}},\tilde{\mathbf{H}}_k\left[\frac{\psi_{k}^*}{\left\|\psi_{k}^*\right\|_{\mathbf{x}_k}}\right]\right\rangle_{\mathbf{x}_k}+ \lambda_*\left\|\bar{\bm\eta}_k^*\right\|_{\mathbf{x}_k}^2\alpha^2\\
		\label{eq_rtr_krylov}
		\le&\; 	 2\phi_l\left(\tilde{\mathbf{H}}_k\right)^2\left\|\bar{\bm\eta}_k^*\right\|_{\mathbf{x}_k}^2 \iota\left(\tilde{\mathbf{H}}_k\right)\left\langle \frac{\bar{\bm\eta}_k^*}{\left\|\bar{\bm\eta}_k^*\right\|_{\mathbf{x}_k}},\tilde{\mathbf{H}}_k\left[\frac{\bar{\bm\eta}_k^*}{\left\|\bar{\bm\eta}_k^*\right\|_{\mathbf{x}_k}}\right]\right\rangle_{\mathbf{x}_k}  +\phi_l(\tilde{\mathbf{H}}_k)^2\lambda_*\left\|\bar{\bm\eta}_k^*\right\|_{\mathbf{x}_k}^2\\
		\le &\;  4\iota\left(\tilde{\mathbf{H}}_k\right)\left(\frac{1}{2} \left\langle  \bar{\bm\eta}_k^*,\tilde{\mathbf{H}}_k[\bar{\bm\eta}_k^*] \right\rangle_{\mathbf{x}_k}+\frac{1}{2}\lambda_*\left\|\bar{\bm\eta}_k^*\right\|_{\mathbf{x}_k}^2\right)\phi_l\left(\tilde{\mathbf{H}}_k\right)^2\\
		\label{eq_rtr_krylov_2}
		= &\; 4\iota\left(\tilde{\mathbf{H}}_k\right)\left(\bar{m}_k(\mathbf{0}_{\mathbf{x}_k})-\bar{m}_k\left(\bar{\bm\eta}_k^*\right)\right)\phi_l\left(\tilde{\mathbf{H}}_k\right)^2.
	\end{align}
	To derive Eq. (\ref{eq_rtr_krylov0}),   $\mathbf{G}_k=-\tilde{\mathbf{H}}_k[\bar{\bm\eta}_k^*]$ from Eq. (\ref{eq_newton}) is used. 
	To derive Eq. (\ref{eq_rtr_krylov}),  we use the definition in Eq. (\ref{eq_induced_hessian_norm}), where for a non-zero $\bm\xi\in T_{\mathbf{x}_k}\mathcal{M}$,  it has
	\begin{equation}
		\frac{\left\|\mathbf{H}_k[\bm\xi]\right\|_{\mathbf{x}_k}}{\left\|\bm\xi\right\|_{\mathbf{x}_k}}\le\left\|\mathbf{H}_k\right\|_{\mathbf{x}_k}=\sup_{{\mathbf{\bm\eta}\in T_{\mathbf{x}_k}\mathcal{M}, \|\mathbf{\bm\eta}\|_{\mathbf{x}_k}\ne0}} \frac{\left\|\mathbf{H}_k[\mathbf{\bm\eta}]\right\|_{\mathbf{x}_k}}{\left\|\bm\eta\right\|_{\mathbf{x}_k}}.
	\end{equation}
	Eq. (\ref{eq_rtr_krylov}) also uses (1) the fact of $\alpha\ge-1$ which comes from the fact of $\alpha$ being in the form of $\frac{a-b}{\max(a,b)}$, and thus $1-\alpha \le 2$, 
	(2) the definition of the smallest and largest eigenvalues in Eqs. (\ref{eq_minimum_eigenvalue_hessian}) and  (\ref{eq_maximum_eigenvalue_hessian}), which gives $\frac{\langle\bm\eta,\tilde{\mathbf{H}[\bm\eta]}\rangle_{\mathbf{x}_k}}{\langle\bm\xi,\tilde{\mathbf{H}[\bm\xi]}\rangle_{\mathbf{x}_k}}\le\frac{\lambda_{max}(\tilde{\mathbf{H}}_k)}{\lambda_{min}(\tilde{\mathbf{H}}_k)}$ for and any unit tangent vectors $\bm\eta $ and $\bm\xi$, and (3) the fact that
	\begin{align}
		\nonumber
		|\alpha|&=\frac{\left|\left\|\bm\xi_l\right\|_{\mathbf{x}_k}-\left\|\bar{\bm\eta}_{k}^*\right\|_{\mathbf{x}_k}\right|}{\max\left(\left\|\bm\xi_l\right\|_{\mathbf{x}_k},\left\|\bar{\bm\eta}_k^*\right\|_{\mathbf{x}_k}\right)} \le \frac{\left\|\bm\xi_l-\bar{\bm\eta}_k^*\right\|_{\mathbf{x}_k}}{\max\left(\left\|\bm\xi_l\right\|_{\mathbf{x}_k},\left\|\bar{\bm\eta}_k^*\right\|_{\mathbf{x}_k}\right)}\\
		&=\frac{\left\|\left(p_{l+1}\left(\tilde{\mathbf{H}}_k\right)-{\rm Id}\right)[\bar{\bm\eta}_k^*]\right\|_{\mathbf{x}_k}}{\max\left(\left\|\bm\xi_l\right\|_{\mathbf{x}_k},\left\|\bar{\bm\eta}_k^*\right\|_{\mathbf{x}_k}\right)}\le \frac{\phi_l(\mathbf{H}_k)\left\|\bar{\bm\eta}_k^*\right\|_{\mathbf{x}_k}}{\max\left(\left\|\bm\xi_l\right\|_{\mathbf{x}_k},\left\|\bar{\bm\eta}_k^*\right\|_{\mathbf{x}_k}\right)}\le\phi_l(\mathbf{H}_k).
	\end{align}
	To derive Eq. (\ref{eq_rtr_krylov_2}),  we use 
	\begin{equation}
		\bar{m}_k(\mathbf{0}_{\mathbf{x}_k})-\bar{m}_k(\bar{\bm\eta}_k^*)=\frac{1}{2}\left\langle \bar{\bm\eta}_k^*,\tilde{\mathbf{H}}_k[\bar{\bm\eta}_k^*]\right\rangle_{\mathbf{x}_k}+\frac{\lambda_*}{2}\left\|\bar{\bm\eta}_k^*\right\|_{\mathbf{x}_k}^2.
	\end{equation}
	Next, as $\bm\eta_k^{*l}$ is the optimal solution in the subspace $\mathcal{K}_l(\mathbf{H}_k, \mathbf{G}_k)$, we have $\bar{m}_k\left(\bm\eta_k^{*l}\right)\le\bar{m}_k\left(\bm\zeta^l\right)$, and hence 
	\begin{equation}
		\label{eq_res}
		\bar{m}_k\left(\bm\eta_k^{*l}\right)-\bar{m}_k\left(\bar{\bm\eta}_k^*\right)\le4\iota\left(\tilde{\mathbf{H}}_k\right)(\bar{m}_k(\mathbf{0}_{\mathbf{x}_k})-\bar{m}_k(\bar{\bm\eta}_k^*))\phi_l\left(\tilde{\mathbf{H}}_k\right)^2.
	\end{equation}

	We then show that the Lanczos method exhibits at least the same convergence property as above for the subsampled Riemannian cubic-regularization subproblem.  Let ${\bm\eta}_k^*$ be the global minimizer for the subproblem $\hat{m}_k$ in Eq. (\ref{eq_sub_problem}).
	${\bm\eta}_k^*$ is equivalent to $\bar{\bm\eta}_k^*$ in the RTR subproblem with $\Delta_k=\left\|{\bm\eta}_k^*\right\|_{\mathbf{x}_k}$ and $\lambda_*=\sigma_k\left\|{\bm\eta}_k^*\right\|_{\mathbf{x}_k}$ \cite{carmon2018analysis}. Then, letting $\bm\eta_k^{*l}$ be the minimizer of $\hat{m}_k$ over $\mathcal{K}_l(\mathbf{H}_k, \mathbf{G}_k)$ satisfying $\left\|{\bm\eta}_k^{*l}\right\|_{\mathbf{x}_k}\le\left\|{\bm\eta}_k^*\right\|_{\mathbf{x}_k}=\Delta_k$,  we have
	\begin{align}
		\nonumber
		\hat{m}_k({\bm\eta}_k^{*l})-\hat{m}_k({\bm\eta}_k^*)&\le \hat{m}_k({\bm\eta}_k^{*l})-\hat{m}_k({\bm\eta}_k^*)\\
		\nonumber
		&=	\bar{m}_k({\bm\eta}_k^{*l})-\bar{m}_k({\bm\eta}_k^*) + \frac{\sigma_k}{3}(\left\|{\bm\eta}_k^{*l}\right\|_{\mathbf{x}_k}^3-\left\|{\bm\eta}_k^*\right\|_{\mathbf{x}_k}^3)\\
		&\le\bar{m}_k({\bm\eta}_k^{*l})-\bar{m}_k({\bm\eta}_k^*) = \bar{m}_k({\bm\eta}_k^{*l})-\bar{m}_k(\bar{\bm\eta}_k^*)
	\end{align}
	Combining this with Eq. (\ref{eq_res}), it has
	\begin{equation}
		\hat{m}_k\left(\bm\eta_k^{*l}\right)-\hat{m}_k\left({\bm\eta}_k^*\right)\le4\iota\left(\tilde{\mathbf{H}}_k\right)(\bar{m}_k(\mathbf{0}_{\mathbf{x}_k})-\bar{m}_k(\bar{\bm\eta}_k^*))\phi_l\left(\tilde{\mathbf{H}}_k\right)^2.
	\end{equation}
	This completes the proof.
	
\end{proof}

\section{Appendix: Proof of Theorem \ref{thm_tCG_converge}}
\label{app_CGconvergence}

\begin{proof} 
	
	We first prove the relationship between $\mathbf{G}_k^i$ and $\mathbf{r}_i$. 
	According to Algorithm \ref{alg_non_linear_tCG},  $\mathbf{r}_0=\mathbf{G}_k^0=\mathbf{G}_k$. Then for $i>0$, we have 
	\begin{align}
		\nonumber
		\mathbf{G}_k^i &= \frac{1}{|\mathcal{S}_g|}\sum_{j\in\mathcal{S}_g} \nabla_{\alpha_i^*\mathbf{p}_i} f_j\left(R_{\mathbf{x}_k^{i-1}}\left(\alpha_i^*\mathbf{p}_i\right)\right)\\
		\label{eq_lem_cg_converege_1}
		&\approx \frac{1}{|\mathcal{S}_g|}\sum_{j\in\mathcal{S}_g}\nabla_{\alpha_i^*\mathbf{p}_i} \left(f_j\left(\mathbf{x}_{k}^{i-1}\right) + \left\langle\mathbf{G}_k^{i-1}, \alpha_i^*\mathbf{p}_i\right\rangle_{\mathbf{x}_k}+\frac{1}{2}\left\langle\alpha_i^*\mathbf{p}_i, \mathbf{H}_k^{i-1}[\alpha_i^*\mathbf{p}_i]\right\rangle_{\mathbf{x}_k}\right) \\
		\label{eq_lem_cg_converege_2}
		&=\mathbf{G}_k^{i-1}+\alpha_i^*\mathbf{H}_k^{i-1}[\mathbf{p}_i]  = \mathbf{r}_{i-1}+\alpha_i^*\mathbf{H}_k^{i-1}[\mathbf{p}_i] =\mathbf{r}_i,
	\end{align}
	where  $\approx$ comes from the first-order Taylor extension and Eq. (\ref{eq_lem_cg_converege_2}) follows the Step (\ref{CG_update_r}) in Algorithm \ref{alg_non_linear_tCG}.

	The exact line search in the Step. (\ref{CG_opt_alpha}) of Algorithm \ref{alg_non_linear_tCG} approximates 
	\begin{equation}
		\label{eq_step_search}
		\alpha_i^*=\arg\min_{\alpha \geq 0}f \left(R_{\mathbf{x}_k^{i-1}}(\alpha\mathbf{p}_i)\right).
	\end{equation}
	Zeroing the derivative of Eq. (\ref{eq_step_search}) with respect to $\alpha$ gives  
	\begin{equation}
		\label{eq_zero_obj}
		\mathbf{0}_{\mathbf{x}_k^i}=\nabla_{\alpha_i^*}f \left(R_{\mathbf{x}_k^{i-1}}\left(\alpha_i^*\mathbf{p}_i\right)\right)= \left\langle\nabla f\left(\mathbf{x}_k^i\right),\mathcal{P}_{\alpha_i^*\mathbf{p}_i}\frac{d(\alpha_i^*\mathbf{p}_i)}{d\alpha_i^*}\right\rangle_{\mathbf{x}_k^i}  \approx \left\langle\mathbf{G}_k^i,\mathcal{P}_{\alpha_i^*\mathbf{p}_i}\mathbf{p}_i \right\rangle_{\mathbf{x}_k^i},
	\end{equation}
	where $\approx$ results from the use of  subsampled gradient $\mathbf{G}_k^i$  to approximate the full gradient $\nabla f(\mathbf{x}_k^i)$. We then show that each $\mathbf{p}_i$ is a sufficient descent direction, i.e., $\left\langle\mathbf{G}_k^{i},\mathbf{p}_{i+1}\right\rangle_{\mathbf{x}_k^{i}}\le-C\left\|\mathbf{G}_k^i\right\|_{\mathbf{x}_k^i}^2$ for some constant $C>0$ \cite{sakai2021sufficient}.
	When $i=0$, $\mathbf{p}_1 = \mathbf{G}_k^0$, and thus $\left\langle\mathbf{G}_k^{0},\mathbf{p}_{1}\right\rangle_{\mathbf{x}_k^{0}}=-\left\|\mathbf{G}_k^0\right\|_{\mathbf{x}_k^0}^2$. 
	When $i>0$, from Step (\ref{CG_p_update}) in Algorithm \ref{alg_non_linear_tCG} and Eq. (\ref{eq_lem_cg_converege_2}),  we have $\mathbf{p}_{i+1} \approx -\mathbf{G}_k^i + \beta_i\mathbf{p}_i$. Applying the inner product to both sides by $\mathbf{G}_k^{i}$, we have
	\begin{align}
		\langle\mathbf{G}_k^i, \mathbf{p}_{i+1}\rangle_{\mathbf{x}_k^i}&\approx-\left\|\mathbf{G}_k^i\right\|_{\mathbf{x}_k^i}^2+\beta_i\langle\mathbf{G}_k^i,\mathcal{P}_{\alpha_i^*\mathbf{p}_i}\mathbf{p}_i\rangle_{\mathbf{x}_k^i}\\
		\label{eq_sufficient_descent}
		&\approx-\left\|\mathbf{G}_k^i\right\|_{\mathbf{x}_k^i}^2\le-C\left\|\mathbf{G}_k^i\right\|_{\mathbf{x}_k^i}^2,
	\end{align}
	with a selected $C>0$.  Here, Eq. (\ref{eq_sufficient_descent}) builds on Eq. (\ref{eq_zero_obj}).
	Given the sufficient descent direction $\mathbf{p}_i$ and the strong Wolfe conditions satisfied by $\alpha_i^*$, Theorem 2.4.1 in \cite{qi2011numerical} shows that the Zoutendijk Condition holds \cite{sato2015new}, i.e.,
	\begin{equation}
		\sum_{i=0}^\infty \frac{\left\langle\mathbf{G}_k^i, \mathbf{p}_{i+1}\right\rangle_{\mathbf{x}_k^i}^2}{\left\|\mathbf{p}_{i+1}\right\|_{\mathbf{x}_k^i}^2}<\infty.
		\label{eq_zou}
	\end{equation}	
	Next, we show that $\beta_i$ is upper bounded. Using Eq. (\ref{eq_lem_cg_converege_2}), we have
	\begin{align}
		\nonumber
		\beta_i&\approx \frac{\left\langle\mathbf{G}_k^i,\mathbf{G}_k^i-\frac{\left\|\mathbf{G}_k^i\right\|_{\mathbf{x}_k^i}}{\left\|\mathbf{G}_k^{i-1}\right\|_{\mathbf{x}_k^{i-1}}}\mathcal{P}_{\alpha_i^*\mathbf{T}_i}\mathbf{G}_k^{i-1}\right\rangle_{\mathbf{x}_k^i}}{2 \left\langle\mathbf{G}_k^{i-1},\mathbf{G}_k^{i-1} \right\rangle_{\mathbf{x}_k^{i-1}}}\\
		\label{eq_beta_bound}
		&=\frac{\left\|\mathbf{G}_k^i\right\|_{\mathbf{x}_k^i}^2-\frac{\left\|\mathbf{G}_k^i\right\|_{\mathbf{x}_k^i}}{\left\|\mathbf{G}_k^{i-1}\right\|_{\mathbf{x}_k^{i-1}}}\cos\theta\left\|\mathbf{G}_k^i\right\|_{\mathbf{x}_k^i}\left\|\mathbf{G}_k^{i-1}\right\|_{\mathbf{x}_k^{i-1}}}{2\left\|\mathbf{G}_k^{i-1}\right\|_{\mathbf{x}_k^{i-1}}^2} \le\frac{\left\|\mathbf{G}_k^i\right\|_{\mathbf{x}_k^i}^2}{\left\|\mathbf{G}_k^{i-1}\right\|_{\mathbf{x}_k^{i-1}}^2},
	\end{align}
	where $c_3>1$ is some constant.

	Now, we prove $\lim_{i\to\infty} \left\|\mathbf{G}_k^i\right\|_{\mathbf{x}_k^i}=0$ by contradiction. Assume that $\lim_{i\to\infty} \left\|\mathbf{G}_k^i\right\|_{\mathbf{x}_k^i}>0$, that is, for all $i$, there exists $\gamma>0$ such that $\left\|\mathbf{G}_k^i\right\|_{\mathbf{x}_k^i}>\gamma>0$.
	Squaring Step (\ref{CG_p_update}) of Algorithm \ref{alg_non_linear_tCG} and applying Eqs. (\ref{eq_wolfe_2}), (\ref{eq_lem_cg_converege_2}), (\ref{eq_sufficient_descent}) and (\ref{eq_beta_bound}), we have
	\begin{align}
		\nonumber
		\left\|\mathbf{p}_{i+1}\right\|_{\mathbf{x}_k^i}^2&\le	\left\|\mathbf{G}_k^{i}\right\|_{\mathbf{x}_k^i}^2 + 2\beta_i \left |\left\langle\mathbf{G}_k^{i},\mathcal{P}_{\alpha_i^*\mathbf{p}_i}\mathbf{p}_{i} \right\rangle_{\mathbf{x}_k^i}\right|+	\beta_i^2\left\|\mathcal{P}_{\alpha_i^*\mathbf{p}_i}\mathbf{p}_{i}\right\|_{\mathbf{x}_k^{i}}^2\\
		\nonumber
		&\le\left\|\mathbf{G}_k^{i}\right\|_{\mathbf{x}_k^i}^2-2c_2\frac{\left\|\mathbf{G}_k^i\right\|_{\mathbf{x}_k^i}^2}{\left\|\mathbf{G}_k^{i-1}\right\|_{\mathbf{x}_k^{i-1}}^2}\left\langle\mathbf{G}_k^{i-1},\mathbf{p}_i\right\rangle_{\mathbf{x}_k^{i-1}}+\beta_i^2\left\|\mathbf{p}_{i}\right\|_{\mathbf{x}_k^{i-1}}^2\\
		\nonumber
		&\le\left\|\mathbf{G}_k^{i}\right\|_{\mathbf{x}_k^i}^2+2c_2C\frac{\left\|\mathbf{G}_k^i\right\|_{\mathbf{x}_k^i}^2}{\left\|\mathbf{G}_k^{i-1}\right\|_{\mathbf{x}_k^{i-1}}^2}\left\|\mathbf{G}_k^{i-1}\right\|_{\mathbf{x}_k^{i-1}}^2+\beta_i^2\left\|\mathbf{p}_{i}\right\|_{\mathbf{x}_k^{i-1}}^2\\
		\label{beta_bound}
		&=\hat{C}\left\|\mathbf{G}_k^{i}\right\|_{\mathbf{x}_k^i}^2+\beta_i^2\left\|\mathbf{p}_{i}\right\|_{\mathbf{x}_k^{i-1}}^2,
	\end{align}
	where $\hat{C}=1+2c_2C>1$.
	Applying this  repetitively, we have
	\begin{align}
		\nonumber
		\left\|\mathbf{p}_{i+1}\right\|_{\mathbf{x}_k^i}^2&\le\hat{C}\left\|\mathbf{G}_k^{i}\right\|_{\mathbf{x}_k^i}^2+\beta_i^2\left(\hat{C}\left\|\mathbf{G}_k^{i-1}\right\|_{\mathbf{x}_k^{i-1}}^2+\beta_{i-1}^2\left\|\mathbf{p}_{i-1}\right\|_{\mathbf{x}_k^{i-2}}^2\right)\\
		\nonumber
		&\le\hat{C}\left(\left\|\mathbf{G}_k^{i}\right\|_{\mathbf{x}_k^i}^2+\beta_i^2\left\|\mathbf{G}_k^{i-1}\right\|_{\mathbf{x}_k^{i-1}}^2+\cdots+ \prod_{j=2}^{i}\beta_j^2\left\|\mathbf{G}_k^{1}\right\|_{\mathbf{x}_k^{1}}^2\right) +\left\|\mathbf{p}_{1}\right\|_{\mathbf{x}_k^{0}}^2 \prod_{j=1}^{i}\beta_j^2\\
		\label{beta_explain}
		&\le\hat{C}\left\|\mathbf{G}_k^{i}\right\|_{\mathbf{x}_k^i}^4\left(\frac{1}{\left\|\mathbf{G}_k^{i}\right\|_{\mathbf{x}_k^i}^2}+\frac{1}{\left\|\mathbf{G}_k^{i-1}\right\|_{\mathbf{x}_k^{i-1}}^2}+\cdots+\frac{1}{\left\|\mathbf{G}_k^{1}\right\|_{\mathbf{x}_k^{1}}^2}+\frac{1}{\hat{C}\left\|\mathbf{G}_k^{0}\right\|_{\mathbf{x}_k^0}^2}\right)\\
		\nonumber
		&\le\hat{C}\left\|\mathbf{G}_k^{i}\right\|_{\mathbf{x}_k^i}^4\sum_{j=0}^i\frac{1}{\left\|\mathbf{G}_k^{j}\right\|_{\mathbf{x}_k^j}^2}
		\le\frac{\hat{C}(i+1)}{\gamma^2}\left\|\mathbf{G}_k^{i}\right\|_{\mathbf{x}_k^i}^4,
	\end{align}
	where Eq. (\ref{beta_explain}) uses Eq. (\ref{eq_beta_bound}) and $\mathbf{p}_1=\mathbf{G}_k^0$, and the last inequality uses $\left\|\mathbf{G}_k^i\right\|_{\mathbf{x}_k^i}>\gamma$. Subsequently, this gives
	\begin{equation}
		\label{eq_norm_conj_bound} 
		\frac{\left\|\mathbf{G}_k^{i}\right\|_{\mathbf{x}_k^i}^4}{\left\|\mathbf{p}_{i+1}\right\|_{\mathbf{x}_k^i}^2}\ge\frac{\gamma^2}{\hat{C}(i+1)}.
	\end{equation}
	Combing Eqs. (\ref{eq_sufficient_descent}) and (\ref{eq_norm_conj_bound}), we have
	\begin{equation}
		\sum_{i=0}^\infty \frac{\left\langle\mathbf{G}_k^i, \mathbf{p}_{i+1}\right\rangle_{\mathbf{x}_k^i}^2}{\left\|\mathbf{p}_{i+1}\right\|_{\mathbf{x}_k^i}^2} =\sum_{i=0}^\infty\frac{\left\|\mathbf{G}_k^{i}\right\|_{\mathbf{x}_k^i}^4}{\left\|\mathbf{p}_{i+1}\right\|_{\mathbf{x}_k^i}^2} \times \frac{\langle\mathbf{G}_k^i, \mathbf{p}_{i+1}\rangle_{\mathbf{x}_k^i}^2}{\left\|\mathbf{G}_k^{i}\right\|_{\mathbf{x}_k^i}^4} \ge \sum_{i=0}^\infty \frac{C^2\gamma^2}{\hat{C}(i+1)}=\infty.
	\end{equation}
	This contradicts Eq. (\ref{eq_zou}) and completes the proof.
	
\end{proof}

\section{Appendix: Subproblem Solvers}
\label{app_sub_solution}

\subsection{Proof of Lemma \ref{lem_lanczos_sol_property}}

\subsubsection{The Case of $l=D$}

\begin{proof} 
	
	\textbf{Assumption \ref{assu_cauchy_eigen_point}:} 
	Regarding the Cauchy condition in Eq. (\ref{eq_cauchy_point}), it is satisfied simply because of Eq. (\ref{eq_lanczos_cauchy}) and Eq. (\ref{eq_cauchy_general_definition}). 
	Regarding the eigenstep condition, the proof is also fairly simple. 
	The solution $\bm\eta_k^*$ from Algorithm \ref{alg_non_linear_Lan} with $l=D$  is the global minimizer over $\mathbb{R}^{D}$. 
	As the subspace spanned by Cauchy steps and eigensteps belongs to $\mathbb{R}^D$, i.e., ${\rm Span}\left\{\bm\eta_k^C, \bm\eta_k^E\right\}\in\mathbb{R}^D$, we have
	\begin{equation}
		\hat{m}_k(\bm\eta_k^*) \le \min_{\bm\eta\in{\rm Span}\left\{\bm\eta_k^C, \bm\eta_k^E \right\}}\hat{m}_k (\bm\eta).
	\end{equation}
	Hence, the solution from Algorithm \ref{alg_non_linear_Lan} satisfies the eigenstep condition in Eq. (\ref{eq_eigen_point})) from Assumption \ref{assu_cauchy_eigen_point}.

	\textbf{Assumption \ref{assu_g}: } 
	As stated in Section 3.3 of \cite{cartis2011adaptive}, any minimizer, including the global minimizer $ \bm\eta_{k}^{*}$ from Algorithm \ref{alg_non_linear_Lan}, is a stationary point of $\hat{m}_k$, and naturally has the property $\nabla_{\bm\eta}\hat{m}_k( \bm\eta_{k}^{*})=\mathbf{0}_{\mathbf{x}_k}$. 
	Hence, it has $\left\|\nabla_{\bm\eta}\hat{m}_k( \bm\eta_{k}^{*})\right\|_{\mathbf{x}_k}=0\le \kappa_\theta\min\left(1, \left\| \bm\eta_k^*  \right\|_{\mathbf{x}_k} \right) \|\mathbf{G}_k\|_{\mathbf{x}_k}$. Assumption \ref{assu_g} is then satisfied.
	
	\textbf{Assumption \ref{assu_cg}: } 
	Given the above-mentioned property of the global minimizer $\bm\eta_k^*$ of $\hat{m}_k$ from Algorithm \ref{alg_non_linear_Lan}, i.e., $\nabla_{\bm\eta}\hat{m}_k( \bm\eta_{k}^{*})=\mathbf{0}_{\mathbf{x}_k}$, and using the definition of $\nabla_{\bm\eta}\hat{m}_k( \bm\eta_{k}^{*})$, it has
	\begin{equation}
		\label{eq_def_model_grad}
		\nabla_{\bm\eta}\hat{m}_k( \bm\eta_{k}^{*})=\mathbf{G}_k+\mathbf{H}_k[\bm\eta_k^*]+\sigma_k\left\|\bm\eta_{k}^{*}\right\|_{\mathbf{x}_k}\bm\eta_{k}^{*}=\mathbf{0}_{\mathbf{x}_k}.
	\end{equation}
	Applying the inner product operation with  $\bm\eta_k^*$ to both sides of Eq. (\ref{eq_def_model_grad}), it has
	\begin{equation}
		\label{eq_lan_assums_1}
		\langle\mathbf{G}_k, \bm\eta_k^*\rangle_{\mathbf{x}_k}+\langle\bm\eta_k^*, \mathbf{H}_k[\bm\eta_k^*]\rangle_{\mathbf{x}_k}+\sigma_k\left\|\bm\eta_{k}^{*}\right\|_{\mathbf{x}_k}^3 = 0,
	\end{equation}
	and this corresponds to Eq. (\ref{eq_sufficient_step0}). 
	As $\bm\eta_k^*$ is a descent direction, we have $\langle\mathbf{G}_k, \bm\eta_k^*\rangle_{\mathbf{x}_k}\le 0$. 
	Combining this with Eq. (\ref{eq_lan_assums_1}), it has
	\begin{equation}
		\langle\bm\eta_k^*, \mathbf{H}_k[\bm\eta_k^*]\rangle_{\mathbf{x}_k}+\sigma_k\left\|\bm\eta_{k}^{*}\right\|_{\mathbf{x}_k}^3\ge0.
	\end{equation}
	and this is Eq. (\ref{eq_sufficient_step}). This completes the proof.
\end{proof}

\subsubsection{The Case of $l<D$}
\begin{proof} 
	
	\textbf{Cauchy condition in Assumption \ref{assu_cauchy_eigen_point}: } Implied by Eq. (\ref{eq_lanczos_cauchy}), any intermediate solution $\bm\eta_k^{*l}$ satisfies the Cauchy condition.
	
	\textbf{Assumption \ref{assu_g}: } As stated in Section 3.3 of \cite{cartis2011adaptive}, any minimizer $\bm\eta^*$ of $\hat{m}_k$ admits the property $\nabla_{\bm\eta}\hat{m}_k( \bm\eta^{*})=\mathbf{0}_{\mathbf{x}_k}$. Since each $\bm\eta_k^{*l}$ is the minimizer of $\hat{m}_k$ over $\mathcal{K}_l$, it has $\nabla_{\bm\eta}\hat{m}_k( \bm\eta_k^{*l})=\mathbf{0}_{\mathbf{x}_k}$. Then, it has $\left\|\nabla_{\bm\eta}\hat{m}_k( \bm\eta_{k}^{*l})\right\|_{\mathbf{x}_k}=0\le \kappa_\theta\min\left(1, \left\| \bm\eta_k^{*l}  \right\|_{\mathbf{x}_k} \right) \|\mathbf{G}_k\|_{\mathbf{x}_k}$. Assumption \ref{assu_g} is then satisfied.

	\textbf{Assumption \ref{assu_cg}: } According to Lemma 3.2 of \cite{cartis2011adaptive}, a global minimizer of $\hat{m}_k$ over a subspace of $\mathbb{R}^D$ satisfies Assumption \ref{assu_cg}. As the solution $\bm\eta_k^{*l}$ in each iteration of Algorithm \ref{alg_non_linear_Lan} is a global solution over the subspace $\mathcal{K}_l$, each $\bm\eta_k^{*l}$ satisfies Assumption \ref{assu_cg}. 
\end{proof}

\subsection{Proof of Lemma \ref{lem_tcg_sol_property}}
\label{cite_proof}

\begin{proof} 
	
	\textbf{For Cauchy Condition of  Assumption \ref{assu_cauchy_eigen_point}:} 
	In the first iteration of  Algorithm \ref{alg_non_linear_tCG},   the step size is optimized along the steepest gradient direction, as  in the classical steepest-descent method of Cauchy, i.e.,
	\begin{equation}
		\label{eq_cg_cauchy}
		\bm\eta_{k}^{1}=\left(\arg\min_{\alpha\in\mathbb{R}} \hat{m}_k(\alpha \mathbf{G}_k)\right)\mathbf{G}_k.
	\end{equation}
	At each iteration $i$ of Algorithm \ref{alg_non_linear_tCG}, the line search process in Eq. (\ref{step_search0})  aims at finding a step size that can achieve a cost decrease, otherwise the step size will be zero, meaning that no strict decrease can be achieved and the algorithm will stop at Step (\ref{CG_return_1}). Because of this, we have
	\begin{equation}
		\hat{m}_k\left(\bm\eta_{k}^{i-1}+\alpha_i \mathbf{p}_i\right)\le \hat{m}_k\left(\bm\eta_{k}^{i-1}\right).
	\end{equation}
	Given $\bm\eta_{k}^{i}=\bm\eta_{k}^{i-1}+\alpha_i^*\mathbf{p}_i$ in Algorithm \ref{alg_non_linear_tCG}, we have
	\begin{equation}
		\hat{m}_k\left(\bm\eta_{k}^{i}\right)\le \hat{m}_k\left(\bm\eta_{k}^{i-1}\right).
	\end{equation}
	Then, considering all $i=1,2,...$, we have
	\begin{equation}
		\hat{m}_k\left( \bm\eta_{k}^* \right)\le \ldots \le \hat{m}_k\left( \bm\eta_{k}^{1}\right) \le \hat{m}_k\left( \bm\eta_{k}^{0} \right).
	\end{equation}
	This shows Algorithm \ref{alg_non_linear_tCG} always returns a solution better than or equal to the Cauchy step. 
	
	\textbf{For $\left\|\nabla_{\bm\eta} \hat{m}_k(\bm\eta)\right\|_{\mathbf{x}_k}\approx 0$:} 
	The approximation concept ($\approx$) of interest here builds on the fact that  $\hat{m}_k(\bm\eta)$ is used as the approximation of the real objective function $f(R_{\mathbf{x}_k}(\bm\eta))$. By assuming $\hat{m}_k(\bm\eta)\approx f(R_{\mathbf{x}_k}(\bm\eta))$, it leads to 
	\begin{equation}
		\label{grad_approx}
		\nabla_{\bm\eta} \hat{m}_k(\bm\eta)\approx \nabla_{\bm\eta}f(R_{\mathbf{x}_k}(\bm\eta)).
	\end{equation}
	Let $\mathbf{G}_{k+1}$ be the subsampled gradient evaluated at $R_{\mathbf{x}_k}(\bm\eta_k^*)$. Based on Theorem \ref{thm_tCG_converge}, it has $\mathbf{G}_{k+1}= \lim_{i\to\infty}\mathbf{G}_k^i$ where $\mathbf{G}_k^i$ is the resulting subsampled gradient after $i$ inner iterations. Since $\mathbf{G}_{k+1}$ is the approximate gradient of the full gradient $\nabla f(R_{\mathbf{x}_k}(\bm\eta_k^*))$, it has $\nabla f(R_{\mathbf{x}_k}(\bm\eta_k^*))=\mathbb{E}[\mathbf{G}_{k+1}]=\mathbb{E}[\lim_{i\to\infty}\mathbf{G}_k^i]$. Hence, it has
	\begin{align}
		& \left\|	\nabla_{\bm\eta} f\left(R_{\mathbf{x}_k}(\bm\eta_k^*)\right)\right\|_{\mathbf{x}_k} \\
		\nonumber
		= &\left\|	\nabla f\left(R_{\mathbf{x}_k}(\bm\eta_k^*)\right)\frac{d\left(R_{\mathbf{x}_k}(\bm\eta_k^*)\right)}{d\bm\eta}\Big|_{\bm\eta=\bm\eta_k^*}\right\|_{\mathbf{x}_k} = \left\|\mathbb{E}\left[\lim_{i\to\infty}\mathbf{G}_k^i\right]\right\|_{\mathbf{x}_k}\le\mathbb{E}\left[\lim_{i\to\infty}\left\|\mathbf{G}_k^i\right\|_{\mathbf{x}_k^i}\right]=0,
	\end{align}
	which indicates the equality holds as $\left\|\nabla_{\bm\eta} f\left(R_{\mathbf{x}_k}(\bm\eta_k^*)\right)\right\|_{\mathbf{x}_k}  = 0$.
	Combining this with Eq. (\ref{grad_approx}), it completes the proof.

	\textbf{For Condition 1 of Assumption \ref{assu_cg}:} 
	Using the definition of $\nabla_{\bm\eta}\hat{m}_k( \bm\eta_{k}^{*})$ as in Eq. (\ref{eq_def_model_grad}), it has
	\begin{align}
		\label{appD_eq1}
		\left\langle \nabla_{\bm\eta}\hat{m}_k( \bm\eta_{k}^{*}), \bm\eta_k^* \right\rangle & = \langle\mathbf{G}_k+\mathbf{H}_k[\bm\eta_k^*]+\sigma_k\left\|\bm\eta_{k}^{*}\right\|_{\mathbf{x}_k}\bm\eta_{k}^{*}, \bm\eta_k^* \rangle_{\mathbf{x}_k} \\
		\nonumber
		& = \langle\mathbf{G}_k, \bm\eta_k^*\rangle_{\mathbf{x}_k}+\langle\bm\eta_k^*, \mathbf{H}_k[\bm\eta_k^*]\rangle_{\mathbf{x}_k}+\sigma_k\left\|\bm\eta_{k}^{*}\right\|_{\mathbf{x}_k}^3.
	\end{align}
	Also,  since $\left\|\nabla_{\bm\eta}\hat{m}_k( \bm\eta_{k}^{*})\right\|_{\mathbf{x}_k}\approx 0$, we have
	\begin{equation}
		\label{appD_eq2}
		\left| \left\langle \nabla_{\bm\eta}\hat{m}_k( \bm\eta_{k}^{*}), \bm\eta_k^* \right\rangle \right| \leq \left\|\nabla_{\bm\eta}\hat{m}_k( \bm\eta_{k}^{*})\right\|_{\mathbf{x}_k} \left\|\bm\eta_k^*\right\|_{\mathbf{x}_k} \approx 0.
	\end{equation}
	Combining Eqs (\ref{appD_eq1}) and (\ref{appD_eq2}) , this results in Eq. (\ref{assump3_app}), which completes the proof.
\end{proof}

\section{Appendix: Proof of Theorem \ref{theorem1}}
\label{appdendix_E}

\subsection{Matrix Bernstein Inequality}
We build the proof on the  matrix Bernstein inequality.  We restate this inequality in the following lemma.
\begin{lem} [Matrix Bernstein Inequality (\cite{gross2011recovering,tropp2015introduction})] 
	Let  $\mathbf{A}_1,...,\mathbf{A}_n$ be independent, centered random matrices with the common dimension $d_1\times d_2$. Assume that each one is uniformly bounded:
	\begin{equation}
		\mathbb{E}[\mathbf{A}_i]=0, \|\mathbf{A}_i\|\le \mu,\  i=1,...,n.	
	\end{equation}
	Given the matrix sum $\mathbf{Z}=\sum_{i=1}^n\mathbf{A}_i$, we define its variance $\nu(\mathbf{Z})$ by
	\begin{equation}
		\begin{split}
			\nu(\mathbf{Z})&:=\max\left\{\left\|\sum_{i=1}^n\mathbb{E}\left[\mathbf{A}_i\mathbf{A}_i^T\right]\right\|, \left\|\sum_{i=1}^n\mathbb{E}\left[\mathbf{A}_i^T\mathbf{A}_i\right]\right\|\right\}.
		\end{split}
	\end{equation}
	Then 
	\begin{equation}
		{\rm Pr}(\left\|\mathbf{Z}\right\|\ge \epsilon)\le (d_1+d_2)\exp\left(\frac{-\epsilon^2/2}{\nu(\mathbf{Z})+\mu \epsilon/3}\right)\ {\rm for\ all}\ \epsilon>0.
	\end{equation}
	\label{lem_Matrix_Bernstein_inequality}
\end{lem}
This lemma supports us to prove Theorem \ref{theorem1} as below.

\subsection{Main Proof}
\begin{proof}
	Following the subsampling process, a total of $|\mathcal{S}_g|$ matrices are uniformly sampled from the set of $n$ Riemannan gradients $\left\{{\rm grad}f_i(\mathbf{x}) \subseteq\mathbb{R}^{d\times r}\right\}_{i=1}^n$. We denote each sampled element as $\mathbf{G}^{(i)}_{\mathbf{x}}$,  and  it has 
	\begin{equation}
		{\rm Pr}\left(\mathbf{G}^{(i)}_\mathbf{x}\right)=\frac{1}{n},\ i=1,...,|\mathcal{S}_g|.
	\end{equation}
	Define the random matrix
	\begin{equation}
		\mathbf{X}_i:=\mathbf{G}^{(i)}_\mathbf{x}-{\rm grad}f(\mathbf{x}),\ i=1,2,..,|\mathcal{S}_g|.
	\end{equation}
	
	\noindent Our focus is the type of problem as in Eq. (\ref{general_optimization_problem}), therefore it has
	\begin{align}
		&\mathbb{E}[\mathbf{X}_i] =\mathbb{E}\left[\mathbf{G}^{(i)}_\mathbf{x}\right] - \mathbb{E}\left[{\rm grad}f(\mathbf{x})\right]=\mathbb{E}\left[{\rm grad}f_i(\mathbf{x})\right] - \mathbb{E}\left[ \frac{1}{n} \sum_{i=1}^n {\rm grad} f_i(\mathbf{x})\right]  =  \mathbf{0}.
	\end{align}
	Define a random variable
	\begin{equation}
		\mathbf{X}:=\frac{1}{|\mathcal{S}_g|}\sum_{i=1}^{|\mathcal{S}_g|} \mathbf{X}_i=\frac{1}{|\mathcal{S}_g|}\sum_{i=1}^{|\mathcal{S}_g|}\left(\mathbf{G}^{(i)}_\mathbf{x}-{\rm grad}f(\mathbf{x})\right).
	\end{equation}
	Its variance satisfies
	\begin{align}
		\nonumber
		\nu(\mathbf{X})&=\max\left\{\frac{1}{|\mathcal{S}_g|^2}\left\|\sum_{i=1}^{|\mathcal{S}_g|}\mathbb{E}\left[\mathbf{X}_i\mathbf{X}_i^T\right]\right\|_\mathbf{x}, \frac{1}{|\mathcal{S}_g|^2}\left\|\sum_{i=1}^{|\mathcal{S}_g|}\mathbb{E}\left[\mathbf{X}_i^T\mathbf{X}_i\right]\right\|_\mathbf{x}\right\}\\
		\label{eq_nu_2}
		&\le\frac{1}{|\mathcal{S}_g|^2}\max\left\{\sum_{i=1}^{|\mathcal{S}_g|}\mathbb{E}\left[\left\|\mathbf{X}_i\mathbf{X}_i^T\right\|_\mathbf{x}\right], \sum_{i=1}^{|\mathcal{S}_g|}\mathbb{E}\left[\left\|\mathbf{X}_i^T\mathbf{X}_i\right\|_\mathbf{x}\right]\right\}
	\end{align}
	Take $\mathbf{G}^{(i)}_{\mathbf{x}}={\rm grad}f_1(\mathbf{x})$ as an example and  applying the definition of $K_{g_{max}}$ in Eq. (\ref{eq_Kgmax}), we have
	\begin{align}
		\nonumber
		\mathbb{E}[\|\mathbf{X}_i\|_\mathbf{x}] &= \mathbb{E}\left[\left\|{\rm grad}f_1(\mathbf{x})-{\rm grad}f(\mathbf{x})\right\|_\mathbf{x}\right] =\mathbb{E}\left[\left\|{\rm grad}f_1(\mathbf{x})-\frac{1}{n}\sum_{i=1}^n{\rm grad}f_i(\mathbf{x})\right\|_\mathbf{x}\right]\\
		\nonumber
		&= \mathbb{E}\left[\left\|\frac{n-1}{n}{\rm grad}f_1(\mathbf{x})-\frac{1}{n}\sum_{i=2}^n{\rm grad}f_i(\mathbf{x})\right\|_\mathbf{x}\right]\\
		\nonumber
		&\le \mathbb{E}\left[ \left( \frac{2(n-1)^2}{n^2} \left\|{\rm grad}f_1(\mathbf{x})\right\|_\mathbf{x}^2+ \frac{2}{n^2}\left\|\sum_{i=2}^n{\rm grad}f_i(\mathbf{x})\right\|_\mathbf{x}^2 \right)^{\frac{1}{2}}\right]\\
		\label{eq_Xi_norm_2}
		&\le \mathbb{E}\left[ \left( \frac{2(n-1)^2}{n^2} K_{g_{max}}^2+ \frac{2(n-2)^2}{n^2}  K_{g_{max}}^2 \right)^{\frac{1}{2}}\right]\le 2K_{g_{max}}.
	\end{align}
	where the first inequality uses $(a+b)^2\le 2a^2+2b^2$. Combining Eq. (\ref{eq_nu_2}) and Eq. (\ref{eq_Xi_norm_2}), it has
	\begin{align}
		\nu(\mathbf{X})\le \frac{1}{|\mathcal{S}_g|^2}\sum_{i=1}^{|\mathcal{S}_g|}\mathbb{E}\left[\left\|\mathbf{X}_i\right\|_\mathbf{x}\right]\mathbb{E}\left[\left\|\mathbf{X}_i^T\right\|_\mathbf{x}\right]
		\le \frac{1}{|\mathcal{S}_g|^2}\sum_{i=1}^{|\mathcal{S}_g|} 4K_{g_{max}}^2
		=\frac{4}{|\mathcal{S}_g|}K_{g_{max}}^2.
	\end{align}

	Now we are ready to apply Lemma \ref{lem_Matrix_Bernstein_inequality}.  Given $\mathbb{E}\left[\frac{1}{|\mathcal{S}_g|}\mathbf{X}_i\right]=\frac{1}{|\mathcal{S}_g|}\mathbb{E}[\mathbf{X}_i]=\mathbf{0}$ and according to the superma definition $\left\|\frac{1}{|\mathcal{S}_g|}\mathbf{X}_i\right\|_\mathbf{x}=\frac{1}{|\mathcal{S}_g|}\|\mathbf{X}_i\|_\mathbf{x}\le\frac{K_{g_{max}}}{|\mathcal{S}_g|}$, the following is obtained from the matrix Bernstein inequality:
	\begin{align}
		{\rm Pr}(\|\mathbf{X}\|_\mathbf{x}\ge \epsilon)&= {\rm Pr}\left(\left\|\frac{1}{|\mathcal{S}_g|}\sum_{i=1}^{|\mathcal{S}_g|}\mathbf{G}^{(i)}_\mathbf{x}-{\rm grad}f(\mathbf{x})\right\|_\mathbf{x}\ge \epsilon\right)\\
		&\le(d+r)\exp\left(\frac{-\epsilon^2/2}{\frac{4}{|\mathcal{S}_g|}K_{g_{max}}^2+\frac{K_{g_{max}}}{|\mathcal{S}_g|} \epsilon/3}\right)\\
		&\le(d+r)\exp\left(\frac{-|\mathcal{S}_g|\epsilon^2}{8(K_{g_{max}}^2+K_{g_{max}})}\right)=\delta,
	\end{align}
	of which the last equality implies
	\begin{equation}
		\epsilon=2\sqrt{\frac{2\left(K_{g_{max}}^2+K_{g_{max}}\right)\ln\left(\frac{d+r}{\delta}\right)}{|\mathcal{S}|_g}}.
	\end{equation}
	In simple words, with a probability at least $1-\delta$, we have $\|\mathbf{X}\|_\mathbf{x} < \epsilon$. 
	Letting 
	\begin{equation}
		\epsilon =2\sqrt{\frac{2\left(K_{g_{max}}^2+K_{g_{max}}\right)\ln\left(\frac{d+r}{\delta}\right)}{|\mathcal{S}|_g}} \le \delta_g,
	\end{equation}
	this results in the sample size bound as in Eq. (\ref{eq_restriction_sg})
	\begin{equation}\
		|\mathcal{S}_g|\ge\frac{8\left(K_{g_{max}}^2+K_{g_{max}}\right)\ln\left(\frac{d+r}{\delta}\right)}{\delta_g^2}.
	\end{equation}
	Therefore, we have $ \|\mathbf{X}\|_\mathbf{x} \le \delta_g$. 
	Expanding $\mathbf{X}$, we have the following satisfied with a  probability at least $1-\delta$:
	\begin{equation}
		\|\mathbf{X}\|_\mathbf{x} = \left\|\frac{1}{|\mathcal{S}_g|}\sum_{i=1}^{|\mathcal{S}_g|}\mathbf{G}^{(i)}_{\mathbf{x}}-{\rm grad}f(\mathbf{x})\right\|_\mathbf{x} = \left\|\mathbf{G}_k-{\rm grad}f(\mathbf{x}_k)\right\|_\mathbf{x}  \le \delta_g,
	\end{equation}
	which is Eq. (\ref{eq_approximate_grad_hess_bound1}) of Condition 1. 
	
	The proof of the other sample size bound follows the same strategy.
	A total of $|\mathcal{S}_H|$ matrices are uniformly sampled from the set of $n$ Riemannan Hessians $\left\{\nabla^2\hat{f}_i(\mathbf{0}_x)[\bm\eta] \subseteq\mathbb{R}^{d\times r} \right\}_{i=1}^n $. We denote each sampled element as $\mathbf{H}^{(i)}_{\mathbf{x}}[\bm\eta]$,  and  it has 
	\begin{equation}
		{\rm Pr}\left(\mathbf{H}^{(i)}_\mathbf{x}[\bm\eta]\right)=\frac{1}{n},\ i=1,...,|\mathcal{S}_H|.
	\end{equation}
	Define the random matrix
	\begin{equation}
		\mathbf{Y}_i:=\mathbf{H}^{(i)}_\mathbf{x}[\bm\eta]-\nabla^2\hat{f}(\mathbf{0}_x)[\bm\eta],\ i=1,2,..,|\mathcal{S}_H|.
	\end{equation}
	For the problem defined in Eq. (1) with a second-order retraction, it has
	\begin{align}
		&\mathbb{E}[\mathbf{Y}_i] =\mathbb{E}\left[\mathbf{H}^{(i)}_\mathbf{x}[\bm\eta]\right] - \mathbb{E}\left[\nabla^2\hat{f}(\mathbf{0}_x)[\bm\eta]\right] \\
		&=\mathbb{E}\left[\nabla^2\hat{f}_i(\mathbf{0}_x)[\bm\eta]\right] - \mathbb{E}\left[ \frac{1}{n} \sum_{i=1}^n\nabla^2\hat{f}_i(\mathbf{0}_x)[\bm\eta]\right]  =  \mathbf{0}.
	\end{align}
	Define a random variable
	\begin{equation}
		\mathbf{Y}:=\frac{1}{|\mathcal{S}_H|}\sum_{i=1}^{|\mathcal{S}_H|} \mathbf{Y}_i=\frac{1}{|\mathcal{S}_H|}\sum_{i=1}^{|\mathcal{S}_H|}\left(\mathbf{H}^{(i)}_\mathbf{x}[\bm\eta]-\nabla^2\hat{f}(\mathbf{0}_x)[\bm\eta]\right).
	\end{equation}
	Its variance satisfies
	\begin{align}
		\nonumber
		\nu(\mathbf{Y})&=\max\left\{\frac{1}{|\mathcal{S}_H|^2}\left\|\sum_{i=1}^{|\mathcal{S}_H|}\mathbb{E}\left[\mathbf{Y}_i\mathbf{Y}_i^T\right]\right\|_\mathbf{x}, \frac{1}{|\mathcal{S}_H|^2}\left\|\sum_{i=1}^{|\mathcal{S}_H|}\mathbb{E}\left[\mathbf{Y}_i^T\mathbf{Y}_i\right]\right\|_\mathbf{x}\right\}\\
		\label{eq_nu_Y2}
		&\le\frac{1}{|\mathcal{S}_H|^2}\max\left\{\sum_{i=1}^{|\mathcal{S}_H|}\mathbb{E}\left[\left\|\mathbf{Y}_i\mathbf{Y}_i^T\right\|_\mathbf{x}\right], \sum_{i=1}^{|\mathcal{S}_H|}\mathbb{E}\left[\left\|\mathbf{Y}_i^T\mathbf{Y}_i\right\|_\mathbf{x}\right]\right\}.
	\end{align}
	Take $\mathbf{H}^{(i)}_\mathbf{x}[\bm\eta]=\nabla^2\hat{f}_1(\mathbf{0}_x)[\bm\eta]$ as an example and  applying the definition of $K_{H_{max}}$ in Eq. (\ref{eq_KHmax}), we have
	\begin{align}
		\nonumber
		\mathbb{E}[\|\mathbf{Y}_i\|_\mathbf{x}] &=\mathbb{E}\left[\left\|\nabla^2\hat{f}_1(\mathbf{0}_x)[\bm\eta]-\frac{1}{n}\sum_{i=1}^n\nabla^2\hat{f}_i(\mathbf{0}_x)[\bm\eta]\right\|_\mathbf{x}\right]\\
		\nonumber
		&= \mathbb{E}\left[\left\|\frac{n-1}{n}\nabla^2\hat{f}_1(\mathbf{0}_x)[\bm\eta]-\frac{1}{n}\sum_{i=2}^n\nabla^2\hat{f}_i(\mathbf{0}_x)[\bm\eta]\right\|_\mathbf{x}\right]\\
		\nonumber
		&\le \mathbb{E}\left[ \left(\frac{2(n-1)^2}{n^2} \left\|\nabla^2\hat{f}_1(\mathbf{0}_x)[\bm\eta]\right\|_\mathbf{x}^2+ \frac{2}{n^2} \left\|\sum_{i=2}^n\nabla^2\hat{f}_i(\mathbf{0}_x)[\bm\eta]\right\|_\mathbf{x}^2 \right)^{\frac{1}{2}}\right]\\
		\nonumber
		&\le \mathbb{E}\left[ \left(\frac{2(n-1)^2}{n^2} K_{H_{max}}^2\left\| \bm\eta\right\|_{\mathbf{x}}^2 + \frac{2(n-2)^2}{n^2} K_{H_{max}}^2\left\| \bm\eta\right\|_{\mathbf{x}}^2 \right)^{\frac{1}{2}}\right] \\
		\label{eq_Yi_norm_2}
		&\le 2K_{H_{max}}\left\| \bm\eta\right\|_{\mathbf{x}}.
	\end{align}
	where the first inequality uses $(a+b)^2\le 2a^2+2b^2$, and $\bm\eta$ is the current moving direction being optimized in Eq. (\ref{eq_sub_problem}). Combining Eq. (\ref{eq_nu_Y2}) and Eq. ((\ref{eq_Yi_norm_2}), we have
	\begin{align}
		\nu(\mathbf{Y})&\le \frac{1}{|\mathcal{S}_H|^2}\sum_{i=1}^{|\mathcal{S}_H|}\mathbb{E}\left[\left\|\mathbf{Y}_i\right\|_\mathbf{x}\right]\mathbb{E}\left[\left\|\mathbf{Y}_i^T\right\|_\mathbf{x}\right]\\
		&\le \frac{1}{|\mathcal{S}_H|^2}\sum_{i=1}^{|\mathcal{S}_H|} 4K_{H_{max}}^2\left\|\bm\eta\right\|_{\mathbf{x}}^2
		=\frac{4}{|\mathcal{S}_H|}K_{H_{max}}^2\left\|\bm\eta\right\|_{\mathbf{x}}^2.
	\end{align}

	We then apply Lemma \ref{lem_Matrix_Bernstein_inequality}. Given $\mathbb{E}\left[\frac{1}{|\mathcal{S}_H|}\mathbf{Y}_i\right]=\frac{1}{|\mathcal{S}_H|}\mathbb{E}[\mathbf{Y}_i]=\mathbf{0}$ and according to the superma definition $\left\|\frac{1}{|\mathcal{S}_H|}\mathbf{Y}_i\right\|_\mathbf{x}=\frac{1}{|\mathcal{S}_H|}\|\mathbf{Y}_i\|_\mathbf{x}\le\frac{K_{H_{max}}\left\|\bm\eta\right\|_{\mathbf{x}}}{|\mathcal{S}_H|}$,  the following is obtained from the matrix Bernstein inequality:
	\begin{align}
		{\rm Pr}(\|\mathbf{Y}\|_\mathbf{x}\ge \epsilon)&= {\rm Pr}\left(\left\|\frac{1}{|\mathcal{S}_H|}\sum_{i=1}^{|\mathcal{S}_H|}\mathbf{H}^{(i)}_\mathbf{x}[\bm\eta]-\nabla^2\hat{f}(\mathbf{0}_x)[\bm\eta]\right\|_\mathbf{x}\ge \epsilon\right)\\
		&\le(d+r)\exp\left(\frac{-\epsilon^2/2}{\frac{4}{|\mathcal{S}_H|}K_{H_{max}}^2\left\|\bm\eta\right\|_{\mathbf{x}}^2+\frac{K_{H_{max}}\left\|\bm\eta\right\|_{\mathbf{x}}}{|\mathcal{S}_H|} \epsilon/3}\right)\\
		&\le(d+r)\exp\left(\frac{-|\mathcal{S}_H|\epsilon^2}{8(K_{H_{max}}^2\left\|\bm\eta\right\|_{\mathbf{x}}^2+K_{H_{max}}\left\|\bm\eta\right\|_{\mathbf{x}})}\right)=\delta,
	\end{align}
	of which the last equality indicates
	\begin{equation}
		\epsilon=2\sqrt{\frac{2\left(K_{H_{max}}^2\left\|\bm\eta\right\|_{\mathbf{x}}^2+K_{H_{max}}\left\|\bm\eta\right\|_{\mathbf{x}}\right)\ln\left(\frac{d+r}{\delta}\right)}{|\mathcal{S}_H|}}.
	\end{equation}
	In simple words, with probability at least $1-\delta$, we have $\|\mathbf{Y}\|_\mathbf{x} < \epsilon$. By letting 
	\begin{equation}
		\epsilon =2\sqrt{\frac{2\left(K_{H_{max}}^2\left\|\bm\eta\right\|_{\mathbf{x}}^2+K_{H_{max}}\left\|\bm\eta\right\|_{\mathbf{x}}\right)\ln\left(\frac{d+r}{\delta}\right)}{|\mathcal{S}_H|}} \le \delta_H\left\|\bm\eta\right\|_{\mathbf{x}},
	\end{equation}
	which results in the sample size bound in Eq. (\ref{eq_restriction_sH})
	\begin{align}
		|\mathcal{S}_H|\ge\frac{8\left(K_{H_{max}}^2+\frac{K_{H_{max}}}{\left\|\bm\eta\right\|_{\mathbf{x}}}\right)\ln\left(\frac{d+r}{\delta}\right)}{\delta_H^2}.
	\end{align}
	And we have $ \|\mathbf{Y}\|_\mathbf{x} \le \delta_H\left\|\bm\eta\right\|_{\mathbf{x}}$. Expanding $\mathbf{Y}$, we have the following satisfied with a  probability at least $1-\delta$:
	\begin{equation}
		\|\mathbf{Y}\|_\mathbf{x} = \left\|\frac{1}{|\mathcal{S}_H|}\sum_{i=1}^{|\mathcal{S}_H|}\mathbf{H}^{(i)}_{\mathbf{x}}[\bm\eta]-\nabla^2\hat{f}(\mathbf{0}_x)[\bm\eta]\right\|_\mathbf{x} = \left\|\mathbf{H}_k[\bm\eta]-\nabla^2\hat{f}(\mathbf{0}_x)[\bm\eta]\right\|_\mathbf{x}  \le \delta_H\left\|\bm\eta\right\|_{\mathbf{x}},
	\end{equation}
	which is Eq. (\ref{eq_approximate_grad_hess_bound2}) of Condition 1. 
\end{proof}

\section{Appendix: Theorem  \ref{theorem2} and Corollary  \ref{coro_1} }
\label{app_theorem2}

\subsection{Supporting Lemmas for Theorem  \ref{theorem2}}
\label{app_theorem2A}

\begin{lem} \textit{Suppose Condition 1 and Assumptions 1, 2 hold, then for the case of $\|\mathbf{G}_k\|\ge\epsilon_g$, we have}
	\begin{equation}
		\begin{split}
			\hat{f}_k(\bm{\eta}_k)-\hat{m}_k(\bm{\eta}_k)\le\left(\frac{L_H}{6}-\frac{\sigma_k}{3}\right)\|\bm{\eta}_k\|_{\bm{x}_k}^3+\delta_g\|\bm{\eta}_k\|_{\bm{x}_k}+\frac{1}{2}\delta_H\|\bm{\eta}_k\|_{\bm{x}_k}^2.
		\end{split}
		\label{eq_lemma_f_minus_m_cauchy}
	\end{equation}
	\label{lem_f_minus_m_cauchy}
\end{lem}

\begin{proof}
	We start from the left-hand side of Eq. (\ref{eq_lemma_f_minus_m_cauchy}), and this leads to
	\begin{align}
		\nonumber
		&\hat{f}_k(\bm{\eta}_k)-\hat{m}_k(\bm{\eta}_k) \\
		\nonumber
		=\;&\hat{f}_k(\bm{\eta}_k)-f(\bm{x}_k)-\left\langle\mathbf{G}_k,\bm{\eta}_k\right\rangle_{\bm{x}_k}-\frac{1}{2}\left\langle\bm{\eta}_k,\mathbf{H}_k[\bm{\eta}_k]\right\rangle_{\bm{x}_k}-\frac{\sigma_k}{3}\|\bm{\eta}_k\|_{\bm{x}_k}^3 \\
		\nonumber
		=\;&\hat{f}_k(\bm{\eta}_k)-f(\bm{x}_k)-\langle\text{grad}f(\bm{x}_k),\bm{\eta}_k\rangle_{\bm{x}_k}-\frac{1}{2}\left\langle\bm{\eta}_k,\nabla^2\hat{f}_k(\bm{0}_{\bm{x}_k})[\bm{\eta}_k]\right\rangle_{\bm{x}_k} \\
		\nonumber
		&+\langle\text{grad}f(\bm{x}_k),\bm{\eta}_k\rangle_{\bm{x}_k}-\langle\mathbf{G}_k,\bm{\eta}_k\rangle_{\bm{x}_k}+\frac{1}{2}\left\langle\bm{\eta}_k,\nabla^2\hat{f}_k(\bm{0}_{\bm{x}_k})[\bm{\eta}_k]\right\rangle_{\bm{x}_k} \\
		\nonumber
		&-\frac{1}{2}\left\langle\bm{\eta}_k,\mathbf{H}_k[\bm{\eta}_k]\right\rangle_{\bm{x}_k}-\frac{\sigma_k}{3}\|\bm{\eta}_k\|_{\bm{x}_k}^3 \\
		\nonumber
		\le\;&\left|\hat{f}_k(\bm{\eta}_k)-f(\bm{x}_k)-\langle\text{grad}f(\bm{x}_k),\bm{\eta}_k\rangle_{\bm{x}_k}-\frac{1}{2}\left\langle\bm{\eta}_k,\nabla^2\hat{f}_k(\bm{0}_{\bm{x}_k})[\bm{\eta}_k]\right\rangle_{\bm{x}_k}\right| \\
		\nonumber
		&+\left|\langle\text{grad}f(\bm{x}_k),\bm{\eta}_k\rangle_{\bm{x}_k}-\langle\mathbf{G}_k,\bm{\eta}_k\rangle_{\bm{x}_k}\right| \\
		\nonumber
		&+\left|\frac{1}{2}\left\langle\bm{\eta}_k,\nabla^2\hat{f}_k(\bm{0}_{\bm{x}_k})[\bm{\eta}_k]\right\rangle_{\bm{x}_k}-\frac{1}{2}\langle\bm{\eta}_k,\mathbf{H}_k[\bm{\eta}_k]\rangle_{\bm{x}_k}\right|-\frac{\sigma_k}{3}\|\bm{\eta}_k\|_{\bm{x}_k}^3 \\
		\nonumber
		\le\;&\frac{1}{6}L_H\|\bm{\eta}_k\|_{\bm{x}_k}^3+\delta_g\|\bm{\eta}_k\|_{\bm{x}_k}+\frac{1}{2}\delta_H\|\bm{\eta}_k\|_{\bm{x}_k}^2-\frac{\sigma_k}{3}\|\bm{\eta}_k\|_{\bm{x}_k}^3 \\
		=\;&\left(\frac{L_H}{6}-\frac{\sigma_k}{3}\right)\|\bm{\eta}_k\|_{\bm{x}_k}^3+\delta_g\|\bm{\eta}_k\|_{\bm{x}_k}+\frac{1}{2}\delta_H\|\bm{\eta}_k\|_{\bm{x}_k}^2, 
		\label{eq_proof_lemma_f_minus_m_cauchy}
	\end{align}
	where the first inequality uses the Cauchy-Schwarz inequality and the second one uses 
	Eqs. (\ref{eq_restricted_lipschitz_hessian_1}), (\ref{eq_approximate_grad_hess_bound1}) and (\ref{eq_approximate_grad_hess_bound2}). It can be seen that the term $\left\langle\mathbf{G}_k,\bm{\eta}_k\right\rangle_{\bm{x}_k}$ can not be neglected because of the condition $\|\mathbf{G}_k\|\ge\epsilon_g$ based on Eq. (\ref{eq_m}). This then results in the term $\delta_g\|\bm{\eta}_k\|_{\bm{x}_k}$ in Eq. (\ref{eq_proof_lemma_f_minus_m_cauchy}).
\end{proof}

\begin{lem} \textit{Suppose Condition 1 and Assumptions 1, 2 hold, then for the case of $\|\mathbf{G}_k\|<\epsilon_g$ and $\lambda_{min}(\mathbf{H}_k)<-\epsilon_H$ , we have}
	\begin{equation}
		\begin{split}
			\hat{f}_k(\bm{\eta}_k)-\hat{m}_k(\bm{\eta}_k)\le\left(\frac{L_H}{6}-\frac{\sigma_k}{3}\right)\|\bm{\eta}_k\|_{\bm{x}_k}^3+\frac{1}{2}\delta_H\|\bm{\eta}_k\|_{\bm{x}_k}^2.
		\end{split}
		\label{eq_lemma_f_minus_m_eigen}
	\end{equation}
	\label{lem_f_minus_m_eigen}
\end{lem}
\begin{proof}
	For each $\bm{\eta}_k$, at least one of the two inequalities is true:
	\begin{align}
		\langle\bm{\eta}_k,\text{grad}f(\bm{x}_k)\rangle_{\bm{x}_k}\le& 0, \\
		\langle-\bm{\eta}_k,\text{grad}f(\bm{x}_k)\rangle_{\bm{x}_k}\le & 0.
	\end{align}
	Without loss of the generality, we assume $\langle\bm{\eta}_k,\text{grad}f(\bm{x}_k)\rangle_{\bm{x}_k}\le 0$ which is also an assumption adopted by \cite{yao2021inexact}, and it has
	\begin{align}
		\nonumber
		&\hat{f}_k(\bm{\eta}_k)-\hat{m}_k(\bm{\eta}_k) \\
		\nonumber
		=\;&\hat{f}_k(\bm{\eta}_k)-f(\bm{x}_k)-\frac{1}{2}\langle\bm{\eta}_k,\mathbf{H}_k[\bm{\eta}_k]\rangle_{\bm{x}_k}-\frac{\sigma_k}{3}\|\bm{\eta}_k\|_{\bm{x}_k}^3 \\
		\nonumber
		=\;&\hat{f}_k(\bm{\eta}_k)-f(\bm{x}_k)-\langle\text{grad}f(\bm{x}_k),\bm{\eta}_k\rangle_{\bm{x}_k}-\frac{1}{2}\left\langle\bm{\eta}_k,\nabla^2\hat{f}_k(\bm{0}_{\bm{x}_k})[\bm{\eta}_k]\right\rangle_{\bm{x}_k} \\
		\nonumber
		&+\langle\text{grad}f(\bm{x}_k),\bm{\eta}_k\rangle_{\bm{x}_k}+\frac{1}{2}\left\langle\bm{\eta}_k,\nabla^2\hat{f}_k(\bm{0}_{\bm{x}_k})[\bm{\eta}_k]\right\rangle_{\bm{x}_k} \\
		\nonumber
		&-\frac{1}{2}\langle\bm{\eta}_k,\mathbf{H}_k[\bm{\eta}_k]\rangle_{\bm{x}_k}-\frac{\sigma_k}{3}\|\bm{\eta}_k\|_{\bm{x}_k}^3 \\
		\nonumber
		\le\;&\left|\hat{f}_k(\bm{\eta}_k)-f(\bm{x}_k)-\langle\text{grad}f(\bm{x}_k),\bm{\eta}_k\rangle_{\bm{x}_k}-\frac{1}{2}\left\langle\bm{\eta}_k,\nabla^2\hat{f}_k(\bm{0}_{\bm{x}_k})[\bm{\eta}_k]\right\rangle_{\bm{x}_k}\right| \\
		\nonumber
		&+\left|\frac{1}{2}\left\langle\bm{\eta}_k,\nabla^2\hat{f}_k(\bm{0}_{\bm{x}_k})[\bm{\eta}_k]\right\rangle_{\bm{x}_k}-\frac{1}{2}\langle\bm{\eta}_k,\mathbf{H}_k[\bm{\eta}_k]\rangle_{\bm{x}_k}\right|-\frac{\sigma_k}{3}\|\bm{\eta}_k\|_{\bm{x}_k}^3 \\
		\nonumber
		\le\;&\frac{1}{6}L_H\|\bm{\eta}_k\|_{\bm{x}_k}^3+\frac{1}{2}\delta_H\|\bm{\eta}_k\|_{\bm{x}_k}^2-\frac{\sigma_k}{3}\|\bm{\eta}_k\|_{\bm{x}_k}^3 \\
		=\;&\left(\frac{L_H}{6}-\frac{\sigma_k}{3}\right)\|\bm{\eta}_k\|_{\bm{x}_k}^3+\frac{1}{2}\delta_H\|\bm{\eta}_k\|_{\bm{x}_k}^2.
		\label{proof_lemma_f_minus_m_eigen}
	\end{align}
	The condition $\|\mathbf{G}_k\|<\epsilon_g$ and $\lambda_{min}(\mathbf{H}_k)<-\epsilon_H$ in this case means that the term $\left\langle\mathbf{G}_k,\bm{\eta}_k\right\rangle_{\bm{x}_k}$ can be neglected based on Eq. (\ref{eq_m}) but the optimization process is not yet finished. The effect is that $\delta_g\|\bm{\eta}_k\|_{\bm{x}_k}$ no more exists in Eq. (\ref{proof_lemma_f_minus_m_eigen}) unlike Eq. (\ref{eq_proof_lemma_f_minus_m_cauchy}).
\end{proof}

\begin{lem} Suppose Condition 1 and Assumptions 1, 2, 3, 4, 5 hold, then when $\|\mathbf{G}_k\|\ge\epsilon_g$ and $\sigma_k\ge L_H$, we have
	\begin{equation}
		\begin{split}
			\left(\frac{L_H}{6}-\frac{\sigma_k}{3}\right)\|\bm{\eta}_k\|_{\bm{x}_k}^3&+\delta_g\|\bm{\eta}_k\|_{\bm{x}_k}+\frac{\delta_H}{2}\|\bm{\eta}_k\|_{\bm{x}_k}^2 \le \delta_g\left\|\bm{\eta}_k^C\right\|_{\bm{x}_k}+\frac{\delta_H}{2}\left\|\bm{\eta}_k^C\right\|_{\bm{x}_k}^2.
		\end{split}
		\label{eq_lem_f_minus_m_cauchy_bound}
	\end{equation}
	\label{lem_f_minus_m_cauchy_bound}
\end{lem}
\begin{proof}
	According to Lemma 6 of  \cite{xu2020newton}, it has
	\begin{equation}
		\begin{split}
			\left\|\bm{\eta}_k^C\right\|_{\bm{x}_k}\ge\frac{1}{2\sigma_k}\left(\sqrt{K_H^2+4\sigma_k\|\mathbf{G}_k\|_{\bm{x}_k}}-K_H\right).
		\end{split}
		\label{eq_cauchy_point_bound}
	\end{equation}
	Then we consider two cases for $\left\|\bm{\eta}_k^C\right\|_{\bm{x}_k}$.
	
	(i) If $\|\bm{\eta}_k\|_{\bm{x}_k}\le\left\|\bm{\eta}_k^C\right\|_{\bm{x}_k}$, since $\sigma_k\ge L_H$, it follows that
	\begin{equation}
		\begin{split}
			\left(\frac{L_H}{6}-\frac{\sigma_k}{3}\right)\|\bm{\eta}_k\|_{\bm{x}_k}^3+\delta_g\|\bm{\eta}_k\|_{\bm{x}_k}+\frac{\delta_H}{2}\|\bm{\eta}_k\|_{\bm{x}_k}^2 \le \delta_g\left\|\bm{\eta}_k^C\right\|_{\bm{x}_k}+\frac{\delta_H}{2}\left\|\bm{\eta}_k^C\right\|_{\bm{x}_k}^2.
		\end{split}
		\label{eq_lem_f_minus_m_cauchy_bound_1}
	\end{equation}
	
	(ii) If $\|\bm{\eta}_k\|_{\bm{x}_k}\ge\left\|\bm{\eta}_k^C\right\|_{\bm{x}_k}$, since $\sigma_k\ge L_H$, we have
	\begin{align}
		\nonumber
		&\; \left(\frac{L_H}{6}-\frac{\sigma_k}{3}\right)\|\bm{\eta}_k\|_{\bm{x}_k}^3+\delta_g\|\bm{\eta}_k\|_{\bm{x}_k}+\frac{\delta_H}{2}\|\bm{\eta}_k\|_{\bm{x}_k}^2 \\
		\nonumber
		\le&\; -\frac{\sigma_k}{6}\|\bm{\eta}_k\|_{\bm{x}_k}^3+\delta_g\|\bm{\eta}_k\|_{\bm{x}_k}+\frac{\delta_H}{2}\|\bm{\eta}_k\|_{\bm{x}_k}^2 \\
		\nonumber
		\le&\; \left(-\frac{\sigma_k}{12}\left\|\bm{\eta}_k^C\right\|_{\bm{x}_k}+\frac{\delta_H}{2}\right)\|\bm{\eta}_k\|_{\bm{x}_k}^2+\left(-\frac{\sigma_k}{12}\left\|\bm{\eta}_k^C\right\|_{\bm{x}_k}^2+\delta_g\right)\|\bm{\eta}_k\|_{\bm{x}_k} \\
		\nonumber
		\le&\; \left(-\frac{\sqrt{K_H^2+4\sigma_k\|\mathbf{G}_k\|_{\bm{x}_k}}-K_H}{24}+\frac{\delta_H}{2}\right)\|\bm{\eta}_k\|_{\bm{x}_k}^2+\\
		\nonumber
		&\;\left(-\frac{\frac{1}{12}\left(\sqrt{K_H^2+4\sigma_k\|\mathbf{G}_k\|_{\bm{x}_k}}-K_H\right)^2\|\mathbf{G}_k\|_{\bm{x}_k}}{4\sigma_k\|\mathbf{G}_k\|_{\bm{x}_k}}+\delta_g\right)\|\bm{\eta}_k\|_{\bm{x}_k} \\
		\nonumber
		\le&\;\left(-\frac{\sqrt{K_H^2+4L_H\epsilon_g}-K_H}{24}+\frac{\delta_H}{2}\right)\|\bm{\eta}_k\|_{\bm{x}_k}^2+ \\
		\nonumber
		&\;\left(-\frac{\frac{1}{12}\left(\sqrt{K_H^2+4L_H\epsilon_g}-K_H\right)^2\epsilon_g}{4L_H\epsilon_g}+\delta_g\right)\|\bm{\eta}_k\|_{\bm{x}_k} \\
		\le&\;0+0\le\delta_g\left\|\bm{\eta}_k^C\right\|_{\bm{x}_k}+\frac{\delta_H}{2}\left\|\bm{\eta}_k^C\right\|_{\bm{x}_k}^2.
		\label{eq_lem_f_minus_m_cauchy_bound_2}
	\end{align}
	The third inequality follows from Eq. (\ref{eq_cauchy_point_bound}). The fourth inequality holds since the function $h(x)=\frac{\left(\sqrt{\alpha^2+x}-\alpha\right)^2}{x}$ is an increasing function of $x$ for $\alpha\ge0$. The penultimate inequality also holds given Eqs. (\ref{eq_restrictions_delta_g}), (\ref{eq_restrictions_delta_h}) and $\frac{1-\tau}{12} \le \frac{1}{12}$ since $\tau\in(0,1]$.
\end{proof}

\begin{lem} Suppose Condition 1 and Assumptions 1, 2, 4, 5 hold, then when $\|\mathbf{G}_k\|\le\epsilon_g$ and $\lambda_{min}(\mathbf{H}_k)\le-\epsilon_H$, if $\sigma_k\ge L_H$, we have
	\begin{equation}
		\begin{split}
			\left(\frac{L_H}{6}-\frac{\sigma_k}{3}\right)\|\bm{\eta}_k\|_{\bm{x}_k}^3&+\frac{\delta_H}{2}\|\bm{\eta}_k\|_{\bm{x}_k}^2 \le\frac{\delta_H}{2}\left\|\bm{\eta}_k^E\right\|_{\bm{x}_k}^2.
		\end{split}
		\label{eq_lem_f_minus_m_eigen_bound}
	\end{equation}
	\label{lem_f_minus_m_eigen_bound}
\end{lem}
\begin{proof}
	From Lemma 7 in \cite{xu2020newton}, we have 
	\begin{equation}
		\begin{split}
			\left\|\bm{\eta}_k^E\right\|_{\bm{x}_k}\ge\frac{\nu}{\sigma_k}|\lambda_{min}(\mathbf{H}_k)|.
		\end{split}
		\label{eq_eigen_point_bound}
	\end{equation}
	Then we consider two cases for $\left\|\bm{\eta}_k^E\right\|_{\bm{x}_k}$.
	
	(i) If $\|\bm{\eta}_k\|_{\bm{x}_k}\le\left\|\bm{\eta}_k^E\right\|_{\bm{x}_k}$, since $\sigma_k\ge L_H$, it follows that
	\begin{equation}
		\begin{split}
			\left(\frac{L_H}{6}-\frac{\sigma_k}{3}\right)\|\bm{\eta}_k\|_{\bm{x}_k}^3+\frac{\delta_H}{2}\|\bm{\eta}_k\|_{\bm{x}_k}^2\le \frac{\delta_H}{2}\|\bm{\eta}_k\|_{\bm{x}_k}^2 \le \frac{\delta_H}{2}\left\|\bm{\eta}_k^E\right\|_{\bm{x}_k}^2.
		\end{split}
		\label{eq_lem_f_minus_m_eigen_bound_1}
	\end{equation}
	
	(ii) If $\|\bm{\eta}_k\|_{\bm{x}_k}\ge\left\|\bm{\eta}_k^E\right\|_{\bm{x}_k}$, since $\sigma_k\ge L_H$, we have
	\begin{align}
		\label{eq_lem_f_minus_m_eigen_bound_2}
		&\;\left(\frac{L_H}{6}-\frac{\sigma_k}{3}\right)\|\bm{\eta}_k\|_{\bm{x}_k}^3+\frac{\delta_H}{2}\|\bm{\eta}_k\|_{\bm{x}_k}^2\\
		\nonumber
		\le&\; -\frac{\sigma_k}{6}\|\bm{\eta}_k\|_{\bm{x}_k}^3+\frac{\delta_H}{2}\|\bm{\eta}_k\|_{\bm{x}_k}^2 \le -\frac{\sigma_k}{6}\left\|\bm{\eta}_k^E\right\|_{\bm{x}_k}\|\bm{\eta}_k\|_{\bm{x}_k}^2+\frac{\delta_H}{2}\|\bm{\eta}_k\|_{\bm{x}_k}^2 \\
		\nonumber
		\le&\; \left(-\frac{\nu}{6}|\lambda_{min}(\mathbf{H}_k)|+\frac{(1-\tau)\nu\epsilon_H}{6}\right)\|\bm{\eta}_k\|_{\bm{x}_k}^2 \le \left(-\frac{\nu\epsilon_H}{6}+\frac{\nu\epsilon_H}{6}\right)\|\bm{\eta}_k\|_{\bm{x}_k}^2 \le  0\le \frac{\delta_H}{2}\left\|\bm{\eta}_k^E\right\|_{\bm{x}_k}^2.
	\end{align}
	The third inequality follows from Eqs. (\ref{eq_restrictions_delta_h}) and (\ref{eq_eigen_point_bound}), and the fourth one holds since $|\lambda_{min}(\mathbf{H}_k)|\ge\epsilon_H$ and $\tau<1$.
\end{proof}

\begin{lem} Suppose Assumptions 1, 2, 3, 4, 5 hold, then when $\|\mathbf{G}_k\|_{\bm{x}_k}\ge \epsilon_g$, if $\sigma_k\ge L_H$, the iteration $k$ is successful, i.e., $\rho_k\ge\tau$  and $\sigma_{k+1}\le\sigma_k$.    \label{lem_success_step_than_cauchy}
\end{lem}
\begin{proof} 
	We have that
	\begin{align}
		\nonumber
		1-\rho_k&=\frac{\hat{f}_k(\bm{\eta}_k)-\hat{m}_k(\bm{\eta}_k)}{\hat{m}_k(\bm{0}_{\bm{x}_k})-\hat{m}_k(\bm{\eta}_k)} \\
		\nonumber
		&\le \frac{\left(\frac{L_H}{6}-\frac{\sigma_k}{3}\right)\|\bm{\eta}_k\|_{\bm{x}_k}^3+\delta_g\|\bm{\eta}_k\|_{\bm{x}_k}+\frac{1}{2}\delta_H\|\bm{\eta}_k\|_{\bm{x}_k}^2}{\hat{m}_k(\bm{0}_{\bm{x}_k})-\hat{m}_k\left(\bm{\eta}_k^C\right)} \\
		\nonumber
		&\le\frac{\delta_g\left\|\bm{\eta}_k^C\right\|_{\bm{x}_k}+\frac{1}{2}\delta_H\left\|\bm{\eta}_k^C\right\|_{\bm{x}_k}^2}{\frac{1}{12}\left\|\bm{\eta}_k^C\right\|_{\bm{x}_k}^2\left(\sqrt{K_H^2+4\sigma_k\|\mathbf{G}_k\|_{\bm{x}_k}}-K_H\right)} \\
		\nonumber
		&\le\frac{4\sigma_k\|\mathbf{G}_k\|_{\bm{x}_k}\delta_g}{\frac{1}{6}\|\mathbf{G}_k\|_{\bm{x}_k}\left(\sqrt{K_H^2+4\sigma_k\|\mathbf{G}_k\|_{\bm{x}_k}}-K_H\right)^2}+\frac{6\delta_H}{\sqrt{K_H^2+4L_H\epsilon_g}-K_H} \\
		\nonumber
		&\le\frac{4L_H\epsilon_g\delta_g}{\frac{1}{6}\epsilon_g\left(\sqrt{K_H^2+4L_H\epsilon_g}-K_H\right)^2}+\frac{6\delta_H}{\left(\sqrt{K_H^2+4L_H\epsilon_g}-K_H\right)} \\
		&\le\frac{1-\tau}{2}+\frac{1-\tau}{2}=1-\tau,
		\label{eq_success_step_than_cauchy}
	\end{align}
	where the first inequality follows from Eqs. (\ref{eq_cauchy_point}), Eq. (\ref{eq_lemma_f_minus_m_cauchy}), the second inequality from Eqs. (\ref{eq_lem_f_minus_m_cauchy_bound}), (\ref{eq_cauchy_point}), (\ref{eq_b}), and the  third  follows from Eq. (\ref{eq_cauchy_point_bound}). The last two inequalities hold given Eqs. (\ref{eq_restrictions_delta_g}), (\ref{eq_restrictions_delta_h}) and the fact that the function $h(x)=\frac{x}{\left(\sqrt{\alpha^2+x}-\alpha\right)^2}$ is decreasing with $\alpha>0$. 
	Consequently, we have $\rho_k\ge\tau$ and that the iteration is successful. Based on Step (\ref{step_src_0}) of Algorithm \ref{alg_inexact_rtr_arc}, we have $\sigma_{k+1}=\max(\sigma_k/\gamma,\epsilon_\sigma)\le\sigma_k$.
\end{proof}

\begin{lem} Suppose Condition 1 and Assumption 1, 2, 4, 5 hold, then when $\|\mathbf{G}_k\|_{\bm{x}_k}< \epsilon_g$ and $\lambda_{min}(\mathbf{H}_k)\le-\epsilon_H$, if $\sigma_k\ge L_H$, the iteration $k$ is successful, i.e., $\rho_k\ge\tau$ and $\sigma_{k+1}\le\sigma_k$.
	\label{lem_success_step_than_eigen}
\end{lem}
\begin{proof} 
	We have that
	\begin{align}
		\nonumber
		1-\rho_k&=\frac{\hat{f}_k(\bm{\eta}_k)-\hat{m}_k(\bm{\eta}_k)}{\hat{m}_k(\bm{0}_{\bm{x}_k})-\hat{m}_k(\bm{\eta}_k)} \le \frac{\left(\frac{L_H}{6}-\frac{\sigma_k}{3}\right)\|\bm{\eta}_k\|_{\bm{x}_k}^3+\frac{1}{2}\delta_H\|\bm{\eta}_k\|_{\bm{x}_k}^2}{\hat{m}_k(\bm{0}_{\bm{x}_k})-\hat{m}_k\left(\bm{\eta}_k^E\right)} \\
		&\le\frac{\frac{1}{2}\delta_H\left\|\bm{\eta}_k^E\right\|_{\bm{x}_k}^2}{\nu|\lambda_{min}(\mathbf{H}_k)\left\|\bm{\eta}_k^E\right\|_{\bm{x}_k}^2/6} \le\frac{3\delta_H}{\nu\epsilon_H}\le1-\tau,
		\label{eq_success_step_than_eigen}
	\end{align}
	where the first and second inequalities follow from Eqs. (\ref{eq_lemma_f_minus_m_eigen}), (\ref{eq_eigen_point}), (\ref{eq_lem_f_minus_m_eigen_bound}), (\ref{eq_c}). The last inequality uses Eq. (\ref{eq_restrictions_delta_h}). 
	Consequently, we have $\rho_k\ge\tau$ which indicates the iteration is successful. Based on Step (\ref{step_src_0}) of Algorithm \ref{alg_inexact_rtr_arc}, we have $\sigma_{k+1}=\max(\sigma_k/\gamma,\epsilon_\sigma)\le\sigma_k$.
\end{proof}

\begin{lem} Suppose Condition 1 and Assumptions 1, 2, 3, 4, 5 hold, then for all $k$, we have
	\begin{equation}
		\sigma_k\le \max\left(\sigma_0,2\gamma L_H\right).
		\label{eq_sigma_bound}
	\end{equation}
	\label{lem_sigma_bound}
\end{lem}
\begin{proof}
	We prove this lemma by contradiction by considering the two following cases. 
	
	(i) If $\sigma_0 \le 2\gamma L_H$, we assume there exists an iteration $k\ge0$ that is the first unsuccessful iteration such that $\sigma_{k+1}=\gamma\sigma_k>2\gamma L_H$. 
	This implies $\sigma_k> L_H$ and $\sigma_{k+1}>\sigma_k$. 
	Since the algorithm fails to terminate at iteration $k$, we have $\|\bm{G_k}\|_{\bm{x}_k}\ge\epsilon_g$ or $\lambda_{min}(\mathbf{H}_k)<-\epsilon_H$. 
	Then, according to Lemmas \ref{lem_success_step_than_cauchy} and \ref{lem_success_step_than_eigen}, iteration $k$ is successful and thus $\sigma_{k+1}=\max(\sigma_k/\gamma, \epsilon_\sigma)\le\sigma_k$. 
	This contradicts the earlier statement of $\sigma_{k+1}>\sigma_k$. 
	We thus have $\sigma_k \le 2\gamma L_H$ for all $k$.
	
	(ii) If $\sigma_0>2\gamma L_H$, similarly, we assume there exists an iteration $k\ge0$ that is the first unsuccessful iteration such that $\sigma_{k+1}=\gamma\sigma_k>\sigma_0$. 
	This implies $\sigma_k> L_H$ and $\sigma_{k+1}>\sigma_k$. 
	Since the algorithm fails to terminate at iteration $k$, we have $\|\bm{G_k}\|_{\bm{x}_k}\ge\epsilon_g$ or $\lambda_{min}(\mathbf{H}_k)<-\epsilon_H$.
	According to Lemmas \ref{lem_success_step_than_cauchy} and \ref{lem_success_step_than_eigen}, iteration $k$ is successful and thus $\sigma_{k+1}=\max(\sigma_k/\gamma, \epsilon_\sigma)\le\sigma_k$, which is a contradiction. Thus, we have $\sigma_k \le \sigma_0$ for all $k$. 
\end{proof}

\subsection{Main Proof of Theorem  \ref{theorem2}}
\label{app_theorem2B}

\begin{proof} 
	Letting $\sigma_b=\max \left(\sigma_0, 2\gamma L_H \right)$, when $\|\mathbf{G}_k\|\ge\epsilon_g$, from Eq. (\ref{eq_cauchy_point}) and Lemma \ref{lem_sigma_bound} we have
	\begin{equation}
		\hat{m}_k(\bm{0}_{\bm{x}_k})-\hat{m}_k(\bm{\eta}_k)\ge\frac{\|\mathbf{G}_k\|_{\bm{x}_k}}{2\sqrt{3}}\min\left(\frac{\|\mathbf{G}_k\|_{\bm{x}_k}}{K_H},\sqrt{\frac{\|\mathbf{G}_k\|_{\bm{x}_k}}{\sigma_k}}\right) \ge\frac{\epsilon_g^2}{2\sqrt{3}}\min\left(\frac{1}{K_H},{\frac{1}{\sqrt{\sigma_b}}}\right).
		\label{eq_subproblem_final_bound_from_cauchy}
	\end{equation}
	When $\|\mathbf{G}_k\|<\epsilon_g$ and $\lambda_{min}(\mathbf{H}_k)\le-\epsilon_H$, from Eq. (\ref{eq_eigen_point}) and Lemma \ref{lem_sigma_bound}, we have
	\begin{equation}
		\begin{split}
			\hat{m}_k(\bm{0}_{\bm{x}_k})-\hat{m}_k(\bm{\eta}_k)
			\ge  \frac{\nu^3}{6\sigma_k^2} |\lambda_{min}(\mathbf{H}_k)|^3
			\ge \frac{\nu^3\epsilon_H^3}{6\sigma_b^2}.
		\end{split}
		\label{eq_subproblem_final_bound_from_eigen}
	\end{equation}
	Let $\mathcal{S}_{succ}$ be the set of successful iterations before Algorithm \ref{alg_inexact_rtr_arc} terminates and $\hat{f}_{min}$ be the function minimum. Since $\hat{f}_k(\bm{\eta}_k)$ is monotonically decreasing, we have
	\begin{align}
		\label{eq_succ_iteration_bound}
		\hat{f}_0(\bm{0}_{\bm{x}_0})-\hat{f}_{min}&\ge\sum_{k=0}^{\infty}\left(\hat{f}_k(\bm{0}_{\bm{x}_k})-\hat{f}_k(\bm{\eta}_k)\right)\ge\sum_{k\in\mathcal{S}_{succ}}\left(\hat{f}_k(\bm{0}_{\bm{x}_k})-\hat{f}_k(\bm{\eta}_k)\right) \\
		\nonumber
		&\ge\tau\left(\sum_{k\in\mathcal{S}_{succ}}\hat{m}_k(\bm{0}_{\bm{x}_k})-\hat{m}_k(\bm{\eta}_k)\right) \\
		\nonumber
		&\ge\tau|\mathcal{S}_{succ}|\min\left(\frac{\nu^3\epsilon_H^3}{6\sigma_b^2},\frac{\epsilon_g^2}{2\sqrt{3}}\min\left(\frac{1}{K_H},{\frac{1}{\sqrt{\sigma_b}}}\right)\right) \ge|\mathcal{S}_{succ}|\tau\kappa\min\left(\epsilon_g^2,\epsilon_H^3\right),
	\end{align}
	where $\kappa=\min\left( \frac{\nu^3}{6\sigma_b^2},\frac{1}{2\sqrt{3}} \min\left(\frac{1}{K_H},\frac{1}{\sqrt{\sigma_b}}\right) \right)$. Let $\mathcal{S}_{fail}$ be the set of unsuccessful iterations and $T$ be the total iterations of the algorithm. Then we have $\sigma_T=\sigma_0\gamma^{|\mathcal{S}_{fail}|-|\mathcal{S}_{succ}|}$. 
	Combining it with $\sigma_T\le\sigma_b$ from Lemma \ref{lem_sigma_bound}, we have
	\begin{align}
		\nonumber
		&\sigma_T=\sigma_0\gamma^{|\mathcal{S}_{fail}|-|\mathcal{S}_{succ}|}\le\sigma_b\\
		\nonumber
		\Longrightarrow&\log\left(\gamma^{|\mathcal{S}_{fail}|-|\mathcal{S}_{succ}|}\right)\le\log\left(\frac{\sigma_b}{\sigma_0}\right)\\
		\nonumber
		\Longrightarrow&\left({|\mathcal{S}_{fail}|-|\mathcal{S}_{succ}|}\right)\le\frac{\log(\sigma_b/\sigma_0)}{\gamma}\\
		\Longrightarrow&|\mathcal{S}_{fail}|\le\frac{\log(\sigma_b/\sigma_0)}{\log{\gamma}}+|\mathcal{S}_{succ}|.
		\label{eq_fail_iteration_bound}
	\end{align}
	Finally, combining Eqs. (\ref{eq_succ_iteration_bound}), (\ref{eq_fail_iteration_bound}), we have
	\begin{align}
		\label{eq_total_iteration_bound}
		T=|\mathcal{S}_{fail}|+|\mathcal{S}_{succ}|&\le\frac{\log(\sigma_b/\sigma_0)}{\log{\gamma}}+2|\mathcal{S}_{succ}| \\
		&\le\frac{\log(\sigma_b/\sigma_0)}{\log{\gamma}}+\frac{2\left(\hat{f}_0(\bm{0}_{\bm{x}_0})-\hat{f}_{min}\right)}{\tau\kappa}\max\left(\epsilon_g^{-2},\epsilon_H^{-3}\right).
		\nonumber
	\end{align}
	This completes the proof.
\end{proof}

\subsection{Main Proof of Corollary  \ref{coro_1}}
\label{app_theorem2C}

\begin{proof}  
	
	Under the given assumptions, when Theorem \ref{theorem2} holds, Algorithm \ref{alg_inexact_rtr_arc} returns an $\left( \epsilon_g,\epsilon_H\right)$-optimal solution in $ T=\mathcal{O}\left(\max\left(\epsilon_g^{-2},\epsilon_H^{-3}\right)\right)$ iterations. 
	Also, according  to Theorem  \ref{theorem1}, at an iteration, Condition \ref{cond_approximate_grad_hess_bound} is satisfied with a probability $(1-\delta)$, where the probability $(1-\delta)$ at the current iteration can be independently achieved by selecting proper subsampling sizes  for the approximate gradient and Hessian.
	Let  $E$ be the event that Algorithm \ref{alg_inexact_rtr_arc} returns an $\left( \epsilon_g,\epsilon_H\right)$-optimal solution and $E_i$ be the event that Condition \ref{cond_approximate_grad_hess_bound} is satisfied at iteration $i$. 
	According on Theorem  \ref{theorem2},  when event $E$ happens, it requires Condition \ref{cond_approximate_grad_hess_bound} to be satisfied for all the iterations, thus we have
	\begin{equation}
		{\rm Pr}(E) = \prod_{i=1}^{T}{\rm Pr}(E_i) =(1-\delta)^{T}=(1-\delta)^{\mathcal{O}\left(\max\left(\epsilon_g^{-2},\epsilon_H^{-3}\right)\right)}.
	\end{equation}
	This completes the proof.
\end{proof}

\section{Appendix: Theorem  \ref{theorem3} and Corollary  \ref{coro_2} }
\label{app_theorem3}

\subsection{Supporting Lemmas for Theorem  \ref{theorem3}}
\label{app_theorem3A}

The proof of Theorem  \ref{theorem3} builds on an existing lemma, which we restate as follows.
\begin{lem}(Sufficient Function Decrease (Lemma 3.3 in \cite{cartis2011adaptive})) Suppose the solution $\bm{\eta}_k$ satisfies Assumption \ref{assu_cg}, then we have
	\begin{equation}
		\begin{split}
			\hat{m}_k(\bm{0}_{\bm{x}_k})-\hat{m}_k(\bm{\eta}_k)\ge\frac{\sigma_k}{6}\|\bm{\eta}_k\|_{\bm{x}_k}^3.
		\end{split}
		\label{eq_sufficient_function_decrease}
	\end{equation}
	\label{lem_sufficient_function_decrease}
\end{lem}
The proof also needs another lemma that we develop as below.
\begin{lem}(Sufficiently Long Steps) Suppose Condition \ref{cond_approximate_grad_hess_bound} and Assumptions 1, 2, 3, 4, 5 hold. If $\delta_g\le\delta_H\le\kappa_\theta\epsilon_g$,  Algorithm \ref{alg_non_linear_tCG} returns a solution $\bm{\eta}_k$ such that
	\begin{equation}
		\begin{split}
			\|\bm{\eta}_k\|_{\bm{x}_k}\ge\kappa_s\sqrt{\epsilon_g},
		\end{split}
		\label{eq_lemma_step_lower_bound}
	\end{equation}
	when $\|\mathbf{G}_{k+1}\|_{\bm{x}_{k+1}}\ge\epsilon_g$ for $k>0$ and when the inner stopping criterion of Eq. (\ref{eq_stepsize_stop_additional}) is used. 
	Here $\kappa_s=\min\left(1/\sqrt{(L_H+2\sigma_b+\frac{\epsilon_g}{3}+\frac{L_l}{3})},1/\sqrt{\frac{5L_H}{3}+\frac{10\sigma_b}{3}+\frac{11L_l}{9}}\right)$ with $L_l > 0$ and $\sigma_b = \max\left(\sigma_0,2\gamma L_H\right)$.
	\label{lemma_step_lower_bound}
\end{lem}
\begin{proof}
	By differentiating the approximate model, we have 
	\begin{align}
		\nonumber
		&\|\nabla\hat{m}_k(\bm{\eta}_{k})\|_{\bm{x}_k}=\bigg|\bigg|\mathbf{G}_k+\mathbf{H}_k[\bm{\eta}_{k}]+\sigma_k\|\bm{\eta}_{k}\|_{\bm{x}_k}\bm{\eta}_{k}\bigg|\bigg|_{\bm{x}_k} \\
		\nonumber
		=&\left\|\mathcal{P}_{\bm{\eta}_{k}}^{-1}\text{grad}\hat{f}_k(\bm{\eta}_{k})+\mathbf{G}_k-\text{grad}f(\bm{x}_k)+\mathbf{H}_k[\bm{\eta}_{k}]-\nabla^2\hat{f}_k(\bm{0}_{\bm{x}_k})[\bm{\eta}_{k}]\right. \\
		\nonumber
		&\left.+\text{grad}f(\bm{x}_k)+\nabla^2\hat{f}_k(\bm{0}_{\bm{x}_k})[\bm{\eta}_{k}]-\mathcal{P}_{\bm{\eta}_{k}}^{-1}\text{grad}\hat{f}_k(\bm{\eta}_{k})+\sigma_k\|\bm{\eta}_{k}\|_{\bm{x}_k}\bm{\eta}_{k}\right\|_{\bm{x}_k} \\
		\nonumber
		\ge&\left\|\mathcal{P}_{\bm{\eta}_{k}}^{-1}\text{grad}\hat{f}_k(\bm{\eta}_{k})\right\|_{\bm{x}_k}- \left\| \mathbf{G}_k-\text{grad}f(\bm{x}_k)\right\|_{\bm{x}_k}- \left\|\mathbf{H}_k[\bm{\eta}_{k}]-\nabla^2\hat{f}_k(\bm{0}_{\bm{x}_k})[\bm{\eta}_{k}]\right\|_{\bm{x}_k}\\
		\nonumber
		&- \left\|\mathcal{P}_{\bm{\eta}_{k}}^{-1}\text{grad}\hat{f}_k(\bm{\eta}_{k})-\text{grad}f(\bm{x}_k)-\nabla^2\hat{f}_k(\bm{0}_{\bm{x}_k})[\bm{\eta}_{k}] \right\|_{\bm{x}_k} - \sigma_k\|\bm{\eta}_{k}\|_{\bm{x}_k}^2\\
		\label{eq_lemma_proof_step_lower_bound_1}
		\ge&\norm{\mathcal{P}_{\bm{\eta}_{k}}^{-1}\text{grad}\hat{f}_k(\bm{\eta}_{k})}_{\bm{x}_{k}}-\delta_g-\delta_H\|\bm{\eta}_{k}\|_{\bm{x}_k}-\frac{L_H}{2}\|\bm{\eta}_{k}\|_{\bm{x}_k}^2-\sigma_b\|\bm{\eta}_{k}\|_{\bm{x}_k}^2,
	\end{align}
	where the first inequality follows from the triangle inequality and the second inequality from 
	Eqs. (\ref{eq_restricted_lipschitz_hessian_2}),  (\ref{eq_approximate_grad_hess_bound1}) and (\ref{eq_approximate_grad_hess_bound2}). 
	Additionally, from Lemma 3.8 in \cite{kasai2018riemannian}, we have
	\begin{equation}
		\begin{split}
			\norm{\text{grad}f_k(\bm{x}_k)}_{\bm{x}_{k}}&-\norm{\mathcal{P}_{\bm{\eta}_{k}}^{-1}\text{grad}\hat{f}_k(\bm{\eta}_{k})}_{\bm{x}_{k}}\\
			&\le\norm{\mathcal{P}_{\bm{\eta}_{k}}^{-1}\text{grad}\hat{f}_k(\bm{\eta}_{k})-\text{grad}f(\bm{x}_k)}_{\bm{x}_k}\le L_l\norm{\bm{\eta}_{k}}_{\bm{x}_{k}},
		\end{split}
		\label{eq_lemma_proof_step_lower_bound_taylor}
	\end{equation}
	where $L_l>0$ is a constant. Then, we have
	\begin{align}
		\nonumber
		\|\mathbf{G}_k\|_{\bm{x}_k}\le &\; \|\mathbf{G}_k-\text{grad}f(\bm{x}_k)\|_{\bm{x}_k}+\|\text{grad}f(\bm{x}_k)\|_{\bm{x}_k} \\ 
		\le &\;\delta_g+L_l\|\bm{\eta}_{k}\|_{\bm{x}_k}+\norm{\mathcal{P}_{\bm{\eta}_{k}}^{-1}\text{grad}\hat{f}_k(\bm{\eta}_{k})}_{\bm{x}_{k}},
		\label{eq_lemma_proof_step_lower_bound_2}
	\end{align}
	with the first inequality following from the triangle inequality and the second from Eqs. (\ref{eq_approximate_grad_hess_bound1}),  (\ref{eq_approximate_grad_hess_bound2}) and  (\ref{eq_lemma_proof_step_lower_bound_taylor}). Then, by combining Eqs. (\ref{eq_stepsize_stop_additional}),  (\ref{eq_lemma_proof_step_lower_bound_1}) and (\ref{eq_lemma_proof_step_lower_bound_2}), with $\theta_k:=\kappa_\theta\min(1,\left\|\bm\eta_k^i\right\|_{\mathbf{x}_k})$ we obtain
	\begin{align}
		&\norm{\mathcal{P}_{\bm{\eta}_{k}}^{-1}\text{grad}\hat{f}_k(\bm{\eta}_{k})}_{\bm{x}_{k}}-\delta_g-\delta_H\|\bm{\eta}_{k}\|_{\bm{x}_k}-\left(\frac{L_H}{2}+\sigma_b\right)\|\bm{\eta}_{k}\|_{\bm{x}_k}^2 \\
		\nonumber
		\le&\; \|\nabla\hat{m}_k(\bm{\eta}_{k})\|_{\bm{x}_k}\le\theta_k\|\mathbf{G}_k\|_{\bm{x}_k} \\ 
		\nonumber
		\le&\; \theta_k\left(\delta_g+L_l\|\bm{\eta}_{k}\|_{\bm{x}_k}+\norm{\mathcal{P}_{\bm{\eta}_{k}}^{-1}\text{grad}\hat{f}_k(\bm{\eta}_{k})}_{\bm{x}_{k}}\right).
	\end{align}
	This results in
	\begin{align}
		\label{ieq}
		&\left(\frac{L_H}{2}+\sigma_b\right)\|\bm{\eta}_{k}\|_{\bm{x}_k}^2+\left(\delta_H+\theta_k L_l\right)\|\bm{\eta}_{k}\|_{\bm{x}_k}+\left(1+\theta_k\right)\delta_g \\
		\nonumber
		\ge&\;\left(1-\theta_k\right)\norm{\text{grad}\hat{f}_k(\bm{\eta}_{k})}_{\bm{x}_{k+1}}\ge\left(1-\theta_k\right)\left(\|\mathbf{G}_{k+1}\|_{\bm{x}_{k+1}}-\delta_g\right).
	\end{align}
	Subsequently, it has
	\begin{equation}
		\label{eq_lemma_proof_step_lower_bound_3}
		\left(\frac{L_H}{2}+\sigma_b\right)\|\bm{\eta}_{k}\|_{\bm{x}_k}^2+\left(\delta_H+\theta_k L_l\right)\|\bm{\eta}_{k}\|_{\bm{x}_k}+2\delta_g \ge\left(1-\theta_k\right)\left(\|\mathbf{G}_{k+1}\|_{\bm{x}_{k+1}}\right).
	\end{equation}
	In the above derivation, we use the property of the parallel transport that preserves the length of the transported vector. 
	The last inequality in Eq. (\ref{ieq}) is based on 
	Eqs. (\ref{eq_approximate_grad_hess_bound1}) and  the triangle inequality.
	
	Now, we consider the following two cases.
	(i) If $\|\bm{\eta}_{k}\|_{\bm{x}_k}\ge 1$, from Eq. (\ref{eq_stepsize_stop_additional}) we have $\theta_k=\kappa_\theta$, and therefore 
	\begin{equation}
		\left(1-\kappa_\theta\right)\left(\|\mathbf{G}_{k+1}\|_{\bm{x}_{k+1}}\right)-2\delta_g\le\left(\frac{L_H}{2}+\sigma_b+\delta_H+\kappa_\theta L_l\right)\|\bm{\eta}_{k}\|_{\bm{x}_k}^2. 
	\end{equation}
	This then gives
	\begin{equation}
		\|\bm{\eta}_{k}\|_{\bm{x}_k}^2\ge\frac{\left(1-\kappa_\theta\right)\left(\|\mathbf{G}_{k+1}\|_{\bm{x}_{k+1}}\right)-2\delta_g}{\frac{L_H}{2}+\sigma_b+\delta_H+\kappa_\theta L_l}\ge\frac{\frac{1}{2}\epsilon_g}{\frac{L_H}{2}+\sigma_b+\frac{\epsilon_g}{6}+\frac{L_l}{6}},
		\label{eq_lemma_proof_step_lower_bound_4}
	\end{equation}
	where the last inequality holds because $\delta_g\le\delta_H\le\kappa_{\theta}\epsilon_g$ and $\kappa_{\theta}<\frac{1}{6}$.
	(ii) If $\|\bm{\eta}_{k}\|_{\bm{x}_k}< 1$, then $\theta_k=\kappa_\theta\|\bm{\eta}_{k}\|_{\bm{x}_k}<\kappa_\theta$. Given Eq. (\ref{eq_lemma_proof_step_lower_bound_2}) and $\delta_g\le\delta_H\le\kappa_{\theta}\epsilon_g$, it has
	\begin{equation}
		\delta_H\le\kappa_{\theta}\epsilon_g\le\kappa_\theta\|\mathbf{G}_k\|_{\bm{x}_k}\le\kappa_\theta\left(\delta_H+L_l\|\bm{\eta}_{k}\|_{\bm{x}_k}+\|\mathbf{G}_{k+1}\|_{\bm{x}_{k+1}}\right).
		\label{eq_lemma_proof_step_lower_bound_5}
	\end{equation}
	Then, we have
	\begin{align}
		&\left(1-\theta_k\right)\left(\|\mathbf{G}_{k+1}\|_{\bm{x}_{k+1}}\right)-2\delta_g \\
		\nonumber
		\le & \left(\frac{L_H}{2}+\sigma_b\right)\|\bm{\eta}_{k}\|_{\bm{x}_k}^2+\left(\delta_H+\theta_k L_l\right)\|\bm{\eta}_{k}\|_{\bm{x}_k} \\
		\nonumber
		\le&\left(\frac{L_H}{2}+\sigma_b\right)\|\bm{\eta}_{k}\|_{\bm{x}_k}^2+\left(\frac{\kappa_\theta }{1-\kappa_\theta}+\kappa_\theta\right)L_l\|\bm{\eta}_{k}\|_{\bm{x}_k}^2+\frac{\kappa_\theta \|\mathbf{G}_{k+1}\|_{\bm{x}_{k+1}}}{1-\kappa_\theta},  
	\end{align}
	which results in
	\begin{equation}
		\label{eq_lemma_proof_step_lower_bound_6}
		\|\bm{\eta}_{k}\|_{\bm{x}_k}^2\ge\frac{	\left(1-\kappa_\theta-\frac{\kappa_\theta}{1-\kappa_\theta}\right)\|\mathbf{G}_{k+1}\|_{\bm{x}_{k+1}}-2\delta_g}{\frac{L_H}{2}+\sigma_b+\left(\frac{\kappa_\theta }{1-\kappa_\theta}+\kappa_\theta\right)L_l}\ge\frac{\frac{3}{10}\epsilon_g}{\frac{L_H}{2}+\sigma_b+\frac{11}{30}L_l}.
	\end{equation}
	This completes the proof.
\end{proof}

\subsection{Main Proof of Theorem  \ref{theorem3}}
\label{app_theorem3B}

Let $\sigma_b = \max\left(\sigma_0,2\gamma L_H\right)$, $\sigma_{min}=\min \left(\sigma_k\right)$ for $k\ge0$ and $\mathcal{S}_{succ}^1$ be the set of successful iterations such that $\|\mathbf{G}_{k+1}\|_{\bm{x}_{k+1}}\ge\epsilon_g$ for $k\in\mathcal{S}_{succ}^1$. As $\hat{f}_k(\bm{\eta}_k)$ is monotonically decreasing, we have
\begin{align}
	\hat{f}_0(\bm{0}_{\bm{x}_0})-\hat{f}_{min}&\ge\sum_{k=0}^{\infty}\left(\hat{f}_k(\bm{0}_{\bm{x}_k})-\hat{f}_k(\bm{\eta}_k)\right)\ge\sum_{k\in\mathcal{S}_{succ}^1}\left(\hat{f}_k(\bm{0}_{\bm{x}_k})-\hat{f}_k(\bm{\eta}_k)\right) \\
	\nonumber
	&\ge\tau\sum_{k\in\mathcal{S}_{succ}^1}\left(\hat{m}_k(\bm{0}_{\bm{x}_k})-\hat{m}_k(\bm{\eta}_k)\right)  \ge\tau|\mathcal{S}_{succ}^1|\min\left(\frac{\nu^3\epsilon_H^3}{6\sigma_k^2},\frac{\sigma_{min}}{6}\|\bm{\eta}_k\|_{\bm{x}_k}^3\right) \\
	\nonumber
	&\ge\tau|\mathcal{S}_{succ}^1|\min\left(\frac{\nu^3\epsilon_H^3}{6\sigma_b^2},\frac{\epsilon_\sigma\kappa_s^3\epsilon_g^{\frac{3}{2}}}{6}\right)  \ge\tau\kappa_1|\mathcal{S}_{succ}^1|\min\left(\epsilon_g^{\frac{3}{2}},\epsilon_H^3\right),
	\label{eq_succ_iteration_bound_optimal}
\end{align}
where $\kappa_1=\frac{1}{6} \min\left(\frac{\nu^3}{\sigma_b^2},\epsilon_\sigma\kappa_s^3\right)$. The fourth inequality follows from Eqs. (\ref{eq_eigen_point}) and (\ref{eq_sufficient_function_decrease}), while the fifth from Eq. (\ref{eq_lemma_step_lower_bound}). 

Let $\mathcal{S}_{succ}^2$ be the set of successful iterations such that $\|\mathbf{G}_{k+1}\|_{\bm{x}_{k+1}}<\epsilon_g$ and $\lambda_{min}(\mathbf{H}_{k+1})<-\epsilon_H$ for $k\in\mathcal{S}_{succ}^2$. Then there is an iteration $t\in\mathcal{S}_{succ}^2$ in which $\|\mathbf{G}_{t}\|_{\bm{x}_{t}}\ge\epsilon_g$ and $\|\mathbf{G}_{t+1}\|_{\bm{x}_{t+1}}<\epsilon_g$. Thus, we have
\begin{equation}
	\hat{f}_0(\bm{0}_{\bm{x}_0})-\hat{f}_{min}\ge\sum_{k=t}^{\infty}\left(\hat{f}_k(\bm{0}_{\bm{x}_k})-\hat{f}_k(\bm{\eta}_k)\right) \ge\hat{f}_0(\bm{0}_{\bm{x}_0})-\hat{f}_t(\bm{\eta}_t)+\sum_{k\in\mathcal{S}_{succ}^2}\left(\hat{f}_k(\bm{0}_{\bm{x}_k})-\hat{f}_k(\bm{\eta}_k)\right),
\end{equation}
and this results in 
\begin{align}
	\nonumber
	\hat{f}_t(\bm{\eta}_t)-\hat{f}_{min}&\ge\sum_{k\in\mathcal{S}_{succ}^2}\left(\hat{f}_k(\bm{0}_{\bm{x}_k})-\hat{f}_k(\bm{\eta}_k)\right) \ge\tau\sum_{k\in\mathcal{S}_{succ}^2}\left(\hat{m}_k(\bm{0}_{\bm{x}_k})-\hat{m}_k(\bm{\eta}_k)\right) \\
	&\ge\tau|\mathcal{S}_{succ}^2|\frac{\nu^3\epsilon_H^3}{6\sigma_k^2} \ge\tau|\mathcal{S}_{succ}^2|\frac{\nu^3\epsilon_H^3}{6\sigma_b^2}\ge\tau\kappa_2|\mathcal{S}_{succ}^2|\epsilon_H^3,
	\label{eq_succ_iteration_bound_optimal_2}
\end{align}
where $\kappa_2=\frac{\nu^3}{6\sigma_b^2}$. The second and third inequalities follow from Eq. (\ref{eq_sigma_bound}) and Eq. (\ref{eq_eigen_point}), respectively. Then, the bound for the total number of successful iterations is
\begin{align}
	\nonumber
	|\mathcal{S}_{succ}|=&\; |\mathcal{S}_{succ}^1| + |\mathcal{S}_{succ}^2| + 1\\
	\nonumber
	\le&\; \frac{\hat{f}_0(\bm{0}_{\bm{x}_0})-\hat{f}_{min}}{\tau\kappa_1}\max \left(\epsilon_g^{-\frac{3}{2}},\epsilon_H^{-3}\right)+\frac{\hat{f}_t(\bm{\eta}_t)-\hat{f}_{min}}{\tau\kappa_2}\epsilon_H^{-3} + 1 \\
	\le&\; \left(\frac{\hat{f}_0(\bm{0}_{\bm{x}_0})-\hat{f}_{min}}{\tau\kappa_1}+\frac{\hat{f}_t(\bm{\eta}_t)-\hat{f}_{min}}{\tau\kappa_2}\right)\max \left(\epsilon_g^{-\frac{3}{2}},\epsilon_H^{-3}\right)+1,
	\label{eq_succ_total_iteration_bound_optimal}
\end{align}
where the extra iteration corresponds to the final successful iteration of Algorithm \ref{alg_inexact_rtr_arc} with $\lambda_{min}(\mathbf{H}_{k+1})\ge-\epsilon_H$. Then, similar to Eq. (\ref{eq_total_iteration_bound}), we have the improved iteration bound for Algorithm \ref{alg_inexact_rtr_arc} given as
\begin{equation}
	T=|\mathcal{S}_{fail}|+|\mathcal{S}_{succ}|\le\frac{\log(\sigma_b/\sigma_0)}{\log{\gamma}}+2|\mathcal{S}_{succ}| \le\frac{\log(\sigma_b/\sigma_0)}{\log{\gamma}}+2C\max\left(\epsilon_g^{-\frac{3}{2}},\epsilon_H^{-3}\right) + 2,
	\label{eq_total_iteration_bound_optimal}
\end{equation}
where $C=\frac{\hat{f}_0(\bm{0}_{\bm{x}_0})-\hat{f}_{min}}{\tau\kappa_1} + \frac{\hat{f}_t(\bm{\eta}_t)-\hat{f}_{min}} {\tau\kappa_2}$.
This completes the proof.
\qed

\subsection{Main Proof of Corollary  \ref{coro_2}}
\label{app_theorem3C}

Although this follows exactly the same way as to prove Corollary  \ref{coro_1}, we repeat it here for the convenience of readers.

\begin{proof}  
	Under the given assumptions, when Theorem \ref{theorem3} holds, Algorithm \ref{alg_inexact_rtr_arc} returns an $\left( \epsilon_g,\epsilon_H\right)$-optimal solution in $ T=\mathcal{O}\left(\max\left(\epsilon_g^{-\frac{3}{2}},\epsilon_H^{-3}\right)\right)$ iterations. 
	Also, according  to Theorem  \ref{theorem1}, at an iteration, Condition \ref{cond_approximate_grad_hess_bound} is satisfied  with a probability $(1-\delta)$, where the probability $(1-\delta)$ at the current iteration can be independently achieved by selecting proper subsampling sizes for the approximate gradient and Hessian.
	Let  $E$ be the event that Algorithm \ref{alg_inexact_rtr_arc} returns an $\left( \epsilon_g,\epsilon_H\right)$-optimal solution and $E_i$ be the event that Condition \ref{cond_approximate_grad_hess_bound} is satisfied at iteration $i$. 
	According on Theorem  \ref{theorem3}, when event $E$ happens, it requires Condition \ref{cond_approximate_grad_hess_bound} to be satisfied for all the iterations, thus we have
	\begin{equation}
		{\rm Pr}(E) = \prod_{i=1}^{T}{\rm Pr}(E_i) =(1-\delta)^{T}=(1-\delta)^{\mathcal{O}\left(\max\left(\epsilon_g^{-\frac{3}{2}},\epsilon_H^{-3}\right)\right)}.
	\end{equation}
	This completes the proof.
\end{proof}

\end{document}